 \documentclass{cup_ims}
 \usepackage{natbib}
\usepackage{color}
\usepackage{braket}

  \usepackage[figuresright]{rotating}
  \usepackage{floatpag,comment}
  \rotfloatpagestyle{empty}

   \usepackage{amsmath}
  \usepackage{amsthm}
  \usepackage{graphicx}
  \usepackage{subfigure}
  
\usepackage[ruled]{algorithm2e}
\usepackage{tikz}
\usetikzlibrary{shapes.geometric}
\usetikzlibrary{arrows}

\usetikzlibrary{fit}
\usetikzlibrary{calc}

  \usepackage{newtxtext}
  \usepackage{newtxmath}
  \usepackage[scaled=0.9]{couriers}

\usepackage{longtable}

\usepackage{makeidx}
\makeindex

  \theoremstyle{plain}
  \newtheorem{theorem}{Theorem}[chapter]
  \newtheorem{lemma}[theorem]{Lemma}

  \theoremstyle{definition}
  \newtheorem{definition}[theorem]{Definition}
  \newtheorem{example}[theorem]{Example}
  \newtheorem*{remark}{Remark}

\graphicspath{{../}{./}}

\newcommand{\Expect}[1]{\mathbb{E} \left[{#1}\right]}
\newcommand{\Expects}[2]{\mathbb{E}_{{#1}} \left[{#2}\right]}
\newcommand{\hExpects}[2]{\hat{\mathbb{E}}_{{#1}} \left[{#2}\right]}

\newcommand{\Var}[1]{\mbox{Var} \left[{#1}\right]}

\newcommand{\Cor}[1]{\mbox{Cor} \left[{#1}\right]}
\newcommand{\Vars}[2]{\mbox{Var}_{{#1}} \left[{#2}\right]}
\newcommand{\Prob}[1]{\mathbb{P} \left({#1}\right)}
\newcommand{\Ker}[1]{\mathbb{Q} \left({#1}\right)}
\newcommand{\hProb}[1]{\hat{\mathbb{P}} \left({#1}\right)}

\newcommand{\Indicator}[1]{\mathbb{I}\left\{{#1}\right\}}

\newcommand{\md}{\mathrm{d}} 

\newcommand{\norm}[1]{\left\lVert #1 \right\rVert}
\newcommand{\ie}{\emph{i.e.}}

\newcommand{\var}{\textrm{var}}
\newcommand{\ba}{\mathbf{a}}
\newcommand{\bb}{\mathbf{b}}
\newcommand{\bmf}{\mathbf{f}}
\newcommand{\bg}{\mathbf{g}}
\newcommand{\bh}{\mathbf{h}}
\newcommand{\bn}{\mathbf{n}}

\newcommand{\bv}{\mathbf{v}}
\newcommand{\bx}{\mathbf{x}}
\newcommand{\bX}{\mathbf X}
\newcommand{\bU}{\mathbf U}
\newcommand{\bA}{\mathbf A}

\newcommand{\bD}{\mathbf D}
\newcommand{\bG}{\mathbf G}
\newcommand{\bH}{\mathbf H}
\newcommand{\bI}{\mathbf I}
\newcommand{\bJ}{\mathbf J}
\newcommand{\bK}{\mathbf K}
\newcommand{\bL}{\mathbf L}
\newcommand{\bM}{\mathbf M}
\newcommand{\bp}{\mathbf p}
\newcommand{\bP}{\mathbf P}
\newcommand{\bQ}{\mathbf Q}
\newcommand{\bZ}{\mathbf Z}
\newcommand{\bY}{\mathbf Y}
\newcommand{\be}{\mathbf e}
\newcommand{\bw}{\mathbf w}
\newcommand{\bW}{\mathbf W}
\newcommand{\bz}{\mathbf z}
\newcommand{\by}{\mathbf y}
\newcommand{\bghat}{\widehat{\mathbf g}}
\newcommand{\bzero}{\mathbf 0}

\newcommand{\btheta}{\boldsymbol\theta}
\newcommand{\bV}{\mathbf V}
\newcommand{\bzeta}{\boldsymbol\zeta}
\newcommand{\bphi}{\boldsymbol\phi}
\newcommand{\bPhi}{\boldsymbol\Phi}
\newcommand{\bpsi}{\boldsymbol\psi}
\newcommand{\bxi}{\boldsymbol\xi}
\newcommand{\bnu}{\boldsymbol\nu}
\newcommand{\bvartheta}{\boldsymbol\vartheta}

\newcommand{\bSigma}{\boldsymbol\Sigma}
\newcommand{\cA}{\mathcal{A}}
\newcommand{\cB}{\mathcal{B}}
\newcommand{\cC}{\mathcal{C}}
\newcommand{\cH}{\mathcal{H}}
\newcommand{\cL}{\mathcal{L}}

\newcommand{\cR}{\mathcal{R}}
\newcommand{\cS}{\mathcal{S}}
\newcommand{\cX}{\mathcal{X}}
\newcommand{\sA}{\mathsf{A}}
\newcommand{\refr}{\lambda_{\mbox{r}}}
\newcommand{\xhat}{\widehat{x}}

\newcommand{\bmu}{\boldsymbol{\mu}}
\newcommand{\bthetahat}{\widehat{\boldsymbol\theta}}

\newcommand{\pitilde}{\widetilde{\pi}}
\newcommand{\pitiltil}{\tilde{\tilde{\pi}}}
\newcommand{\data}{\mathcal{D}}
\newcommand{\likelihood}[2]{L\left({#1}; {#2}\right)}
\newcommand{\prior}[1]{\pi_0\left({#1} \right)}

\newcommand{\Lpspace}{\mathcal{L}}

\newcommand{\Normal}{\mathsf{N}}
\newcommand{\Leap}{\mathsf{Leap}}
\newcommand{\qmap}{\mathsf{q}}

\newcommand{\Ihat}{\widehat{I}}
\newcommand{\IACT}{\mathsf{IACT}}

\newcommand{\st}{\btheta} 

\newcommand{\Vsp}{\mathbb{V}}
\newcommand{\kernel}{\mathsf{k}}

\DeclareMathOperator*{\argmax}{arg\,max}
\DeclareMathOperator*{\argmin}{arg\,min}

\begin{document}

  \title{Scalable Monte Carlo for Bayesian Learning}

  \author{Paul Fearnhead, Christopher Nemeth, Chris J. Oates and Chris Sherlock}

  \frontmatter
  \maketitle
  \tableofcontents

  \mainmatter
 


\chapter*{Preface}

At the time of writing, science, industry, and society are being transformed by the emergence of a new generation of powerful machine learning and artificial intelligence (AI) methodologies.
The safe use of such algorithms demands a probabilistic viewpoint, enabling reasoning in settings where data are noisy or limited, and endowing predictions with an appropriate degree of confidence for downstream decision-making and mitigation of risk.
Yet, it remains true that fundamental probabilistic operations, such as conditioning on an observed dataset, are not easily performed at the scale required.

The aim of this book is to provide a graduate-level introduction to advanced topics in Markov chain Monte Carlo (MCMC), as applied broadly in the Bayesian computational context.
Most, if not all of these topics (stochastic gradient MCMC, non-reversible MCMC, continuous time MCMC, and new techniques for convergence assessment) have emerged as recently as the last decade, and have driven substantial recent practical and theoretical advances in the field.
A particular focus is on methods that are \emph{scalable} with respect to either the amount of data, or the data dimension, motivated by the emerging high-priority application areas in machine learning and AI.
Throughout this book, the clear presentation of ideas is prioritised over a rigorous technical treatment of all mathematical details; appropriate references for further reading are provided in the end-notes of each chapter.
In particular, we will limit the use of measure theory; the reader should assume that all sets and functions are measurable with respect to an appropriate sigma-algebra, and all continuous distributions should be assumed to be absolutely continuous with respect to Lebesgue measure and all densities should be assumed to be densities with respect to Lebesgue measure.

This book has been indirectly shaped by the researchers and colleagues -- too numerous to name individually -- who have contributed to recent progress in the field.
Special gratitude must go to Rebekah Fearnhead, Heishiro Kanagawa, Tam\'{a}s Papp and Lorenzo Rimella, for proof-reading the manuscript, to Richard Howey for typesetting the figures, and to Natalie Tomlinson and Anna Scriven for their encouragement and typesetting support.
The authors also wish to thank Daniel Andrade for pointing out errata in the original text.
The authors are grateful for financial support from the Engineering and Physical Sciences Council (through grants EP/W019590/1, EP/R018561/1, EP/R034710/1, EP/V022636/1 and EP/Y028783/1), the Alan Turing Institute, and the Leverhulme Trust.

\vspace{50pt}

\begin{center}
\begin{tabular}{ll}
Paul Fearnhead & Chris J. Oates \\
Christopher Nemeth & \emph{Newcastle University, UK} \\
Chris Sherlock & \\
\emph{Lancaster University, UK} & \\
\end{tabular}
\end{center}

\begin{table}
\renewcommand{\arraystretch}{1.3}
\begin{center}
\textbf{Common Notation}
\begin{tabular}{lp{8cm}}
$n$&total number of iterations of an algorithm or  Monte Carlo sample size. \\
$d$&dimension (of parameter space).\\
$N$&number of elements in the dataset.\\
$\mathcal{D}$ & the dataset $\{\by_1,\ldots,\by_N\}$,.\\
$\boldsymbol{\theta}$&the parameter.\\
$L(\st;\mathcal{D})$ &the likelihood function.\\
$\ell(\st;\mathcal{D})$& the log-likelihood function.\\
$\pi_0(\st)$ & the prior density.\\
$\pi(\st|\mathcal{D})$ & the posterior density, often abbreviated to $\pi(\btheta)$. \\ 
$\mathbb{I}$& the indicator function.\\
$\mathbf{I}_d$& the $d\times d$ identity matrix.\\
$x_i$& the $i$th component of the vector $\bx$.\\
$\bx_k$& the $k$th vector in a sequence $(\bx_k)_{k=1,2,\dots}$.\\
$x_k^{(i)}$ & the $i$th component of the vector $\bx_k$.\\
i.i.d. & independent and identically distributed.\\
$\stackrel{\mathsf{D}}{=}$ & equal in distribution. \\
$\stackrel{\mathsf{D}}{\rightarrow}$ & converges in distribution. \\
$\delta_{\mathbf{x}}$ & the Dirac distribution, which places all mass at $\mathbf{x}$. \\
$\mathsf{N}(\cdot;\bmu,\mathbf{V})$ & the density of a normal random variable with mean $\bmu$ and covariance $\mathbf{V}$. \\
$\mathsf{U}_d(\cdot)$ & the density for a uniform random variable on the $d$-dimensional sphere, 
$\mathbb{S}^{d-1} \subset \mathbb{R}^d$. \\
$\mathcal{L}^p(\pi)$ & the set of measurable functions $f$ with $\int |f(\mathbf{x})|^p \; \mathrm{d}\pi(\mathbf{x}) < \infty$. \\
$C^s(\mathbb{R}^d,\mathbb{R}^p)$ & the set of functions $f : \mathbb{R}^d \rightarrow \mathbb{R}^p$ for which continuous derivatives exist of orders up to $s \in \{0,1,\dots\} \cup \{\infty\}$. \\
$x_n=O(a_n)$& there exist $n_0$ and $M$ such that for all $n\ge n
_0$, $|x_n|\le M a_n$.\\
$X_n=O_p(a_n)$& for any $\epsilon>0$ there exist $n_0$ and $M$ such that  for all $n \ge n_0$, $\Prob{|X_n|> M a_n}< \epsilon$.
\end{tabular}
\end{center}
\end{table}


\chapter{Background}
\label{chap:background}


This book describes some recent developments in \emph{scalable Monte Carlo} algorithms and their applications within Bayesian learning: what exactly does this mean?

Monte Carlo methods are a class of computational methods that involve repeated sampling to numerically approximate quantities of interest. 
We specifically focus on Monte Carlo integration\index{Monte Carlo integration} methods, which are sampling-based methods for evaluating or approximating the value of integrals. 
Such methods are widely used across science and engineering, but our motivation comes particularly from Bayesian statistics.  One of the key quantities in Bayesian statistics is the posterior distribution\index{posterior distribution}, which encapsulates our belief 
regarding unknown parameters of a model given our prior belief and an observed dataset. 
We can then obtain estimates of the parameters, or quantify our uncertainty about the parameters, in terms of expectations with respect to the posterior distribution\index{posterior distribution}. For example, a common estimate of a parameter is the posterior expectation of that parameter; the  predictive probability of future observations is the expectation of the density/mass function of the future observation taken with respect to the posterior distribution. Calculating these expectations involves evaluating an integral, and the idea of Monte Carlo is to use samples from the posterior to estimate such integrals.

The main challenge with using Monte Carlo in Bayesian statistics is often in deriving an efficient algorithm to sample from the posterior distribution\index{posterior distribution}. Markov chain Monte Carlo\index{Markov chain Monte Carlo} (MCMC) is a general and widely used class of methods for sampling from a distribution, based on simulating a Markov process that has the posterior distribution\index{posterior distribution} as its stationary distribution\index{stationary distribution}. 

In recent years, there has been interest in applying MCMC\index{Markov chain Monte Carlo} to ever-increasingly complex and challenging problems. For example, the dimension, $d$ say, of the parameter space of the models we wish to fit to data, or the number of data points, $N$ say, in our data set can be large. As either $d$ or $N$ increases, the efficiency of MCMC methods may reduce. For example, as $d$ increases, we may need to have more iterations of our MCMC algorithm to achieve the required level of accuracy, while as $N$ increases, the computational cost per iteration of a standard algorithm will increase. \emph{Scalable MCMC}\index{Markov chain Monte Carlo} methods are specifically those methods that can scale well as either or both  $d$ and $N$ increase.

The remainder of this introductory chapter will cover background relevant to scalable MCMC\index{Markov chain Monte Carlo}. Section~\ref{sec.MonteCarloWhole} will introduce Monte Carlo methods, explain why Monte Carlo integration\index{Monte Carlo integration} is widely used, and explain how it is relevant to Bayesian statistics. This will be followed by an introduction to some of the statistical models and applications that will be used to demonstrate the methods in this book, as well as an informal and brief introduction to some of the concepts from stochastic processes that will be used in later chapters.  Finally, the chapter ends with a short introduction to kernel methods in preparation for a deeper exposition in Chapter \ref{chap:stein}.

\section{Monte Carlo Methods}
\label{sec.MonteCarloWhole}

\subsection{What Is Monte Carlo Integration?}\index{Monte Carlo integration}
\label{subsec: what is MC}

Assume we have a distribution of interest. For simplicity of presentation, here and for the remainder of this chapter, we assume that the distribution is continuous on $\mathbb{R}^d$. Let $\bX$ denote a random variable with this distribution, and let $\pi(\bx)$ denote the corresponding probability density function for $\bX$; we will also use $\pi$ to refer to the distribution itself when that is necessary.  Then the expectation of some function $h$ of $\bX$ is an integral
\[
I=\Expect{h(\bX)} = \int h(\bx)\pi(\bx) \; \mbox{d}\bx.
\]
This expectation is well-defined, that is, $h$ is integrable with respect to $\pi$, if $\int |h(\bx)|\pi(\bx)~\md \bx<\infty$. We abbreviate this to $h\in \Lpspace^1(\pi)$ and throughout this section, we assume that this holds true.
If we can sample from $\pi(\cdot)$ then we can estimate this expectation/integral by (i) drawing $n$ independent realisations, $\bx_1,\ldots,\bx_n$, from $\pi(\cdot)$ and (ii) calculating the sample average of the values $h(\bx_1), \ldots, h(\bx_n)$. This gives an estimate of $I$, namely
\[
\hat{I} = \frac{1}{n} \sum_{k=1}^n h(\bx_k).
\]
This is called a \emph{Monte Carlo} estimate of $I$, as it is obtained from independent, random samples from $\pi(\cdot)$.

The Monte Carlo estimator can be interpreted as being based on $n$ independent random variables $\bX_1,\dots,\bX_n$, of which $\bx_1,\dots,\bx_n$ are realisations. Is it a good estimator? This is impossible to answer in generality, but we can at least describe some good properties that the estimator can admit. 
First, since the $\bX_i$ are i.i.d., $\Expect{\hat{I}}=\Expect{h(\bX_1)}=I$, so the estimator is \emph{unbiased}. 
Second, and more importantly, the strong law of large numbers\index{strong law of large numbers} tells us that we can make our estimate arbitrarily accurate, with high probability, if we choose $n$ large enough. 
Formally, provided $I$ is well-defined, that is, $h\in \Lpspace^1(\pi)$, and our samples from $\pi(\cdot)$ are independent, then as $n\to \infty$,
\begin{equation}
\label{eqn.SLLNiid}
\frac{1}{n} \sum_{k=1}^n h(\bX_k) \to I~~~\mbox{almost surely.}
\end{equation}
\emph{Almost sure} convergence means that the collection of outcomes where the convergence does not occur has a combined probability of $0$.

Thus, with high probability, our Monte Carlo estimator will be accurate if we choose $n$ large enough, but the result does not tell us how large $n$ needs to be, nor how accurate the estimator will be for a given value of $n$. However, provided that $\int h(\bx)^2\pi(\bx) \; \mbox{d}\bx < \infty$, which we abbreviate to $h\in \Lpspace^2(\pi)$, we can use the central limit theorem\index{central limit theorem} to answer these questions. Again assume that our samples from $\pi(\cdot)$ are independent, and define
\[
V=\int \{h(\bx)-I\}^2\pi(\bx) \; \mbox{d}\bx.
\]
Then, the central limit theorem\index{central limit theorem} states that
\[
\sqrt{n} \left( \frac{\frac{1}{n} \sum_{k=1}^n h(\bX_k)-I}{\sqrt{V} } \right) \stackrel{\mathsf{D}}{\rightarrow} \mathsf{N}(0,1),
\]
as $n\rightarrow \infty$. 
Here the convergence is in distribution, and we have convergence to a standard normal distribution in the limit.

One way of interpreting this result is that, for large enough $n$, approximately,
\[
\frac{1}{n} \sum_{k=1}^n h(\bX_k) \sim \mathsf{N}\left(I,\frac{V}{n} \right).
\]
That is our estimator will be approximately normally distributed, with mean equal to the integral, $I$, and a variance that is $V/n$. This shows that the quantity $V$ governs how easy it is to estimate $I$ via Monte Carlo integration\index{Monte Carlo integration}, and the accuracy depends on both $V$ and $n$. The order of the error of a Monte Carlo estimator is $\sqrt{V/n}$ and thus, the Monte Carlo error decays with sample size at a rate of $n^{-1/2}$.

\subsection{Importance Sampling}\index{importance sampling}

What if we are interested in calculating or approximating a more general integral, $I=\int_{\Omega} g(\bx)\mbox{d}\bx$, of some function $g$ over a region $\Omega$? We can use Monte Carlo sampling to estimate this integral by rewriting the integral as an expectation with respect to some density function $q(\cdot)$ defined on $\Omega$ as follows,
\[
I=\int_{\Omega} \frac{g(\bx)}{q(\bx)} q(\bx) \; \mbox{d}\bx = \Expect {h(\bX)},
\]
where $h(\bx)=g(\bx)/q(\bx)$. 
If $I$ is well-defined, that is $\int |g(\bx)|\mbox{d}\bx<\infty$, then $h \in \Lpspace^1(q)$,
and $I$ can be estimated using Monte Carlo integration\index{Monte Carlo integration} as expressed above, based on independent realised samples $\bx_1,\ldots,\bx_n$ from $q(\cdot)$ by calculating the arithmetic mean of $h(\bx_1),\ldots,h(\bx_n)$. This process is called \emph{importance sampling}\index{importance sampling}, and $q(\cdot)$ is known as the proposal distribution\index{proposal distribution}. 

For this Monte Carlo estimator to be feasible, we have two constraints on $q$. 
First, we need $q(\bx)>0$ whenever $g(\bx)>0$, in order for $h(\bx)$ to be well-defined. 
Second, we need to be able to easily sample from $q(\cdot)$. The choice of $q(\cdot)$ will affect the accuracy of the estimator, with the variance of our estimator for a Monte Carlo sample of size $n$ being $V/n$ where
\[
V=\int \left(\frac{g(\bx)}{q(\bx)} - I \right)^2q(\bx) \; \mbox{d}{\bx}.
\]
This variance will be small if $g(\bx)/q(\bx)$ is roughly constant, and one can show that the optimal choice of $q(\cdot)$ in terms of minimising $V$ is $q(\bx) \propto |g(\bx)|$. If $g(\bx)$ is non-negative everywhere (or non-positive everywhere), then such a choice of $q$ will give an estimator that has zero-variance, that is, an exact estimator. More generally, the variance $V$ will be large if there are values of $\bx$ for which $g(\bx)/q(\bx)$ is large. This leads to a rule of thumb that, if $\Omega$ is unbounded, one wants $q(\bx)$ to have heavier tails than $|g(\bx)|$ to avoid this ratio blowing up as $\|\bx\|\rightarrow\infty$.

\subsection{Monte Carlo or Quadrature?}

It is natural to ask why one should use Monte Carlo integration\index{Monte Carlo integration} when there are alternative numerical integration\index{numerical integration} methods, such as quadrature\index{quadrature}. To see the potential benefits of Monte Carlo methods, consider estimating an integral on the unit hypercube $[0,1]^d$. We can then compare quadrature with Monte Carlo integration\index{Monte Carlo integration} based on samples from a uniform distribution on $[0,1]^d$.

First, consider $d=1$. In this case, quadrature methods tend to be much more accurate than the Monte Carlo methods. We have seen that the Monte Carlo variance, if we have $n$ Monte Carlo samples, is $O(1/n)$, which means that the error of our Monte Carlo estimator will be $O_p(n^{-1/2})$. 

\begin{figure}
    \centering
    \includegraphics[width=0.7\textwidth]{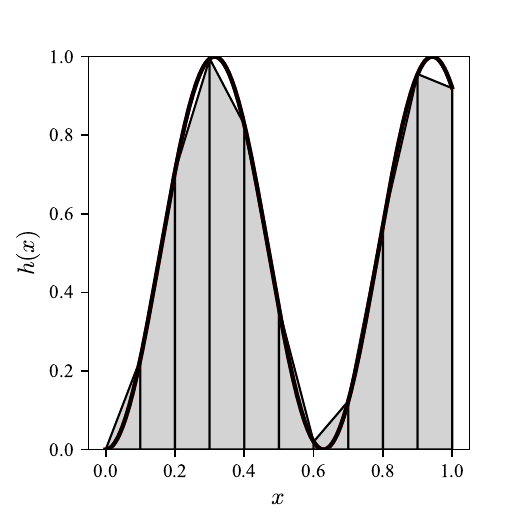}
    \caption{Example of trapezoid rule. We can estimate the integral, by (i) setting $x_1,\dots,x_n$ to be evenly spaced points on $[0,1]$; (ii) creating $n-1$ trapezoids based on joining up the points $(x_k,h(x_k))$ (shaded in regions); and (iii) estimating the integral by the sum of the areas of the trapezoids.
    }
    \label{fig:ch1_trapezoid}
\end{figure}

By comparison, a simple numerical method is the trapezoidal rule. This involves evaluating the integrand, $h(x)$ at a set of equally spaced points, $x_1,\dots,x_n$, on $[0,1]$, 
and approximating the integral using the total area of the trapezoids formed by joining up the points $(x_k,h(x_k))$ for $k=1,\ldots,n$; see Figure \ref{fig:ch1_trapezoid}. Assuming our integrand has a bounded second derivative $|h''(x)|<L$ for some $L$, then we can bound the error in the estimate of the integral as $L\delta^2/12$, where $\delta=1/(n-1)$ is the width of each trapezoid. This gives an error that decays like $O(1/n^2)$, which is much better than the Monte Carlo method. Furthermore, higher-order quadrature\index{quadrature} methods, such as Simpson's rule, can obtain even faster decay of the approximate error with $n$, if the integrand is sufficiently smooth.

So, quadrature\index{quadrature} methods can be more accurate for one-dimensional integrals, at least for functions whose second derivatives are bounded. However, now consider higher-dimensional integrals involving functions $h(\bx)$, the only information about which we have is that the second-order (partial) derivatives are bounded. Then we can apply a cubature rule based on a grid of $m+1$ equally spaced points in each dimension.  The spacing of these points will be $\delta=1/m$ and there will be $n=(m+1)^d$ points in total.  If we have a cubature whose error decays like $\delta^r$, for some power $r$, for example, $r=2$ for the trapezoidal rule, then the error decays at a rate of $m^{-r}\approx n^{-r/d}$. For large $d$, this convergence will be slower than the $n^{-1/2}$ rate of Monte Carlo integration\index{Monte Carlo integration}, explaining why Monte Carlo is often the default method for numerically approximating high-dimensional integrals. To overcome this \emph{curse of dimension}\index{curse of dimension} in cubature, it is usually necessary to identify a sense in which the integrand $h(\bx)$ is effectively low-dimensional, which can be difficult or impossible depending on the applied context.

\subsection{Control Variates}\index{control variate}
\label{sec.ControlVariatesBasics}

Let us return to the problem of estimating the expectation of some function of a random variable,
\[
I=\Expect{h(\bX)}=\int h(\bx)\pi(\bx) \; \mbox{d}\bx,
\]
where $\pi(\bx)$ is the density of $\bX$. We have seen how we can estimate this using a sample from $\pi(\cdot)$, and that the accuracy of this estimator is proportional to 
\[
V=\int \left\{h(\bx)-I\right\}^2\pi(\bx) \; \mbox{d}\bx = 
\int h(\bx)^2 \pi(\bx) \; \mbox{d}\bx -I^2.
\]
The latter expression is just the standard expression for the variance of $h(\bX)$. This shows that it is easier to estimate expectations of functions that vary less when evaluated at $\bX$.

Assume that we know the expectation of a set of random variables $g_1(\bX),\ldots,g_J(\bX)$, each a transformation of $\bX$. Without loss of generality, we can assume that these random variables have mean zero, that is,
\[
\Expect{g_j(\bX)}=0, ~~ \mbox{for } j=1,\ldots,J,
\]
as, if this is not the case, we can define new random variables equal to the old random variables minus their expectations. Then, for any constants $\gamma_1,\ldots,\gamma_n$,
\begin{equation} \label{eq:ch1_cv}
I=\Expect{h(\bX)} - \sum_{j=1}^J \gamma_j \Expect{g_j(\bX)} =
\Expect{h(\bX)- \sum_{j=1}^J \gamma_j g_J(\bX)}.
\end{equation}
By suitable choice of the constants $\gamma_1,\ldots,\gamma_J$, the variability of the random variable $h(\bX)-\sum_{j=1}^J \gamma_j g_j(\bX)$ can be made smaller than that of $h(\bX)$, and thus a Monte Carlo estimate of $I$ based on (\ref{eq:ch1_cv}) will have smaller Monte Carlo variance than the basic Monte Carlo estimator. We call $\sum_{j=1}^J \gamma_j g_j(\bX)$ a \emph{control variate}\index{control variate} for $h(\bX)$. Heuristically, we want to choose $\gamma_1,\ldots,\gamma_J$ so that $h(\bX) \approx \gamma_0+\sum_{j=1}^J \gamma_j g_j(\bX)$, which means that $h(\bX)-\sum_{j=1}^J \gamma_j g_j(\bX)$ is approximately constant.

As a simple example, consider estimating the expectation of $\sin(X)$ where $X$ has a standard normal distribution $\mathsf{N}(0,1)$. 
We know that this expectation is 0 as the distribution of $X$ is symmetric about $0$ and $\sin(-x)=-\sin(x)$. We will compare the simple Monte Carlo estimator of the expectation with estimates using control variates\index{control variate} with the functions $g_1(x)=x$ and $g_2(x)=x^2-1$. By using a Taylor expansion of $\sin(x)$ at $x=0$ we have $\sin(x)\approx x$ for small $x$, and thus a simple choice of control variate is $g_1(x)$. 

We show pictorially the benefit of using this control variate\index{control variate} in Figure \ref{fig:ch1_cv}, where we see that $\sin(x)-x \approx 0$ for most $x$
values sampled from the standard normal distribution. This reduces the Monte Carlo variance of the estimate of $\Expect{h(X)}$ by close to a factor of 2.
\begin{figure}
    \centering
    \includegraphics[width=\textwidth]{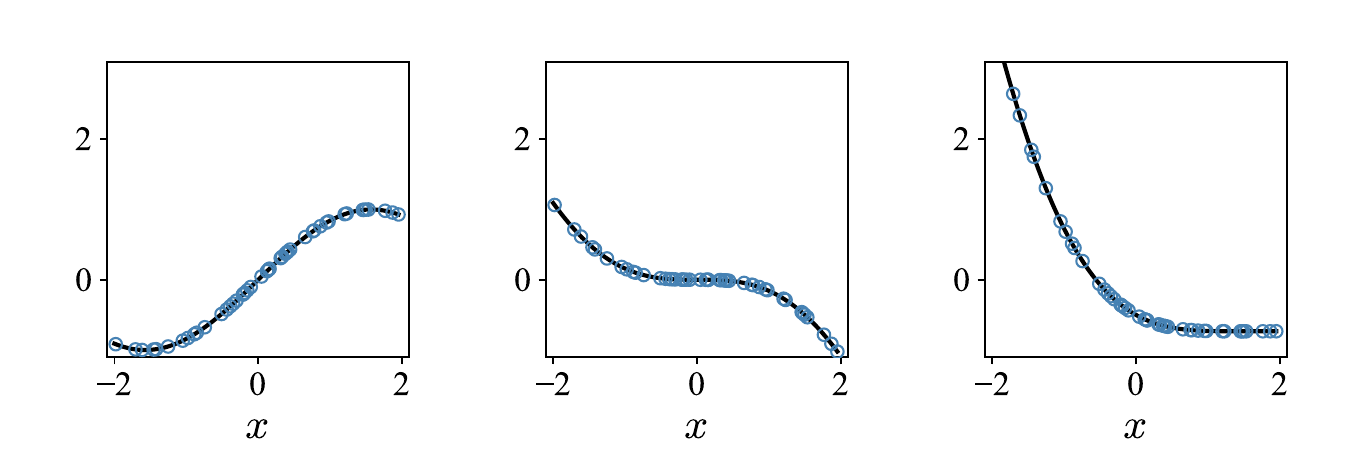}
    \caption{Example of control variates for estimating $\Expect{\sin(X)}$, where $X$ has a standard normal distribution $\mathsf{N}(0,1)$. Each plot shows the function whose expectation is being estimated and 50 values used in the Monte Carlo estimate (dots). From left to right the functions are respectively: $h(x)=\sin(x)$, $h(x)=\sin(x)-x$, and $h(x)=\sin(x) - \pi x/2 + (x^2-1)/2$. 
    The expectation of each function is constructed to be the same. The effect of introducing control variates\index{control variate} in the middle and right-hand plot is to flatten out the function we are integrating -- in the middle plot, this happens for $x\approx0$ and for the right-hand plot for $x\approx \pi/2$. The variability of the function values, i.e. the dots, is smallest for the middle plot and largest for the right-hand plot.
    }
    \label{fig:ch1_cv}
\end{figure}

Care must be taken with control variates\index{control variate}, however. For example, if we perform a Taylor expansion of $\sin(x)$ at $x=\pi/2$ we get $\sin(x) \approx 1-(x-\pi/2)^2/2$, which suggests using $-g_2(x)/2+\pi g_1(x)/2$ as a control variate. However, this choice leads to an increase in the Monte Carlo variance by over a factor of 3. Figure \ref{fig:ch1_cv} shows that the function $\sin(x)-\pi x/2 + (x^2-1)/2$ is roughly constant for $x\approx \pi/2$, but overall it is more variable across the range $x\in[-2,2]$, where most of the probability mass of $\mathsf{N}(0,1)$ lies.

This example shows that the choice of $\gamma_1,\ldots,\gamma_J$ is important when using control variates\index{control variate}. In some situations, there may be a natural way of choosing these -- for example, based on a Taylor expansion of the function of interest around the mode of the distribution of $\bX$. However, it is also possible to choose these values based on simulation. Ideally, we would choose $\gamma_1,\ldots,\gamma_J$ to minimise the Monte Carlo variance
\[
\int \left\{h(\bx)-\sum_{j=1}^J \gamma_j g_j(\bx)\right\}^2\pi(\bx) \; \mbox{d}\bx - I^2,
\]
and we can obtain a Monte Carlo estimate of this. 
If $\bx_1,\ldots,\bx_n$ are realised samples from $\pi(\cdot)$, then we can choose $\gamma_1,\ldots,\gamma_J$ to minimise
\[
\sum_{k=1}^n \left(h(\bx_k)-\sum_{j=1}^J \gamma_j g_j(\bx_k) \right)^2,
\]
which just involves minimising a sum of squares criterion. If we let $\mathbf{h}$ be the $n\times 1$ vector whose $i$th entry is $h(\bx_i)$, $\boldsymbol{\gamma}$ be the $J \times 1$ vector whose $i$th entry is $\gamma_i$, and $\bZ$ be the $n \times J$ matrix whose $(i,j)$th entry is $g_j(\bx_i)$, then, assuming $\bZ$ is of full rank,
the least-squares estimate of $\boldsymbol{\gamma}$ is
\[
\hat{\boldsymbol{\gamma}} = (\bZ^{\top} \bZ)^{-1}\bZ^{\top} \mathbf{h}.
\]
These estimates $\hat{\boldsymbol{\gamma}}$ depend on the Monte Carlo samples, and thus for the Monte Carlo estimate of $I$ to be unbiased, we need to use a new set of Monte Carlo samples from $\bX$ for estimating $I$ using the $\hat{\boldsymbol{\gamma}}$.

While we have presented the idea of control variates\index{control variate} for estimating expectations of functions, similar ideas can be used with importance sampling\index{importance sampling} for estimating general integrals.

\subsection{Monte Carlo Integration and Bayesian Statistics}\index{Monte Carlo integration}
\label{sec.MCBayes}

One of the most important applications of Monte Carlo methods occurs within Bayesian statistics. To explain why, consider the problem of making inferences, from data, about the parameter of a statistical model. We will use the notation $\data$ to denote data in general. In some situations, we will need to distinguish individual data points, and in those settings, we will assume $\data=\{\by_1,\ldots,\by_N\}$, with $\by_i$ being the $i$th data point and $N$ being the size of our dataset. 

We further assume that we have a model for the data. Let the model depend on a parameter $\btheta$, and denote the likelihood of the data under our model by $\likelihood{\btheta}{\data}$. The likelihood is the probability, or probability density, of observing data $\data$ under our model if the parameter is $\btheta$. In Bayesian statistics, we represent beliefs, or uncertainty, about the parameter, $\btheta$, through probability distributions. Our beliefs about $\btheta$ before seeing the data are given by a prior, $\prior{\btheta}$, and, once we observe data, Bayes' theorem provides the update to the posterior distribution\index{posterior distribution}:
\begin{equation}
\label{eqn.postispriortimeslike}
\pi(\btheta\mid \data) \propto \prior{\btheta} \likelihood{\btheta}{\data}.
\end{equation}
Where it will not cause confusion, we may drop the explicit conditioning on the data in the posterior and write $\pi(\btheta)$ rather than $\pi(\btheta\mid \data)$.

Assuming the correctness of our model,
the posterior distribution\index{posterior distribution} contains all information about the parameter, $\btheta$, that can be logically deduced from our prior belief and the dataset. 
From it, we can then obtain a point estimate for $\btheta$, such as its posterior mean, and quantify uncertainty in terms of the posterior probability of $\btheta$ lying in a given set of values. However, in most applications, the posterior distribution\index{posterior distribution} is intractable, meaning that it cannot be explicitly calculated. 
The central challenge is that the posterior density $\pi(\btheta\mid\data)$ is known, via Bayes' theorem, only up to a normalising constant.

The intractability of the posterior distribution\index{posterior distribution} is a key motivator for Monte Carlo methods. If we can draw samples from $\pi(\btheta\mid\data)$,  
then we can obtain simple, and often accurate, Monte Carlo estimates of posterior quantities of interest.
Given realisations $\btheta_1,\ldots,\btheta_n$ sampled from $\pi(\btheta\mid\data)$, and a function $h(\btheta)$ whose expectation
\begin{equation*}
    I := \Expects{\pi}{h(\btheta)} = \int h(\btheta) \pi(\btheta\mid\data) \; \mathrm{d}\btheta
\end{equation*}
is of interest, define

\begin{equation}
\label{eqn.gen.MCaverage.pi}
\Ihat_n(h) := \frac{1}{n}\sum_{k=1}^n h(\btheta_k).
\end{equation}
As mentioned earlier, for any function $h \in \Lpspace^1(\pi)$, the strong law of large numbers\index{strong law of large numbers} \eqref{eqn.SLLNiid} tells us that we can estimate $\Expects{\pi}{h(\btheta)}$ as accurately as we desire using Monte Carlo integration\index{Monte Carlo integration}, and provided enough samples are taken:  $\Ihat_n(h)to \Expects{\pi}{h(\btheta)}$ almost surely as $n\to \infty$. 
Moreover, if $h \in \Lpspace^2(\pi)$ then the central limit theorem\index{central limit theorem} states that the Monte Carlo error, $\Ihat_n(h)-\Expects{\pi}{h(\btheta)}$ is $O_p(n^{-1/2})$. 

For example, the vector of posterior means can be estimated by
\[
\hat{\btheta}=\frac{1}{n} \sum_{k=1}^n \btheta_k,
\]
and the posterior probability of $\btheta\in\mathcal{B}$ for some set $\mathcal{B}$ can be estimated by the proportion of Monte Carlo samples in $\mathcal{B}$
\[
{\hProb{\btheta\in\mathcal{B}}} = \frac{1}{n} \sum_{k=1}^n \Indicator{\btheta_k\in\mathcal{B}},
\]
where $\Indicator{\btheta_k\in\mathcal{B}}$ is the indicator function of the event $\btheta_k\in\mathcal{B}$.

The challenge with this Monte Carlo approach to Bayesian statistics is the difficulty in sampling from $\pi(\btheta)$, particularly if $\btheta$ is high-dimensional. Of the Monte Carlo integration\index{Monte Carlo integration} methods we have mentioned so far, importance sampling\index{importance sampling} offers an alternative when we are unable to sample from $\pi$ directly. Consider estimating the posterior expectation for some function $h(\btheta)$, so $h(\btheta)=\btheta$ would give us the posterior mean of $\btheta$ and $h(\btheta)=\Indicator{\btheta\in\mathcal{B}}$ would give us the posterior probability of $\btheta\in\mathcal{B}$. Let $q(\btheta)$ be a proposal distribution\index{proposal distribution} with the same support as the posterior. 
Then we have
\[
\Expect{h(\btheta) \mid \data} = \int h(\btheta)\pi(\btheta) \; \mbox{d}\btheta = \int \frac{h(\btheta)\pi(\btheta)}{q(\btheta)} q(\btheta) \; \mbox{d}\btheta.
\]
It is common to define \textit{weights} $w(\btheta):=\pi(\btheta)/q(\btheta)$. Then given an independent sample $\btheta_1,\dots,\btheta_n$ from $q(\btheta)$, we can estimate the posterior expectation by the importance sampling\index{importance sampling} estimator
\[
\frac{1}{n} \sum_{k=1}^n w(\btheta_k)h(\btheta_k).
\]

There are two issues with this estimator. The first is that as we only know the posterior up to a constant of proportionality, we only know the weights up to a constant of proportionality. However, the constant of proportionality can be estimated by setting $h(\btheta)=1$, whence $\Expect{h(\btheta)}=1$ 
as the expectation of a constant is the constant. Thus we can use the unnormalised posterior density in the definition of the weights and estimate the normalising constant as $(1/n)\sum_{k=1}^n w(\btheta_k)$. The posterior expectation of $h(\btheta)$ is then estimated as
\[
\sum_{k=1}^n \frac{w(\btheta_k)}{\sum_{j=1}^n w(\btheta_j)}h(\btheta_k),
\]
which requires knowing the posterior density only up to a constant of proportionality. Often we define normalised weights, $w^*(\btheta_k)=w(\btheta_k)/\sum_{j=1}^n w(\btheta_j)$, and we can then view the weighted samples $(\btheta_k,w^*(\btheta_k))$, for $k=1,\ldots,n$, as a discrete approximation to the posterior. 

The second issue is that the Monte Carlo variances of our estimators of posterior expectations depend on the variability of the weights. Often this will be large if $\btheta$ is high-dimensional. To see this, consider a toy example where the posterior has independent components. Assume each component is normal with mean 0 and variance $\sigma^2$, and the importance-sampling\index{importance sampling} proposal distribution\index{proposal distribution} is also independent over components, but with a standard normal distribution, that is, with mean zero and unit variance, for each component.
The importance sampling\index{importance sampling} weight for a realisation $\btheta=(\theta_{1},\ldots,\theta_{d})$ is
\[
w(\btheta)= \sigma^{-d} \exp\left\{ \frac{\sigma^2-1}{2\sigma^2} \sum_{i=1}^d \theta_{i}^2 \right\}.
\]
Now $\sum_{i=1}^d \theta_{i}^2 $ has a $\chi^2_d$ distribution under the proposal, and using the moment generating function of a $\chi^2_d$ distribution, we obtain the Monte Carlo variance of the weight:
\[
\var\{ w(\btheta) \} = \sigma^{-d} \left(2-\sigma^2 \right)^{-d/2} -1.
\]
Writing $\sigma^2=1+\epsilon$, for some $\epsilon>0$, this variance is $(1/\sqrt{1-\epsilon^2})^{d}-1$, which increases exponentially with $d$. 
The focus of the MCMC\index{Markov chain Monte Carlo} methods that we introduce in Chapter~\ref{chap:intro-mcmc} is to produce sampling algorithms that avoid this exponential curse of dimensionality\index{curse of dimension}.

\section{Example Applications}

In later chapters, we will demonstrate the Monte Carlo methods on some example models that we now introduce. Whilst these models are somewhat simple to describe, their posteriors exhibit many of the features of more challenging posterior distributions\index{posterior distribution}, in particular, with respect to scalable sampling.

\subsection{Logistic Regression} \label{sec:ch1-logistic}

Logistic regression\index{logistic regression} models the relationship between a binary response and a set of covariates. Denote the responses by $y_1,\ldots,y_N$ and the covariates by $d$-dimensional vectors $\bx_1,\ldots,\bx_N$. Then, logistic regression\index{logistic regression} models the data (the responses) as conditionally independent, given a $d$-dimensional parameter $\btheta$ and the covariates, and that
\[
\Prob{Y=y_j|\bx_j,\btheta} = \frac{\exp\{y_j\bx_j^{\top}\btheta\}}{1+\exp\{\bx_j^{\top}\btheta\}}.
\]
An intercept term can be included in the model by setting the first coordinate of each of $\bx_1,\ldots,\bx_N$ to be 1.

Our interest will be in sampling from the posterior distribution\index{posterior distribution} of $\btheta$. To define the posterior, we need to specify a prior $\pi_0(\btheta)$, and we will assume that our prior is Gaussian with mean $\bzero$ and variance $\bSigma_{\btheta}$. This leads to a posterior distribution\index{posterior distribution}, $\pi(\btheta|\mathcal{D})$, which can be written succinctly up to a  multiplicative constant as
\[
\pi(\btheta|\mathcal{D}) \propto \exp\left\{-\frac{1}{2}\btheta^{\top}\bSigma_{\btheta}^{-1}\btheta \right\} \prod_{j=1}^N \frac{\exp\{y_j\bx_j^{\top}\btheta\}}{1+\exp\{\bx_j^{\top}\btheta\}}.
\]

This is a canonical, albeit relatively simple, test problem for sampling methodologies. 
When we consider sampling methods for this model, we will drop the explicit conditioning on data $\mathcal{D}$ and use $\pi(\btheta)$ to denote the target distribution of the sampler. 
The samplers we consider will often use gradient\index{gradient} information about their target distribution, and we have
\begin{equation} \label{eq:ch1-logisticgradient}
\frac{\partial \log \pi(\btheta)}{\partial \theta_{i}} = -\left[\btheta^{\top}\bSigma_{\btheta}^{-1}\right]_{i} + \sum_{j=1}^N x_j^{(i)}\left\{y_j - \frac{\exp\{\bx_j^{\top}\btheta\}}{1+\exp\{\bx_j^{\top}\btheta\}}\right\},
\end{equation}
where $x_j^{(i)}$ indicates the $i$th component of $\bx_{j}$.

\subsection{Bayesian Matrix Factorisation} \label{sec:ch1-BMF}

Bayesian matrix factorisation\index{Bayesian matrix factorisation} attempts to find a representation of a high-dimensional matrix as the product of two lower-dimensional matrices. 
Consider an $n\times m$ matrix $\bY$, and let $\btheta=\{\bU,\bV\}$ where $\bU$ and $\bV$ are  $n\times d$ and $d\times m$ matrices, respectively. Then the approximation is $\bY\approx\bU\bV$. If $d \ll \min\{m,n\}$ then this can lead to a substantial reduction in dimension, and the model can be viewed as attempting to find low-dimensional structure in $\bY$.

The interpretation of this model is that each row of $\bV$ is a \emph{factor}, and we aim to approximate each row of $\bY$ as a linear combination of these factors. The entries in $\bU$ are called \emph{factor loadings}, and give the relative weight of each factor for each row of $\bY$. 

One common approach to fitting these models is to use a Gaussian working model; thus, up to an additive constant, the log-likelihood is 
\[
L(\btheta; \data) = -{nm}\log \sigma-\frac{1}{2\sigma^2} \left\{ \sum_{i=1}^n\sum_{j=1}^m \left(Y_{i,j}- \sum_{k=1}^d U_{i,k}V_{k,j}\right)^2
\right\},
\]
where $\sigma^2$ is the variance of the difference between entries of $\bY$ and $\bU\bV$.
In Bayesian matrix factorisation\index{Bayesian matrix factorisation}, we then introduce a prior on the parameters $\bU$ and $\bV$. Often, the prior for each entry is Gaussian, or is a mixture of a Gaussian and a point mass at zero, as this encourages sparsity in the factors, which potentially aids the interpretation of $\bU$ and $\bV$. It is also possible to introduce a prior over the number of factors, $d$, with the priors for the entries of $\bU$ and $\bV$ potentially depending on $d$.

\subsection{Bayesian Neural Networks for Classification} 
\label{sec:ch1-BNN}

Artificial neural networks\index{Bayesian neural network} are a flexible and popular class of models used in machine learning for solving supervised learning problems, such as regression and classification tasks. In the case of classification, assume that $y_1, y_2, \ldots,y_N$ are observed data, where each $y_j$ represents one of $G$ classes; that is, $y_j \in \{1,2,\ldots,G\}$. Assuming $d-$dimensional vectors $\bx_1,\bx_2,\ldots,\bx_N$ for the covariates, then under a simple two-layer neural network model, the probability of a particular class $y_j$ is
\begin{equation}
    \label{eq:ch1-nn-model}
    \mathbb{P}(Y=y_j | \bx_j, \btheta) \propto \exp(\mathbf{A}^\top_{y_j} \sigma(\mathbf{B}^\top \bx_{j} + \mathbf{b}) + a_{y_j}), 
\end{equation}
where $\mathbf{b}$ is a $d_h$-dimensional vector and $\mathbf{B}$ is a $d\times d_h$ matrix, with $d_h$ the dimension of the variables in the hidden layer. The function $\sigma:\mathbb{R}^{d_h} \rightarrow (0,1)^{d_h}$ is a vector softmax function with $\sigma(\bz)_i=\exp(z_i)/\{\sum_{j=1}^{d_h} \exp(z_j)\}$ for $i=1,\ldots,d_h$. The notation $\mathbf{A}_i$ refers to the $i$-th column of the $d_h\times G$ matrix $\mathbf{A}$. The parameters of the model $\btheta=\mbox{vec}(\mathbf{a},\mathbf{A},\mathbf{b},\mathbf{B})$ are represented by vectors $\mathbf{a},\mathbf{b}$, commonly referred to as \textit{biases}, and matrices $\mathbf{A}$ and $\mathbf{B}$ are commonly referred to as \textit{weights}. 

Taking a Bayesian approach to parameter estimation, we can place independent Gaussian priors on each of the elements of the biases and weights in $\btheta.$ Monte Carlo algorithms can be used to sample from the posterior, 

\begin{equation}
    \label{eq:ch1-bnn-post}
    \pi(\btheta|\mathcal{D}) \propto \pi_0(\btheta) \prod_{j=1}^N \mathbb{P}(y_j|\bx_j,\btheta).
\end{equation}

For Bayesian neural network\index{Bayesian neural network} models, the dataset sizes tend to be very large, and approximating the posterior distribution\index{posterior distribution} requires Monte Carlo methods which are scalable to large datasets. In Chapter \ref{chap:sgld}, we will use stochastic gradient MCMC algorithms to approximate the posterior distribution of the Bayesian neural network\index{Bayesian neural network} model.\index{posterior distribution} 

\section{Markov Chains} \label{sec:ch1-MC}
This section describes discrete-time Markov chains\index{Markov chain!discrete-time}, focusing on the concepts that will be required to understand the MCMC\index{Markov chain Monte Carlo} method and its efficiency: the stationary distribution\index{stationary distribution}, reversibility\index{reversibility}, convergence to the stationary distribution, ergodic averages, integrated auto-correlation time,\index{integrated auto-correlation time} and effective sample size. 

\begin{definition}
A \emph{discrete-time Markov chain}\index{Markov chain!discrete-time} on a state space $\cX$ is a collection of random variables $\{X_k\}_{k=0}^\infty$ with each $X_k\in \cX$, such that for any 
$\cA \subseteq \cX$, 
\begin{equation}
\Prob{X_{k+1}\in \cA\mid X_{k}=x_{k},\dots,X_0=x_0}=\Prob{X_{k+1}\in \cA\mid X_{k}=x_{k}};
\end{equation}
conditional on the current state, the distribution of the next state is independent of all previous states.
\end{definition}

In this chapter, we will only consider \emph{homogeneous} Markov chains, where the distribution of the next state given the current state does not depend on the value of $k$.
Such a chain has a stationary distribution\index{stationary distribution}, $\nu$, if $X_k\sim \nu\implies X_{k+1}\sim \nu$. If the chain also has a unique limiting distribution, then this must be $\nu$ since, by repeated induction, if $X_j\sim \nu$ then $X_k\sim \nu$ for all $k>j$, including as $k\to \infty$. 

The following two examples of Markov chains on the vertices of an $m$-sided polygon illustrate different ways that a chain can be stationary. We label the vertices of the polygon from $0$ to $m-1$, increasing in a clockwise direction; thus, $\cX=\{0,1,\dots,m-1\}$.

\begin{example}
\label{example.ngon.nrev}
(See Figure~\ref{fig.ngonrandomwalk}, left.) Let $\{X_k\}_{k=0}^\infty$ be a Markov chain on the vertices of an $m$-sided polygon where the state at time $k+1$ is obtained from the state at time $k$ by moving to the next vertex in a clockwise direction. If at time $k$ the chain is equally likely to be at each of the vertices, then this is still the case at time $k+1$. The stationary distribution\index{stationary distribution} has $\Prob{X_k=x}=1/m$ for $x\in\cX$.  
\end{example}

\begin{example}
\label{example.ngon.rev.np}
(See Figure~\ref{fig.ngonrandomwalk}, right.) 
Let $\{X_k\}_{k=0}^\infty$ be a Markov chain on the vertices of an $n$-sided polygon where the state at time $k+1$ is obtained from the state at time $k$ by performing one of the following moves, each of which has a probability of $1/3$: move to the next vertex in an anti-clockwise direction; do not move; move to the next vertex in a clockwise direction. As with Example \ref{example.ngon.nrev} the stationary distribution\index{stationary distribution} has $\Prob{X_k=x}=1/m$ for $x\in\cX$.
\end{example}

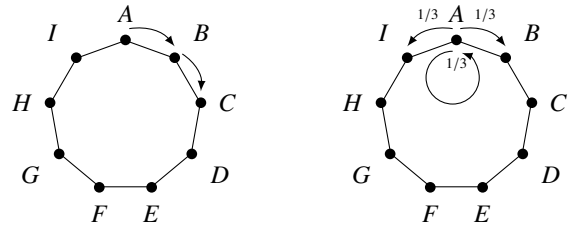
\begin{figure}
  \begin{subfigure}
    \centering
\begin{tikzpicture}
\def\numSides{9}

\node[draw, circle, minimum size=2cm, regular polygon, regular polygon sides=\numSides] (a) {};

\foreach \x in {1,2,...,\numSides}
  \fill (a.corner \x) circle[radius=2pt];

\node[label=$A$] at (a.corner 1) {};
\node[label=40:$B$] at (a.corner 9) {};
\node[label=0:$C$] at (a.corner 8) {};
\node[label=340:$D$] at (a.corner 7) {};
\node[label=270:$E$] at (a.corner 6) {};
\node[label=270:$F$] at (a.corner 5) {};
\node[label=200:$G$] at (a.corner 4) {};
\node[label=180:$H$] at (a.corner 3) {};
\node[label=140:$I$] at (a.corner 2) {};

\draw[-latex,black] ($(a.corner 1)+(0.05,0.15)$) to[bend left]  ($(a.corner 9)+(0.0,0.15)$);

\draw[-latex,black] ($(a.corner 9)+(0.1,0.05)$) to[bend left]  ($(a.corner 8)+(0.0,0.15)$);

\end{tikzpicture}
  \end{subfigure}
  \quad\quad\quad 
  \begin{subfigure}
    \centering
\begin{tikzpicture}
\def\numSides{9}

\node[draw, circle, minimum size=2cm, regular polygon, regular polygon sides=\numSides] (a) {};

\foreach \x in {1,2,...,\numSides}
  \fill (a.corner \x) circle[radius=2pt];

\node[label=$A$] at (a.corner 1) {};
\node[label=40:$B$] at (a.corner 9) {};
\node[label=0:$C$] at (a.corner 8) {};
\node[label=340:$D$] at (a.corner 7) {};
\node[label=270:$E$] at (a.corner 6) {};
\node[label=270:$F$] at (a.corner 5) {};
\node[label=200:$G$] at (a.corner 4) {};
\node[label=180:$H$] at (a.corner 3) {};
\node[label=140:$I$] at (a.corner 2) {};

\node[label=45:\tiny$1/3$] at (a.corner 1) {};
\node[label=135:\tiny$1/3$] at (a.corner 1) {};
\node[label=270:\tiny$1/3$] at (a.corner 1) {};

\draw[-latex,black] ($(a.corner 1)+(0.05,0.15)$) to[bend left]  ($(a.corner 9)+(0.0,0.15)$);
 
\draw[-latex,black] ($(a.corner 1)+(-0.05,0.15)$) to[bend right]  ($(a.corner 2)+(0.0,0.15)$);

\draw[-latex,black] ($(a.corner 1)+(-0.05,-0.15)$) arc
    [
        start angle=90,
        end angle=430,
        x radius=0.35cm,
        y radius =0.35cm
    ] ;

\end{tikzpicture}
  \end{subfigure}
  \caption{9-sided polygon where the Markov chain only moves clockwise (left figure), as in Example \ref{example.ngon.nrev} or moves either a clockwise or anti-clockwise direction with probability $1/3$ (right figure), as in Example 
\ref{example.ngon.rev.np}.\label{fig.ngonrandomwalk}}
\end{figure}

\subsection{Reversible Markov Chains}

Example \ref{example.ngon.nrev} has a clear flow in a clockwise direction and, because of this, is an example of a non-reversible Markov chain; these will be discussed in detail in Chapter \ref{chap:non-reversible}.  By contrast, in Example \ref{example.ngon.rev.np}, consider any two adjacent vertices: at stationarity, the probability of being at the first and moving to the second is the same as the probability of being at the second and moving to the first. Indeed, this is true of any pair of vertices, with the probability being $0$ if they are not adjacent. This is an example of a reversible Markov chain.

\begin{definition}
\label{defn.reversible}
A Markov chain $\{X_k\}_{k=1}^\infty$ with a state space of $\cX$ is \emph{reversible} with respect to a distribution $\nu$ when, for any two sets $\cB,\cC\subseteq \cX$, if $X_k\sim \nu$ then $\Prob{X_k \in \cB, X_{k+1}\in \cC}
=
\Prob{X_{k}\in \cC,X_{k+1}\in \cB }$.
\end{definition}

Consider the decomposition
$$
\Prob{X_k \in \cB, X_{k+1}\in \cC}=\Prob{X_k \in \cB}\Prob{ X_{k+1}\in \cC|X_k\in \cB}.
$$
The first term on the right-hand side is the amount of probability mass in $\cB$ at time $k$ and the second term is the fraction of that mass that moves to C at time $k+1$, so the product is the amount of probability mass moving from $\cB$ to $\cC$. If the chain is reversible with respect to $\nu$ and $X_k\sim \nu$, then this is also the amount of mass moving from $\cC$ to $\cB$. Given this balance, referred to as \emph{detailed balance}\index{detailed balance},  we would expect the total amount of probability mass in any set to remain constant. Indeed, 
setting $\cC=\cX$ in Definition \ref{defn.reversible}, we see that reversibility\index{reversibility} implies that if $X_k\sim \nu$,  $\Prob{X_{k}\in \cB}=\Prob{\cX_{k+1} \in \cB}$. Since this is also true for all $\cB$, $X_{k+1}\sim \nu$.

\subsection{Convergence, Averages, and Variances}
\label{sec.MCcvgESS}

In Example \ref{example.ngon.rev.np}, whatever the value or distribution of $X_0$, as $k\to \infty$ the distribution of $X_k$ converges to the stationary distribution\index{stationary distribution}. For simplicity of presentation, we show this when $m=2m'+1$ is odd. For all $x_0,x\in \cX$, $\Prob{X_{m'}=x|X_0=x_0}\ge 1/3^{m'}$ since it takes at most $m'$ moves in a single direction to reach $x$, and if the chain arrives earlier, we include the probability of it staying at $x$ until time $m'$. Thus, the transition probability after $m'$ steps can be written as a mixture: 
\begin{equation}
\label{eqn.TPmixture}
\Prob{X_{m'}=x|X_0=x_0}= \delta \nu(x) + (1-\delta) q(x|x_0),
\end{equation} for some conditional probability mass function $q$ and with $\delta = m/3^{m'}$. 
The distribution at the start of a given iteration can always be thought of as a mixture of $\nu$ and some other distribution, where the mixture probability for $\nu$ could be $0$. We can imagine that there is a hidden coin, and if it is showing `heads', then the distribution of the chain is $\nu$. Since $\nu$ is the stationary distribution\index{stationary distribution}, if the coin is currently showing `heads', it will still be showing heads after a further $m'$ moves. If the coin is showing `tails',  then  \eqref{eqn.TPmixture} tells us that there is a probability of at least $\delta$ that it will be showing heads after the next $m'$ moves.
Equivalently, the mixture probability of the component that is not $\nu$ has been multiplied by $1-\delta$ or less. After $km'$ iterations, it is, therefore, at most $(1-\delta)^k\to 0$ as $k\to \infty$.

However, convergence to a stationary distribution\index{stationary distribution} does not occur for all Markov chains. The chain in Example \ref{example.ngon.nrev} is deterministic: if $X_0=0$, then $X_{km}=0$ for all integers, $k$. The following examples illustrate two further cases.

\begin{example}
\label{ex.ngon.periodic} (See Figure \ref{fig:enter-label}.) 
    Alter Example \ref{example.ngon.rev.np} so that the chain cannot remain at its current vertex but must move either clockwise or anti-clockwise by a single vertex, each with a probability of $1/2$. As with the Examples \ref{example.ngon.nrev} and \ref{example.ngon.rev.np}, the stationary distribution\index{stationary distribution} has $\Prob{X_k=x}=1/m$ for $x\in\cX$.
\end{example}

 If $n$ is an even number and $X_0$ is even, then the chain in Example \ref{ex.ngon.periodic} only visits even-numbered states at even-numbered times and odd-numbered states at odd-numbered times. Such chains are termed \emph{periodic}\index{period} and clearly do not converge to their stationary distribution\index{stationary distribution}.

 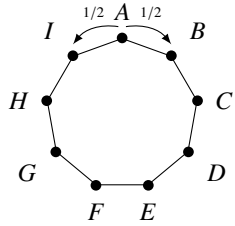
\begin{figure}
     \centering
     \begin{tikzpicture}
\def\numSides{9}

\node[draw, circle, minimum size=2cm, regular polygon, regular polygon sides=\numSides] (a) {};

\node[label=$A$] at (a.corner 1) {};
\node[label=40:$B$] at (a.corner 9) {};
\node[label=0:$C$] at (a.corner 8) {};
\node[label=340:$D$] at (a.corner 7) {};
\node[label=270:$E$] at (a.corner 6) {};
\node[label=270:$F$] at (a.corner 5) {};
\node[label=200:$G$] at (a.corner 4) {};
\node[label=180:$H$] at (a.corner 3) {};
\node[label=140:$I$] at (a.corner 2) {};

\node[label=45:\tiny$1/2$] at (a.corner 1) {};
\node[label=135:\tiny$1/2$] at (a.corner 1) {};

\foreach \x in {1,2,...,\numSides}
  \fill (a.corner \x) circle[radius=2pt];

\draw[-latex,black] ($(a.corner 1)+(0.05,0.15)$) to[bend left]  ($(a.corner 9)+(0.0,0.15)$);
 
\draw[-latex,black] ($(a.corner 1)+(-0.05,0.15)$) to[bend right]  ($(a.corner 2)+(0.0,0.15)$);

\end{tikzpicture}

     \caption{9-sided polygon with Markov transitions described in Example \ref{ex.ngon.periodic}.}
     \label{fig:enter-label}
 \end{figure}

\begin{example}
\label{ex.reducible}
    Consider a Markov chain of the form in Example \ref{example.ngon.rev.np}, but on two separate $m$-sided polygons with no movement between the two. A chain with separate regions between which there can be no movement is termed \emph{reducible}\index{reducible} because it can be reduced to simpler component parts.
\end{example}

A reducible\index{reducible} chain does not even have a single stationary distribution\index{stationary distribution}. In Example \ref{ex.reducible} for any $\beta\in[0,1]$ the distribution with probabilities $\beta/m$ for each vertex on the first polygon and $(1-\beta)/m$ for each vertex on the second polygon is stationary. 

A chain which is not reducible is termed \emph{irreducible}\index{irreducibility}, and a chain which is not periodic is termed \emph{aperiodic}\index{aperiodicity}.

\subsubsection{Ergodic Averages}

The \emph{ergodic theorem}\index{ergodic theorem} for a Markov chain on a general state space, $\mathcal{X}$,  states that provided the chain satisfies natural generalisations of irreducibility\index{irreducibility} and aperiodicity\index{aperiodicity} and has a proper stationary distribution\index{stationary distribution}, $\nu$, then as $k\to \infty$, the distribution of $X_k$ converges to that stationary distribution\index{stationary distribution}. 
Furthermore, subject to the same conditions, for any $h \in \Lpspace^1(\nu)$, samples from the Markov chain satisfy a strong law of large numbers\index{strong law of large numbers}:
\begin{equation}
\label{eqn.SLLNMC}
\Ihat_n(h):=\frac{1}{n}\sum_{k=1}^n h(X_k)
\to
\Expects{\nu}{h(X)}
\end{equation}
almost surely as $n\to \infty$.

\subsubsection{Integrated Auto-Correlation Time and Effective Sample Size} \label{sec:ch1-IACFESS}

Let us assume that $X_0$ is, in fact, drawn from the stationary distribution\index{stationary distribution}. Define $\sigma^2_h:=\Vars{\nu}{h(X)}$ and assume $\sigma^2_h<\infty$. For $k\in\{0,1,2,\dots\}$, the lag-$k$ auto-correlation\index{auto-correlation} is $\rho_k:=\Cor{h(X_0),h(X_k)}=\Cor{h(X_j),h(X_{j+k})}$ since the Markov chain is time-homogeneous. If the $X_j$ were independent samples from $\nu$ then $n \Var{\Ihat_n(h)}=\sigma^2_h$. 
For a stationary Markov chain with
\begin{equation}
\label{eqn.rhocond}
    \sum_{k=1}^\infty |\rho_k|<\infty,
\end{equation}
it holds that
\begin{equation}
\label{eqn.asymp.var}
    \lim_{n\to \infty}n\Var{\Ihat_n(h)}= \sigma^2_h ~\IACT_h,
\end{equation}
where 
\begin{equation}
    \IACT_h:=1+2\sum_{k=1}^\infty \rho_k,
\end{equation}
is the \emph{integrated auto-correlation time}\index{integrated auto-correlation time}.
To see why this is the case, first, without loss of generality, take $h$ to have $\Expects{\nu}{h(X)}=0$; if this is not true initially, we subtract off the expectation: the variance properties are unchanged. Then
\begin{align}
n\Var{\Ihat_n(h)}
&=
\frac{1}{n}\Expect{\sum_{k=1}^n \sum_{j=1}^n h(X_j)h(X_k)}
=\frac{\sigma^2_h}{n}\sum_{k=1}^n \sum_{j=1}^n\rho_{|k-j|}.
\end{align}
But $\sum_{k=1}^n \sum_{j=1}^n\rho_{|k-j|}=n\rho_0+2\sum_{k=1}^n(n-k)\rho_k$, so
\[
\frac{n}{\sigma^2_h}\Var{\Ihat_n(h)}=1 + 2\sum_{k=1}^n \left( 1-\frac{k}{n}\right) \rho_k
=
1+2\sum_{k=0}^\infty \max\left(0,1-\frac{k}{n} \right) \rho_k.
\]
Given \eqref{eqn.rhocond}, the dominated converge theorem\index{dominated convergence theorem} permits us to exchange the ordering of the limit as $n\to \infty$ and the sum over $k$, which provides the limit \eqref{eqn.asymp.var}.

The practical consequence of \eqref{eqn.asymp.var} is that, for finite $n$,
\begin{equation}
\label{eqn.VarIhat.IACT}
\Var{\Ihat_n(h)} \approx \frac{\sigma^2_h}{n/\IACT_h},
\end{equation}
the same as the variance if $n/\IACT_h$ i.i.d. samples from $\nu$ had been used. The quantity $n/\IACT_h$ is, therefore, known as the \emph{effective sample size}\index{effective sample size}. Since they relate directly to the inverse variance of $\Ihat_n(h)$, effective sample size and the inverse of the integrated auto-correlation time\index{integrated auto-correlation time} are useful measures of the efficiency of a Markov chain for estimating $\Expects{\nu}{h(X)}$. 

\section{Stochastic Differential Equations}
\label{sec.SDEs}

The Langevin stochastic differential equation\index{stochastic differential equation} (SDE) is the basis for the Metropolis-Adjusted Langevin algorithm (Section \ref{sec:ch2-MALA}) and for stochastic gradient Langevin methods (Chapter \ref{chap:sgld}). It is also key to understanding the efficiency of various Metropolis--Hastings algorithms when the dimension, $d$, is high (see Chapter \ref{chap:intro-mcmc}). We start with a heuristic introduction to SDEs\index{stochastic differential equation} before considering a special case of the Langevin diffusion\index{Langevin diffusion} known as the Ornstein--Uhlenbeck process\index{Ornstein--Uhlenbeck process} and then moving onto the general Langevin diffusion.

Consider a differential equation of the form
\[
\frac{\md x_t}{\md t}=a(x_t,t),
\]
with a known initial value for $x_0$.
Discretising time leads to the simple Euler\index{Euler-Maruyama}  approximation
\[
\delta x_t \approx a(x_t,t) \delta t,
\]
where $\delta x_t=x_{t+\delta t}-x_t$. 
Setting $\delta t=T/m$, starting with $x_0$, and recursively applying the Euler update $m$ times leads to an approximation $\xhat_T$, which approaches the true value $x_T$ as $m\to \infty$.

Instead of deterministic updates, we might wish to allow for the addition of random noise with scale proportional to $b(x_t,t)$. The initial value, $X_0$, may now be random, and setting $\delta X_t := X_{t+\delta t}-X_t$ leads to one possible update 
\[
\delta X_t \approx a(X_t,t) \delta t + b(X_t,t) \epsilon_t, \qquad \epsilon_t \sim \mathsf{N}(0,\delta t),
\]
where the Gaussian noise terms $\epsilon_t$ are independent of all previous randomness, and $X_t$ has become a random variable. 
A noise distribution of the form $\mathsf{N}(0,\delta t)$ is chosen because it is self-consistent. For example, with $a(X_t,t)=a$ and $b(x_t,t)=b$, after two time steps initialised at $X_0=x_0$, we have
\[
X_{2\delta t} \approx x_0+a\delta t+b \epsilon_{\delta t} +a\delta t+b \epsilon_{2 \delta t} = x_0+2a\delta t+b \tilde{\epsilon}_{2\delta t}
\]
where $\tilde{\epsilon}_{2\delta t} \sim \mathsf{N}(0,2\delta t)$, since the two noise terms $\epsilon_{\delta t}$, $\epsilon_{2 \delta t}$ are independent. 
However, the right-hand side of this expression is exactly of the same form we would get from a single time step of size $2\delta t$ to obtain $X_{2\delta t}$ from $X_0$.

The process with $a=0$, $b=1$ and $X_0=0$ consists of a sequence of mean-zero Gaussian increments, each with a variance of $\delta t$. This is a discretisation of a process known as \emph{Brownian motion}\index{Brownian motion}, which is often denoted by $W_t$. In particular, we have that
\[
\delta W_t = W_{t+\delta t}-W_t\sim \mathsf{N}(0,\delta t),
\]
and $W_t\sim \mathsf{N}(0,t)$.
From the definition of $W_t$, we may rewrite the noisy update as
\begin{equation}
\label{eqn.EulerMar}
\delta X_t \approx a(X_t,t) \delta t + b(X_t,t) \delta W_t.
\end{equation}
Consider this process on some interval $[0,T]$, with $\delta t=T/m$ and $X_0=x_0$, for some initial value $x_0$. We create a continuous-time extension of the process by, for any $t\in [0,T]$, defining $X_t=X_{\delta t\lfloor t/\delta t\rfloor}$, where $\lfloor\cdot \rfloor$ is the floor function. Subject to some regularity conditions, the limit as $m\to \infty$ exists  and is written
\[
\md X_t = a(X_t,t) \md t + b(X_t,t) \md W_t.
\]
This is known as a \emph{stochastic differential equation}  (SDE)\index{stochastic differential equation}, and \eqref{eqn.EulerMar} is the Euler--Maruyama\index{Euler-Maruyama} approximation to it. Subject to the initial condition, the \emph{solution} to this SDE is the stochastic process $\{X_t\}_{t\in [0,T]}$ obtained as the limit ias $\delta t \to 0$ of the continuous-time extension of the discrete-time process defined through  \eqref{eqn.EulerMar}.

The heuristic above describes a one-dimensional SDE and its Euler--Maruyama\index{Euler-Maruyama} discretisation; however, it is straightforward to extend these to higher dimensions with $\bX_t\in \mathbb{R}^d$, $\ba:\mathbb{R}^d\times [0,\infty)\to \mathbb{R}^d$, $\bW_t\in \mathbb{R}^k$ and the $d\times k$ matrix $\bb:\mathbb{R}^d\times[0,\infty)\to \mathbb{R}^{dk}$.

A stochastic process that satisfies an SDE is called a \emph{diffusion}. 
For the most part, we will deal with time-homogeneous diffusions, where $a$ and $b$ have no explicit time dependence; however, time-inhomogeneous diffusions will be used in Chapter \ref{chap:sgld}.

\subsection{The Ornstein--Uhlenbeck Process}\index{Ornstein--Uhlenbeck process}
\label{sec.DefineOU}
Consider the SDE
\[
\md X_t = -\frac{1}{2\sigma^2}b^2 X_t \md t + b \md W_t.
\]
The Euler--Maruyama\index{Euler-Maruyama} discretisation gives
\[
X_{t+\delta t} \approx X_t+\delta t X_t=
\left(1-\frac{b^2}{2\sigma^2} \delta t\right) X_t +b \delta  W_t.
\]
Since $\delta W_t$ is Gaussian distributed and independent of $X_t$, if $X_t$ is Gaussian so is $X_{t+\delta t}$. Moreover, if $\Expect{X_t}=0$, then $\Expect{X_{t+\delta t}}=0$. Finally, if $\Var{X_t}=\sigma^2$ then \[
\Var{X_{t+\delta t}}= \left(1-\frac{b^2}{2\sigma^2}\delta t\right)^2 \sigma^2 +b^2 \delta t
=
\sigma^2+\frac{1}{4\sigma^4}b^4 \delta t^2.
\]
In the limit $m\to \infty$, as the number of increments is increased, with $\delta t =T/m\downarrow 0$, the term in $\delta t^2$ becomes irrelevant: the variance does not change. Thus, if $X_0\sim \mathsf{N}(0,\sigma^2)$ then $X_t\sim \mathsf{N}(0,\sigma^2)$ for all $t>0$; the SDE is stationary. Shifting the coordinate system by $m$ we see that the slightly more general SDE
\begin{equation}
\label{eqn.OUprocess}
\md X_t = -\frac{1}{2\sigma^2}b^2 (X_t-m) \md t + b \md W_t
\end{equation}
has a stationary distribution\index{stationary distribution} of $\mathsf{N}(m,\sigma^2)$. The process arising from the SDE \eqref{eqn.OUprocess} is known as the \emph{Ornstein Uhlenbeck} (OU)\index{Ornstein--Uhlenbeck process} process. Substituting $s=b^2 t$, the SDE becomes $\md X_s= -(X_s-m)/(2\sigma^2)\md s + \md W_s$, which explains why $b^2$ is termed the \emph{speed} of the diffusion. Figure \ref{fig:ch1_OU} presents three realisations of OU processes with stationary distribution\index{stationary distribution} $\mathsf{N}(m,1)$ and started from the corresponding $m/2$. Each diffusion has a different speed, and the effect of this on the convergence to, and mixing within, the stationary distribution\index{stationary distribution} is clearly visible. 

\begin{figure}
    \centering
     \includegraphics[width=0.7\textwidth]{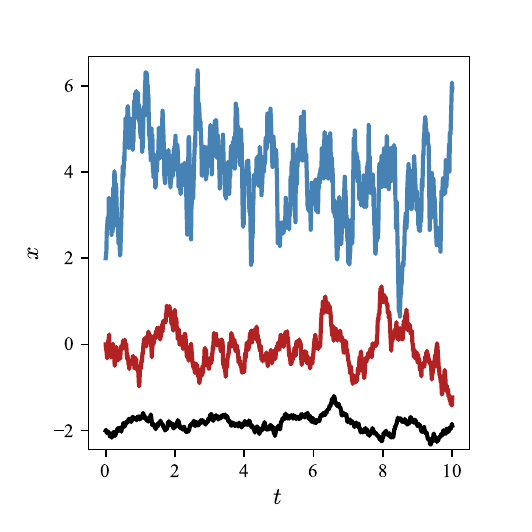}
    \caption{Three realisations of the Ornstein--Uhlenbeck processes\index{Ornstein--Uhlenbeck process}, all with $\sigma=1$, and on the time interval $[0,10]$. Other parameter settings are $x_0=2$, $m=4$ and $b=3$; $x_0=m=0$ and $b=1$; $x_0=-2$, $m=-4$ and $b=1/3$.
    }
    \label{fig:ch1_OU}
\end{figure}

Consider an OU process with $m=0$, started from an initial distribution of $\mathsf{N}(m_0,v_0)$ and let $m_t$ and $v_t$ be the expectation and variance at time $t$. Taking expectations of the discretisation of this process gives  $m_{t+\delta t}=m_t-b^2/(2\sigma^2) m_t \delta t$. The limit  $\delta t\to 0$ gives 
$$
\frac{\md m_t}{\md t}=-\frac{b^2}{2\sigma^2} m_t,
$$
and $m_t=m_0\exp(-b^2 t/2\sigma^2)$. Similarly, 
$$
v_{t+\delta t}=\left(1-\frac{b^2}{2\sigma^2}\delta t\right)^2 v_t +b^2\delta t = v_t +b^2\left(1-\frac{v_t}{\sigma^2}\right)\delta t +o(\delta t),
$$
from which 
$$
\frac{\md v_t}{\md t}
=
b^2\left(1-\frac{v_t}{\sigma^2}\right).
$$
This has the solution $v_t=\sigma^2+(v_0-\sigma^2)\exp(-b^2 t/\sigma^2)$. Shifting the co-ordinate system by $m$, we obtain, the time $t$ distribution of the OU process \eqref{eqn.OUprocess},
$$
X_t\sim \mathsf{N}\left(m+(m_0-m)\exp\left[-\frac{b^2 t}{2\sigma^2}\right], \sigma^2+(v_0-\sigma^2)\exp\left[-\frac{b^2t}{\sigma^2}\right]\right).~~~(*)
$$

\subsection{The Infinitesimal Generator} \label{sec:generator}

The \emph{infinitesimal generator}\index{infinitesimal generator|see{generator}} (or, simply, \emph{generator})\index{generator} of a continuous-time stochastic process acts on a function $h$ of the process:
\begin{equation}
    \label{eqn.InfGenerator}
    (\cL h)(\bx):=\left.\frac{\partial}{\partial t}\Expect{h(\bX_t)|\bX_0=\bx}\right|_{t=0}
    =
    \lim_{\delta t\downarrow 0}
    \frac{\Expect{h(\bX_{\delta t})}-h(\bx)}{\delta t}.
\end{equation}
The set of functions for which the limit exists for all $\bx$ is called the \emph{domain} of the infinitesimal generator\index{generator}. Subject to regularity conditions, this includes the set of compactly supported functions with a second derivative that is continuous, denoted $C_0^2$. 

For processes defined by an SDE, we can gain some insight into their generator\index{generator} by considering a Taylor expansion. For simplicity of presentation, we consider $x\in\mathbb{R}$:
\[
\frac{1}{\delta t}\Expect{h(X_{\delta t})-h(x)}
=
\frac{1}{\delta t}
\Expect{(X_{\delta t}-x)h'(x)+\frac{1}{2}(X_{\delta t}-x)^2h''(x) 
+\dots} .
\]
The Euler--Maruyama\index{Euler-Maruyama} approximation of the SDE is $X_{\delta t}-x\approx a(x) \delta t+b(x)\delta W_t$. Thus, 
$\Expect{X_{\delta t}-x}\approx a(x)\delta t$ and $\Expect{(X_{\delta t}-x)^2}\approx b(x)^2\delta t+a(x)^2[\delta t]^2$, with all higher-order expectations at most $o(\delta t)$. Thus, we might expect that
\[
 (\cL h)(x)=a(x)\frac{\md h}{\md x}+\frac{1}{2}b(x)^2 \frac{\md^2 h}{\md x^2},
\]
and this is indeed the case. 
For a multivariate diffusion, the generator\index{generator} is
\begin{equation}
(\cL h)(\bx)=\sum_{i=1}^d \left.a_i\frac{\partial h}{\partial x_i}\right|_{\bx} +\frac{1}{2}\sum_{i=1}^d\sum_{j=1}^d
\left.(bb^\top)_{i,j}\frac{\partial^2 h}{\partial x_i\partial x_j}\right|_{\bx}.
\end{equation}

Generators\index{generator} of diffusion processes are used in the next subsection to derive the stationary density of two classes of diffusion that appear repeatedly in Chapters \ref{chap:intro-mcmc} and \ref{chap:sgld}. Generators of diffusions are also used in Chapter  \ref{chap:stein} for the assessment and improvement of algorithms. Finally, Chapter \ref{chap:continuous-time} employs the generators of another class of continuous-time stochastic processes to determine the processes' stationary distributions\index{stationary distribution}.

\subsection{Langevin Diffusions}\index{Langevin diffusion}
\label{sec.IntroLangDiffusions}
We now describe two classes of diffusion, the overdamped and underdamped Langevin diffusions\index{Langevin diffusion!underdamped}, where the stationary density forms an explicit part of the SDE formulation.

\subsubsection{The Overdamped Langevin Diffusion}\index{Langevin diffusion!overdamped}
Consider a positive, differentiable density function $f(\bx)$ for $\bx \in \mathbb{R}^d$, and the following SDE:
\begin{equation}
    \label{eqn.Langevin}
    \md \bX_t = \frac{1}{2}\nabla \log f(\bX_t) b^2 \md t+b \md \bW_t.
\end{equation}
A solution to this SDE is known as an  \emph{overdamped Langevin diffusion}\index{Langevin diffusion!overdamped}. 
The OU process \eqref{eqn.OUprocess} with $f(x)=\mathsf{N}(x;m,1/a)$ is a special case of this class of diffusions, and in this case, as seen in Section \ref{sec.DefineOU}, $f$ is the density of the stationary process. In fact, this is true in general: the stationary density of the overdamped Langevin diffusion\index{Langevin diffusion!overdamped} \eqref{eqn.Langevin} is $f$. To see this in one dimension, consider the infinitesimal generator\index{generator} of the diffusion,
\[
(\cL h)(x)
=
\frac{1}{2}b^2\frac{f'(x)}{f(x)}h'(x)+\frac{1}{2}b^2h''(x).
\]
This is the rate of change of the expectation of $h(X_t)$ at $t=0$, when started from $X_0=x$. Suppose instead that $X_0$ has a density of $f$. Then, the rate of change of the expectation of $X_t$ at $t=0$ can be calculated by taking expectations with respect to $X_0$. This is
\[
\frac{1}{2}b^2\int \left\{\frac{f'(x)}{f(x)}h'(x)+h''(x)\right\} f(x)~\md x
=
\frac{1}{2}b^2\int \left\{f(x)h'(x)\right\}' ~\md x= 0
\]
for all sufficiently smooth $h$ with compact support. 
Thus, if $X_0\sim f$, $\left.\frac{\md}{\md t} \Expect{h(X_t)}\right|_{t=0} =0$. This is true for \emph{all} $h\in C_0^2$, and so the distribution of $X_t$ does not change as $t$ increases from $0$. The distribution at time $0$ must, therefore, be the stationary distribution\index{stationary distribution} of the Langevin diffusion \eqref{eqn.Langevin}, and $f$ is the corresponding stationary density.

When Langevin diffusions\index{Langevin diffusion} are employed in a Bayesian setting, $f(\bx)$ is often a posterior density whose normalising constant is, typically, intractable. The fact that the calculation of $\nabla \log f(\bX_t)$ does not require this normalising constant is crucial to the practical use of these diffusions.

\subsubsection{The Underdamped Langevin Diffusion} \label{ch1:sec:underdamped}
The \emph{underdamped Langevin diffusion}\index{Langevin diffusion!underdamped} extends the state space to include a velocity component, $\bP_t$:
\begin{align}
\label{eqn.underdampedA}
\md \bX_t &= \bP_t \md t,\\
\md \bP_t & = - \gamma \bP_t \md t + c \nabla \log f(\bX_t) \md t +\sqrt{2\gamma c} ~\md \bW_t
\label{eqn.underdampedB}
\end{align}
Intuitively, dividing \eqref{eqn.underdampedB} through by $\gamma$ and taking the limit as $\gamma \to \infty$ and $c\to \infty$ with $c/\gamma=b^2/2$ fixed, we obtain the overdamped Langevin diffusion\index{Langevin diffusion!overdamped}, so the latter is a limiting case of the underdamped diffusion\index{Langevin diffusion!underdamped}.

The underdamped Langevin diffusion\index{Langevin diffusion!underdamped} targets $f(\bx)g(\bp)$, where 
\[
g(\bp)=\frac{1}{\sqrt{2\pi c}}\exp\left(-\frac{1}{2 c} \|\bp\|^2\right).
\]
To see this, we, again, restrict ourselves to the one-dimensional case to simplify the presentation, and, again, we start from the generator\index{generator},
\[
(\cL h)(x,p)=p h_x(x,p)-\gamma p h_p(x,p) +  c\frac{f'(x)}{f(x)}h_p(x,p) + \gamma c h_{p,p}(x,p),
\]
where we have used subscripts to denote the differentiation of $h$ with respect to $x$ or $p$.  The quantity $(\cL h)(x,p)$ is the rate of change of the expectation of $h(X_t,P_t)$ at $t=0$, when started at $X_0=x$ and $P_0=p$. Thus if $X_0$ and $P_0$ have respective densities of $f(x)$ and $g(p)$, then the rate of change of $\Expect{h(X_t,P_t)}$ at $t=0$ is
\[
\iint \left\{ ph_x-\gamma p h_p + c\frac{f'(x)}{f(x)}h_p + \gamma c h_{p,p}\right\} f(x) g(p) \md p \md x.
\]
In the manipulations that follow, we will twice use the fact that $g'(p)=-p g(p)/c$. First, integration by parts gives
\[
\int \gamma c h_{p,p} g(p) \; \md p=\int p\gamma h_p g(p) \; \md p,
\]
so the second and fourth terms cancel. Second, two integrations by parts, first with respect to $p$ and then with respect to $x$, give
\begin{align*}
\iint cf'(x) h_p g(p) \; \md p \md x
&=
\iint f'(x) h pg(p) \; \md p \md x\\
&=
-\iint f(x) h_x pg(p) \; \md p \md x,
\end{align*}
so the first and third terms cancel. The argument is completed analogously to that for the overdamped Langevin diffusion\index{Langevin diffusion!overdamped}.

In Chapter \ref{chap:sgld}, we explore further the overdamped and underdamped Langevin diffusions as practical algorithms for scalable Monte Carlo inference in the large-data setting and show that the discretisation of these diffusion processes leads to important special cases of the general framework for stochastic gradient MCMC algorithms. 

\section{The Kernel Trick}\index{kernel trick}
\label{sec.RKHS}

Chapter \ref{chap:stein} introduces the \emph{kernel Stein discrepancy} and uses it 
 to measure the discrepancy between a sample of points and a distribution of interest. Practical use of the methodology is made feasible by the ability to reduce what appears to be an infinite amount of computation -- maximising a quantity over an uncountably infinite set of possible functions -- to only a finite number of arithmetic operations. The key mechanism for this simplification is often called \emph{the kernel trick}, and the setting for its use is a  \emph{reproducing kernel Hilbert space}\index{reproducing kernel Hilbert space}. 

This section first explains the kernel trick\index{kernel trick} in the more familiar setting of a finite-dimensional inner-product space, before extending to the more general setting required for Chapter \ref{chap:stein}.
Whilst many of the concepts introduced are much more general, our presentation focuses on the specific setting of relevance: the vectors of our inner-product space\index{inner product space} are functions, the associated field is $\mathbb{R}$ and the inner product\index{inner product} is an integral with respect to a probability distribution. 

Throughout, $f(\cdot)$, $g(\cdot)$ \emph{etc}. are functions from $\cX\to \mathbb{R}$, where $\cX$ is $\mathbb{R}^d$ or some closed or open subset of $\mathbb{R}^d$; $f(\bx)$, $g(\bx)$ \emph{etc} denote the function evaluated at $\bx\in \cX$. The probability distribution $\nu$ is assumed to have a density $\nu(\bx)$ on $\cX$.

\subsection{Finite-Dimensional Inner-Product Spaces}

Let $0(\cdot)$ be the function such that $0(\bx)=0$ for all $\bx\in \cX$.
A set, $\Vsp$, of functions from $\cX \to \mathbb{R}$ is a \emph{vector space}\index{vector space} over $\mathbb{R}$ if the following axioms are satisfied:

\begin{enumerate}
    \item $0(\cdot)\in \Vsp$.
    \item $f(\cdot)\in \Vsp\implies -f(\cdot)\in \Vsp$.
    \item $f(\cdot), g(\cdot)\in \Vsp\implies f(\cdot)+g(\cdot)\in \Vsp$.
    \item $f(\cdot)\in \Vsp$ and $a\in \mathbb{R}\implies a f(\cdot)\in \Vsp$.
\end{enumerate}

\emph{Aside}: The associativity, commutativity, and distributivity axioms of a general vector space\index{vector space} are satisfied automatically when the elements are functions and from $\cX$ to $\mathbb{R}$ and the field is $\mathbb{R}$. 

Every finite-dimensional vector space\index{vector space} has a dimension, $n$, such that there is a set of $n$ vectors $\{b_1(\cdot),\dots, b_n(\cdot)\}$ which satisfy the following two properties:
\begin{enumerate}
    \item \textbf{Linear independence}: If there are $a_1,\dots,a_n \in \mathbb{R}$ such that $\sum_{i=1}^n a_i b_i(\cdot)=0$ then $a_i=0$ for all $i\in\{1,\dots,n\}$.
    \item \textbf{Spanning} $\Vsp$: for each $f(\cdot)\in \Vsp$ there are $a_1,\dots,a_n\in \mathbb{R}$ such that $f(\cdot)=\sum_{i=1}^n a_i b_i(\cdot)$. 
\end{enumerate}
The set $\{b_1(\cdot),\dots,b_n(\cdot)\}$ is called a \emph{basis}.

\begin{example}
\label{example.Ltwo.running}
 It is straightforward to check that the set \begin{align*}
   \Vsp
   &=\{f(\cdot): f(x)=c\sin(x+\theta):c\in \mathbb{R},\theta\in [0,2\pi)\}\\
   &=
  \{f(\cdot): f(x)=a \sin x + b \cos x; a,b\in \mathbb{R}\}
  \end{align*} satisfies Axioms 1--4, whatever the domain, $\cX\subseteq \mathbb{R}$.
   We may take $b_1(\cdot)=\sin(\cdot)$ and $b_2(\cdot)=\cos(\cdot)$. However, we may also take $b_1(\cdot)=\sin(\cdot)+3\cos(\cdot)$ and $b_2(\cdot)=\cos(\cdot)$, for example.
\end{example}

For any vector space\index{vector space} $\Vsp$ of functions from $\cX\to \mathbb{R}$ and any distribution $\nu$ with a probability density function on $\cX$ of $\nu(x)$, we define the \emph{inner product}\index{inner product} 
\begin{equation}
\label{eqn.def.inner.prod}
\braket{f(\cdot),g(\cdot)}_{\nu}
=\int f(\bx)g(\bx)\nu(\bx) \; \md \bx,
\end{equation}
where here and throughout this section, if the integral range is not specified, then it is $\cX$. We refer to this inner product as $\braket{\cdot,\cdot}_{\nu}$.

The inner product\index{inner product} defined by \eqref{eqn.def.inner.prod} clearly satisfies two of the three defining properties of an inner product: $\braket{f(\cdot),g(\cdot)}=\braket{g(\cdot),f(\cdot)}$ and
$
\braket{f(\cdot)+g(\cdot),h(\cdot)}=\braket{f(\cdot),h(\cdot)}+\braket{g(\cdot),h(\cdot)}
$.
However, we have only that $\braket{f(\cdot),f(\cdot)}=0 \Leftrightarrow$ $f(\bx)= 0(\bx)$ $\nu$-almost everywhere, rather than $\braket{f(\cdot),f(\cdot)}=0\Leftrightarrow f(\cdot)= 0(\cdot)$. Each $f$ belongs to an equivalence class of functions that are equal $\nu$-almost everywhere. This set of equivalence classes forms a vector space,\index{vector space} and \eqref{eqn.def.inner.prod} defines an inner product on this space, not on the space of functions, $\Vsp$. To keep the presentation in this section as straightforward as possible, our wording ignores this distinction, but the more rigorous reader may wish to replace any vector space\index{vector space} of functions and inner product\index{inner product} between these functions with the corresponding vector space\index{vector space} of equivalence classes of functions and inner products between these equivalence classes.  

The inner product\index{inner product} provides a norm, called the \emph{induced norm}, the square of which is 
\[
\|f(\cdot)\|^2_\nu
=
\braket{f(\cdot),f(\cdot)}_\nu
=\int f(\bx)^2 \nu(\bx) \; \md \bx.
\]

\begin{example}
(Example \ref{example.Ltwo.running} continued). Let $\cX = [0,2\pi]$ and 
let $\nu$ be the uniform distribution on $[0,2\pi]$. For any $f(\cdot),g(\cdot)\in \Vsp$,
\[
\braket{f(\cdot),g(\cdot)}_\nu
=
\frac{1}{2\pi}\int_0^{2\pi} f(x) g(x) \; \md x
~~~\mbox{and}~~~
\|f(\cdot)\|^2_\nu=\frac{1}{2\pi}\int_0^{2\pi} f(x)^2 \; \md x.
\]
\end{example}

\begin{example}
\label{example.fd.squareInt}
For a general vector space\index{vector space} $\Vsp$ of functions of the form $\cX\to \mathbb{R}$, let $\Vsp_\nu$ be the elements of V that have a finite norm induced by $\nu$:
\[
\Vsp_\nu=\left\{f(\cdot)\in \Vsp: \int f(\bx)^2 \nu(\bx) \; \md \bx<\infty\right\}.
\]
Then $\Vsp_\nu$ is also a vector space\index{vector space}, since Axioms 1, 2 and 4 are satisfied trivially, and Axiom 3 is satisfied since for any   $f(\cdot),g(\cdot)\in \Vsp_\nu$,
\begin{align*}
\|f(\cdot)+g(\cdot)\|_\nu^2
    &=
    \braket{f(\cdot)+g(\cdot),f(\cdot)+g(\cdot)}_\nu\\
    &=
    \|f(\cdot)\|_\nu^2+2\braket{f(\cdot),g(\cdot)}_\nu+\|g(\cdot)\|^2_{\nu}\\
&\le
\|f(\cdot)\|_\nu^2+2 \|f(\cdot)\|_\nu \|g(\cdot)\|_\nu +  \|g(\cdot)\|_{\nu}^2
<\infty,
\end{align*}
where the third line uses the Cauchy--Schwarz\index{Cauchy--Schwarz} inequality, which, in this case, is the familiar inequality $\Expect{f(\bX)g(\bX)}^2\le \mathbb{E}[f(\bX)^2]\Expect{g(\bX)^2}$, where $\bX$ has a density $\nu$ on $\cX$. 
\end{example}

Henceforth, for narrative simplicity, we will assume that $\Vsp$ is a finite-dimensional vector space\index{vector space} with dimension $n$. Section \ref{sec.general.kernels} extends the narrative to potentially infinite-dimensional spaces.

When considering the inner product\index{inner product} $\braket{\cdot,\cdot}_\nu$, two vectors $f(\cdot),g(\cdot)\in \Vsp$ are said to be \emph{orthogonal}\index{orthogonal} if 
$\braket{f(\cdot),g(\cdot)}_\nu=0$ and the basis vectors, $e_1(\cdot),\dots,e_n(\cdot)$ are said to be \emph{orthonormal}\index{orthonormal} if they are orthogonal and each has a norm of $1$: for each $j,k\in\{1,\dots,n\}$,
\[
\|e_j(\cdot)\|_\nu=1
~~~\mbox{and}~~~
\braket{e_j(\cdot),e_k(\cdot)}_\nu=0
\]
whenever $j\ne k$. 
We will reserve the symbols $\{e_k(\cdot)\}_{k=1}^n$ for any set of $n$ orthonormal\index{orthonormal} basis functions.  

The representation of  $f(\cdot)$ in terms of an orthonormal\index{orthonormal} basis
\[
f(\cdot)=\sum_{j=1}^n f_j e_j(\cdot)
\]
is termed an \emph{orthonormal decomposition} of $f(\cdot)$. Since the $e_i(\cdot)$ are orthonormal\index{orthonormal}, the projection of $f(\cdot)$ onto $e_k(\cdot)$ is $f_k$:
\begin{align}
\nonumber
\braket{f(\cdot),e_k(\cdot)}_\nu&=
\left\langle
\sum_{j=1}^n f_j e_j(\cdot),e_k(\cdot)
\right\rangle_\nu
=\sum_{j=1}^n f_j\braket{e_j(\cdot),e_k(\cdot)}_\nu\\
&= f_k.
\label{eqn.project}
\end{align}

Furthermore, the squared norm of $f(\cdot)$ is the sum of the squares of the orthonormal\index{orthonormal} projections:
\begin{align}
\nonumber
\|f(\cdot)\|^2
&=
\left\langle\sum_{j=1}^n f_j e_j(\cdot),\sum_{k=1}^n f_k e_k(\cdot)
    \right\rangle
    =
    \sum_{j=1}^n\sum_{k=1}^n f_j f_k\braket{e_j(\cdot),e_k(\cdot)}\\
    &=\sum_{j=1}^n f_j^2.
\label{eqn.genPythag}
\end{align}

\begin{example}
\label{example.orthonormal}
In Example \ref{example.Ltwo.running}, since $\int_0^{2\pi} \sin^2 x \; \md x = \int_0^{2\pi} \cos^2 x \; \md x=\pi$ and $\int_0^{2\pi} \sin x \cos x \; \md x = 0$,
\[
e_1(\cdot)=\sqrt{2}\sin(\cdot)
~~~\mbox{and}~~~
e_2(\cdot)=\sqrt{2}\cos(\cdot)
\]
form an orthonormal\index{orthonormal} basis for $\Vsp$ when $\nu$ is the uniform distribution on $[0,2\pi]$. Any function $f(\cdot)\in  \Vsp$ can be written as $f(\cdot)=f_1 e_1(\cdot)+ f_2 e_2(\cdot)$. For example set
\begin{equation}
\label{eqn.fd.f.example}
f(x)=\sin(x+\pi/6)=\frac{\sqrt{3}}{2}\sin x + \frac{1}{2} \cos x.
\end{equation}
So $f_1=\sqrt{3}/(2\sqrt{2})$ and $f_2=1/(2\sqrt{2})$. Also
\[
\|f(\cdot)\|^2_\nu=f_1^2+f_2^2=\frac{3}{8}+\frac{1}{8}=\frac{1}{2}
=
\frac{1}{2\pi}\int_0^{2\pi}
\sin^2(x+\pi/6) \; \md x.
\]

\end{example}

\subsection{Kernels in a Finite-Dimensional Inner-Product Space}
\label{sec.kernels.from.fd}

As in the previous subsection, let $\Vsp$ be an $n$-dimensional vector space\index{vector space} of functions from $\cX$ to $\mathbb{R}$ and let $\nu$ be a probability distribution on $\cX$ with a probability density of 
$\nu(\bx)$, $\bx\in \cX$. Finally, let $\{e_k(\cdot)\}_{k=1}^n$ be a set of basis functions that is orthonormal\index{orthonormal} with respect to the inner product\index{inner product} \eqref{eqn.def.inner.prod}. 

Let $\lambda_1,\dots,\lambda_n$ be a set of non-negative scalars, and consider the following real-valued function on $\cX\times \cX$:
\begin{equation}
\label{eqn.define.fd.kernel}
\kernel(\bx,\by)=\sum_{j=1}^n \lambda_j e_j(\bx)e_j(\by).
\end{equation}

Clearly, $\kernel(\cdot,\cdot)$ is \emph{symmetric}: $\kernel(\by,\bx)=\kernel(\bx,\by)$. Moreover, $\kernel(\cdot,\cdot)$ is \emph{positive semidefinite}: for any finite $J<\infty$,  $c_1,\dots,c_J\in \mathbb{R}$ and $\bx_1,\dots\bx_J\in \mathbb{R}^d$,
\begin{align*}
\sum_{j=1}^J\sum_{k=1}^J c_j c_k \kernel(\bx_j,\bx_k)
&=
\sum_{j=1}^J\sum_{k=1}^J c_j c_k \sum_{l=1}^n \lambda_l e_l(\bx_j)e_l(\bx_k)\\
&=
\sum_{l=1}^n \lambda_l 
\sum_{j=1}^J\sum_{k=1}^J c_j c_k
e_l(\bx_j)e_l(\bx_k)\\
&=
\sum_{l=1}^n \lambda_l
\left\{\sum_{j=1}^J c_j e_l(\bx_j)\right\}^2
\ge 0.
\end{align*}

Any function $\kernel(\cdot,\cdot): \cX\times \cX\to \mathbb{R}$ which is both symmetric and positive semidefinite is called a \emph{kernel}. 

\begin{example}
\label{example.choose.kernel}
Continuing Example \ref{example.Ltwo.running}, let $\kernel:[0,2\pi]\times [0,2\pi]\to\mathbb{R}$ be
\begin{align*}
\kernel(x,y)&=
\frac{1}{2} e_1(x)e_1(y) + \frac{3}{2} e_2(x) e_2(y)
=\sin x \sin y + 3\cos x \cos y\\&= 2\cos(y-x)+\cos(y+x).
\end{align*}
This is symmetric and positive definite by construction. 
\end{example}

Given the definition of $\kernel(\cdot,\cdot)$ in \eqref{eqn.define.fd.kernel}, define
\begin{equation}
\label{eqn.def.kx.fn}
\kernel(\bx,\cdot)=
\sum_{j=1}^n \lambda_j e_j(\bx)e_j(\cdot),
\end{equation}
and $\kernel(\cdot,\bx)=\kernel(\bx,\cdot)$.
Since $e_j(\bx)\in \mathbb{R}$, $\kernel(\bx,\cdot)\in \Vsp$. Furthermore, for $f(\cdot) \in \Vsp$ define the operator $T_\kernel$ via
\begin{equation}
\label{eqn.defineTkernel}
T_\kernel f(\cdot)
=
\int \kernel(\cdot,\by) f(\by) \nu(\by) \; \md \by.
\end{equation}
Then $T_\kernel$ is a \emph{linear operator}, since for any $a,b\in \mathbb{R}$ and $f(\cdot),g(\cdot) \in \Vsp$, 
\begin{align*}
T_\kernel\left\{af(\cdot)+bg(\cdot)\right\}
=aT_\kernel f(\cdot) + bT_\kernel g(\cdot).
\end{align*}
Now, writing $f(\cdot)=\sum_{k=1}^n f_k e_k(\cdot)$,
\begin{align*}
(T_\kernel f(\cdot))(\bx)
&=
\int \kernel(\bx,\by) f(\by) \nu(\by) \md \by
\\
&=
\braket{\kernel(\bx,\cdot),f(\cdot)}_\nu
=
\left\langle
\sum_{j=1}^n \lambda_j e_j(\bx)e_j(\cdot),\sum_{k=1}^n f_k e_k(\cdot)
\right\rangle_\nu\\
&=
\sum_{j=1}^n\sum_{k=1}^n\lambda_j e_j(\bx) f_k\braket{e_j(\cdot),e_k(\cdot)}_\nu
=\sum_{k=1}^n\lambda_k f_k e_k(\bx).
\end{align*}
So $T_\kernel f(\cdot)=\sum_{k=1}^n\lambda_k f_k e_k(\cdot)$ and, hence, $T_\kernel f(\cdot) \in \Vsp$, too. Moreover, considering $f(\cdot)=e_j(\cdot)$, we see that $T_\kernel e_j(\cdot)=\lambda_je_j(\cdot)$; each $e_j(\cdot)$ is an eigenfunction\index{eigenfunction} of $T_\kernel$ with a corresponding eigenvalue\index{eigenvalue} of $\lambda_j$.

\begin{example}
Continuing Example \ref{example.Ltwo.running}, with the kernel from Example \ref{example.choose.kernel},
\[
\kernel(x,\cdot)=\sin x \sin(\cdot) + 3\cos x \cos(\cdot) = 2\cos(\cdot - x)+\cos(\cdot+x).
\]

Let $f(\cdot)$ be as defined in \eqref{eqn.fd.f.example}.
Then, using the definite integrals at the start of Example \ref{example.orthonormal},
\begin{align*}
T_\kernel f(\cdot)
&=
\frac{1}{2\pi}
\int_0^{2\pi}\left\{\sin(\cdot)\sin y+3\cos(\cdot)\cos y\right\}\left\{\frac{\sqrt{3}}{2}\sin y + \frac{1}{2} \cos y\right\} \; \md y\\
&=
\frac{1}{4\pi}\int_0^{2\pi} \sqrt{3}\sin(\cdot)\sin^2 y + 3\cos(\cdot) \cos^2 y  \; \md y\\
&=
\frac{1}{4}\left\{\sqrt{3} \sin(\cdot)+3\cos(\cdot)\right\}.
\end{align*}
Since 
\[e_1(\cdot)=\sqrt{2}\sin(\cdot),~  e_2(\cdot)=\sqrt{2}\cos(\cdot),~ f_1=\frac{\sqrt{3}}{2\sqrt{2}},~  f_2=\frac{1}
{2\sqrt{2}},
\]
$\lambda_1=1/2$ and $\lambda_2=3/2$, $T_\kernel f(\cdot)$ is, therefore,
\[
\lambda_1 f_1 e_1(\cdot) + \lambda_2 f_2 e_2(\cdot),
\]
as we would hope.
\end{example}

\subsection{A New Inner Product and the Kernel Trick in Finite Dimensions}
\label{sec.fdkerneltrick}

Let $\{e_j(\cdot)\}_{j=1}^n$ be an orthonormal\index{orthonormal} basis for $\Vsp$ and let $\kernel$ be defined through \eqref{eqn.define.fd.kernel} with respect to this basis.
For $f(\cdot),g(\cdot)\in \Vsp$ with
\begin{equation}
\label{eqn.decompose.fg}
f(\cdot)=\sum_{j=1}^n f_j e_j(\cdot)
~~~\mbox{and}~~~
g(\cdot)=\sum_{j=1}^n g_j e_j(\cdot),
\end{equation}
the inner product\index{inner product} with respect to $\nu$ is the sum of the products of the orthogonal\index{orthogonal} projections,
\begin{align*}
\braket{f(\cdot),g(\cdot)}_\nu
=&
\left\langle
\sum_{j=1}^n f_j e_j(\cdot),\sum_{k=1}^n g_k e_k(\cdot)\right\rangle_\nu
=
\sum_{j=1}^n
\sum_{k=1}^n
f_j g_k \braket{e_j(\cdot),e_k(\cdot)}_\nu\\
&=
\sum_{j=1}^n f_j g_j.
\end{align*}

We now define a new inner product\index{inner product}
\begin{equation}
\label{eqn.inner.prod.Kfd}
\braket{f(\cdot),g(\cdot)}_\kernel
=
\sum_{j=1}^n \frac{f_j g_j}{\lambda_j},
\end{equation}
where the $\{\lambda_j\}_{j=1}^n$, are exactly those from the definition of $\kernel$ and are the eigenvalues\index{eigenvalue} of the operator $T_\kernel$.

This inner product\index{inner product} may be rephrased in terms of a set of eigenfunctions\index{eigenfunction} that are orthonormal\index{orthonormal} with respect to $\braket{\cdot,\cdot}_\kernel$: $\{e'_j(\cdot)\}_{j=1}^n$ with $e'_j(\cdot)=\sqrt{\lambda_j}e_j(\cdot)$. With respect to this basis, the vector  
\[f(\cdot)=\sum_{j=1}^n f'_j e'_j(\cdot)\]
with $f'_j=f_j/\sqrt{\lambda_j}$. Using an analogous decomposition for $g(\cdot)$, 
\[
\braket{f(\cdot),g(\cdot)}_\kernel = \sum_{j=1}^n f'_j g'_j,
\]
as expected. Finally,
\[
\kernel(\bx,\by)=\sum_{j=1}^n e'_j(\bx) e'_j(\by).
\]
\begin{example}
    In Example \ref{example.choose.kernel}, $e_1'(\cdot)=\sin(\cdot)$ and $e_2'(\cdot)=\sqrt{3}\cos(\cdot)$. Clearly,
    \[
 \kernel(x,y)=\sin x \sin y + 3\cos x  \cos y = e_1'(x)e_1'(y)+e_2'(x)e_2'(y).
    \]
    For $f(\cdot)$ as in Example \ref{eqn.fd.f.example},
\[
f(\cdot)=
\frac{\sqrt{3}}{2}\sin(\cdot)+\frac{1}{2}\cos(\cdot)
=
\frac{\sqrt{3}}{2}\sin(\cdot)+\frac{\sqrt{3}}{6}\sqrt{3}\cos(\cdot),
\]
so $f'_1=\sqrt{3}/2$ and $f'_2=\sqrt{3}/{6}$. Thus,
    \[ \|f(\cdot)\|^2_{\kernel}
    =
    \frac{3}{4}+\frac{3}{36}
    =
    \frac{5}{6}.
    \]  
\end{example}

\subsubsection{The Kernel Trick}\index{kernel trick}

From the definition \eqref{eqn.def.kx.fn} and with $f(\cdot)$ decomposed as in \eqref{eqn.decompose.fg},
\begin{align}
\nonumber
\braket{\kernel(\bx,\cdot),f(\cdot)}_\kernel
&=
\left\langle
\sum_{j=1}^n\lambda_je_j(\bx)e_j(\cdot),
\sum_{k=1}^n f_k e_k(\cdot)
\right\rangle_\kernel
=\sum_{j=1}^n
\frac{\lambda_j e_j(\bx) f_j}{\lambda_j}\\
&=f(\bx).
\label{eqn.kernelOne}
\end{align}
Moreover, choosing $f(\cdot)$ to be $\kernel(\by,\cdot)$,   $f_j=\lambda_j e_j(\by)$ from \eqref{eqn.def.kx.fn}, and, hence,
\begin{equation}
\label{eqn.kernelTwo}
\braket{\kernel(\bx,\cdot),\kernel(\by,\cdot)}_{\kernel}
=
\sum_{j=1}^n \lambda_j e_j(\bx) e_j(\by)
=
\kernel(\bx,\by).
\end{equation}
Together, \eqref{eqn.kernelOne} and \eqref{eqn.kernelTwo} enable the evaluation of inner products in $\braket{\cdot,\cdot }_\kernel$ \emph{without needing to know the original basis functions $e_1(\cdot),\dots,e_n(\cdot)$ nor the associated values $\lambda_1,\dots,\lambda_n$}. Indeed, we do not even need to know $\nu$. This is known as \emph{the kernel trick}\index{kernel trick}, and we will exemplify its use in Section \ref{sec.using.kernel.trick}. First, we generalise to a much broader class of kernels.

\subsection{General Kernels}
\label{sec.general.kernels}
In Section \ref{sec.kernels.from.fd} we created a kernel via \eqref{eqn.define.fd.kernel} using a known orthonormal\index{orthonormal} basis for the inner-product  space\index{inner product space} $\Vsp$, with the inner product\index{inner product} specified by \eqref{eqn.def.inner.prod} according to the density $\nu$. However, a kernel is   \emph{any} positive-definite symmetric function and we are interested in kernels $\kernel: \cX\times \cX\to \mathbb{R}$.

\begin{example}
\label{example.GaussianKernel}
The Gaussian kernel\index{kernel (Gaussian)} is
\[
\kernel(\bx,\by)=\exp\left(-\|\by-\bx\|^2\right),
\]
where $\|\cdot\|$ represents the standard Euclidean norm. 
This is clearly symmetric. To see that $\kernel$ is also positive semidefinite on $\cX=\mathbb{R}^d$, note that
\[
\bZ\sim \mathsf{N}_d\left(\bx,\frac{1}{4}\bI_d\right)
~~~\mbox{and}~~~
\bY|\bZ\sim \mathsf{N}_d\left(\bZ,\frac{1}{4}\bI_d\right)
\implies 
\bY\sim \mathsf{N}_d\left(\bx,\frac{1}{2}\bI_d\right),
\]
from which
\[
\exp\left(-\|\by-\bx\|^2\right)
=
\gamma\int \exp\left(-2\|\by-\bz\|^2\right)\exp\left(-2\|\bx-\bz\|^2\right) \md \bz,
\]
where $\gamma=2^d/\pi^{d/2}$.
Hence $\sum_{j=1}^J\sum_{k=1}^J
c_j c_k \kernel(\bx_j,\bx_k)$ is 
\begin{align*}
\gamma\sum_{j=1}^J\sum_{k=1}^J&
c_j c_k 
\int \exp\left(-2\|\bx_j-\bz\|^2\right)\exp\left(-2\|\bx_k-\bz\|^2\right)\md \bz\\
&=
\gamma
\int
\sum_{j=1}^J\sum_{k=1}^J
c_j c_k 
\exp\left(-2\|\bx_j-\bz\|^2\right)\exp\left(-2\|\bx_k-\bz\|^2\right)\md \bz\\
&=
\gamma \int
\left\{
\sum_{j=1}^J
c_j\exp\left(-2\|\bx_j-\bz\|^2\right)
\right\}^2\md \bz\\
&\ge 0.
\end{align*}
\end{example}

When specifying $\kernel$ in Example \ref{example.GaussianKernel}, we have not specified a vector space\index{vector space}, nor a density $\nu$, nor an associated inner product\index{inner product}. However, since $\kernel$ is a kernel, we might hope that if we do specify $\nu$ and the inner product\index{inner product} $\braket{\cdot,\cdot}_\nu$ in \eqref{eqn.def.inner.prod}, then there might be a vector space\index{vector space} with a basis that is orthonormal\index{orthonormal} with respect to $\braket{\cdot,\cdot}_\nu$ such that $\kernel$ has the decomposition
\eqref{eqn.define.fd.kernel}. If this were the case, then we would know that there was a new inner product\index{inner product} $\braket{\cdot,\cdot}_\kernel$ such that \eqref{eqn.kernelOne} and \eqref{eqn.kernelTwo}  held. Hence, we could evaluate inner products with respect to $\kernel$ without knowing the basis itself nor the eigenvalues\index{eigenvalue} of $T_\kernel$, nor, even, the details about $\nu$. 

The decomposition in \eqref{eqn.def.inner.prod} does not hold in general, but Mercer's theorem\index{Mercer's theorem} and generalisations of it tell us that an analogous decomposition, but with $n$ potentially infinite, holds widely.

Specifically, let $\cX$ be $\mathbb{R}^d$ or a closed or open subset of $\mathbb{R}^d$, $\kernel(\cdot,\cdot):\cX\times \cX\to \mathbb{R}$ be a kernel and $\nu(\bx), ~\bx \in \cX$, be a probability density on $\cX$. Then, provided $\kernel(\bx,\by)$ is a continuous function of $\bx$ and $\by$, and
\begin{equation}
\label{eqn.Sun.condition}
\int \kernel(\bx,\by)^2 \nu(\by) \; \md \by<\infty~~~\mbox{for all}~\bx \in \cX,
    \end{equation}
the linear operator $T_\kernel$ defined in \eqref{eqn.defineTkernel} has at most countably many positive (and no negative) eigenvalues\index{eigenvalue} $\lambda_1,\lambda_2,\dots$ with corresponding eigenfunctions\index{eigenfunction} $e_1(\cdot),e_2(\cdot),\dots$ which are orthonormal\index{orthonormal} with respect to the inner product\index{inner product} $\braket{\cdot,\cdot}_\nu$ defined in \eqref{eqn.def.inner.prod}. Furthermore, 
 $\kernel$ can be decomposed as
\[
\kernel(\bx,\by)=\sum_{j=1}^\infty \lambda_k e_j(\bx)e_j(\by)
\]
and the set $\{\sqrt{\lambda_j}e_j(\cdot)\}_{j=1}^\infty$ forms an orthonormal\index{orthonormal} basis with respect to the inner product\index{inner product}
\[
\braket{f(\cdot),g(\cdot)}_\kernel=\sum_{j=1}^\infty \frac{f_j g_j}{\lambda_j},
\]
where \begin{equation}  f(\cdot)=\sum_{j=1}^\infty f_j e_j(\cdot)~~~\mbox{and} ~~~g(\cdot)=\sum_{j=1}^\infty g_j e_j(\cdot).
\label{eqn.fg.full.decomp}
\end{equation}

The space in which $e_1(\cdot),e_2(\cdot),\dots$ lie is a generalisation of the vector space\index{vector space} of Example \ref{example.fd.squareInt} to the \emph{Hilbert space}\index{Hilbert space}, $\cH_\nu$, of functions $f(\cdot):\cX\to \mathbb{R}$ with the inner product\index{inner product} $\braket{\cdot,\cdot}_\nu$ and such that
$\|f(\cdot)\|_\nu^2=\int f(\bx)^2 \nu(\bx) \; \md \bx<\infty$. Likewise, the orthonormal\index{orthonormal} basis $\{\sqrt{\lambda_j}e_j(\cdot)\}_{j=1}^\infty$ lies in the \emph{reproducing kernel Hilbert space}\index{reproducing kernel Hilbert space}, $\cH_{\kernel}$, of functions with the inner product\index{inner product} $\braket{\cdot,\cdot}_\kernel$ and such that $\|f(\cdot)\|_{\kernel}<\infty$. A \emph{Hilbert space}\index{Hilbert space} $\mathcal{H}$ is an inner-product space\index{inner product space} with a potentially infinite set of basis vectors that is \emph{complete}; informally, it contains no `holes', so that for any sequence $f_1,f_2,\dots$ with $\sum_{j=1}^\infty f_j^2<\infty$ then (e.g. considering $\cH_\kernel$) $\lim f(\cdot)=\lim_{n\to \infty} \sum_{j=1}^n f_j e'_j(\cdot)$ exists,  with distance measured through the norm induced by the inner product\index{inner product}, and is in $\cH_\kernel$.

Thus, the simplifications of the inner products\index{inner product} in  \eqref{eqn.kernelOne} and \eqref{eqn.kernelTwo} continue to hold; in general, the intermediate steps must replace $n$ with $\infty$.

\begin{example} 
For the Gaussian kernel\index{kernel (Gaussian)} of Example \ref{example.GaussianKernel}, $\kernel(\bx,\cdot)=\exp(-\|\bx-\cdot\|^2)$ and
\[
\left\langle
\exp(-\|\bx-\cdot\|^2)
,
\exp(-\|\by-\cdot\|^2)
\right\rangle_\kernel
=\exp(-\|\by-\bx\|^2).
\]
Also, for any $f(\cdot) \in \cH_\kernel$,
\[
\left\langle
\exp(-\|\bx-\cdot\|^2)
,
f(\cdot)
\right\rangle_\kernel
=f(\bx).
\]
\end{example}

\subsubsection{Trace-Class Kernels}\index{kernel (trace-class)}

A kernel where $\int \mathsf{k}(\bx,\bx)\nu(\bx) \; \md \bx=c < \infty$ is referred to as \emph{trace class}. This property has important consequences for the set of eigenvalues\index{eigenvalue}, $\lambda_1,\lambda_2,\dots$, of $T_\kernel$, since
 \begin{align*}
  \int \mathsf{k}(\bx,\bx)\nu(\bx) \;\md \bx
  &=
  \int \sum_{k=1}^\infty \lambda_k e_k(\bx)e_k(\bx)\nu(\bx) \; \md \bx\\
  &=
  \sum_{k=1}^\infty \lambda_k \int e_k(\bx)^2\nu(\bx) \; \md \bx
  =\sum_{k=1}^\infty \lambda_k.
\end{align*}
Thus, $\sum_{k=1}^\infty \lambda_k=c$. Since each $\lambda_k\ge 0$, we have $\lim_{k\to \infty}\lambda_k=0$.

The Gaussian kernel\index{kernel (Gaussian)} of Example \ref{example.GaussianKernel} is trace class\index{kernel (trace-class)} with $c=1$ since $\kernel(\bx,\bx)=1$ for all $\bx \in \mathbb{R}^d$. The kernel we will meet in in Chapter \ref{chap:stein} is also of trace class, following a similar reasoning. 

Without loss of generality, we label the eigenvalues\index{eigenvalue} $\lambda_1,\lambda_2\dots$ in order of decreasing size (choosing any one of the possibilities if some of the $\lambda_j$ are not unique). 
With the decomposition of $f(\cdot)$ in \eqref{eqn.fg.full.decomp},
\[
\|f(\cdot)\|_{\kernel}^2
=
\sum_{j=1}^\infty \frac{f_j^2}{\lambda_j}
\ge
\frac{1}{\lambda_1}
\sum_{j=1}^\infty f_j^2
=\frac{1}{\lambda_1}\|f(\cdot)\|_{\nu}^2.
\]
Thus, $\|f(\cdot)\|_{\kernel}<\infty\implies \|f(\cdot)\|_{\nu}<\infty$ and hence $
\cH_\kernel \subseteq \cH_\nu$. 
In general, $\cH_\kernel$ is strictly smaller than $\cH_\nu$ and the more quickly the eigenvalues\index{eigenvalue} of $T_\kernel$ decay, the smaller the space $\cH_\kernel$.

\subsection{The Power of the Kernel Trick}
\label{sec.using.kernel.trick}

Suppose we have values $\bx_1,\dots,\bx_m\in \cX$ and we are interested in 
\[
\Vsp^*
=
\left\{g(\cdot):g(\cdot)=\sum_{j=1}^m g_j \kernel(\bx_j,\cdot), \; g_1,\dots,g_m\in \mathbb{R}\right\}.
\]
First, for any $g(\cdot)=\sum_{j=1}^m g_j \kernel(\bx_j,\cdot)$,
\begin{align}
\nonumber
\|g(\cdot)\|^2_\kernel
&=
\left\langle
\sum_{j=1}^m g_j \kernel(\bx_j,\cdot),
\sum_{k=1}^m g_k \kernel(\bx_k,\cdot)
\right\rangle_\kernel
=
\sum_{j=1}^m\sum_{k=1}^m
g_j \braket{\kernel(\bx_j,\cdot),\kernel(\bx_k,\cdot}g_k\\
&=
\sum_{j=1}^m\sum_{k=1}^m
g_j\kernel(\bx_j,\bx_k)g_k<\infty.
\label{eqn.normInKernelRep}
\end{align}
So, $\Vsp^*\subseteq \cH_\kernel$. Second, for any $f(\cdot)\in \cH_\kernel$,
\begin{equation}
\label{eqn.dotInKernelRep}
\braket{f(\cdot),g(\cdot)}_\kernel
=\sum_{j=1}^m g_j\braket{f(\cdot),\kernel(\bx_j,\cdot)}_\kernel=\sum_{j=1}^m g_j f(\bx_j).
\end{equation}
 
Suppose there is a particular function of interest, $f(\cdot)\in \cH_\kernel$, and we would like to construct the function $g(\cdot)\in \Vsp^*$ that most closely resembles $f(\cdot)$ in shape. We could find the unit vector in $\Vsp^*$ that has the largest component in the $f(\cdot)$ direction,
\[
\widehat{g}(\cdot)=\argmax_{g(\cdot)\in \Vsp^*: \|g(\cdot)\|=1} \braket{f(\cdot),g(\cdot)}_{\kernel}.
\]
The size of the inner product\index{inner product}, $\braket{f(\cdot),\widehat{g}(\cdot)}_{\kernel}$, is a measure of the ability of $\Vsp^*$ to represent $f(\cdot)$.

Define $\bmf=[f(\bx_1),\dots,f(\bx_m)]^\top$ and $\bg=[g_1,\dots g_m]^\top$ and let $\bK$ be the matrix with elements $K_{i,j}=\kernel(\bx_i,\bx_j)$. Then, \eqref{eqn.normInKernelRep} and \eqref{eqn.dotInKernelRep} become
\[
\braket{f(\cdot),g(\cdot)}_\kernel= \bg^\top \bmf
~~~\mbox{and}~~~
\|g(\cdot)\|_\kernel^2
=
\bg^\top \bK \bg.
\]
To find $\widehat{g}(\cdot)$ we must find the vector $\widehat{\bg}$ that maximises $\bg^\top \bmf$ subject to $\bg^\top \bK \bg=1$.

Let $\bA$ be a square matrix such that $\bA \bA^\top=\bK$ and set $\bh=\bA^\top \bg$. Then, equivalently, we wish to maximise $ \bh^{\top}\bA^{-1} \bmf$ such that $\|\bh\|=1$. We must find the unit $m$-vector with the largest component in the $\bA^{-1} \bmf$ direction, which is
\[
\widehat{\bh}=\frac{\bA^{-1}\bmf}
{\sqrt{\left(\bA^{-1}\bmf\right)^\top\bA^{-1} \bmf}}
\implies
\widehat{\bg}
=
\frac{\bA^{-\top}\bA^{-1}\bmf}
{\sqrt{\bmf^\top \bA^{-\top}\bA^{-1}\bmf}}
=
\frac{\bK^{-1}\bmf}
{\sqrt{\bmf^\top \bK^{-1} \bmf}},
\]
since $\bghat=\bA^{-\top} \widehat{\bh}$.
The inner product\index{inner product} $\widehat{\bg}^\top\bmf$ is
\[
\frac{\bmf^\top \bK^{-1}\bmf}
{\sqrt{\bmf^\top \bK^{-1} \bmf}}
=
\sqrt{\bmf^\top \bK^{-1} \bmf}
.
\]
This calculation \emph{only} requires us to be able to evaluate $f(\bx_j)$ and $\kernel(\bx_j,\bx_k)$ for $j,k=1,\dots,m$. We do not need to know the eigenfunctions\index{eigenfunction} $e_1(\cdot),\dots$ nor eigenvalues\index{eigenvalue} $\lambda_1,\dots$ of $T_\kernel$. Indeed, we do not even need to know $\nu$; only that \eqref{eqn.Sun.condition} is satisfied.

\begin{example}
\label{example.kernelAppr}
Let $\cX=\mathbb{R}$ and let $\kernel$ be the one-dimensional case of the Gaussian kernel\index{kernel (Gaussian)} in Example \ref{example.GaussianKernel}. We find the approximations to the function
\[
f(x)=\frac{1}{1+x^2},
\]
using gradually more and more kernel functions $\kernel(x_j,x)$. For  points, $x_1,\dots,x_J$, $\bK$ is the matrix with elements $K_{i,j}=\exp[-(x_i-x_j)^2]$, and $\bmf$ is the vector with $f_j=f(x_j)$. We set $(x_1,x_2,x_3,x_4,x_5,x_6,x_7)=(-3,\dots,3)$ and approximate $f(x)$ using just $x_1$ then $x_1,\dots,x_3$, then $x_1,\dots,x_5$ and finally $x_1,\dots,x_7$. Figure \ref{fig:ch1_kernelapprox} compares the four approximations with the truth. Each time new points are added to the set, the approximation improves, but it matters where the points are added; some basis vectors are more helpful than others. 
\end{example}

\begin{figure}
    \centering
    \includegraphics[width=0.7\textwidth]{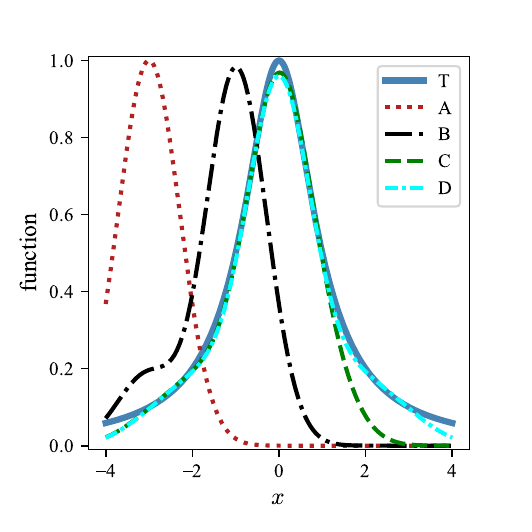}
    \caption{The function $f(x)=1/(1+x^2)$ (T) and kernel-based approximations to $f(x)$ from Example \ref{example.kernelAppr}. Curves use A: $x=-3$, B: $x=-3,-2,-1$, C: $x=-3,\dots,1$ and D: $x=-3,\dots,3$. }
\label{fig:ch1_kernelapprox}
\end{figure}

\section{Chapter Notes}

There are many texts that cover the introductory material from this chapter in more depth and rigour than we have allowed; we suggest a few on each topic. 

Basic Monte Carlo and importance sampling\index{importance sampling} are covered in \cite{ripley2009stochastic} \cite{rubinstein2008}. For an introduction to Bayesian statistics and the use of Monte Carlo methods for Bayesian analysis, see \cite{bernardo2009bayesian}, \cite{robert2007bayesian} and \cite{robert1999monte}. 

\cite{norris1998markov} provides a gentle introduction to Markov chains on discrete state spaces, while \cite{meyn2012markov} gives a thorough treatment on general state spaces; a less thorough but more readily accessible treatment for general state spaces is given by \cite{roberts2004general}. \cite{Geyer1992} describes methods for estimating the integrated auto-correlation time\index{integrated auto-correlation time} from a sample of the chain when the Markov chain is reversible; for the non-reversible chains of Chapter \ref{chap:non-reversible} the integrated auto-correlation can be estimated by fitting an auto-regressive process to the time series $\{h(X_k)\}_{k=1}^n$ or by estimating the spectral density of the series at a frequency of $0$ \cite[e.g.][]{HeidWelch1981}.  

Stochastic differential equations and diffusions are the subject of \cite{oksendal2013stochastic}, \cite{rogers2000diffusions} and \cite{rogers2000diffusions2}.  An alternative to simple Monte Carlo, which attempts to obtain better convergence rates with the Monte Carlo sample size $n$, is quasi-Monte Carlo. See, for example, \cite{caflisch1998monte} for an introduction and \cite{l2002recent} for work on randomised quasi-Monte Carlo.

Chapter 1 of \cite{Conway2010} introduces Hilbert spaces in general, and kernels and reproducing kernel Hilbert spaces are covered in Chapter 6 of \cite{RasmussenWilliams2005}. Mercer's theorem is usually stated for a compact $\cX$; we have used the generalisation to non-compact spaces in \cite{Sun2005}.


\chapter{Reversible MCMC and Its Scaling}\index{reversible}
\label{chap:intro-mcmc}

Building on the introductions to Bayesian statistics, Monte Carlo methods and Markov chains in Chapter \ref{chap:background}, this section introduces Markov chain Monte Carlo (MCMC)\index{Markov chain Monte Carlo} algorithms as a generic computational solution to the challenge of using Monte Carlo methods to sample from the posterior distribution\index{posterior distribution} and, hence, estimate posterior expectations of quantities of interest. 

As described in Chapter \ref{chap:background}, if it is possible to sample directly from the posterior, $\pi(\btheta):=\pi(\btheta|\data)$ (see equation \eqref{eqn.postispriortimeslike}) then for any function $h$ with $\Expects{\pi}{h^2(\btheta)}<\infty$, it is possible to estimate $\Expects{\pi}{h(\btheta)}$ via the Monte Carlo average \eqref{eqn.gen.MCaverage.pi}, the typical error of which is of size $n^{-1/2}$, where $n$ is the number of samples. Unfortunately, it is usually not possible, or is computationally infeasible, to generate independent and identically distributed samples from $\pi$. Importance sampling\index{importance sampling} provides, perhaps, the most natural alternative to direct sampling; however, as exemplified in Section \ref{sec.MCBayes} the variance of importance sampling\index{importance sampling} estimators typically degrades exponentially quickly with dimension.

\emph{Markov chain Monte Carlo}\index{Markov chain Monte Carlo} (MCMC) is a generalisation of the Monte Carlo method that, as we will see, has several favourable properties when it comes to facilitating computation in Bayesian statistics problems. In this context, 
the aim of MCMC is to construct a Markov chain, $\{\btheta_k\}_{k=1}^\infty$ whose limiting distribution is the posterior distribution\index{posterior distribution} of interest, so that samples from a sufficiently long chain, except, perhaps, those near the beginning, arise approximately from the posterior and can be used to create Monte Carlo approximations to expectations as in \eqref{eqn.gen.MCaverage.pi}, via the ergodic average defined in \eqref{eqn.SLLNMC}. 

The workhorse of MCMC is the Metropolis--Hastings\index{Metropolis--Hastings} algorithm.  
We describe the general Metropolis--Hastings\index{Metropolis--Hastings} algorithm and show that the resulting chain satisfies detailed balance\index{detailed balance} with respect to $\pi$. We then investigate particular special cases: the independence sampler\index{independence sampler}, the random walk Metropolis\index{random walk Metropolis} algorithm, the Metropolis-adjusted Langevin algorithm, and Hamiltonian Monte Carlo. For each of these cases, we overview the behaviour as the dimension $d\to \infty$, motivating the need for further scalable methods.

Throughout this chapter, we denote the support of $\pi$ by $\Theta$; for example, $\Theta$ might be $\mathbb{R}^d$ for some $d\in \mathbb{N}$.


\section{The Metropolis--Hastings Algorithm}
\label{sec.MHgen}

The idea of the \emph{Metropolis--Hastings}\index{Metropolis--Hastings} algorithm is to define the dynamics of a Markov chain\index{Markov chain!discrete-time} by specifying an arbitrary proposal distribution for the next state of the Markov chain, and then having a mechanism where this proposal is either accepted or rejected. If it is rejected, the state of the Markov chain is unchanged. As we will see, it is generally possible to choose the acceptance probability\index{acceptance probability} to depend on the target distribution, so that the resulting Markov chain will have the target distribution as its stationary distribution\index{stationary distribution}.

The Metropolis--Hastings\index{Metropolis--Hastings} algorithm is given in Algorithm \ref{alg:MH}. The posterior density, $\pi$, appears in both the numerator and denominator of the acceptance probability\index{acceptance probability}, $\alpha(\btheta_k,\btheta')$, so terms involving the typically intractable density $\pi(\btheta)$ can be replaced with the tractable product of the prior and the likelihood, $\prior{\btheta}\likelihood{\btheta}{\data}$; we do not need to know the normalising constant, $p(\data)=\int_\Theta \prior{\btheta}\likelihood{\btheta}{\data}\mbox{d}\btheta$, as it will cancel in the ratio.

As well as $\pi(\btheta)$, and an initial value for the parameter vector, the Metropolis--Hastings\index{Metropolis--Hastings} algorithm requires a proposal density, $q(\btheta'|\btheta)$.
Common choices of the density $q$ include the following:

\begin{description}
\item[Metropolis--Hastings Independence Sampler (MHIS)] $q(\btheta'\mid \btheta):=q'(\btheta')$ for some density $q'$. The proposal does not depend on the current state; for example, $q'$ could be the same as a sensible importance sampling\index{importance sampling} proposal\index{proposal} distribution (see Section \ref{sec.MCBayes}).
\item[Random walk Metropolis (RWM)]\index{random walk Metropolis} $q(\btheta'\mid \btheta)=q'(\btheta'-\btheta)$, where $q'$ is a density such that for any vector $\bv\in \Theta$, $q'(\bv)=q'(-\bv)$. For example $\btheta'| \btheta\sim \mathsf{N}(\btheta,\lambda^2 \bI_d)$, where $\bI_d$ is the $d\times d$ identity matrix and $\lambda>0$.
\item[Metropolis-adjusted Langevin algorithm (MALA)]\index{Metropolis--adjusted Langevin algorithm} adds a specific form of offset to a Gaussian RWM proposal\index{proposal}. For example, 
  $\btheta'|\btheta\sim \mathsf{N}(\btheta+\frac{1}{2}\lambda^2\nabla \log \pi(\btheta),\lambda^2 \bI_d)$.
  \item[Hamiltonian Monte Carlo (HMC)] starting from $\btheta$ and with a random momentum, such as $\bp\sim \mathsf{N}(\bzero,M\bI_d)$, Hamiltonian dynamics are approximately integrated forwards on a potential surface of $-\log \pi(\btheta)$. The proposal, $\btheta'$, is the position after some fixed time $T$. 
\end{description}

Later in this chapter, we describe and investigate these classes of proposals in more detail and examine their relative efficiencies. 

\begin{algorithm}
    \caption{Metropolis--Hastings algorithm}
    \KwIn{Density $\pi(\btheta)$, initial value $\btheta_0$ and proposal density $q(\btheta'|\btheta)$.}
    \For{$k \in 0, \dots, n-1$} {
        Propose $\btheta'$ from $q(\btheta'|\btheta_k)$ \\
        Calculate the acceptance probability\index{acceptance probability}:
        \begin{equation}\label{eqn.MHaccProb}\alpha(\btheta_k,\btheta'):=\min\left(1,\frac{\pi(\btheta')q(\btheta_k|\btheta')}{\pi(\btheta_k)q(\btheta'|\btheta_k)}\right).\end{equation}\\
        With a probability of $\alpha(\btheta_k,\btheta')$ accept the proposal, $\btheta_{k+1}\gets\btheta'$; otherwise reject it, $\btheta_{k+1}\gets \btheta_k$.    }
    \label{alg:MH}
\end{algorithm}

That the Metropolis--Hastings\index{Metropolis--Hastings} algorithm has a stationary density of $\pi$, follows directly from the fact that it is reversible\index{reversible} with respect to $\pi$ (see Section \ref{sec:ch1-MC}). We now show that this is the case. First, notice that
\[
\pi(\btheta)q(\btheta'| \btheta)\alpha(\btheta,\btheta')
=
\pi(\btheta')q(\btheta|\btheta')\alpha(\btheta',\btheta),
\]
since both are $\min\left(\pi(\btheta)q(\btheta'| \btheta),\pi(\btheta')q(\btheta| \btheta')\right)$. Now suppose that $\btheta_k\sim \pi$, let $\cB,\cC\subseteq \Theta$ and let $\sA$ be the event that the proposal is accepted. Then
\begin{align*}
  \Prob{\sA,\btheta_k\in \cB, \btheta_{k+1}\in \cC}
  &=
  \int_{\btheta_k\in \cB} \int_{\btheta'\in \cC} \pi(\btheta_k)q(\btheta'\mid \btheta_k)\alpha(\btheta_k,\btheta')~\md \btheta' \md \btheta_k\\
  &=\int_{\btheta_k\in \cB} \int_{\btheta'\in \cC}\pi(\btheta')q(\btheta_k\mid \btheta')\alpha(\btheta',\btheta_k)~\md \btheta_k \md \btheta'\\
  &=\int_{\btheta'\in \cB}\int_{\btheta_k\in\cC} \pi(\btheta_k)q(\btheta'\mid \btheta_k)\alpha(\btheta_k,\btheta')~\md \btheta_k \md \btheta'\\
  &=
  \Prob{\sA,\btheta_k\in \cC, \btheta_{k+1}\in \cB}
\end{align*}
where, on the penultimate line, we have switched the labels. Since $\btheta_k=\btheta_{k+1}$ on a rejection, we also have that
\[\Prob{\sA^\complement,\btheta_k\in \cB,\btheta_{k+1}\in \cC}
  =
  \Prob{\sA^\complement, \btheta_k\in \cC, \btheta_{k+1}\in \cB}.\]
Summing the two equalities above for $\sA$ and $\sA^\complement$ gives
\[
\Prob{\btheta_k\in\cB,\btheta_{k+1}\in\cC}=\Prob{\btheta_k\in\cC,\btheta_{k+1}\in\cB},\]
as required.

\subsubsection{Burn-In, Mixing, Estimators, and their Variance}

Typically, the initial value for the Markov chain is not sampled from $\pi$, since if it were possible to do this then there would be no need for MCMC. Hence, the marginal distributions of early points in the Markov chain might not be sufficiently close to $\pi$. In practice, we discard such early points from the sample; the terms \emph{warm-up}\index{warm-up} or \emph{burn-in}\index{burn-in} are applied to both this initial period and the samples that arise from it. Here, we imagine that there are $b$ burn-in samples, $\btheta_{-b+1},\dots\btheta_0$ and that the remaining samples, which are deemed to be from a chain that has approximately converged, are $\btheta_1,\dots,\btheta_n$. The expectation of any function $h(\btheta)$ with respect to the posterior is then estimated via:
\begin{equation}
\label{eqn.ergodic.av.posterior}
\widehat{I}_n(h):=\frac{1}{n}\sum_{k=1}^nh(\btheta_k).
\end{equation}
Following the exposition in Section \ref{sec:ch1-IACFESS}, subject to conditions, including that $\Expects{\pi}{h(\btheta)^2}<\infty$, the variance of this estimator may be approximated as in  \eqref{eqn.VarIhat.IACT}; this is an approximation both because the Markov chain has not fully converged after the $b$ burn-in\index{burn-in} iterations, and because $n$ is finite. As with standard Monte Carlo estimates, the standard error decreases as $n^{-1/2}$; however, the constant of proportionality is (typically) higher, reflecting the fact that the samples are (typically) positively correlated.  

In most MCMC algorithms, the positive correlation arises from two separate sources: firstly, a proposal may be rejected, in which case the new position of the chain is the same as the old position; second, most types of Metropolis-Hastings\index{Metropolis--Hastings} algorithm are \emph{local}: the proposal is, in some sense, close to the current value when compared to the size of the main posterior mass. Consequently, the chain can take many iterations to move from one part of the posterior to another. The act of moving around the posterior is termed \emph{mixing,}\index{mixing} and in this book, we informally refer to the number of iterations taken to move substantially within the context of the posterior distribution\index{posterior distribution} as the \emph{mixing time}. 

\subsubsection{Running Example}
The following running example of an isotropic Gaussian target distribution serves to demonstrate some properties of the Metropolis--Hastings\index{Metropolis--Hastings} algorithm in practice. In the next section, it will be used to illustrate the different measures of the efficiency of a Metropolis--Hastings\index{Metropolis--Hastings} Markov chain algorithm.

\begin{example}
  \label{example.RWMa}
Given $d\in \mathbb{N}$, and $\btheta=(\theta_1,\dots,\theta_d)^\top$, let 
\[
\pi_d(\btheta)=
\mathsf{N}(\btheta;0,\bI_d)
\equiv
\frac{1}{(2\pi)^{d/2}}\exp\left(-\frac{1}{2}\sum_{i=1}^d\theta_i^2\right)\equiv
\frac{1}{(2\pi)^{d/2}}\exp\left(-\frac{1}{2}\|\btheta\|^2\right),
\]
where $\|\cdot\|$ denotes the Euclidean norm, and $\bI_d$ the $d\times d$ identity matrix.
\end{example}

For now, we explore the above target using the RWM algorithm described above: 
\[
q(\btheta'|\btheta)=\mathsf{N}\left(\btheta';\btheta,\lambda^2 \bI_d\right)
\equiv
\frac{1}{(2\pi)^{d/2}\lambda^d}\exp\left(-\frac{1}{2\lambda^2}\|\btheta'-\btheta\|^2\right).
\]  
 Figure \ref{fig.RWMa} shows plots from $n=1000$ iterations of the algorithm in Example \ref{example.RWMa} with $d=1$. The leftmost plot starts from $\btheta_0=20$ while the other two start from $\btheta_0=1$. More than $99.7\%$ of the posterior mass lies between $\btheta=-3$ and $\btheta=3$, and so the chain that was started outside of this region first heads towards the main mass. Once it has arrived, it then explores the region for the remainder of the time, $n$. 
 The exploration is slow because the scale of the proposed jumps, $\lambda=0.2$, is small compared with the size of the region. With larger proposed jumps, $\lambda=2$, the exploration is much more rapid. However, with jumps of size $\lambda=20$, most of the proposals are outside of the high-density region, so the acceptance ratio\index{acceptance ratio} is small and the proposals are rejected. Thus, even though the proposed jumps are large, the algorithm does not explore the range of posterior values quickly. 

\begin{figure}[!t]
\centering
\includegraphics[width=\textwidth]{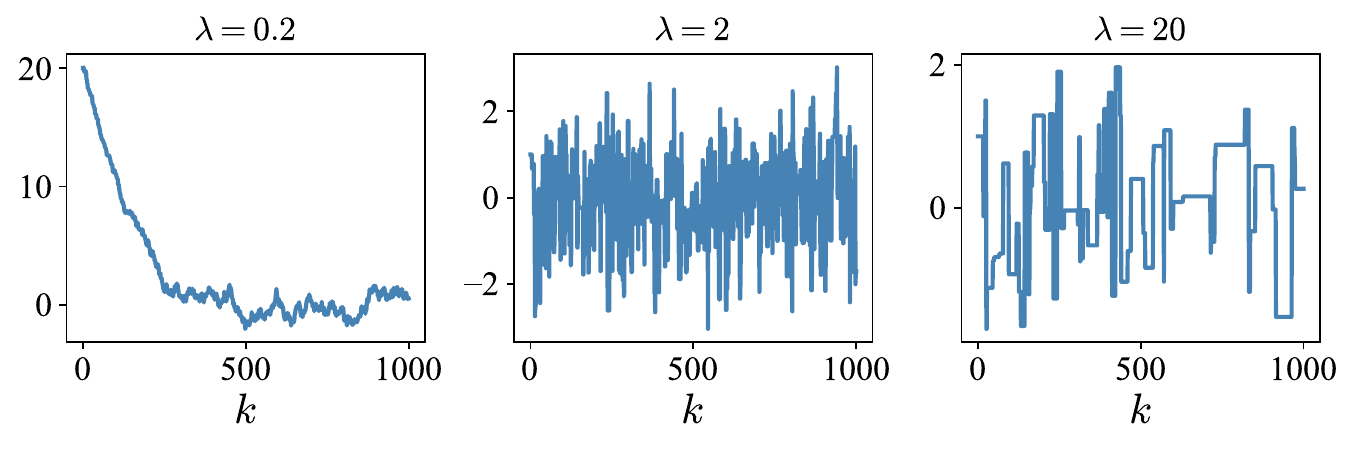}
 \caption{Trace plots from three RWM runs on $\pi_d$ from Example \ref{example.RWMa} with $d=1$, using initial values of $20$, $1$ and $1$, and scalings\index{scaling} of $0.2$, $2$ and $20$ respectively. 
 \label{fig.RWMa}}
\end{figure}

Theory tells us that the distribution of $\btheta_k$ converges to $\pi$ as $k\to \infty$. For a finite $k$, $\btheta_k$ will not be an exact draw from $\pi$ but it might be close. In Figure \ref{fig.RWMa} (left), we might deem the distribution sufficiently close after approximately $300$ iterations and so we might discard $\btheta_0,\dots,\btheta_{299}$ as \emph{burn-in}\index{burn-in} and take $\{\btheta_{300}, \dots,\btheta_{1000}\}$ to be an approximate, correlated, sample from $\pi$ for use in a Monte Carlo average, $\Ihat(h)$, of the form \eqref{eqn.gen.MCaverage.pi}. The runs illustrated in Figure \ref{fig.RWMa} (middle and right) started from a sensible value in the posterior and so $\{\btheta_1,\dots,\btheta_{1000}\}$ might reasonably be used. 

Let us now ignore the need for burn-in\index{burn-in} when $\btheta_0=20$ and $\lambda=0.2$, or consider a thought experiment where this algorithm was also started from $\btheta_0=1$. From Figure \ref{fig.RWMaHist}, the sample obtained when $\lambda=2$ appears to represent $\pi$, which is symmetric about a single mode at $0$ and has support beyond $\pm 2$, much better than either of the other two samples. Thus, we might expect estimates, $\Ihat(h)$, obtained from the algorithm when $\lambda=2$ to be, in some sense, more accurate.

\begin{figure}[!t]
\centering
\includegraphics[width=\textwidth]{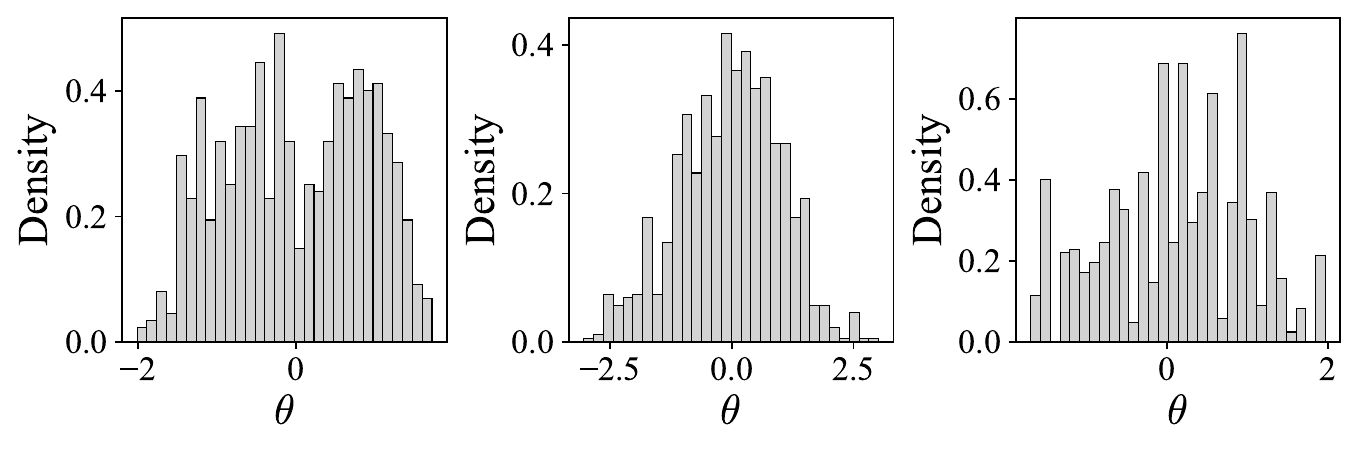}
 \caption{Histograms of the samples obtained from the algorithms in Figure \ref{fig.RWMa}. The left plot (corresponding to $\lambda=0.2$, $\btheta_0=20$) was created after discarding $\{\btheta_{0},\dots,\btheta_{299}\}$ as burn-in\index{burn-in}; for the other two runs (centre: $\lambda=2$; right: $\lambda=20$) only $\btheta_0$ was discarded. 
  \label{fig.RWMaHist}}
\end{figure}

The \emph{empirical acceptance rate} for a Metropolis--Hastings\index{Metropolis--Hastings} algorithm is the fraction of the $n$ proposals that were accepted. 
For the three RWM algorithms, these were respectively $0.881$, $0.485$ and $0.059$; the smaller the proposed jumps, the closer $\pi(\btheta')$ typically is to $\pi(\btheta_k)$ and so the higher the acceptance rate\index{acceptance rate}. For a stationary Metropolis--Hastings\index{Metropolis--Hastings} Markov chain, the empirical acceptance rate approximates the true acceptance rate at stationarity:
\[
\alpha=\Expects{\btheta\sim \pi,\btheta'\sim q(\cdot|\btheta)}{\alpha(\btheta,\btheta')}
\]
Aspects of the proposal\index{proposal} for the RWM, MALA, and HMC are often tuned by targeting an empirical acceptance rate\index{acceptance rate} that is neither too high nor too low. In later sections, the acceptance rate will provide us with an intuitive entrance into the behaviour of the canonical Metropolis--Hastings\index{Metropolis--Hastings} algorithms as the dimension of the parameter vector increases. Here, it will be helpful to define the \emph{acceptance ratio}\index{acceptance ratio}: $\rho(\btheta,\btheta'):=\pi(\btheta')q(\btheta|\btheta')/\{\pi(\btheta)q(\btheta'|\btheta)\}$, so that $\alpha(\btheta,\btheta')=\min[1,\rho(\btheta,\btheta')]$.

  \subsection{Component-wise Updates and Gibbs Moves} \label{ch2:sec:Gibbs}
Algorithm \ref{alg:MH}, the Metropolis--Hastings\index{Metropolis--Hastings} algorithm, and all of the special cases that we will examine in this chapter, are written so that a single iteration consists of a proposal\index{proposal} to change the entire $\btheta$ vector and a decision on whether or not to accept this proposal\index{proposal}. However, it is also possible, and sometimes helpful, to sequentially update subsets of the components of $\btheta$. Indeed, each iteration of the very first Metropolis--Hastings\index{Metropolis--Hastings} algorithm \cite[]{metropolis1953equation} cycled through pairs of components ($x$ and $y$ coordinates of each particle in a lattice of a large number of particles), applying a RWM\index{random walk Metropolis} update one pair at a time.

Denote the set of components to be updated by $\btheta^{(i)}$ and the remaining components by $\btheta^{(-i)}$. By writing the target as $\pi(\btheta)=\pi(\btheta^{(-i)})\pi(\btheta^{(i)}|\btheta^{(-i)})$, and the proposal\index{proposal} as $q_i(\btheta^{'(i)}|\btheta^{(i)},\btheta^{(-i)})$,  essentially the same argument as for a proposal\index{proposal} that changes all components shows that the component-wise\index{component-wise} propose/accept-reject step with an acceptance probability\index{acceptance probability} of
\[
\min\left(1,\frac
{\pi(\btheta^{'(i)}|\btheta^{(-i)})q_i(\btheta^{(i)}|\btheta^{'(i)},\btheta^{(-i)})}
{\pi(\btheta^{(i)}|\btheta^{(-i)})q_i(\btheta^{'(i)}|\btheta^{(i)},\btheta^{(-i)})}\right)
\]
satisfies detailed balance\index{detailed balance} with respect to the full posterior, $\pi$. Of course, if only that move were used, some components would never be updated, the algorithm would be reducible\index{reducible}, and ergodic averages would, therefore, not converge to the corresponding true expectations. 
The composition of many such moves over different components, typically, does not satisfy this detailed balance\index{detailed balance} condition, but, since each move preserves $\pi$, the composition does, too. 

In the special case where the proposal\index{proposal} for the $i$th block of components is 
\[
q_i(\btheta^{'(i)}|\btheta_k)
:=  
\pi(\btheta^{'(i)}|\btheta_k^{(-i)}),
\]
the acceptance probability\index{acceptance probability} is $1$ and the move is called a \emph{Gibbs} move\index{Gibbs move}.
 Such a move is only feasible when it is possible to sample from $\pi(\btheta^{(i)}|\btheta_k^{(-i)})$ which most usually occurs when, conditional on $\btheta_k^{(-i)}$, the prior for $\btheta^{(i)}$ is conjugate with its likelihood. For example, when $y_1,\dots,y_N$ are independent realisations from a $\mathsf{N}(\mu,1/\tau)$ distribution and $\mu$ and $\tau$ have independent priors with $\mu$ following a Student-t distribution and $\tau\sim \mathsf{Gam}(a,b)$, then \emph{a posteriori} $\tau|\mu\sim \mathsf{Gam}(a+n/2,b+\frac{1}{2}\sum_{j=1}^n(y_i-\mu)^2)$; this property is sometimes called \emph{conditional conjugacy}. Gibbs moves\index{Gibbs move} offer a further advantage when compared with many other Metropolis-Hastings\index{Metropolis--Hastings} moves, over and above the fact that the acceptance probability\index{acceptance probability} is $1$: the updated parameter component is sampled from the full range of its conditional posterior. When components are close to independent, this contrasts with algorithms such as the RWM\index{random walk Metropolis} and MALA, where, in moderate to high dimensions, the moves are local -- the proposed value is close to the current value.  However, when, as is typically the case, components are correlated, the conditional posterior for a component can have a much smaller range than its marginal posterior, and so the Gibbs moves\index{Gibbs move}, too, are, in effect, local. Whilst they are a useful tool in the MCMC armoury, Gibbs moves\index{Gibbs move} are not the focus of this book and for further information, we refer the interested reader to the general texts cited in Section \ref{sec.Ch2.ChapNotes}.

\subsection{The Metropolis--Hastings Independence Sampler}

Consider the \emph{Metropolis--Hastings independence sampler} (MHIS)\index{independence sampler} in the case where $q(\btheta)=\pi(\btheta)$. In this case $\alpha(\btheta_k,\btheta')=1$ and every proposal\index{proposal} is accepted. Since proposals are from $\pi$, the MHIS provides us with an i.i.d. sample from $\pi$. Of course, in practice, we are typically not able to sample from $\pi$, but this suggests a heuristic for the MHIS: choose a proposal so that the acceptance rate\index{acceptance rate} is as close as possible to $1$.

For unimodal posteriors in low dimensions a reasonable approximation can often be obtained by first using a numerical method to find the posterior mode and then choosing a proposal\index{proposal} that matches the mode and the curvature of the log-posterior at this point. 
This strategy does not scale favourably to high dimensions, however, as the following simple example shows. 

Consider the isotropic unit Gaussian posterior from Example \ref{example.RWMa} and  use the following MHIS proposal\index{proposal}:
\[
q(\btheta'|\btheta)=\mathsf{N}(\bzero,\sigma^2 \bI_d)
\equiv
\frac{1}{(2\pi)^{d/2}\sigma^d}\exp\left(-\frac{1}{2\sigma^2}\|\btheta'\|^2\right).
\]
The acceptance ratio\index{acceptance ratio}, $\rho(\btheta,\btheta')$, is
\[
\frac{\exp(-\frac{1}{2}\|\btheta'\|^2)\exp(-\frac{1}{2\sigma^2}\|\btheta\|^2)}
{\exp(-\frac{1}{2}\|\btheta\|^2)\exp(-\frac{1}{2\sigma^2}\|\btheta'\|^2)}
=
\exp\left\{\frac{1}{2}\left(1-\frac{1}{\sigma^2}\right)\left(\|\btheta\|^2-\|\btheta'\|^2\right)\right\},
\]
so
\[
\frac{1}{d}\log \rho(\btheta,\btheta')
=
\frac{1}{2}\left(1-\frac{1}{\sigma^2}\right)\left(\frac{1}{d}\|\btheta\|^2-\frac{1}{d}\|\btheta'\|^2\right).
\]
If the chain is stationary, then $\|\btheta\|^2=\sum_{i=1}^d\theta_i^2\sim \chi^2_d$ since $\theta_i\stackrel{iid}{\sim}\mathsf{N}(0,1)$. Thus $\Expect{\|\btheta\|^2/d}=1$ and $\Var{\|\btheta\|^2/d}=2/d$, and the same properties hold for $\|\btheta'/\sigma\|^2/d$. Thus, in high dimensions, to a first-order approximation, 
$\|\btheta\|^2/d\approx 1$ and $\|\btheta'\|^2/d\approx \sigma^2$ and
\[
\frac{1}{d}\log \rho(\btheta,\btheta')\approx
\frac{1}{2}\left(1-\frac{1}{\sigma^2}\right)\left(1-\sigma^2\right).
\]
This gives a first-order approximation to the acceptance rate\index{acceptance rate} of
\[
\min\left(1,\exp\left\{-\frac{d}{2\sigma^2}\left(\sigma^2-1\right)^2\right\}\right),
  \]
  which grows exponentially small with dimension unless $\sigma=1$. Alternatively, stabilising the acceptance rate above zero requires  $\sigma^2= 1 + O(1/\sqrt{d})$; the approximation must become more and more accurate as $d\to \infty$. 

  The exponential decrease in acceptance rate\index{acceptance rate} with dimension is closely linked with the exponential increase in the variance of the weights with dimension in the importance sampling\index{importance sampling} example at the end of Section \ref{sec.MCBayes}. 
  In high dimensions, a sufficiently accurate and tractable approximation, $q$, is rarely available; consequently, importance sampling\index{importance sampling} and MHIS are rarely used, except in relatively simple, low-dimensional scenarios.

  \subsection{The Random Walk Metropolis Algorithm}
  \label{sec.rwm}

The \emph{random walk Metropolis}\index{random walk Metropolis} (RWM) algorithm was the first MCMC algorithm to ever be used. Unlike the independence sampler, it does not require an accurate global approximation to the posterior and can be tuned so that it works even in very high dimensions. Furthermore, unlike the algorithms that we shall explore subsequently, it does not require the gradient of the log posterior. 

The most frequently used RWM proposal\index{proposal}, the so-called \emph{preconditioned}\index{preconditioned} RWM, has the form $\btheta'|\btheta=\btheta+\Normal(\bzero,\lambda^2 \bV)$, where $\bV$ is an estimate of the posterior variance matrix and $\lambda$ is a tunable scaling\index{scaling} parameter. This enhancement can increase the efficiency of the algorithm by many orders of magnitude when the components of $\btheta$ are highly correlated and/or vary on very different length scales. However, whatever the proposal\index{proposal}, the RWM constraint, that $q(\btheta'|\btheta)=q(\btheta|\btheta')$, means that the acceptance probability\index{acceptance probability} simplifies to $\min\{1,\pi(\btheta')/\pi(\btheta)\}$.

\subsubsection{Scaling of RWM with Dimension}\index{scaling}

We again consider the isotropic Gaussian posterior of Example \ref{example.RWMa} and show how, when the chain is at equilibrium, the RWM algorithm using a $\mathsf{N}(\btheta,\lambda^2 \bI_d)$ proposal\index{proposal} can be made to work no matter what the dimension.

Write the proposal\index{proposal} as $\btheta'=\btheta+\lambda \bZ$, where $\bZ\sim \mathsf{N}(\bzero,\bI_d)$. The log acceptance ratio\index{acceptance ratio} is then
  \begin{align}
  \nonumber
\log \rho(\btheta,\btheta')&=  -\frac{1}{2}\|\btheta+\lambda \bZ\|^2+\frac{1}{2}\|\btheta\|^2
  =
-\lambda \|\btheta\|~\bthetahat\cdot \bZ-\frac{1}{2}\lambda^2\|\bZ\|^2\\
\label{eqn.RWMlognormal}
  &\stackrel{\mathsf{D}}{=}
  -\lambda \|\btheta\|~Z'-\frac{1}{2}\lambda^2\|\bZ\|^2,  
  \end{align}
  where $Z'\sim \mathsf{N}(0,1)$ and $\bthetahat=\btheta/\|\btheta\|$.
  Now $\|\bZ\|^2\sim \chi^2_d$ and, at equilibrium, $\|\btheta\|^2\sim \chi^2_d$. By the same argument as used for the MHIS, we might make a first approximation of $\|\bZ\|^2\approx d$ and $\|\btheta\|\approx \sqrt{d}$, from which it appears that the acceptance ratio\index{acceptance ratio} must decay exponentially quickly with dimension. However, this need not be the case, since we can control the scaling\index{scaling}, $\lambda$. The fact that $\|Z\|^2/d\approx 1$ and $\|\btheta\|/\sqrt{d}\approx 1$ suggests setting
  \begin{equation}
    \label{eq.RWMscaling}
\lambda=\frac{\ell}{\sqrt{d}}
\end{equation}
for some fixed $\ell>0$. In this case 
\begin{align*}
\alpha(\btheta,\btheta')
&\stackrel{\mathsf{D}}{=}
\min\left[1,\exp\left\{-\ell Z'~\frac{\|\btheta\|}{\sqrt{d}}-\frac{1}{2}\ell^2\frac{1}{d}\|\bZ\|^2\right\}\right]\\
&\approx
\min\left[1,\exp\left\{-\ell Z'-\frac{1}{2}\ell^2\right\}\right].
\end{align*}
This quantity is stable away from zero and does not depend on dimension. Taking expectations, elementary calculus gives
\[
\Expects{\btheta\sim \pi,\btheta'\sim q(\cdot|\btheta)}{\alpha(\btheta,\btheta')}\approx
2\Phi\left(-\frac{1}{2}\ell\right),
\]
where $\Phi$ is the cumulative distribution function of a standard normal random variable. 
This equation describes how, for a high-dimensional Gaussian target,  the acceptance rate\index{acceptance rate} decreases as the (dimensionally adjusted) scaling\index{scaling}, $\ell$, increases. 

Indeed, much more is true. Figure \ref{fig.RWMb} shows trace plots for $\theta_1$, the first component of $\btheta$, when $d=50$ and when $d=500$ and using a scaling\index{scaling} of $\lambda=\ell/\sqrt{d}$ with $\ell=2$. The behaviours of the trace plots appear similar, except that when $d=500$ the time scale over which the process explores the posterior is ten times that when $d=50$.

\begin{figure}[ht]
\centering
 \includegraphics[width=0.95\textwidth]{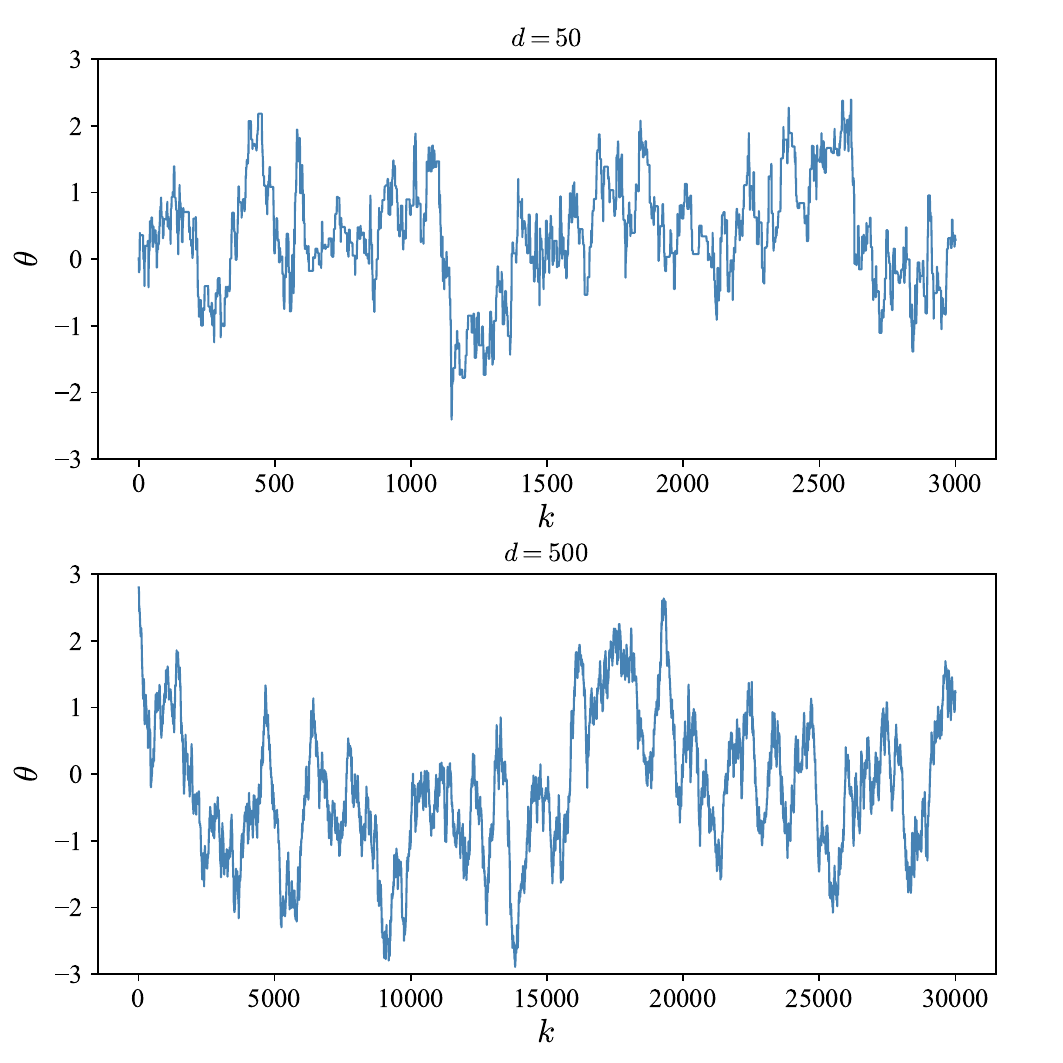}
 \caption{Trace plots for the first component, $\theta_1$, of $\btheta$ for a RWMH on $\pi(\btheta)\propto \exp(-\frac{1}{2}\|\btheta\|^2)$ in $d=50$ (top) and $d=500$ (bottom). Both algorithms were initialised using a sample from $\pi$ and each used a scaling\index{scaling} of $\lambda=\ell/\sqrt{d}$ with $\ell=2$. 
 \label{fig.RWMb}}
\end{figure}

As dimension goes to infinity, with a scaling\index{scaling} of $\ell/\sqrt{d}$ and with time sped up by a factor of $d$ (essentially running for $nd$ iterations rather than $n$), the path of the first component approaches (in distribution) the path of the stochastic differential equation \eqref{eqn.scaling.limit},  below. For simplicity of notation, we denote the first component by $\theta$ and denote its marginal distribution by $f(\theta)\propto\exp(-\theta^2/2)$.
\begin{equation}
\label{eqn.scaling.limit}
\md \theta_t = \frac{1}{2}\left[\log f(\theta_t)\right]'~h(\ell) \md t + \sqrt{h(\ell)}\md W_t,
\end{equation}
where $h(\ell)=\ell^2\times 2\Phi(-\ell/2)$.
This is the OU process defined in \eqref{eqn.OUprocess}, with $b=\sqrt{h(\ell)}$; it  has a $\mathsf{N}(0,1)$ stationary distribution\index{stationary distribution}.  
Here, $h(\ell)$ can be thought of as the \emph{speed} of the diffusion, with a larger value corresponding to a diffusion that will converge to stationarity more quickly, and can be maximised with respect to $\ell$, giving $\ell_{\mathrm{opt}}\approx 2.38$. This corresponds to an acceptance rate\index{acceptance rate} of $2\Phi(-\ell_{\mathrm{opt}}/2)\approx 0.234$, and leads to the well-known advice to choose the RWM scaling\index{scaling} so that the acceptance rate\index{acceptance rate} is approximately $1/4$.

Of more importance for us is that the limiting process is approached by letting $\lambda=\ell/\sqrt{d}$ and speeding up time by a factor of $d$. Reversing this logic, in dimension $d$, the first component moves $d$ times more slowly than the diffusion. In other words, \emph{the time or, equivalently, the number of iterations taken by the RWM to explore the posterior in dimension $d$ is proportional to $d$}.

Figure \ref{fig.RWMACF} emphasises this linear dependence on dimension by continuing the example in Figure \ref{fig.RWMb}. In dimension $d=50$, the algorithm is run for  $n=10000$ iterations, and for each component, the auto-correlations\index{auto-correlation} are calculated up to a lag of $300$. The dotted blue line shows the component-wise\index{component-wise} average of each auto-correlation. For $d=500$, $n=100000$ iterations were used and auto-correlations\index{auto-correlation} up to a lag of $3000$ were calculated. The dashed red line shows the component-wise\index{component-wise} averages plotted against lag$/10$. The curves are almost indistinguishable and the resulting estimated integrated auto-correlation times\index{integrated auto-correlation time} are, respectively, $70$ and $718$. The corresponding effective sample sizes\index{effective sample size} are, therefore, almost identical, even though the experiment with $d=500$ used ten times the number of iterations.

\begin{figure}[ht]
\centering
  \includegraphics[width=0.7\textwidth]{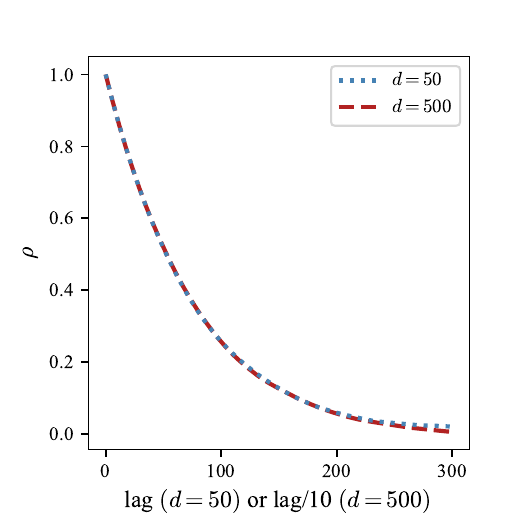}
 \caption{Component-wise\index{component-wise} average auto-correlation plots for a RWM on $\pi(\btheta)\propto \exp(-\frac{1}{2}\|\btheta\|^2)$ in $d=50$ and $d=500$. Both algorithms were initialised using a sample from $\pi$ and each used a scaling\index{scaling} of $\lambda=\ell/\sqrt{d}$ with $\ell=2$. \label{fig.RWMACF}}
\end{figure}

The above arguments have been made rigorous and applied to more complex targets such as $\pi(\btheta)=\prod_{i=1}^d C_if(C_i\theta_i)$, for a large class of density functions $f$ \cite[see][for example]{Roberts:2001}. The limiting process for the first coordinate becomes a Langevin diffusion\index{Langevin diffusion} \eqref{eqn.Langevin} with a stationary density of $C_1 f(C_1 \theta_1)$, and in all cases, the time taken by the RWM to explore the target is proportional to $d$.

\subsection{The Metropolis-Adjusted Langevin Algorithm}
\label{sec:ch2-MALA}

The \emph{Metropolis-adjusted Langevin algorithm} (MALA)\index{Metropolis--adjusted Langevin algorithm} differs from the RWM proposal\index{proposal} of $\btheta'|\btheta \sim \mathsf{N}(\btheta,\lambda^2 \bI_d)$ through an additional deterministic offset of $\frac{1}{2}\lambda^2 \nabla \log \pi(\btheta)$. We motivate this proposal and then generalise it to allow preconditioning\index{preconditioned} via a positive definite variance matrix, $\bV$; as with the RWM, this can bring dramatic efficiency improvements in practice.

The deterministic offset can be seen as an additional movement in the `uphill' direction, that is biasing the proposal\index{proposal} to move to areas of higher posterior density; however, there is a deeper motivation. The proposal\index{proposal} can be written as
\begin{equation}
\label{eqn.MALA.disc.simple}
\btheta'|\btheta =\btheta+\frac{1}{2}\lambda^2 \nabla \log \pi(\btheta) + \boldsymbol{\epsilon}
\end{equation}
where $\boldsymbol{\epsilon}\sim \Normal(0,\lambda^2\bI_d)$.

Substituting $\lambda^2=\delta t$, we see that the proposal\index{proposal} is exactly the Euler--Maruyama\index{Euler--Maruyama} discretisation of the Langevin diffusion\index{Langevin diffusion} that has a stationary distribution\index{stationary distribution} of $\pi$, \eqref{eqn.Langevin} (with $b=1$). 
In particular, in the hypothetical limit as $\lambda\downarrow 0$ the algorithm should require no accept-reject step to target $\pi$. In this sense, it is a natural form for the proposal.

We now derive the preconditioned\index{preconditioned} MALA proposal\index{proposal}. 
For a general positive-definite $\bV$, let $\bA$ be a square matrix such that $\bA\bA^\top=\bV$, and consider $\bpsi=\bA\btheta$. Multiplying \eqref{eqn.MALA.disc.simple} by $\bA$ gives
\[
\bpsi'|\bpsi =\bpsi+\frac{1}{2}\lambda^2 \bA \nabla_{\btheta} \log \pi(\btheta) + \bA \boldsymbol{\epsilon},
\]
where we have made explicit that the gradient\index{gradient} is with respect to $\btheta$. The density for $\bpsi$ is $\widetilde{\pi}(\bpsi)\propto \pi(\bA^{-1}\bpsi)=\pi(\btheta)$. Further, $\nabla_\theta =\bA^\top \nabla_\psi$, so
\[
\bpsi'|\bpsi =\bpsi+\frac{1}{2}\lambda^2 \bV \nabla_{\bpsi} \log \widetilde{\pi}(\bpsi) + \bA \boldsymbol{\epsilon}.
\]
Since $\bA \boldsymbol{\epsilon}\sim \Normal(\bzero,\lambda^2\bV)$, this
 corresponds to the \emph{preconditioned}\index{preconditioned} MALA proposal\index{proposal}:
\[
\btheta'|\btheta \sim
\Normal\left(\btheta+\frac{1}{2}\lambda^2 \bV \nabla \log \pi(\btheta),\lambda^2 \bV\right).
\]
When the posterior is unimodal, preconditioned\index{preconditioned} MALA is often most efficient when $\bV$ is an approximation to the posterior variance.

We now explore the scaling\index{scaling} of MALA with dimension and the sensitivity to large gradients\index{gradient}. In both of these analyses, the following simplification of part of the log acceptance ratio\index{acceptance ratio} will be helpful. For the MALA proposal\index{proposal} in \eqref{eqn.MALA.disc.simple}, and writing $\bg(\btheta)$ for the gradient at $\btheta$,

\begin{align}
\nonumber
\log \frac{q(\btheta|\btheta')}{q(\btheta'|\btheta)}
&=
\frac{1}{2\lambda^2}
\|\btheta'-\btheta-\frac{\lambda^2}{2}\bg(\btheta)\|^2
-
\frac{1}{2\lambda^2}
\|\btheta-\btheta'-\frac{\lambda^2}{2}\bg(\btheta')\|^2\\
&=
\nonumber
\frac{1}{8}\left[
\bg(\btheta)+\bg(\btheta')
\right]^\top
\left[
\lambda^2 \bg(\btheta)-\lambda^2 \bg(\btheta')
+4(\btheta-\btheta')
\right]
\\
&=
\label{eqn.simply.MALA.ratio}
-\frac{1}{8}\left[
\bg(\btheta)+\bg(\btheta')
\right]^\top
\left[
\lambda^2 \bg(\btheta)+\lambda^2 \bg(\btheta')
+4\boldsymbol{\epsilon}
\right].
\end{align}

\subsubsection{Scaling of MALA with Dimension}\index{scaling}
Consider the isotropic Gaussian running example, Example \ref{example.RWMa}, where $\bg(\btheta)=-\btheta$. For the proposal noise, we write $\boldsymbol{\epsilon}=\lambda \bZ$, where $\bZ\sim \Normal(\bzero,\bI_d)$. Using \eqref{eqn.simply.MALA.ratio}, the log acceptance ratio\index{acceptance ratio} for MALA is
\[
\log \rho(\btheta,\btheta')
=
\frac{1}{2}\|\btheta\|^2-\frac{1}{2}\|\btheta'\|^2
-\frac{\lambda^2}{8}\|\btheta+\btheta'\|^2
+\frac{1}{2}\lambda(\btheta+\btheta')^\top \bZ.
\]
Substituting $\btheta'=(1-\frac{1}{2}\lambda^2) \btheta+\lambda \bZ$ from \eqref{eqn.MALA.disc.simple} and collecting terms, we obtain
\begin{equation}
\label{eqn.GaussMALA}
\log \rho(\btheta,\btheta')
=-\frac{\lambda^3}{8}\left[
\lambda\left\{\|\bZ\|^2-\|\btheta\|^2\right\}+\frac{1}{4}\lambda^3\|\btheta\|^2+(2-\lambda^2)\btheta^\top \bZ
\right].
\end{equation}
Since $\|\bZ\|^2$ and $\|\btheta\|^2$ are both $\chi^2_d$, each has an expectation of $d$ and their difference is $O_p(d^{1/2})$; further, $\btheta^\top \bZ\sim\Normal(0,\|\btheta\|^2)$. 

To understand the relative sizes of the terms, we informally write:
\[
\log \rho(\btheta,\btheta')
=
\lambda^3 \left[
\lambda O_p\left(d^{1/2}\right)
+\lambda^3 O_p(d)
+(2-\lambda^2) O_p\left(d^{1/2}\right)
\right].
\]
Thus, with $\lambda=\ell/d^{1/6}$, the first term vanishes and the second and third are $O_p(1)$, leading to an acceptance ratio\index{acceptance ratio} that is $O_p(1)$.

As for the RWM, it is possible to obtain a limiting diffusion of the form \eqref{eqn.scaling.limit} for the first component of $\btheta$ as the dimension goes to infinity. For MALA, however,  the required scaling\index{scaling} is $\lambda=\ell/d^{1/6}$, time is sped up by a factor of $d^{1/3}$ (essentially running for $nd^{1/3}$ iterations rather than $n$) and, for the Gaussian target, the speed of the diffusion is $h(\ell)=2\ell^2 \Phi(-\ell^3/4)$. 
Optimising the speed with respect to the scaling\index{scaling} leads to a recommended acceptance rate\index{acceptance rate} of approximately, $57.4\%$. As with the RWM, the more important point for us is that the limiting OU process mixes in a time of $O(1)$, so the original process, before it has been sped up, mixes in a time of $O(d^{1/3})$. 
This is considerably faster than the $O(d)$ mixing\index{mixing} of the RWM.

As for the RWM, the above result, which is specific to a $\Normal(\bzero, \bI_d)$ target, has been generalised to targets of the form $\pi(\btheta)=\prod_{i=1}^d C_i f(C_i \theta_i)$, again leading to a Langevin diffusion\index{Langevin diffusion} for the first component in the limit as $d\to \infty$, an optimal acceptance rate\index{acceptance rate} of $57.4\%$, and requiring time to be sped up by a factor of $d^{1/3}$. 

The above product results for MALA rely on the existence and good behaviour of all derivatives of $f$ up to the $7$th order, and that the process was started from stationarity. \cite{ChrRobRos2005} investigates the behaviour of MALA on a high-dimensional isotropic Gaussian target when the algorithm is started close to the mode. When a scaling\index{scaling} of $\lambda=\ell/d^{1/6}$ is used, in the limit as $d\to \infty$ the process, sped up by a factor of $d^{1/3}$, does not move. Substituting $\btheta=0$, for example, into \eqref{eqn.GaussMALA}, we see that $\log \rho=-\lambda^4\|\bZ\|^2/8$. Since $\|\bZ\|^2=O_p(d)$, the acceptance probability\index{acceptance probability} is approximately $\exp[-\ell^4 d^{1/3}]$. If, instead, a scaling\index{scaling} of $\lambda=\ell/d^{1/4}$ is used then the acceptance rate\index{acceptance rate} remains $O_p(1)$ as $d\to \infty$. This new process, sped up by a factor of $d^{1/2}$, moves deterministically towards the region of the main posterior mass. Reductions in efficiency can also occur if only lower derivatives of the target are well-behaved.

\subsubsection{Sensitivity to Gradients}

Whilst the scaling properties of MALA are favourable compared with those of the RWM and MHIS, the performance of MALA is notoriously sensitive to large gradients\index{gradient}. We 
 illustrate this with a simple example in one dimension.

\begin{example}
\label{example.BadMALA}
Let $\theta \in \mathbb{R}$ and for some  
    $a>0$ let 
$\pi(\theta)\propto \exp\left(-\frac{1}{a}|\theta|^a\right)$,
so $\nabla \log \pi(\theta)=  -\theta|\theta|^{a-2}$ and $|\nabla \log \pi(\theta)|=|\theta|^{a-1}$.

    When $a>2$, whatever the (fixed) value of $\lambda$, for large enough $\theta$,  $\Expect{\theta'|\theta}$ is dominated by the term $\frac{1}{2}\lambda^2 \nabla \log \pi(\theta)$, so that (with a very high probability) the proposal has the opposite sign to the current value and a much larger magnitude,  $\lambda^2|\theta|^{a-1}/2$.

    Writing $\epsilon=\lambda Z$, where $Z\sim \mathsf{N}(0,1)$, the log acceptance ratio\index{acceptance ratio} for MALA is $\rho(\theta,\theta')=\log \pi(\theta')-\log \pi(\theta)+B(\theta,\theta')$, where, from \eqref{eqn.simply.MALA.ratio},
    \begin{align}
     B(\theta,\theta')
    &=
    \frac{1}{8}
    \{\theta|\theta|^{a-2}+\theta'|\theta'|^{a-2}\}\{4\lambda Z - \lambda^2 \theta |\theta|^{a-2}-\lambda^2 \theta'|\theta'|^{a-2}\}.
 \label{eqn.MALAoneLogAcc}
    \end{align}  
    The highest order term in \eqref{eqn.MALAoneLogAcc} arises from the product of $\theta'$ terms, so it is negative and of order $|\theta|^{2(a-1)^2}$. The difference in log posteriors is dominated by $\log \pi(\theta')\sim -|\theta|^{a(a-1)}$, which is, again, large and negative. Hence, the acceptance probability\index{acceptance probability} is almost certainly very close to $0$. 
     Unsurprisingly, since the proposal is even further from the main mass than the current value is, the proposal is almost certainly rejected. As $\theta$ moves further and further into the tail of the target, the average (over the proposal distribution) acceptance probability for MALA becomes arbitrarily small and the algorithm converges increasingly slowly.  
\end{example}

In Example \ref{example.BadMALA}, similar poor behaviour occurs even with $a=2$ (a Gaussian posterior), provided $\lambda^2>2$. More generally, MALA can become "stuck" anywhere that $\|\nabla \log \pi\|$ is large. In particular, the basic MALA algorithm should be used with caution if the user suspects that the posterior has tails that are lighter than Gaussian. Mitigations against this behaviour are  briefly discussed in the Chapter Notes.

\section{Hamiltonian Monte Carlo} \label{sec:ch2-HMC}
We have seen that when the dimension $d$ is high, MALA can maintain a high acceptance rate\index{acceptance rate} with a scaling\index{scaling} of $\ell/d^{1/6}$, whereas the RWM requires a scaling\index{scaling} of $\ell/d^{1/2}$. In other words, MALA can propose much larger sensible jumps than the RWM. As we shall see, \emph{Hamiltonian Monte Carlo} (HMC)\index{Hamiltonian Monte Carlo} allows even larger jumps than MALA, whilst maintaining a high acceptance rate\index{acceptance rate}.
HMC can be viewed as using a Metropolis--Hastings algorithm\index{Metropolis--Hastings}, but with a more intricate proposal\index{proposal} mechanism than those seen so far. 

One may consider $-\log \pi(\btheta)$ as a \emph{potential energy} surface. 
Intuitively, one may think of this as a physical surface on which a `particle' with mass $M$ currently sits at a `height' (strictly, potential energy) of $-\log \pi(\btheta)$ above a current parameter value,  $\btheta\in \mathbb{R}^d$. To obtain the proposal, the particle is given a random momentum, $\bp\in \mathbb{R}^d$ drawn from a symmetric distribution. The true frictionless motion that the particle would undergo along the potential surface, according to Hamiltonian dynamics, is approximated numerically. The proposal\index{proposal} is the position $\btheta'\in \mathbb{R}^d$ after a time $T$, a tuning parameter. 

As we shall see, the log-acceptance ratio\index{acceptance ratio} for the algorithm can be written as 
\[
\log \rho(\btheta,\btheta')=
-\log \pi(\btheta)+\frac{1}{2M}\bp^\top\bp-\left\{-\log \pi(\btheta')+\frac{1}{2 M}{\bp'}^\top \bp'\right\} ,
\]
where $\bp'$ is the momentum at time $T$. The term, $\bp^\top \bp/(2M)$ is the \emph{kinetic energy} of the particle, and $-\log \pi(\btheta)$ is the potential energy, so the acceptance rate\index{acceptance rate} is $\min[1,\exp(-\delta E)]$, where $\delta E$ is the change in total energy over time $T$.
Frictionless motion conserves the total energy so that under the exact dynamics the acceptance probability\index{acceptance probability} is $1$. Numerical integration approximates the dynamics, using an integration step size\index{step size}, $\epsilon$. A smaller $\epsilon$ gives a more accurate numerical scheme and a higher average acceptance rate\index{acceptance rate}, but for a given $T$ it also requires more numerical steps and, hence, a larger computational cost.

We now provide a more rigorous description of a standard version of the algorithm, including an explanation of the acceptance probability\index{acceptance probability} that leads to a stationary distribution\index{stationary distribution} of $\pi$. 

The first component of the algorithm is a positive-definite mass matrix\index{mass matrix}, $\bM$, the inverse of which plays a similar role to the preconditioning\index{preconditioned} matrix $\bV$ used in the RWM and MALA. The \emph{mass} of an object, as used in the intuitive explanation above, is the ratio between the magnitude of a force that is applied and the magnitude of the acceleration that results and is a scalar. For additional generality, in HMC, we imagine that this ratio can be different along each of a set of $d$  orthogonal principal axes leading to a mass matrix\index{mass matrix} rather than a scalar.

The core component of the HMC algorithm is the numerical integration scheme, which repeatedly uses the leapfrog\index{leapfrog dynamics} step to deterministically evolve the position and momentum from a time $t$ to a time $t+\epsilon$: $(\btheta_{t+\epsilon},\bp_{t+\epsilon})=\Leap(\btheta_t,\bp_t;\epsilon,\bM)$, where
\[
\bp_*=\bp_t+\frac{1}{2}\epsilon\nabla \log \pi(\btheta_t),
~~~
\btheta_{t+\epsilon}=\btheta_t+\epsilon \bM^{-1}\bp_*,
~~~
\bp_{t+\epsilon}=\bp_*+\frac{1}{2}\epsilon\nabla \log \pi(\btheta_{t+\epsilon}).
\]
HMC uses the leapfrog\index{leapfrog dynamics} scheme rather than, for example, the Euler or Runge--Kutta schemes because the leapfrog scheme possesses two key properties that will be discussed shortly.

HMC repeats the leapfrog\index{leapfrog dynamics} step $L$ times, where $L\epsilon=T$, to obtain the proposal, $\btheta'=\btheta_T$ as depicted in Figure \ref{fig.HMCleaps}. The proposed momentum is, in fact, $\bp'=-\bp_T$ and we denote the transformation: $(\btheta,\bp)\to(\btheta',\bp')$ by $\Leap^L_-$.

\begin{figure}[ht]
\centering
  \includegraphics[width=0.7\textwidth]{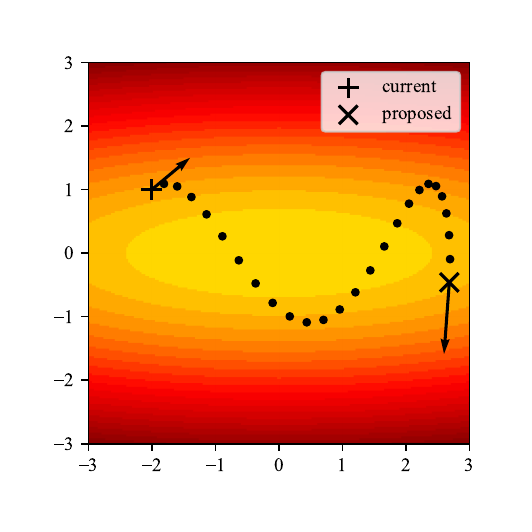}
 \caption{Initial point $\btheta=\btheta_0=+$ and final point $\btheta'=\btheta_{2.5}=\times$ after $L=25$ leapfrog\index{leapfrog dynamics} steps using a time interval of $\epsilon=0.1$ and a mass matrix\index{mass matrix} of $\bM=\bI_2$. The momentum at the current and proposed point (before the moment flip) is proportional to the size of the corresponding arrow and intermediate points appear as small solid circles. \label{fig.HMCleaps}}
\end{figure}

Standard HMC proposes  $\bp$ from a $\Normal(\bzero,\bM)$ distribution, and can be viewed as targeting a joint density of $(\btheta,\bp)$ that is the product of the density for $\bp$ and the posterior:
\[
\pitilde(\btheta,\bp)=
\pi(\btheta)(2\pi)^{-d/2}\mathsf{det}(\bM)^{-1/2}\exp\big[-\frac{1}{2}\bp^\top \bM^{-1} \bp\big]
\]
At the end of each iteration, we discard $\bp$, and the marginal for $\btheta$ is $\pi$, as required. 
Algorithm \ref{alg:HMC} details the standard version of the Hamiltonian Monte Carlo\index{Hamiltonian Monte Carlo} algorithm.

\begin{algorithm}
    \caption{Hamiltonian Monte Carlo}
    \KwIn{Density $\pi(\btheta)$, initial value $\btheta_0$, mass matrix\index{mass matrix} $\bM$, time interval $T$, number of leapfrog steps $L$.}
    $\epsilon\gets T/L$.\\
    \For {$k\in 0,\dots,n-1$} {
    Sample $\bp\sim \Normal(\bzero,\bM)$.\\
    $(\btheta',\bp')\gets\Leap^L_-(\btheta_k,\bp)$.\\
    Calculate the acceptance probability\index{acceptance probability}:
    \[\alpha(\btheta_k,\bp;\btheta',\bp')
    :=\min\left(1,\frac{\pitilde(\btheta',\bp')}{\pitilde(\btheta_k,\bp)}\right).\]\\
    With a probability of $\alpha(\btheta_k,\bp;\btheta',\bp')$ accept the proposal, $\btheta_{k+1}\gets\btheta'$; otherwise reject it, $\btheta_{k+1}\gets \btheta_k$.    }
    \label{alg:HMC}
\end{algorithm}

HMC combines a momentum refresh with a Metropolis--Hastings\index{Metropolis--Hastings} step which uses a deterministic proposal\index{proposal} $(\btheta,\bp)\gets (\btheta',\bp')\equiv \Leap^L_-(\btheta,\bp;\epsilon;\bM)$; finally the new momentum is discarded. The momentum refreshment preserves $\pitilde$ as it samples directly from the correct conditional. We now explain why the accept-reject step with a deterministic proposal\index{proposal} also preserves $\pitilde$. 

The leapfrog\index{leapfrog dynamics} step possesses two key properties:
\begin{description}
    \item[Property 1] $\Leap$ has a Jacobian\index{Jacobian} of $1$.
    \item[Property 2] If $\Leap(\btheta,\bp)=(\btheta',\bp')$ then $\Leap(\btheta',-\bp')=(\btheta,-\bp)$.
\end{description}
Property 1 arises because the Leapfrog is a composition of three transformations each of which has a Jacobian\index{Jacobian} of $1$. Property 2 is straightforward to verify, and emulates frictionless dynamics in that if after moving for some time the momentum of an object is suddenly reversed, after the same amount of time again the object will end up back where it started, moving with the same speed as when it started but in the opposite direction. The composition of $L$ leapfrog steps\index{leapfrog dynamics} possesses the same property: in Figure \ref{fig.HMCleaps}, starting at the $\times$ but with a momentum given by the reverse of the corresponding arrow, and proceeding for $25$ leapfrog\index{leapfrog dynamics} steps leads to the $+$ position, but with a momentum of exactly the reverse of the true initial momentum that the corresponding arrow represents.

 Naturally, $\Leap^L_-$, the composition of $L$ leapfrog\index{leapfrog dynamics} steps, combined with a momentum flip, also has a Jacobian\index{Jacobian} of $1$. Moreover, $\Leap^L_-$ is self-inverse: $\Leap^L_-(\Leap^L_-(\btheta,\bp))=(\btheta,\bp)$; equivalently, from $(\btheta',\bp')$ we would propose $(\btheta,\bp)$.

As in Section \ref{sec.MHgen}, the detailed balance\index{detailed balance} condition is trivial under rejection, so we focus on acceptances.  Let $\mathsf{A}$ be the event of an acceptance and write $(\btheta_k,\bp_k)$ and $(\btheta_{k+1},\bp_{k+1})$ for the position and momentum before and after the acceptance step (and before the momentum is discarded). Then, for $\cB \in \mathbb{R}^{2d}$ and $\cC\in \mathbb{R}^{2d}$, and writing  $\Leap^L_-(\cA)$ for the image of a set $\cA$ under $\Leap^L_-$,
\begin{eqnarray*}
& & \hspace{-50pt} \Prob{\mathsf{A}, (\btheta_k,\bp_k)\in \cB, (\btheta_{k+1},\bp_{k+1})\in \cC} \\
&=&
\iint_{\cB\cap \Leap^L_-(\cC)}
\pitilde(\btheta,\bp)\alpha(\btheta,\bp;\btheta',\bp') ~\md(\btheta,\bp)\\
&=&
\iint_{\Leap^L_-(\cB)\cap\cC}
\pitilde(\btheta',\bp')\alpha(\btheta',\bp';\btheta,\bp)) ~\md(\btheta',\bp')\\
&=&
\Prob{\mathsf{A}, (\btheta_k,\bp_k)\in \cC, (\btheta_{k+1},\bp_{k+1})\in \cB},
\end{eqnarray*}
where on both intermediate lines we have used that $\Leap^L_-$ is self-inverse, and the penultimate line uses that the Jacobian\index{Jacobian} of $(\btheta,\bp)\to (\btheta',\bp')$ is $1$.

\subsubsection{Scaling of HMC with Dimension}\index{scaling}

Given a particular integration time, $T$, the smaller the step size\index{step size}, $\epsilon$, the more accurate the leapfrog\index{leapfrog dynamics} scheme, and the closer the acceptance rate\index{acceptance rate} is to $1$. At the same time, the computational cost is proportional to the number of leapfrog\index{leapfrog dynamics} steps, $L=\lceil T/\epsilon\rceil$. A large $\epsilon$ leads to many rejections and a small $\epsilon$ leads to a high computational cost,  suggesting that there is an optimal choice of $\epsilon$ between these two extremes.

In this analysis, we consider a general product target, $\pi(\btheta)=\prod_{i=1}^d f(\theta_i)$, and assume an identity mass matrix\index{mass matrix}, $\bM=\bI_d$.
In this case, the evolution of each $(\theta_i,p_i)$ by $\Leap^L_-$ does not depend on any of the other components.
The acceptance ratio\index{acceptance ratio} for HMC is
\[
\rho(\btheta,\bp;\btheta',\bp')
=
\frac{\pitilde(\btheta',\bp')}{\pitilde(\btheta,\bp)}
=
\prod_{i=1}^d \rho_1^{(i)}
\]
where $(\btheta',\bp')=\Leap^L_-(\btheta,\bp)$, a deterministic function, and \[
\rho_1^{(i)}=\frac{f(\theta_i') \Normal(p_i';0,1)}{f(\theta_i)\Normal(p_i;0,1)}.
\] 
At stationarity, after cancellations, and using the unit Jacobian\index{Jacobian} of $\Leap^L_-$, we have
\begin{align}
\nonumber
\Expects{\theta_i \sim f,p_i\sim \Normal(0,1)}{\rho_1^{(i)}}
&=
\int f(\theta_i')\Normal(p_i';0,1)~\md \theta_i\md p_i\\
\label{eqn.EoneHMC}
&=
\int f(\theta_i')\Normal(p_i';0,1)~\md \theta_i' \md p_i'
=
1.
\end{align}
Moreover, the $\rho_1^{(i)}$ are i.i.d., so, by the central limit theorem\index{central limit theorem}, approximately, 
\[
\log \rho = \sum_{i=1}^d \log \rho_i\sim \Normal(d \Expect{\log \rho_1},d\Var{\log \rho_1}),
\]
from which we see that $\rho$ has approximately a lognormal distribution. From \eqref{eqn.EoneHMC}, and the component-wise\index{component-wise} independence, $\Expect{\rho}=1$, so $\Expect{\log \rho}=-\frac{1}{2}\Var{\log \rho}$ and, hence,  $\Expect{\log \rho_1}\approx -\frac{1}{2}\Var{\log \rho_1}$. This gives the same distribution for the log-acceptance ratio\index{acceptance ratio} as we found for the RWM in \eqref{eqn.RWMlognormal} with the same scaling of the expectation and variance with dimension if $\lambda$ (for the RWM) or $\epsilon$ (for HMC) is kept fixed. For the RWM, this necessitated taking $\lambda^2\propto 1/d$; however, for Hamiltonian dynamics approximated by the leapfrog\index{leapfrog dynamics} integrator with step size\index{step size} $\epsilon$ over a time period $T$, the error in the total energy is $O(\epsilon^2)$; \emph{i.e.}, $\Expect{|\log \rho_1|}=O(\epsilon^2)$.  Thus 
\begin{eqnarray*}
& & \hspace{-40pt} \Var{\log \rho_1}+\frac{1}{4}\Var{\log \rho_1}^2 \\
&=&
\Var{\log \rho_1}+\Expect{\log \rho_1}^2
=
\Expect{(\log \rho_1)^2}=O(\epsilon^4).
\end{eqnarray*}
Setting $\epsilon=O(d^{-1/4})$ gives $\Var{\log \rho_1}=O(1/d)$, so both $\Expect{\log \rho}$ and $\Var{\log \rho}$ are $O(1)$, as required for the acceptance ratio\index{acceptance ratio} to be well-behaved.
Taking $\epsilon\propto d^{-1/4}$, and $T$ fixed as dimension increases, implies that for a given amount of movement in each component, the number of leapfrog\index{leapfrog dynamics} steps, and hence the computational cost, increases in proportion to $d^{1/4}$. Contrasting this with a cost of $d^{1/3}$ for MALA and $d$ for the RWM shows why HMC is often the algorithm of choice for high-dimensional targets.

\subsubsection{Tuning HMC}
After a more rigorous scaling\index{scaling} analysis than our heuristic explanation, \cite{Beskos:2013} concludes that given $T$, in the high-dimensional limit, $\epsilon$ should be chosen so that the acceptance rate\index{acceptance rate} is around $65\%$; this limit is approached slowly, however, and in practice, it is often found that a higher acceptance rate is optimal. 

The main difficulty with tuning HMC is in choosing the integration time, $T$. For example, for a $\Normal(0,\sigma^2)$ target, $\pi$, using a momentum of $p\sim \Normal(0,1)$, it is straightforward to show that if $\theta_0\sim \pi$ then under the true Hamiltonian dynamics, $\mathsf{Cor}[\theta_0,\theta_T]=\cos(T/\sigma)$. If the target is a product of Gaussians, each with a different length scale, then the auto-correlations\index{auto-correlation} between the current values and the proposals for each coordinate have different periods. The periodicity means that increasing $T$ does not monotonically decrease the auto-correlation\index{auto-correlation} and the different periods\index{period} mean that the minimum correlation over all components, which upper bounds the minimum of the lag-1 auto-correlations\index{auto-correlation}, is an erratic function of $T$. Hence, the overall efficiency can, and often does, behave erratically as $T$ is varied. This is illustrated in Figure \ref{fig.HMCcorr} where the optimal choice of $T$ is around $8$--$9$, but slight deviations from this range lead to substantial decreases in efficiency.

\begin{figure}[!t]
\centering
  \includegraphics[width=\textwidth]{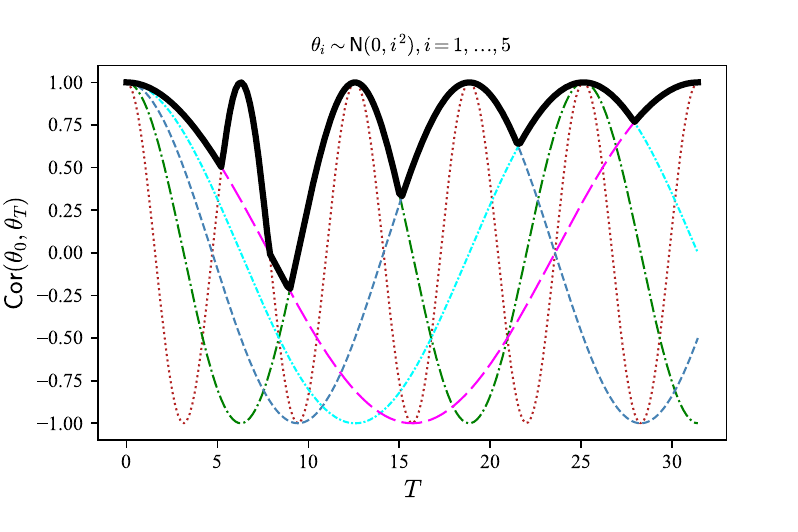}
 \caption{$\mathsf{Cor}(\theta_0,\theta_T)$ against $T$ for all $5$ components of $\btheta$ when $\pi(\btheta)\propto \exp[-\frac{1}{2}\sum_{i=1}^5\theta_i^2/i^2]$, $\btheta_0\sim \pi$ and $\bp\sim \Normal(\bzero,\bI_5)$ (non-solid lines). The thick solid line is the pointwise maximum over components. \label{fig.HMCcorr}
 }
\end{figure}


\section{Chapter Notes}
\label{sec.Ch2.ChapNotes}

This chapter has only touched the surface on many established aspects of MCMC and variations on the Metropolis--Hastings algorithm of \cite{Hastings:1970}.\index{Metropolis--Hastings} The first MCMC algorithm was the RWM-within-Gibbs algorithm of \cite{metropolis1953equation}, whilst the MALA was suggested and studied in \cite{besag1994} and  \cite{Roberts:1996}, and HMC was proposed in \cite{DuaneHMC1987}.  There are many texts and review articles devoted to MCMC, including \cite{Geyer1992}, \cite{robert1999monte}, \cite{gamerman2006markov}, and \cite{brooks2011handbook}.\index{Markov chain Monte Carlo}

The first high-dimensional RWM scaling\index{scaling} result showing a limiting diffusion appears in \cite{RobertsGelmanGilks1997} and applies to a product target, $\pi(\btheta)\propto \prod_{i=1}^d f(\theta_i)$; this is extended to MALA in \cite{Roberts:1998} and, for both the RWM and MALA, to targets of the form $\prod_{i=1}^d C_i f(C_i\theta_i)$ in \cite{Roberts:2001}. \cite{Sherlock/Roberts:2009} tackle the RWM on spherical and elliptical targets, showing that in some situations, the optimal acceptance rate\index{acceptance rate} can be less than 0.234, \cite{Sherlock:2015} extends the analysis for product targets to the pseudo-marginal RWM. Of the many other scaling\index{scaling} results for these two algorithms, we highlight the following: \cite{ChrRobRos2005} examines the transient phases of the algorithms, \cite{BeskosRobertsStuart:2009} considers a change of measure from a product law and, finally,  \cite{Kamatani2020} considers the RWM on spherically symmetric scale-mixtures of Gaussians and shows that when the tails are heavier than exponential, although the optimal scaling\index{scaling} is still $\ell/d$, the norm of the process, $\|\btheta\|$, mixes in a time of $O(d^2)$ rather than $O(d)$, suggesting that the RWM may be too costly on some, more realistic heavy-tailed targets.

Example  \ref{example.BadMALA} illustrated the poor behaviour of MALA in the tails of targets with tails that are lighter than Gaussian. 
A simple solution is to truncate the gradient\index{gradient} term \cite[]{Roberts:1996}. \cite{LivZan2022} provides an alternative use of gradients within the proposal that leads to the same limiting behaviour as MALA, but is automatically robust to issues with light-tails.

A single iteration of HMC approximates  Hamiltonian dynamics over a finite time $T$, whatever the dimension. Thus, if, as in the earlier scaling\index{scaling} analysis, $T$ is kept fixed, there can be no limiting diffusion for a product target. A similar scaling\index{scaling} analysis to ours appears in \cite{Neal:2011}, which itself is based on \cite{Creutz1988}; a more rigorous analysis is given by \cite{Beskos:2013}. 

Recent works have mitigated against the erratic dependence of the HMC efficiency on the integration time, $T$. Techniques include introducing randomness into the length of the path \cite[]{Neal:2011,BouSan2017,HofRadSou2021}, randomising the choice of proposal point from those along the path \cite[][the former also automatically choosing $T$ at each iteration and the latter adjusting $T$ according to a natural length scale of the target]{hoffman2014no,SheUrbLud2023} or by jittering the momentum after each leapfrog\index{leapfrog dynamics} step \cite[]{DurVog2023}.

Further variations on the HMC algorithm include the truly non-reversible \cite{horowitz1991generalized}, which is discussed in more detail in Section \ref{sec:ch5_lifting}, and \cite{SohMud2014}, which tries to mitigate one of the key issues with \cite{horowitz1991generalized}.

A separate strand of methodological developments for reversible MCMC starts with position-dependent preconditioning\index{preconditioned} of MALA and extends to Riemann manifold Hamiltonian Monte Carlo\index{Hamiltonian Monte Carlo} \cite[]{girolami2011riemann} and Riemann manifold MALA \cite[]{Xifaraetal2014} which itself feeds into the Stochastic Gradient Riemannian Langevin Dynamics described in Section \ref{ch3:sgmcmc-general}.


\chapter{Stochastic Gradient MCMC\index{Markov chain Monte Carlo} Algorithms}
\label{chap:sgld}

Chapter \ref{chap:intro-mcmc} introduced MCMC algorithms\index{Markov chain Monte Carlo} as a simulation-based approach to approximate distributions of interest. A drawback of the algorithms introduced in Chapter \ref{chap:intro-mcmc} is that their computational time scales poorly with large datasets. In this chapter, we will explore a class of algorithms that can be viewed as approximations of the algorithms introduced in Chapter \ref{chap:intro-mcmc}. We introduce the stochastic gradient Langevin algorithm and extensions of this algorithm, which are popular Bayesian inference methods in the field of machine learning.  
Compared to traditional MCMC algorithms, we will now replace the gradient\index{gradient} of the log density of the target distribution with a stochastic approximation. This stochastic approximation is generated using a subsample of the full dataset to produce an approximate MCMC\index{Markov chain Monte Carlo} algorithm. This class of stochastic gradient MCMC\index{stochastic gradient Markov chain Monte Carlo} algorithms is computationally faster than standard MCMC\index{Markov chain Monte Carlo} algorithms but at the expense of introducing a small asymptotic bias that can be corrected post-hoc. 
Through this chapter, we will discuss the motivation behind these algorithms, and their extensions, and provide empirical comparisons to traditional  MCMC\index{Markov chain Monte Carlo} algorithms.

\section{The Unadjusted Langevin Algorithm}\index{unadjusted Langevin algorithm}
\label{sec:ch3-langevin-diffusion}

Recall that we aim to sample from a posterior distribution\index{posterior distribution} with density $\pi(\st)$, where for this chapter, $\st$ is a $d-$dimensional vector in $\mathbb{R}^d$. It is assumed for the methods we discuss in this chapter that $\log\pi(\st)$ is continuous and differentiable almost everywhere. Simulating a stochastic process that has $\pi$ as its stationary distribution\index{stationary distribution} is a well-established method for generating samples approximately from $\pi(\st)$. By sampling from such a process for an extended period and discarding the initial burn-in\index{burn-in} samples, we obtain a set of samples that approximate $\pi(\st)$. The accuracy of the approximation depends on how quickly the stochastic process converges to its stationary distribution\index{stationary distribution} from the initial point, relative to the length of the burn-in\index{burn-in} period, as well as on the time for the chain to mix within the stationary distribution. The MCMC method (see Chapter \ref{chap:intro-mcmc}) is the most widely used technique for sampling in this way. 

With $b=1$, the \textit{overdamped Langevin diffusion}\index{Langevin diffusion!overdamped} first introduced in \eqref{eqn.Langevin} of Section \ref{sec.IntroLangDiffusions} is
\begin{equation} \label{eq:LangevinSDE}
\mbox{d}\st_t=\frac{1}{2} \nabla \log\pi(\st_t) \mbox{d}t + \mbox{d}W_t,
\end{equation}
where $\frac{1}{2}\nabla \log\pi(\st_t)$ is the drift term and $W_t$ denotes $d$-dimensional Brownian motion. In this chapter, we sometimes refer to the solution to this stochastic differential equation\index{stochastic differential equation} (SDE)\index{SDE|see{stochastic differential equation}} simply as \emph{the Langevin diffusion}.
Under mild regularity conditions, the Langevin diffusion has $\pi$ as its stationary distribution\index{stationary distribution}. As detailed in Section \ref{sec.SDEs} and, in particular \eqref{eqn.EulerMar}, this equation can be interpreted as defining the dynamics of a Markov process over infinitesimally small time intervals. 
That is, for a small time-interval $\delta>0$, the Langevin diffusion has a discrete-time analogue given by the Euler--Maruyama\index{Euler-Maruyama} approximation,
\begin{equation} \label{eq:Euler}
    \st_{t+\delta} = \st_t + \frac{\delta}{2} \nabla \log\pi(\st_t) + \sqrt{\delta} \bZ, \quad t \geq 0
\end{equation}
where $\bZ$ is a vector of $d$ independent standard Gaussian random variables. 
This discrete-time update equation is commonly known as the \textit{unadjusted Langevin algorithm}\index{unadjusted Langevin algorithm} (ULA) or the \textit{Langevin Monte Carlo} algorithm. The discrete-time sequence $\{\st_t\}_{t \geq 0}$ generated by \eqref{eq:Euler} differs from the sequence produced by the process in \eqref{eq:LangevinSDE}. 
The update equation given in  (\ref{eq:Euler}) provides a straightforward and practically implementable method for generating approximate samples from the overdamped Langevin diffusion\index{Langevin diffusion!overdamped}. 
To generate samples over a duration $T=n\delta$, where $n$ is an integer, we begin by setting the initial state of the process to $\st_0$, and then repeatedly simulate the process using  (\ref{eq:Euler}) to obtain values at times $\delta,2\delta,\ldots,n\delta$. We will use the notation $\st_k$ to refer to the state of the process at time $k\delta$.
 As with the MCMC algorithms discussed in Chapter \ref{chap:intro-mcmc}, an estimate of any expectation is obtained via a Monte Carlo average: $\Expects{\pi}{h(\btheta)}\approx \frac{1}{n}\sum_{k=1}^n h(\btheta_k)$.

To improve the accuracy of the Euler--Maruyama discretisation \eqref{eq:Euler} when sampling from the Langevin diffusion at a fixed time $T$, we can decrease $\delta$. As $\delta$ becomes smaller, the discretisation error decreases, and the approximation becomes more accurate. In theory, we can achieve any desired degree of accuracy in approximating the SDE\index{stochastic differential equation}  \eqref{eq:LangevinSDE} by selecting $\delta$ small enough. 
However, for a fixed $T$, the computational cost increases in proportion to $1/\delta$. Alternatively, given a fixed computational budget, $T$ decreases in proportion to $\delta$. The longer $T$ is, the more information about the diffusion's stationary distribution\index{stationary distribution} we collect, and hence, for a fixed computational budget, the variance of any estimate from the samples increases as the bias decreases. In practice, therefore, the choice of $\delta$ requires a compromise between the bias and the variance of our estimators.
 
\section{Approximate vs. Exact MCMC}
\label{sec:ch3-appr-mcmc-using}

The overdamped Langevin diffusion\index{Langevin diffusion!overdamped} has $\pi$ as its stationary distribution\index{stationary distribution}, and therefore it is natural to consider this stochastic process as the basis for an MCMC\index{Markov chain Monte Carlo} algorithm. In fact, if it were possible to simulate exactly the dynamics of the Langevin diffusion, then we could use the resulting realisations at a set of discrete time points as our MCMC\index{Markov chain Monte Carlo} output. However, for general $\pi(\st)$, the Langevin dynamics\index{Langevin dynamics} are intractable, and therefore it is necessary to resort to using samples generated by the Euler--Maruyama approximation \eqref{eq:Euler}.\index{Euler-Maruyama}

This is most commonly seen with the Metropolis-adjusted Langevin Algorithm (MALA) (see Section \ref{sec:ch2-MALA}).\index{Metropolis--adjusted Langevin algorithm}\index{MALA|see{Metropolis--adjusted Langevin algorithm}} This algorithm uses the Euler--Maruyama approximation \eqref{eq:Euler}\index{Euler-Maruyama} over an appropriately chosen time-interval, $\delta$, to define the proposal distribution of a standard Metropolis--Hastings\index{Metropolis--Hastings} sampler (see Algorithm \ref{alg:MH}). 
Simulated values are then either accepted or rejected based on the Metropolis--Hastings acceptance probability \eqref{eqn.MHaccProb}.\index{Metropolis--Hastings} 
Such an algorithm has good theoretical properties, and in particular, can scale better to high-dimensional problems than the simpler random walk Metropolis\index{random walk Metropolis} algorithm. See Section \ref{sec:ch2-MALA} for a more detailed description of the MALA\index{Metropolis--adjusted Langevin algorithm} algorithm and its dimensional scaling.

A simpler algorithm is the just-described \emph{unadjusted Langevin algorithm}\index{unadjusted Langevin algorithm} \eqref{eq:Euler}, which simulates from the Euler--Maruyama\index{Euler-Maruyama} approximation of the Langevin diffusion but does not use a Metropolis--Hastings\index{Metropolis--Hastings} accept-reject step, and so the stationary distribution\index{stationary distribution} of the resulting Markov chain is not $\pi$. Hence, even once the Markov chain has essentially converged, the Monte Carlo samples are from an approximation to $\pi$ rather than from $\pi$ itself. Because of this, estimators of expectations are typically biased, even as the number of samples, $n$, grows to infinity. Computationally, such an algorithm is quicker per iteration, but often this saving is small as the cost of calculating $\nabla \log\pi(\st)$, which is required for one step of the ULA\index{ULA|see{unadjusted Langevin algorithm}} algorithm, typically scales at least linearly with the dataset size.
If the MALA\index{Metropolis--adjusted Langevin algorithm} algorithm is optimally tuned, then approximately $40\%$ of the samples will be rejected, which leads to wasted computation compared to ULA\index{unadjusted Langevin algorithm} where all samples, albeit biased, are accepted. However, this is counteracted by the larger step sizes that are possible with MALA.

\subsubsection{Example: Sampling from a Gaussian Posterior Distribution}

To illustrate the computational and statistical accuracy trade-offs between the ULA\index{unadjusted Langevin algorithm} and MALA\index{Metropolis--adjusted Langevin algorithm} schemes, we consider a simple bivariate Gaussian posterior distribution\index{posterior distribution}, which we shall use as a running example throughout this chapter. We assume that data arise as realisations from a Gaussian location model with mean parameter $\st$ assumed to be unknown and the variance $\bV$ is known. 
We select a conjugate Gaussian prior for the unknown $\st$ which leads to the generative model \begin{equation}
\label{eq:running-gaussian}
    \by_j|\st \sim \mathsf{N}\left(\st,\bV
\right), \quad\quad \st \sim \mathsf{N}(\mathbf{0},\bI_2), \quad \mbox{for} \ j=1,\ldots,N,
\end{equation} 
where  we set $\bV = \begin{pmatrix}
1 & 0 \\
0 & 10
\end{pmatrix}$ and $\bI_2$ is a 2-dimensional identity matrix. For this simple model, it is possible to derive a tractable posterior distribution\index{posterior distribution}  $\st | \by \sim \mathsf{N}(\bmu_N,\bSigma_N),$ where $\bSigma_N = (N\bV^{-1} + \bI_2)^{-1}$ and $\bmu_N = \bSigma_N(\bV^{-1}\sum_{j=1}^N \by_j)$.
We can use both ULA\index{unadjusted Langevin algorithm} and MALA\index{Metropolis--adjusted Langevin algorithm} schemes to sample from the posterior distribution\index{posterior distribution} and compare the Monte Carlo accuracy of both algorithms against the known ground-truth posterior distribution\index{posterior distribution}. We measure the distributional accuracy between the true posterior $\pi$ and a Monte Carlo approximation $\tilde{\pi}$ using the Wasserstein-2 distance\index{Wasserstein metric}, 
\begin{equation}
    \label{eq:w2-distance}
    \mathsf{d}_{W_2}^2(\pi,\tilde{\pi}) = \inf_{\zeta \in \Gamma(\pi,\tilde{\pi})} \int_{\mathbb{R}^d \times \mathbb{R}^d}   \| \st-\st'\|^2 \mathrm{d}\zeta(\st,\st'),
\end{equation}
where the $\inf$ is taken with respect to all joint distributions $\zeta$ which have $\pi$ and $\tilde{\pi}$ as their marginal distributions. In the setting where both $\pi$ and $\tilde{\pi}$ are Gaussian, there is a tractable closed-form expression for the Wasserstein-2 distance\index{Wasserstein metric},
\begin{equation*}
    \mathsf{d}_{W_2}^2(\mathsf{N}(\bmu_a,\bSigma_a),\mathsf{N}(\bmu_b,\bSigma_b)) = \|\bmu_a-\bmu_b\|_2^2 + \mbox{trace}(\bSigma_a+\bSigma_b-2(\bSigma_a^{1/2}\bSigma_b\bSigma_a^{1/2})^{1/2}). 
\end{equation*}

In Figure \ref{fig:ula_vs_mala}, we calculated the approximate Wasserstein-2 distance\index{Wasserstein metric} between the true posterior and a moment-matched Gaussian approximation to the Monte Carlo samples generated by the ULA\index{unadjusted Langevin algorithm}/MALA\index{Metropolis--adjusted Langevin algorithm} algorithms. We simulated $N=1000$ (left panel) and $N=10000$ (right panel) data points from the model \eqref{eq:running-gaussian} and ran ULA\index{unadjusted Langevin algorithm}/MALA\index{Metropolis--adjusted Langevin algorithm} for $n=1000$ iterations. For each $N$, MALA and ULA used the same step size, $\delta=1/N$. Since the true posterior is Gaussian, we expect the moment-matched Gaussian approximation for the MALA\index{Metropolis--adjusted Langevin algorithm} sampler to get more and more accurate as the number of iterations increases. Since the target is Gaussian, the ULA\index{unadjusted Langevin algorithm} update is a linear Gaussian autoregression, so the stationary distribution\index{stationary distribution} for ULA\index{unadjusted Langevin algorithm} is also  Gaussian, so the moment-matched Gaussian approximation to this will also get more and more accurate as the number of iterations increases.

Comparing ULA\index{unadjusted Langevin algorithm} and MALA\index{Metropolis--adjusted Langevin algorithm} for a given computational time, the accuracy of ULA\index{unadjusted Langevin algorithm} is initially comparable to that of MALA, if not slightly better. However, with a larger computational budget, i.e., more Monte Carlo iterations, ULA\index{unadjusted Langevin algorithm} is less accurate due to the asymptotic bias from discretising the Langevin diffusion. When taking into account the reduced computational cost of ULA\index{unadjusted Langevin algorithm}, this means that ULA\index{unadjusted Langevin algorithm} is better for small computational budgets, whereas for moderate to large computational budgets, MALA\index{Metropolis--adjusted Langevin algorithm} is better. Note that the computational budget required for MALA\index{Metropolis--adjusted Langevin algorithm} to display improved statistical efficiency over ULA\index{unadjusted Langevin algorithm} is dependent on the dataset size. This is highlighted in Figure \ref{fig:ula_vs_mala}, where $N=1000$ in the left panel and $N=10000$ in the right panel.

The reason that ULA\index{unadjusted Langevin algorithm} is competitive with MALA\index{Metropolis--adjusted Langevin algorithm} only for very small computational budgets is that the computational gain per iteration is only roughly twofold (i.e. not calculating the accept-reject ratio roughly halves the cost, as the gradient still needs to be calculated), and this is only a small gain relative to the bias that is introduced. If, on the other hand, there was a way of implementing ULA, or something like ULA, which was $O(N)$ faster, then the computational benefit compared to MALA\index{Metropolis--adjusted Langevin algorithm} would be more significant. To achieve such a speed-up, this would require an algorithm where the cost of calculating or approximating the gradient\index{gradient} is only $O(1)$ -- this is the key idea behind the \textit{stochastic gradient Langevin dynamics}\index{stochastic gradient Langevin dynamics} algorithm, which will be explored in detail in the remainder of this chapter.

\begin{figure}
    \centering
    \includegraphics[scale=0.60]{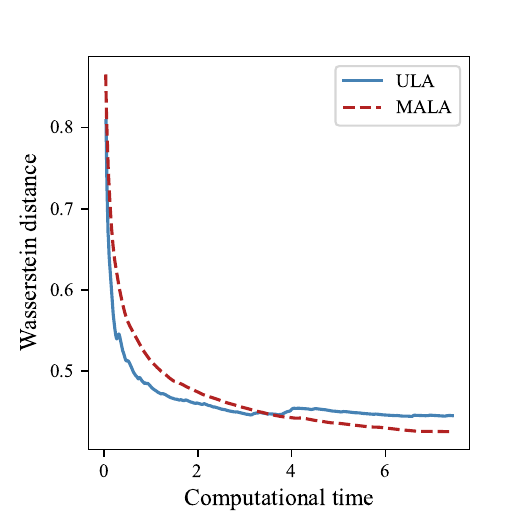}
    \includegraphics[scale=0.60]{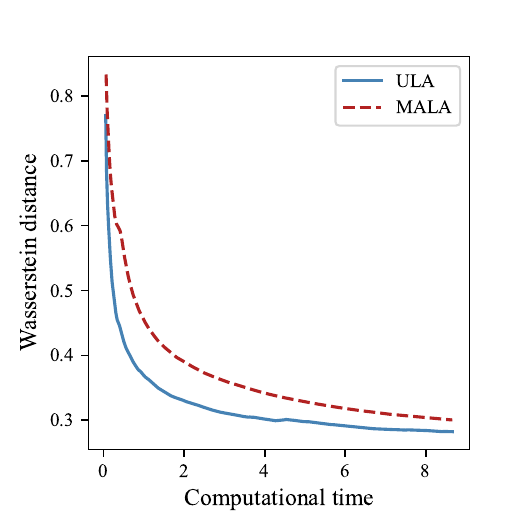}
    \caption{Wasserstein distance between the true $\pi$ and approximate posterior distribution against computational time (in seconds), where the approximate posterior $\tilde{\pi}$ is generated using the ULA and MALA\index{Metropolis--adjusted Langevin algorithm} schemes. Left panel $N=1000$ and right panel $N=10000$.}
    \label{fig:ula_vs_mala}
\end{figure}

\section{Stochastic Gradient Langevin Dynamics}
\label{sec:ch3-sgld-alg-sec}

We have seen how the ULA\index{unadjusted Langevin algorithm} algorithm is computationally faster than MALA, at the expense of introducing a bias that produces samples not from the desired invariant distribution $\pi$, but a distribution close to $\pi$.  Even without the Metropolis-Hastings acceptance probability, the ULA\index{unadjusted Langevin algorithm} algorithm still incurs a cost in calculating the gradient\index{gradient} of the log-posterior density and, under common assumptions, this computational cost scales linearly with the data set size. 

Recent interest in Bayesian analysis has considered the challenge of scalable inference in the presence of large datasets, where the log-posterior density is defined as a sum over data points. For instance, if we consider data $\by_1,\ldots,\by_N$ that are conditionally independent given $\st$, then $\pi(\st)\propto \pi_0(\st)\prod_{j=1}^N f(\by_j|\st)$. 
Here, $\pi_0(\st)$ is the prior density, and $f(\by_j|\st)$ is the likelihood for the $j$th observation. In this context, we define the log-posterior density, up to an additive constant that does not depend on $\st$, as

\begin{equation}
    \label{eq:log-post-dens}
\log\pi(\st)=\sum_{j=1}^N \log\pi_j(\st), \ \ \mbox{where} \ \ \log\pi_j(\st)=  \log f(\by_j|\st)+ \frac{1}{N} \log \pi_0(\st). 
\end{equation}
The computational bottleneck for ULA\index{unadjusted Langevin algorithm} is in calculating $\nabla \log\pi(\st)$, which can be expensive if we have a large dataset. For a dataset with $N$ independent observations, where the log-posterior density is a sum over $N$ independent terms \eqref{eq:log-post-dens}, the computational cost of evaluating the log-posterior density, or its gradient\index{gradient}, is $O(N)$ for each iteration of ULA. 

A solution to this problem is to use \emph{stochastic gradient Langevin dynamics}\index{stochastic gradient Langevin dynamics} \cite[SGLD,][]{Welling:2011}, which avoids calculating $\nabla \log\pi(\st)$, and instead uses an unbiased estimator of the gradient\index{gradient} at each iteration. It is trivial to obtain an unbiased estimate using a random subsample of the terms in the sum. The simplest implementation is to choose $m \ll N$ and estimate $\nabla \log\pi(\st)$ with
\begin{equation} \label{eq:U-hat-simple}
\nabla^{(m)} \log\pi(\st) = \frac{N}{m} \sum_{j \in \mathcal{S}_m} \nabla \log\pi_j(\st),
\end{equation}
where $\mathcal{S}_m$ is a random sample of size $m$ taken without replacement from $\{1,\ldots,N\}$. 
We call this the \emph{simple} estimator of the gradient\index{gradient} and use the superscript $(m)$ to denote the subsample size used in constructing our estimator. The resulting SGLD\index{SGLD|see{stochastic gradient Langevin dynamics}} algorithm is given in Algorithm \ref{alg:SGLD}. Essentially, the SGLD\index{stochastic gradient Langevin dynamics} algorithm is the same as ULA\index{unadjusted Langevin algorithm} \eqref{eq:Euler} and simulates an Euler--Maruyama discretisation of the Langevin diffusion. The only difference is that the true gradient\index{gradient} is replaced with the estimated gradient\index{gradient} \eqref{eq:U-hat-simple}. 

Using an estimator for the gradient\index{gradient} adds additional noise, with a variance of $O(\delta^2),$ and therefore the stochastic dynamics no longer follow the ULA\index{unadjusted Langevin algorithm} update equation \eqref{eq:Euler}; instead, 
 the SGLD\index{stochastic gradient Langevin dynamics} algorithm targets a  distribution\index{posterior distribution} that is typically more dispersed than $\pi$. However, for sufficiently small $\delta$, this additional variance becomes negligible compared with the injected Gaussian noise of \eqref{eq:Euler}, which has a variance of $\delta$. 
It is possible to generalise Algorithm \ref{alg:SGLD} to the setting of adaptive step sizes $\delta_k$ which are dependent on the iteration number $k$. This is not commonly used in practice, and for simplicity, we work with the constant step size\index{step size} version given in Algorithm \ref{alg:SGLD}.
A justification for using the SGLD\index{stochastic gradient Langevin dynamics} algorithm with a decaying step size\index{step size} can be given by an informal argument along the lines that if the step size\index{step size} is $\delta_k \downarrow 0$ for $k\rightarrow\infty$, with $\sum_k \delta_k=\infty$, then, under suitable regularity conditions, the discretisation error and the effect of the gradient\index{gradient} noise vanish in the limit, and weighted averages of the samples are consistent estimators of posterior expectations (see Section \ref{sec:ch3-sgld-theory}).

\begin{algorithm}
    \caption{Stochastic Gradient Langevin Dynamics (SGLD)}
    \KwIn{$\st_0$, $\delta$, $m$, $n$.}
    \For{$k \in 0, \dots, n-1$} {
        Draw a subset $\mathcal{S}_m \subset \{1,\ldots,N\}$ of size $m$  \\
        Estimate $\nabla^{(m)} \log\pi(\st_k)$ using \eqref{eq:U-hat-simple} \\
        Draw $\bZ_k \sim \mathsf{N}( 0, \delta \mathbf{I} )$ \\
        Update $\st_{k+1} \leftarrow \st_k + \frac{\delta}{2} \nabla^{(m)} \log\pi(\st_k)  + \bZ_k$
    }
    \label{alg:SGLD}
\end{algorithm}

The advantage of SGLD\index{stochastic gradient Langevin dynamics} is that, if the subsample size $m$ is much smaller than the full dataset size $N$, the per-iteration cost of the algorithm can be much smaller than that of either MALA\index{Metropolis--adjusted Langevin algorithm} or ULA. For large data applications, SGLD\index{stochastic gradient Langevin dynamics} has been empirically shown to perform better than standard MCMC\index{Markov chain Monte Carlo} when there is a fixed computational budget \citep{ahn2015large,li2016scalable}.  In challenging examples, performance has been based on measures of predictive accuracy on a held-out test dataset, rather than on how accurately the samples approximate the true posterior distribution\index{posterior distribution}. Furthermore, the conclusions from such studies will clearly depend on the computational budget, with larger budgets favouring exact methods such as MALA\index{Metropolis--adjusted Langevin algorithm};  see the theoretical results in Section \ref{sec:ch3-sgld-theory} and empirical results in Section \ref{sec:ch3-appr-mcmc-using}.

The SGLD\index{stochastic gradient Langevin dynamics} algorithm is closely related to \emph{stochastic gradient descent} (SGD) \citep{robbins1951stochastic}, an efficient algorithm for finding the local maxima of a function. The only difference is the inclusion of the additive Gaussian noise at each iteration of SGLD\index{stochastic gradient Langevin dynamics}. Without the noise, but with a suitably decreasing step size\index{step size}, SGD would converge to a local maximum of the density $\pi(\st)$. Again, SGLD\index{stochastic gradient Langevin dynamics} has been shown empirically to out-perform stochastic gradient descent  \citep{chen2014stochastic}, at least in terms of predictive accuracy -- intuitively this is because SGLD\index{stochastic gradient Langevin dynamics} will produce samples from the region around the estimate obtained by SGD, and thus can average over the uncertainty in the parameters. This strong link between SGLD\index{stochastic gradient Langevin dynamics} and SGD may also explain why the former performs well when compared to exact MCMC\index{Markov chain Monte Carlo} methods, at least in terms of predictive accuracy.

\subsection{Controlling Stochasticity in the Gradient Estimator}
\label{sec:ch3-sgld-grad-var}

The key ingredient of SGLD\index{stochastic gradient Langevin dynamics} is found in replacing the true gradient\index{gradient} with an unbiased estimator. 
The more accurate this estimator is, the lower the bias will be for the same computational cost, and thus it is natural to consider alternatives to the simple estimator (\ref{eq:U-hat-simple}).
One way of reducing the variance of a Monte Carlo estimator is to use control variates\index{control variate} (see Section \ref{sec.ControlVariatesBasics} for a detailed explanation), which in our setting involves choosing a set of simple functions $g_j$, $j=1,\ldots,N$, which we refer to as control variates\index{control variate}, and whose sum $\sum_{j=1}^N g_j(\st)$ can be evaluated for any $\st$. 

We can rewrite the full-data gradient\index{gradient} of the log-posterior density as
\[
 \sum_{j=1}^N \nabla \log\pi_j(\st) = \sum_{j=1}^N g_j(\st) + \sum_{j=1}^N \left(\nabla \log\pi_j(\st) - g_j(\st)  \right), 
\]
and from this, we can obtain an unbiased estimator 
\begin{equation}
\label{eq:cv-grad-est}
\sum_{j=1}^N g_j(\st)+\frac{N}{m} \sum_{j \in \mathcal{S}_m} (\nabla \log\pi_j(\st) - g_j(\st) ),    
\end{equation}
where again $\mathcal{S}_m$ is a random sample of size $m$  drawn from $\{1,\ldots,N\}$. The intuition behind this idea is that if each $g_j(\st)\approx \nabla \log\pi_j(\st)$, then this estimator can have a much smaller variance than the simple subsampled gradient estimator \eqref{eq:U-hat-simple}.

One approach to choosing the control variate\index{control variate} function $g_j(\st)$ that is often used in practice is to (i) use SGD to find an approximation, $\widehat{\st}$, to the mode of the distribution $\pi$; and (ii) set $g_j(\st)=\nabla \log\pi_j(\widehat{\st})$. This leads to the following control variate\index{control variate} estimator,
\begin{equation}
\label{eq:cv-estimator}    
 \nabla^{(m)} \log\pi_{\mathrm{cv}}(\st)= \sum_{j=1}^N \nabla \log\pi_j(\widehat{\st}) + \frac{N}{m} \sum_{j \in \mathcal{S}_m} \left( \nabla \log\pi_j(\st)-\nabla \log\pi_j(\widehat{\st}) \right).
\end{equation}
Implementing such an estimator involves an initial up-front cost for finding a suitable $\widehat{\st}$ and then calculating, storing, and summing $\nabla \log\pi_j(\widehat{\st})$ for $j=1,\ldots,N$. For these types of control variate\index{control variate} approaches, the main cost is from finding a suitable $\widehat{\st}$, although, once found, we can then use $\widehat{\st}$ as a starting value for the SGLD\index{stochastic gradient Langevin dynamics} algorithm, replacing $\st_0$ with $\widehat{\st}$ in Algorithm \ref{alg:SGLD}, which can significantly reduce the burn-in\index{burn-in} phase.

The advantage of using the control variate-based\index{control variate} estimator can be seen if we compare the variance bounds of this estimator against the simple estimator. If we assume that each $\log \pi_j(\st)$ is twice continuously differentiable on $\mathbb{R}^d$ and has Lipschitz-continuous gradients, then there exist positive constants $L_j>0$ for all  $j=1,\ldots,N$, such that
\begin{equation}
\label{eq:ass-lipschitz-j}
 \|\nabla \log\pi_j(\st)-\nabla \log\pi_j(\st')\| \leq L_j \|\st-\st'\|.
\end{equation}

Lemma \ref{cv-lemma} and several subsequent results provide bounds on the trace of the variance matrix of each estimator of $\nabla \log \pi(\st)$. Since variances are non-negative, this bounds the variances of each individual component. Also, since all eigenvalues of a variance matrix are non-negative, it bounds the largest of these; i.e., the variance of the worst-behaved linear combination of components. Finally, for any $d$-vector random variable $\bxi$ with $\Expect{\bxi}=0$,
\[
\mathsf{tr}(\Var{\bxi})
=
\Expect{\sum_{i=1}^d\xi_i^2}
=
\Expect{\|\bxi\|^2}
\geq
\Var{\|\bxi\|},
\]
so the bounds also apply to the variance of the Euclidean norm of the gradient. For any random vector $\bxi$ with $\Expect{\bxi}=0$, we refer to the important quantity of $\Expect{\|\bxi\|^2}=\mathsf{tr}(\Var{\bxi})$ as its \emph{pseudo-variance}.

\begin{lemma}
\label{cv-lemma}
Assume condition \eqref{eq:ass-lipschitz-j}, then there are a constant $C_2>0$ and, for each $\st$, a constant $C_1=C_1(\st)>0$ such that the pseudo variances of the simple gradient estimator \eqref{eq:U-hat-simple} and control variate-based\index{control variate} gradient estimator \eqref{eq:cv-estimator} have the following bounds:
\begin{align}
\label{eq:var-grad-bounds}
     \mathsf{tr}\left(\mathrm{Var} \left[ \nabla^{(m)} \log\pi(\st) \right]\right) &\leq C_1 \frac{N^2}{m},\\ \mathsf{tr}\left(\mathrm{Var} \left[ \nabla^{(m)} \log\pi_{\mathrm{cv}}(\st) \right]\right) 
     &\leq C_2 \|\st-\widehat{\st}\|^2 \frac{N^2}{m}. \label{eq:var-grad-bounds2}
 \end{align}
\end{lemma}

\begin{proof}
   We prove this result for the control variate-based\index{control variate} gradient estimators \eqref{eq:cv-estimator}; the result for the simple SGLD\index{stochastic gradient Langevin dynamics} estimator \eqref{eq:U-hat-simple} follows analogously.
  
   We first define $\bxi:= \nabla^{(m)} \log\pi_{\mathrm{cv}}(\st) - \nabla \log\pi(\st)$, so that $\bxi$ measures the noise in the gradient estimate and has mean zero. The trace of the variance in the noise is then given by 
\begin{align*}
    &\Expect{\norm{\bxi}^2} 
    = \Expect{ \norm{ \nabla^{(m)} \log\pi_{\mathrm{cv}}(\st) - \nabla \log\pi(\st) }^2} \\
    &= \Expect{\norm{  \frac{N}{m} \sum_{j \in \mathcal{S}_m} \left( \nabla \log\pi_j(\st)-\nabla \log\pi_j(\widehat{\st}) \right) - \left( \nabla \log\pi(\st) - \nabla \log\pi(\widehat{\st}) \right) }^2} \\
    &=
 \Expect{\norm{  \frac{N}{m} \sum_{j \in \mathcal{S}_m} \left\{\left( \nabla \log\pi_j(\st)-\nabla \log\pi_j(\widehat{\st}) \right) - \frac{1}{N}\left( \nabla \log\pi(\st) - \nabla \log\pi(\widehat{\st}) \right)\right\} }^2} \\   
    &\leq \frac{N^2}{m^2} \Expect {\sum_{j \in \mathcal{S}_m} \norm{ \left( \nabla \log\pi_j(\st)-\nabla \log\pi_j(\widehat{\st}) \right) - \frac{1}{N}\left( \nabla \log\pi(\st) - \nabla \log\pi(\widehat{\st}) \right) }^2},
\end{align*}
where the final line holds with equality in the setting of subsampling with replacement, because the $m$ summands are then independent with mean zero and so the cross terms vanish. For subsampling without replacement, the sampled indices will be negatively correlated and thus will lead to lower variance. For any random variable $\bX$, we have  $\Expect{\norm{\bX - \Expect{\bX}}^2} \leq \Expect{\norm{\bX}^2}$. Each index in $\mathcal{S}_m$ is uniformly distributed on $1,\ldots,N$, so the term subtracted inside the norm above is the expectation of $\nabla \log\pi_j(\st)-\nabla \log\pi_j(\widehat{\st})$. Using this result, and the Lipschitz assumption \eqref{eq:ass-lipschitz-j}, leads to  
\begin{align*}
    \Expect{\norm{\bxi}^2}
    &\leq \frac{N^2}{m^2} \Expect{ \sum_{j \in \mathcal{S}_m} \norm{\nabla \log\pi_j(\st)-\nabla \log\pi_j(\widehat{\st})}^2} \\
    &= \frac{N^2}{m} \Expect{\norm{\nabla \log\pi_j(\st)-\nabla \log\pi_j(\widehat{\st})}^2} \\
    &\leq \frac{N^2}{m} \frac{1}{N}\sum_{j=1}^N  \left( L_j\norm{\st-\widehat{\st}}\right)^2 = \frac{N^2}{m} \norm{\st-\widehat{\st}}^2 \frac{1}{N}\left(\sum_{j=1}^N L_j^2 \right), 
\end{align*}
where the second line follows from the exchangeability of the indices $j$ and, on that line, $j$ is a single index sampled uniformly at random from $1,\dots,N$. This proves the stated result and gives  $C_2=\frac{1}{N}\sum_{j=1}^N L_j^2$.
\end{proof}

If we make the further assumption that there exists a positive constant $L>0$ such that, for all $j=1,\ldots,N$,
 \begin{equation}
\label{eq:ass-lipschitz-all}
 \|\nabla \log\pi_j(\st)-\nabla \log\pi_j(\st')\| \leq L \|\st-\st'\|
\end{equation}
holds, then it is straightforward to show that we have the following Lipschitz bound on the gradient of the log-posterior,
 \begin{equation}
\label{eq:ass-lipschitz-post}
 \|\nabla \log\pi(\st)-\nabla \log\pi(\st')\| \leq LN \|\st-\st'\|,
\end{equation}
which now leads to an updated constant $C_2=L^2$ in \eqref{eq:var-grad-bounds2} of Lemma \ref{cv-lemma}.

 By comparing the upper bounds on the variance of the gradients in \eqref{eq:var-grad-bounds} and \eqref{eq:var-grad-bounds2}, we can see that when $\st$ is close to $\widehat{\st}$, we would expect the variance in the control variate\index{control variate} estimator to be smaller than for the simple estimator. Furthermore, in many big data settings where $N$ is large, we would expect by the Bernstein--von Mises theorem \citep[e.g.][]{lecam1986} a value of $\st$ drawn from the posterior distribution\index{posterior distribution} to be of distance $O(N^{-1/2})$ from the mode of the distribution; provided that $\widehat{\st}$ is itself within $O(N^{-1/2})$ of the mode, we then expect $\|\st-\widehat{\st}\|^2$ to be $O(N^{-1})$. Therefore, compared to the $O(N^2/m)$ variance from the simple estimator \eqref{eq:var-grad-bounds}, we would expect to see a reduced $O(N/m)$ variance \eqref{eq:var-grad-bounds2} for the
 control variate\index{control variate} gradient estimator. This simple argument suggests that, for the same level of accuracy, we can reduce the computational cost of SGLD\index{stochastic gradient Langevin dynamics} by $O(N)$ if we use control variate-based\index{control variate} gradient estimators. This is supported by a number of theoretical results which show that if we ignore the pre-processing cost of finding $\widehat{\st}$, the computational cost per effective sample of SGLD\index{stochastic gradient Langevin dynamics} with control variates\index{control variate} is $O(1)$, rather than the $O(N)$ cost for SGLD\index{stochastic gradient Langevin dynamics} with the simple gradient estimator \eqref{eq:U-hat-simple}.

A further consequence of these bounds on the variance is that they suggest that if $\st$ is far from $\widehat{\st}$, then the variance when using control variates\index{control variate} can be larger, potentially substantially larger than that of the simple estimator. This point is illustrated in the top-left panel of Figure \ref{fig:stoc_grads}, where the variance in the simple gradient estimator \eqref{eq:U-hat-simple} is approximately constant for all $\st$. However, the variance in the control variate\index{control variate} gradient estimator increases as $\st$ moves away from $\widehat{\st}$. Two natural ways of addressing this issue have been proposed in the literature. One option is to only use the control variate\index{control variate} estimator when $\st$ is close enough to $\widehat{\st}$ \citep{Fearnhead:2018}, though it is up to the user to define what is `close enough' in practice. The second approach, a Langevin interpretation of the optimisation algorithm SAGA, which was proposed by \cite{Dubey:2016}, is to update $\widehat{\st}$ whilst running the SGLD\index{stochastic gradient Langevin dynamics} algorithm. This can be done efficiently by using $g_j(\st)=\nabla \log\pi_j(\st_{k_j})$, where $\st_{k_j}$ is the value of $\st$ at the most recent iteration of the SGLD\index{stochastic gradient Langevin dynamics} algorithm where $\nabla \log\pi_j(\st)$ was evaluated. This involves updating the storage of $g_j(\st)$ and its sum at each iteration; importantly, the latter can be done with an $O(m)$ cost. 

An alternative approach to reducing the variance, for both the simple and control variate-based\index{control variate} gradient estimators, is to use non-uniformly generated subsamples within the gradient\index{gradient} estimator. This approach, introduced by \cite{putcha2023preferential} and known as \textit{preferential subsampling}, generates subsamples $\mathcal{S}_m \subset \{1,\ldots,N\}$ of size $m$ with replacement, where the probability of drawing the $j$th data point is $p_j$ and the expected number of times that the $j$th data point appears in the subsample is $mp_j$. This then leads to a new simple, as opposed to CV-based, unbiased gradient\index{gradient} estimator
\begin{align}
\label{eq:ps-grad-estimator}
 \nabla^{(m)} \log\pi_{\mathrm{ps}}(\st) = \frac{1}{m}\sum_{j \in \mathcal{S}_m} \frac{\nabla \log\pi_j(\st)}{p_j},
\end{align}
where $p_j>0$ for all $j=1,\ldots,N$ and $\sum_{j=1}^N p_j=1$.
 The simple gradient estimator \eqref{eq:U-hat-simple} is given as a special case when $p_j = 1/N$ for all $j$. 

If we define the error in the stochastic gradient as $\bxi=\nabla^{(m)} \log\pi_{\mathrm{ps}}(\st) - \nabla \log\pi(\st)$, then taking expectations over the random subsample, whose distribution depends on the weights $\mathbf{p}=(p_1,\ldots,p_N)$, leads to the pseudo-variance of $\nabla^{(m)} \log \pi_{\mathrm{ps}}(\st),$
$$\Expect{\|\bxi\|^2}=\mathsf{tr}\left(\mbox{Var}[\nabla^{(m)} \log\pi_{\mathrm{ps}}(\st)]\right).$$

\begin{lemma}
    The optimal weights $\mathbf{p}^*$ which minimise the pseudo-variance
    $$
    \min_{\mathbf{p}:p_j \in [0,1], \sum_j p_j=1 } \mathsf{tr}\left(\mathrm{Var}[\nabla^{(m)} \log\pi_{\mathrm{ps}}(\st)]\right)
    $$
 are equivalently found by minimising 
$$
 \min_{\mathbf{p}}  \sum_{j=1}^N \frac{1}{p_{j}} \big\|\nabla \log\pi_j (\st) \big\|^2,
$$
 resulting in optimal weights of the form 
 \begin{equation}
\label{eq:sgld-opt-weights}
p_{j}^* = \frac{\| \nabla \log\pi_j (\st)\|}{\sum_{i=1}^N \| \nabla \log\pi_i (\st)  \| }\, \text{ for } j=1, \ldots, N. 
\end{equation}
\end{lemma}
\begin{proof}
    A full proof of this result is given in \cite{putcha2023preferential}. In brief, this result follows from two steps. Firstly, a straightforward expansion of $\mathsf{tr}\left(\mbox{Var}[\nabla^{(m)} \log\pi_{\mathrm{ps}}(\st)]\right)$ with the gradient estimator given in \eqref{eq:ps-grad-estimator} leads to  $\frac{1}{m}  \sum_{j=1}^N \frac{1}{p_{j}} \big\|\nabla \log\pi_j (\st) \big\|^2 + C,$ where $C=-\|\nabla\log\pi(\st)\|^2/m \leq 0$ does not depend on $\mathbf{p}$. Minimising this term with respect to the constraint $\{\mathbf{p}:p_j \in [0,1], \sum_j p_j=1\}$ follows by using Lagrange multipliers.  

\end{proof}

In practice, the optimal weights $p_j^*$, which would minimise the gradient variance, are dependent on the current iterate $\st_k$ of the SGLD\index{stochastic gradient Langevin dynamics} algorithm. This means that for each iteration $k$ in Algorithm \ref{alg:SGLD}, every $p_j$ for all $j=1,\ldots,N$ would need to be updated, which is an $O(N)$ calculation and would defeat the purpose of using subsamples of size $m \ll N$ in the SGLD\index{stochastic gradient Langevin dynamics} algorithm. We could instead replace the optimal weights $p_j^*$ \eqref{eq:sgld-opt-weights} with approximate weights $\widehat{p}_j,$ where in \eqref{eq:sgld-opt-weights} we replace $\st$ with an estimate of the posterior mode $\widehat{\st}$. As discussed in the case of control variate\index{control variate} gradient estimators, finding $\widehat{\st}$ requires a one-off pre-processing cost.  

Weighted sampling can be combined with the control variate\index{control variate} estimator \eqref{eq:cv-grad-est} with a natural choice of weights that increase with the size of the derivative of $\nabla \log\pi_j(\st)$ at $\widehat{\st}$. We can also use stratified sampling to try to ensure each subsample\index{subsample} is representative of the full data. However, regardless of the choice of gradient estimator, an important question is: how large should the subsample\index{subsample} size $m$ be? A simple rule is to choose $m$ so that, in one iteration of SGLD\index{stochastic gradient Langevin dynamics}, the variance of the noise from the gradient term is dominated by the variance of the injected noise. The former is $\delta^2/4$ times the pseudo-variance of the gradient estimator and the latter is proportional to $\delta$\index{step size}. If we choose $\delta$ proportional to $1/N$ (the square of the target size), then the bounds in Lemma \ref{cv-lemma} suggest that $m$ must grow in proportion to $N$ for the simple estimator, but that a fixed $m$ suffices for the control variate\index{control variate} estimator when $\st$ is within $O(N^{-1/2})$ of $\widehat{\st}$. The choice of step size is discussed further in Section \ref{ch3:implementation-guidance}.

Subsample\index{subsample} size could also be dynamically adjusted whilst running the SGLD\index{stochastic gradient Langevin dynamics} algorithm. The idea, which is particularly relevant when using control variates\index{control variate}, is that the accuracy of the estimator of the gradient can vary considerably with $\btheta$. To counteract this, it may be more efficient to have a larger subsample\index{subsample} size when the variance would be larger. One simple approach is to specify an upper bound, say $V_0$, on the variance that we would like to achieve, and consider how to vary subsample\index{subsample} size to achieve this.

An extension of the result in Lemma \ref{cv-lemma} can be derived for weighted subsampling. If the subsample\index{subsample} in the control variate\index{control variate} gradient estimator \eqref{eq:cv-estimator} is drawn with probabilities $p_j$, and each sampled term is weighted by $1/(mp_j)$ as in \eqref{eq:ps-grad-estimator}, then the variance of the gradient estimator satisfies

\begin{equation}
    \label{eq:var-bound-stoch-grad}
     \mathsf{tr}\left(\Var { \nabla^{(m)} \log\pi_{\mathrm{cv}}(\st) }\right) \leq \frac{1}{m} \|\st - \widehat{\st} \|^2 
     \sum_{j=1}^N \frac{L_j^2}{p_j}
     ,
\end{equation}
where $p_j$ are the subsample\index{subsample} weights and $L_j$ are Lipschitz constants on the gradient components \eqref{eq:ass-lipschitz-j}. A similar result can be derived for the simple gradient estimator \eqref{eq:U-hat-simple}. As with the control variate bound in Lemma \ref{cv-lemma}, the bound in \eqref{eq:var-bound-stoch-grad} is also $O(N)$ for fixed $m$ and bounded $L_j$. The $O(1/N)$ term $\|\btheta-\widehat{\btheta}\|^2$ cancels with the summation over $N$ terms. However, each of the $N$ terms in the sum is $O(N)$ when each $p_j$ is of order $1/N$. Choosing appropriate weights $p_j$ for each $j=1,\ldots,N$ reduces the multiplier of $N$; the bound is smallest when $p_j\propto L_j$.

Given our specified upper bound $V_0>0$, we can plug this into \eqref{eq:var-bound-stoch-grad}. By rearranging the inequality, we see that, for given weights $p_j$, the variance is at most $V_0$ whenever the subsample\index{subsample} size satisfies
$$
m \geq \frac{1}{V_0}\|\st - \widehat{\st} \|^2 \bigg(\sum_{j=1}^N \frac{L_j^2}{p_j}\bigg).
$$
For a fixed bound $V_0$, fixed $N$ and fixed weights $p_j$, this subsample\index{subsample} size is $m \propto \|\st - \widehat{\st} \|^2,$  which suggests that for SGLD with control variates\index{control variate}, the subsample\index{subsample} size should increase at a rate which is quadratic in the distance between the current iterate of the SGLD\index{stochastic gradient Langevin dynamics} chain $\st_k$ and the mode of the posterior distribution\index{posterior distribution} $\widehat{\st}$. Whilst the constant of proportionality may be hard to calculate, a user can choose a constant based on a reasonable average subsample\index{subsample} size they want to achieve -- and this would still keep the bound on the variance of the gradient estimate the same for all iterations of SGLD\index{stochastic gradient Langevin dynamics}.

\subsection{Example: The Value of Control Variates}

Recall the tractable Gaussian posterior example $\pi(\st) = \mathsf{N}(\bmu_N,\bSigma_N)$ in \eqref{eq:running-gaussian}, where $\bSigma_N = (N\bV^{-1} + \bI_2)^{-1}$ and $\bmu_N = \bSigma_N(\bV^{-1}\sum_{j=1}^N \by_j)$. From this posterior distribution\index{posterior distribution}, we can calculate the full-data gradient\index{gradient}, as well as the simple stochastic \eqref{eq:U-hat-simple} and control variate-based\index{control variate} gradients \eqref{eq:cv-estimator}. The gradient\index{gradient} for the $j$th component is $\nabla\log\pi_j(\st)=  \nabla\log f(\by_j|\st)+(1/N)\nabla\log \pi_0(\st)$ and for this example we can easily derive the posterior mode $\widehat{\st}= \bSigma_N(\bV^{-1}\sum_{j=1}^N \by_j)$ and use this within the control variate\index{control variate} gradient estimator. For the full-data, simple stochastic, and control variate\index{control variate} gradient estimators, we have the following,  

\begin{align*}
\nabla \log \pi(\st) &= \sum_{j=1}^N \nabla\log\pi_j(\st) = \sum_{j=1}^N \bV^{-1} \by_j - (N\bV^{-1} +\mathbf{I}_2)\st, \\
\nabla^{(m)} \log \pi(\st) &= \frac{N}{m}\sum_{j \in \mathcal{S}_m} \nabla\log\pi_j(\st)=  \frac{N}{m}\sum_{j \in \mathcal{S}_m} \bV^{-1}\by_j - \left(N\bV^{-1} +\mathbf{I}_2\right)\st,
\end{align*}
\begin{align*}
\nabla^{(m)} \log\pi_{\mathrm{cv}}(\st) &= \sum_{j=1}^N \nabla \log\pi_j(\widehat{\st}) + \frac{N}{m} \sum_{j \in \mathcal{S}_m} \left( \nabla \log\pi_j(\st)-\nabla \log\pi_j(\widehat{\st}) \right) \\
&= \sum_{j=1}^N \bV^{-1} \by_j - (N\bV^{-1} +\mathbf{I}_2)\widehat{\st}
+ \left(N\bV^{-1}+\mathbf{I}_2\right)\left(\widehat{\st}-\st\right)  \\
&= \sum_{j=1}^N \bV^{-1}\by_j - (N\bV^{-1} +\mathbf{I}_2)\st.
\end{align*}
 For this Gaussian example, the control variate-based\index{control variate} gradient estimator is equal to the full-data gradient\index{gradient}. This holds for any choice of $\widehat{\st}$, and not only for the posterior mode, because $\nabla\log\pi_j(\st)-\nabla\log\pi_j(\widehat{\st})$ does not depend on $j$. The simple gradient estimator \eqref{eq:U-hat-simple} gives an unbiased estimate of the full-data gradient\index{gradient}, however, for small subsample\index{subsample} sizes, this estimator can lead to poor posterior approximations. For this particular model, it is possible to separate the data $\by_j$ from the parameters $\st$ in such a way that the gradient estimator can be updated at each Monte Carlo iteration without re-evaluating the data, i.e., the data component of the gradient could be pre-computed and stored. Therefore, for this particular model, the simple SGLD gradient estimator would not be recommended as the model structure easily leads to more efficient gradient estimators. However, deriving better gradient estimators, such as for this Gaussian posterior model, is usually only possible for simple models and therefore the simple unbiased gradient estimator \eqref{eq:U-hat-simple} is still a popular choice within the SGLD scheme (Algorithm \ref{alg:SGLD}) for general models. 
 
 Figure \ref{fig:stoc_grads} shows the posterior approximation for the Gaussian model using the three gradient estimators, where stochastic gradients are calculated using $1\%$ of the full dataset. The ULA\index{unadjusted Langevin algorithm} sampler (top right) without a Metropolis--Hastings\index{Metropolis--Hastings} correction produces an accurate, albeit slightly biased, approximation similar to the SGLD with control variates\index{control variate} (SGLD-CV) algorithm (bottom right), which has the same gradient\index{gradient} estimator for this model. The SGLD\index{stochastic gradient Langevin dynamics} algorithm (bottom left) has the correct posterior mean but produces an over-dispersed approximation to the posterior variance. Reducing the stochasticity in the gradient estimator, for example, by increasing the subsample\index{subsample} size, will lead to improved empirical performance. This is a well-known feature of the SGLD\index{stochastic gradient Langevin dynamics} algorithm and its theoretical properties are discussed in the next section.

\begin{figure}
\centering
  \includegraphics[width=\textwidth]{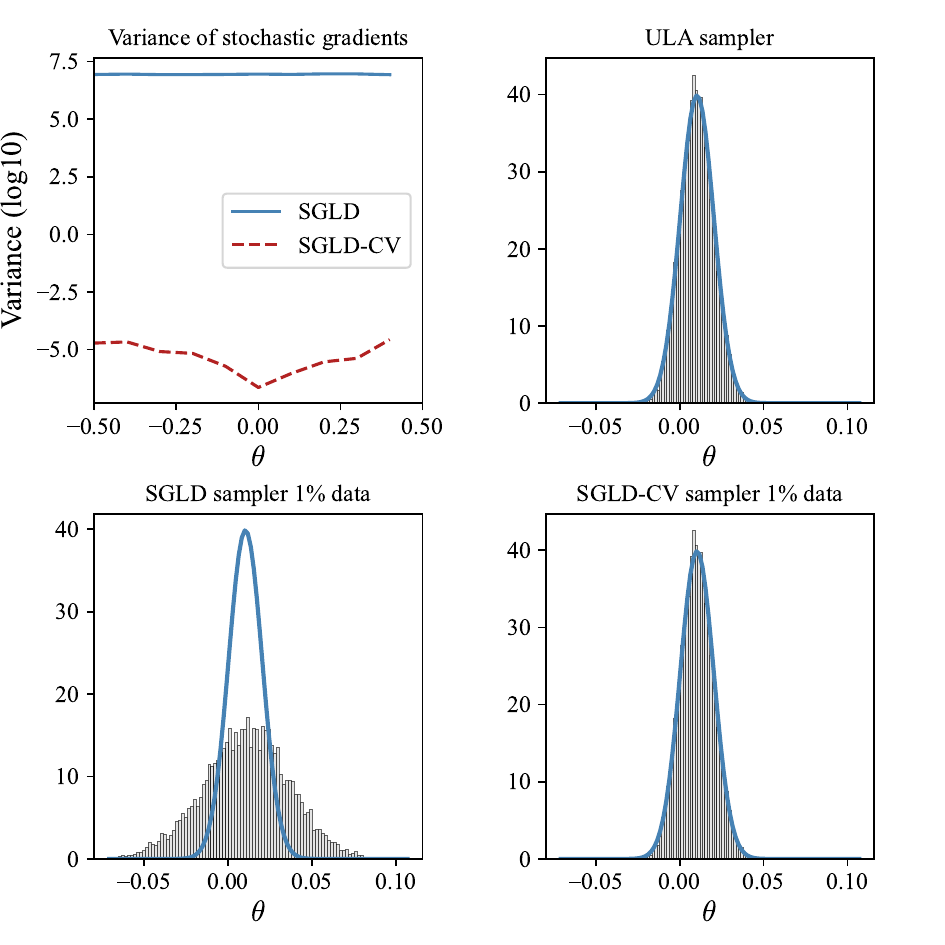}
 \caption{Top left: Variance of the estimated gradients taken over a range of $\theta_1$ for the first component of $\st$.  The variance for the SGLD estimator is stable across $\theta_1,$ whereas the variance in the SGLD-CV gradient estimator increases as $|\theta_1-\widehat{\theta}_1|$  increases. 
 The remaining plots give the posterior approximation to the first marginal component $\theta_1$ for ULA, SGLD and SGLD-CV, with the solid black line representing the true density.} \label{fig:stoc_grads}
\end{figure}

\subsection{Convergence Results for Stochastic Gradient Langevin Dynamics} \label{sec:ch3-sgld-theory}

The SGLD\index{stochastic gradient Langevin dynamics} algorithm provides a simple and efficient approach for sampling from a posterior distribution\index{posterior distribution} $\pi$. 
However, a key question is whether errors can accumulate over many SGLD\index{stochastic gradient Langevin dynamics} iterations, leading to poor approximate samples.
Fortunately, under suitable regularity conditions on $\pi$, theoretical results indicate that SGLD\index{stochastic gradient Langevin dynamics} can avoid persistent error accumulation. A key assumption is that the drift in the underlying Langevin diffusion pushes the state $\st$ towards regions of high probability under $\pi$. This ensures that the diffusion is geometrically ergodic, that is, it forgets its initial position at an exponential rate. As a result, the SGLD-generated samples tend to remain in the regions of high probability under $\pi$.

There are two main theoretical approaches for analysing the accuracy of SGLD\index{stochastic gradient Langevin dynamics}. (i) We can consider the accuracy of estimating expectations of the form $\Expects{\pi}{h(\st)},$ for some test function  $h(\st)$ under the posterior distribution\index{posterior distribution} $\pi$. This involves taking the average of $h(\st_k)$ over $n$ SGLD\index{stochastic gradient Langevin dynamics} iterates $\{\st_k\}_{k=1}^n$. The mean squared error (MSE) between this Monte Carlo average and the true expectation $\Expects{\pi}{h(\st)}$ provides one measure of SGLD\index{stochastic gradient Langevin dynamics} accuracy. \citet{Teh:2016} studied this error metric in the setting where the SGLD\index{stochastic gradient Langevin dynamics} step size\index{step size} $\delta_k$ decays over iterations. (ii) Alternatively, we can bound the error between the distribution of the SGLD\index{stochastic gradient Langevin dynamics} iterates and the posterior $\pi$ after $n$ steps. This involves analysing the total variation, or Wasserstein distance\index{Wasserstein metric}, between the marginal distribution of the SGLD\index{stochastic gradient Langevin dynamics} chain at iteration $n$ (i.e. $\tilde{\pi}_n$) and $\pi$. Since SGLD\index{stochastic gradient Langevin dynamics} is based on the Langevin diffusion, its ergodicity properties can be leveraged to prove the marginal distributions converge to $\pi$. 

Consider approach (i). The MSE of the SGLD\index{stochastic gradient Langevin dynamics} estimator can be upper bounded by $C_N(\delta^2 +1/(n\delta)),$ where $C_N$ is a constant dependent on the dataset size $N$. The $\delta^2$ term reflects the squared bias and $1/(n\delta)$ is the variance term. For a fixed computational budget $n$, the bias increases if the step size\index{step size} $\delta$ increases, while the variance decreases with increasing $\delta$. \cite{Teh:2016} showed that in the setting of a decreasing step size\index{step size} $\delta_k$, in order to minimise the asymptotic MSE, the optimal $\delta$ decays at rate of $n^{-1/3}$. This yields an MSE rate of $n^{-2/3}$, which is slower than the $n^{-1}$ rate for standard Monte Carlo methods. The slower convergence arises from controlling both bias and variance, which is similar to other asymptotically biased Monte Carlo methods. On the other hand, for larger computational budgets (i.e. larger $n$), exact MCMC\index{Markov chain Monte Carlo} will outperform SGLD\index{stochastic gradient Langevin dynamics}, because, unlike for exact MCMC\index{Markov chain Monte Carlo} methods, the bias from SGLD\index{stochastic gradient Langevin dynamics} with a fixed step size\index{step size} is non-vanishing, and removing it with a decreasing step size\index{step size} slows the rate of convergence. 

For approach (ii), one quantifies the accuracy of SGLD\index{stochastic gradient Langevin dynamics} via the accuracy of the marginal distribution of the SGLD\index{stochastic gradient Langevin dynamics} samples, $\st_k$, at iteration $k$. We denote this marginal distribution by $\tilde{\pi}_k$. Accuracy is commonly measured by the Wasserstein distance\index{Wasserstein metric} \eqref{eq:w2-distance} between $\tilde{\pi}_k$ and the posterior distribution\index{posterior distribution} $\pi$, as this makes the analysis more tractable. 
However, care is needed when interpreting the Wasserstein distance, as it is not scale-invariant, i.e., changing the units scales the distance. Additionally, for a fixed level of accuracy in each marginal, the distance grows as $d^{1/2}$ with dimension $d$. 

Much of the theory for ULA\index{unadjusted Langevin algorithm} and SGLD\index{stochastic gradient Langevin dynamics} assumes the log posterior density, $\log\pi(\st),$ is smooth and strongly concave. The key assumptions for analysis of these algorithms are that $\log\pi(\st)$ is $l$-strongly concave \eqref{eq:ass-convex} and $\log\pi(\st)$ is continuously differentiable with $L$-Lipschitz gradients \eqref{eq:ass-lipschitz}. This means there exist constants $0 < l \leq L$ such that for all $\st$ and $\st'$:
\begin{align}
  \label{eq:ass-lipschitz}
 \|\nabla \log\pi(\st')-\nabla \log\pi(\st)\| \leq L \|\st-\st'\|, ~~\text{and}~~ \\
 \label{eq:ass-convex}
 - \langle \nabla \log\pi(\st) - \nabla \log\pi(\st'),\st-\st'\rangle\geq l\,\|\st-\st'\|^2,
\end{align}
where we define $\kappa = L/l$ as the condition number. 
The first condition \eqref{eq:ass-lipschitz} bounds how fast the Langevin drift can change, and thus controls the step size (which should be $\delta<1/L$). The second assumption \eqref{eq:ass-convex} ensures that the drift term pushes $\st$ towards high-density regions, making the Langevin diffusion geometrically ergodic. 
Together, these assumptions imply an upper and lower bound on the directional derivatives of the log posterior density. The bounds enable stable discretisation and prevent persistent error accumulation in the SGLD\index{stochastic gradient Langevin dynamics} algorithm. They also imply $\log\pi(\st)$ is uni-modal. 
By leveraging strong concavity, the resulting theory provides step size\index{step size} conditions and rates of convergence for ULA/SGLD\index{stochastic gradient Langevin dynamics}. When $\log\pi(\st)$ is non-strongly concave, alternative assumptions are required \citep{raginsky2017non,majka2018non}, for instance that $-\log\pi(\st)$ is dissipative\index{distantly dissipative}, i.e. $-\langle \st,\nabla\log\pi(\st) \rangle \geq a\|\st\|^2-b,$ for some $a>0$ and $b \geq 0$.

Under the above assumptions \eqref{eq:ass-lipschitz}-\eqref{eq:ass-convex}, \cite{Dalalyan:2017} showed that running the SGLD\index{stochastic gradient Langevin dynamics} algorithm with a step size\index{step size} parameter $\delta \leq 4/(l+L)$, the Wasserstein-2 distance\index{Wasserstein metric} $\mathsf{d}_{W_2}(\tilde{\pi}_n,\pi)$ between the SGLD\index{stochastic gradient Langevin dynamics} marginal $\tilde\pi_n$ at iteration $n$ and the posterior $\pi$ can be bounded as:
\begin{equation} \label{eq:Wasserstein-bd1}
 \mathsf{d}_{W_2}(\tilde\pi_n,\pi) \leq (1-l \delta/2)^n \mathsf{d}_{W_2}(\tilde\pi_0,\pi) + C_3(\delta d)^{1/2} + C_4 \sigma (\delta d)^{1/2},
\end{equation}
where $l, C_3, C_4$ are constants, $d$ is the dimension, and $\sigma^2$ is an upper bound on the variance of any individual component of the stochastic gradient\index{stochastic gradient} \eqref{eq:var-grad-bounds}. In the case of ULA, $\sigma^2=0$. The first term in the upper bound decays exponentially, controlling bias from the initial distribution of the algorithm $\tilde\pi_0$ from which the Markov chain is initialised. The second term represents the error that results from the Euler--Maruyama\index{Euler-Maruyama} discretisation of the Langevin diffusion. The final term relates to the variance in the stochastic gradient\index{stochastic gradient}. In the case of the simple gradient estimator \eqref{eq:U-hat-simple}, the variance is $O(N^2/m)$ and the final term in the bound is $O(N\sqrt{(\delta d/m)})$, which is then the dominating term in the upper bound.

Given that the main motivation of SGLD\index{stochastic gradient Langevin dynamics} is to perform Bayesian inference over large-scale data, a natural question is to ask how does SGLD\index{stochastic gradient Langevin dynamics} scale as data size $N$ increases? One way of addressing this is to ask, as $N$ increases, what is the computational cost of running SGLD\index{stochastic gradient Langevin dynamics} so that we have a fixed level of approximation? We need to define our measure of approximation appropriately to account for the fact that $\pi$ will change as $N$ increases: under certain assumptions, such as the Bernstein--von Mises \citep{lecam1986} assumption, the variance will decay as $1/N$. This has been investigated by \cite{Nagapetyan:2017} and \cite{Baker:2017}, who consider using control variates\index{control variate} as a pre-processing step, which has a computational cost that is linear in $N$. Ignoring the cost of this pre-processing step for SGLD\index{stochastic gradient Langevin dynamics}, using control variates\index{control variate} asymptotically has a computational cost per effective sample that is constant. By comparison, the computational cost per effective sample of SGLD\index{stochastic gradient Langevin dynamics} with the simple estimator of the gradient \eqref{eq:U-hat-simple} is linear in $N$.

\subsubsection{Example: Theoretical Properties on a Gaussian Target Distribution} \label{sec:ch3-theory-example}

We can gain insight into the properties of the SGLD\index{stochastic gradient Langevin dynamics} algorithm by returning to our running Gaussian example \eqref{eq:running-gaussian}. Recall that the posterior under this model is a bivariate Gaussian distribution $\st |\by \sim \mathsf{N}(\bmu_N,\bSigma_N)$. In \eqref{eq:running-gaussian}, we assumed that the covariance matrix $\bSigma_N$ was diagonal; however, we will now instead consider a general symmetric positive definite matrix. We can express the variance matrix in terms of some rotation matrix $\bP$ and a diagonal matrix $\bD$, whose entries satisfy the condition $\sigma_1^2\geq \sigma_2^2,$ i.e., $\bSigma_N=\bP^{\top}\bD\bP$. Under this Gaussian posterior model, the drift term of the Langevin diffusion is
\[
\nabla \log\pi(\st)=- \bSigma_N^{-1}(\st-\bmu_N) = -\bP^{\top}\bD^{-1}\bP(\st-\bmu_N).
\]
The $k+1$th iteration of the SGLD\index{stochastic gradient Langevin dynamics} algorithm is
\begin{equation}
  \label{eq:gaussian-example-dynamics}
 \st_{k+1}=  \st_{k}-\frac{\delta}{2}\bP^{\top}\bD^{-1}\bP(\st_{k}-\bmu_N)+\delta \bnu_k + \sqrt{\delta}\bZ_k,
\end{equation}
where $\bZ_k$ is a vector of two independent standard Gaussian random variables and $\bnu_k$ is half the error in our stochastic estimate of $\nabla \log\pi(\st_{k})$, so that $\delta\bnu_k$ is the error that the estimate induces in the update. The entries of $\bD^{-1}$ correspond to the constants that appear in assumptions \eqref{eq:ass-lipschitz} and \eqref{eq:ass-convex}, with $l=1/\sigma_1^2$ and $L=1/\sigma_2^2,$ which are discussed further in Section \ref{ch3:implementation-guidance} in terms of their impact on the step size parameter $\delta$.

We can simplify the SGLD\index{stochastic gradient Langevin dynamics} algorithm by considering updates on the transformed state $\tilde{\st}=\bP(\st-\bmu_N)$, which gives SGLD\index{stochastic gradient Langevin dynamics} updates,
\begin{align*}
 \tilde{\st}_{k+1} =& \tilde{\st}_{k}-\frac{\delta}{2}\bD^{-1}\tilde{\st}_{k}+\delta\bP\bnu + \sqrt{\delta}\bP\bZ \\
 =&\left(\begin{array}{cc}1-\delta/(2\sigma_1^2) & 0 \\ 0 & 1-\delta/(2\sigma_2^2)  \end{array} \right)\tilde{\st}_{k}+\delta\bP\bnu_k + \sqrt{\delta}\bP\bZ_k,
\end{align*}
where the variance of $\bP\bZ_k$ is still the identity as 
$\bP$ is a rotation matrix.

The SGLD\index{stochastic gradient Langevin dynamics} update is a vector auto-regressive process. This process will have a stationary distribution\index{stationary distribution} if $\delta<4\sigma_2^2=4/L$, otherwise the process will produce trajectories which will tend to infinity in at least one of the two components. A similar assumption on the step size\index{step size} is required to establish the upper bound on $\mathsf{d}_{W_2}(\tilde\pi_n,\pi)$ in  \eqref{eq:Wasserstein-bd1} for log-concave posterior distributions. 

For simplicity in the manner of convergence, we choose $\delta<2\sigma_2^2$  and then define $\lambda_j=\delta/(2\sigma_j^2)<1$. This then leads to the following SGLD\index{stochastic gradient Langevin dynamics} dynamics for each component, $j=1,2$
\begin{equation}
  \label{eq:1d-example}
  \tilde{\theta}_{k}^{(j)}=(1-\lambda_j)^k \tilde{\theta}_0^{(j)}+\sum_{i=0}^{k-1} (1-\lambda_j )^{k-1-i}\left(\delta\bP\bnu_i + \sqrt{\delta}\bP \bZ_i \right)^{(j)},  
\end{equation}
where $\tilde{\theta}_k^{(j)}$ is the $j$th component of $\tilde{\st}_k$. From this, we immediately see from the first term on the right-hand side of \eqref{eq:1d-example}
that the SGLD\index{stochastic gradient Langevin dynamics} algorithm forgets its initial condition exponentially quickly. However, the rate of exponential decay is slower for the component with the larger marginal variance, i.e., $\sigma_1^2$. Furthermore, as the step size\index{step size} $\delta$ is constrained by the smaller marginal variance $\sigma_2^2$, this rate will have to be slow if $\sigma_2^2 \ll \sigma_1^2$; this suggests that there are benefits from re-scaling the posterior so that the marginal variances of different components are approximately equal.

Taking the expectation of \eqref{eq:1d-example},  with respect to $\bnu$ and $\bZ$, and letting $k \rightarrow \infty$, results in SGLD\index{stochastic gradient Langevin dynamics} dynamics that have the correct limiting mean, but with an inflated variance. This is most easily seen if we assume that the variance of $\bP\bnu_k$, $\bSigma_{\bnu}$ say, is independent of position. In this case, the marginal variances of the stationary distribution\index{stationary distribution} of SGLD\index{stochastic gradient Langevin dynamics} are the diagonal entries of
\[
\left(\begin{array}{cc} (1-(1-\lambda_1)^2 )^{-1} & 0 \\ 0 &  (1-(1-\lambda_2)^2)^{-1}\end{array} \right) (\delta^2\bSigma_{\bnu}+\delta\bI_2),
\]
where $\bI_2$ is a 2-dimensional identity matrix. The marginal variance for component $j$ is thus 
\[
\sigma_j^2\frac{1+\delta\Sigma_{\bnu,jj}}{1-\delta/(4\sigma_j^2)} = \sigma_j^2\left(1+\delta\Sigma_{\bnu,jj}\right)+\frac{\delta}{4}+O(\delta^2).
\]
The inflation in variance comes both from the noise in the estimate of $\nabla \log\pi(\st)$, which is the $\delta\Sigma_{\bnu,jj}$ factor, and the Euler approximation, which supplies the additive constant, $\delta/4$. 
For more general posterior distributions, the mean of the stationary distribution\index{stationary distribution} of the SGLD\index{stochastic gradient Langevin dynamics} algorithm will not necessarily be correct, but we would expect the mean to be more accurate than the variance, with the variance of SGLD\index{stochastic gradient Langevin dynamics} being greater than that of the true posterior.
This analysis further suggests that, for posterior distributions which are close to Gaussian, it may be possible to perform a better correction to compensate for the inflation of the variance. For example, we could replace $\bZ_k$ with Gaussian random variables where the covariance matrix is chosen such that the covariance matrix of $\delta\bnu_k+\sqrt{\delta}\bZ_k$ becomes $\delta\bI_2$.

\section{A General Framework for Stochastic Gradient MCMC}\index{SGMCMC|see{stochastic gradient Markov chain Monte Carlo}}
\label{ch3:sgmcmc-general}

Stochastic gradient MCMC\index{stochastic gradient Markov chain Monte Carlo} is not limited to approximating the Langevin diffusion. We can construct other diffusion processes that also have $\pi$ as their stationary distribution\index{stationary distribution}. By approximating the dynamics of these alternative diffusions, we can develop stochastic gradient\index{stochastic gradient} versions of many popular MCMC\index{Markov chain Monte Carlo} algorithms beyond the Langevin method. 

\cite{ma2015complete} proposed a general framework for stochastic gradient MCMC\index{stochastic gradient Markov chain Monte Carlo} that extends the approach beyond the Langevin diffusion. This framework allows us to develop stochastic gradient\index{stochastic gradient} analogues of algorithms like HMC\index{Hamiltonian Monte Carlo} (as introduced in Section \ref{sec:ch2-HMC}), which leverages Hamiltonian dynamics. Stochastic gradient HMC\index{stochastic gradient Hamiltonian Monte Carlo} (SGHMC)\index{SGHMC|see{stochastic gradient Hamiltonian Monte Carlo}} has been shown to display improved mixing\index{mixing} properties compared to stochastic gradient Langevin dynamics\index{stochastic gradient Langevin dynamics}. 

Beyond both SGLD\index{stochastic gradient Langevin dynamics} and SGHMC\index{stochastic gradient Hamiltonian Monte Carlo}, stochastic gradient MCMC\index{stochastic gradient Markov chain Monte Carlo} approaches provide a general and flexible framework to adapt many MCMC\index{Markov chain Monte Carlo} algorithms to large datasets where full-data MCMC\index{Markov chain Monte Carlo} is infeasible. By approximating the dynamics of various diffusions with the correct stationary distribution\index{stationary distribution}, it is possible to develop fast, scalable MCMC\index{Markov chain Monte Carlo} algorithms tailored to different posterior geometries and datasets.

We introduce a general class of diffusion models that may also include auxiliary variables, and we denote the full state by $\bvartheta \in \mathbb{R}^{d_{\bvartheta}}$. This state contains our variable of interest $\st$, as well as an auxiliary variable $\bp$. For example, in the case of the overdamped Langevin diffusion\index{Langevin diffusion!overdamped}, $\bvartheta=\st$, and extending the state space to include a velocity variable  $\bp$ allows for the underdamped Langevin diffusion\index{Langevin diffusion!underdamped} of Section \ref{sec.IntroLangDiffusions}, which relates to  Hamiltonian MCMC, giving $\bvartheta=(\st,\bp)$. 

We define the general SDE\index{stochastic differential equation} for $\bvartheta$ as

\begin{equation}
\label{eq:sde}
\mbox{d}\bvartheta = \frac{1}{2}\mathbf{b}(\bvartheta)\mbox{d}t + \sqrt{\bD(\bvartheta)}\mbox{d}W_t,  
\end{equation}
where $\mathbf{b}(\bvartheta)$ is the \textit{drift term}, $\bD(\bvartheta)$ is a positive semi-definite \textit{diffusion matrix} with matrix square root $\sqrt{\bD(\bvartheta)}$ and $W_t$ denotes $d_{\bvartheta}$-dimensional Brownian motion. \cite{ma2015complete} provide a general framework for how to choose $\mathbf{b}(\bvartheta)$ and $\bD(\bvartheta)$ to achieve a desired stationary distribution\index{stationary distribution}. Given a general function $H(\bvartheta)$, where $\exp\{-H(\bvartheta)\}$ is integrable, simulating from an SDE\index{stochastic differential equation} with drift term,
\begin{align}
\label{eq:sde-f}
\mathbf{b}(\bvartheta) =& -\left[\bD(\bvartheta)+\bQ(\bvartheta)\right]\nabla H(\bvartheta) + \Gamma(\bvartheta),  \ \mbox{where}\\
 \Gamma_i(\bvartheta) =& \sum_{j=1}^{d_{\bvartheta}} \frac{\partial}{\partial \bvartheta_j}(\bD_{ij}(\bvartheta)+\bQ_{ij}(\bvartheta)), \nonumber
\end{align}
 ensures that the stationary distribution\index{stationary distribution} of (\ref{eq:sde}) is proportional to $\exp\{-H(\bvartheta)\}$. The matrix $\bQ(\bvartheta)$ is a skew-symmetric curl matrix that controls the deterministic traversing effects of the SDE\index{stochastic differential equation}  sampler, whereas $\bD(\bvartheta)$ controls the diffuse dynamics of the process. The general SDE\index{stochastic differential equation}  \eqref{eq:sde} can be decomposed into (i) \textit{Riemannian-Langevin dynamics} which are reversible and controlled by $\bD(\bvartheta)$ and (ii) \textit{deterministic Hamiltonian dynamics} which are irreversible and controlled by $\bQ(\bvartheta)$. For the Langevin-type algorithms, where $\bQ(\bvartheta)=0,$ it can be shown that they are able to quickly converge towards a mode of the distribution and diffusively explore the region around the mode. For the deterministic HMC\index{Hamiltonian Monte Carlo} algorithm, where $\bD(\bvartheta)=0,$ the algorithm excels at deterministically traversing along level sets of the Hamiltonian. 
 
To approximately sample from the SDE\index{stochastic differential equation} we discretise the dynamics using the same Euler--Maruyama\index{Euler-Maruyama} scheme used for the Langevin diffusion (Section \ref{sec:ch3-langevin-diffusion}), that is,
\begin{equation}
\label{eq:sgmcmc-update}
\bvartheta_{k+1} = \bvartheta_k - \frac{\delta}{2} \left[(\bD(\bvartheta_k)+\bQ(\bvartheta_k))\nabla H(\bvartheta_k)-\Gamma(\bvartheta_k)\right] + \sqrt{\delta} \bZ, \quad k \geq 0,  
\end{equation}
where $\bZ \sim \mathsf{N}(0,\bD(\bvartheta_k))$. If $\bvartheta=\st$, the posterior distribution\index{posterior distribution} $\pi$ is the stationary distribution\index{stationary distribution} for the choice $H(\bvartheta)=-\log\pi(\st)$. If we include an auxiliary variable, i.e., $\bvartheta = (\st,\bp)$, and we choose $H(\bvartheta)=-\log\pi(\st)+K(\bp)$ for some user-chosen function $K(\cdot)$, then $\pi$ is the $\btheta$-marginal of the stationary distribution\index{stationary distribution}. 

From the general SDE\index{stochastic differential equation}  framework \eqref{eq:sde}, we can obtain \emph{stochastic gradient MCMC}\index{stochastic gradient Markov chain Monte Carlo} (SGMCMC) algorithms by replacing $\nabla H(\bvartheta_k)$ in (\ref{eq:sgmcmc-update}) with an unbiased estimate $\widehat{\nabla} H(\bvartheta_k)$ that is evaluated on a subsample\index{subsample} of the dataset. As shown in Section \ref{sec:ch3-sgld-theory} and Figure \ref{fig:stoc_grads}, the statistical efficiency of stochastic gradient\index{stochastic gradient} algorithms is strongly tied to the variance of the gradient estimate. We can see that the variability of the gradient estimate inflates the noise input into \eqref{eq:sgmcmc-update} -- which will lead to a stationary distribution\index{stationary distribution} whose variance is also inflated. Therefore, where feasible, we should correct for this. 
If we define $\bV(\bvartheta_k) = \Var{\widehat{\nabla} H(\bvartheta_k)}$ as the variance of the stochastic gradient\index{stochastic gradient}, then the conditional variance of $\bvartheta_{k+1}$ given $\bvartheta_{k}$ will be inflated by $\delta^2\mathbf{B}(\bvartheta_k)$, where
\begin{equation}
\label{eq:sde-var-inflate}
\mathbf{B}(\bvartheta_k) = \frac{1}{4}(\bD(\bvartheta_k)+\bQ(\bvartheta_k)) \bV(\bvartheta_k) (\bD(\bvartheta_k)+\bQ(\bvartheta_k))^\top.
\end{equation}
Correcting for this inflation in the SGMCMC\index{stochastic gradient Markov chain Monte Carlo} setting leads to a modified discrete-time algorithm in \eqref{eq:sgmcmc-update}, where $\nabla H(\bvartheta_k)$ is replaced by the stochastic approximation $\widehat{\nabla} H(\bvartheta_k)$, but we also reduce the variance of the noise term, so $\bZ \sim \mathsf{N}(0,\bD(\bvartheta_k)-\delta\mathbf{B}(\bvartheta_k))$. The challenge with this idea in practice is how to estimate $\mathbf{B}(\bvartheta_k)$?

Through different choices of $H(\bvartheta)$, $\bD(\bvartheta)$ and $\bQ(\bvartheta)$, we can derive several SGMCMC\index{stochastic gradient Markov chain Monte Carlo} algorithms as special cases of the general discretised SDE\index{stochastic differential equation}  \eqref{eq:sgmcmc-update}. Some of these special cases will be introduced in the following sections.  

\subsubsection{Stochastic Gradient Langevin Dynamics}\index{stochastic gradient Langevin dynamics}

The SGLD\index{stochastic gradient Langevin dynamics} algorithm (introduced in Section \ref{sec:ch3-sgld-alg-sec}) follows from the dynamics of the general SDE\index{stochastic differential equation}  \eqref{eq:sgmcmc-update} by setting
\begin{equation*}
    \bvartheta = \st, \quad H(\bvartheta) = -\log\pi(\st), \quad \bD(\bvartheta)= \mathbf{I}, \quad \bQ(\bvartheta) = \mathbf{0}
\end{equation*}
to give dynamics
\begin{equation*}
     \st_{k+1} = \st_k + \frac{\delta}{2} \left[\nabla \log\pi(\st_k)\right] + \sqrt{\delta} \bZ, \quad k \geq 0,  
\end{equation*}
which is the same Euler--Maruyama\index{Euler-Maruyama} approximation of the Langevin diffusion introduced in \eqref{eq:Euler}, but where in practice $\nabla \log\pi(\st_k)$ would be replaced with a stochastic gradient\index{stochastic gradient} estimator, using, for example, \eqref{eq:U-hat-simple}, \eqref{eq:cv-estimator} or \eqref{eq:ps-grad-estimator}.

\subsubsection{Stochastic Gradient Hamiltonian Monte Carlo}\index{stochastic gradient Hamiltonian Monte Carlo}

The popular HMC\index{Hamiltonian Monte Carlo} algorithm introduced in Section \ref{sec:ch2-HMC} can also be derived as a special case of the general SDE\index{stochastic differential equation}  dynamics \eqref{eq:sgmcmc-update}. As discussed in Section \ref{sec:ch2-HMC}, the HMC\index{Hamiltonian Monte Carlo} algorithm introduces a momentum component $\bp$ to improve the mixing\index{mixing} of the Markov chain and a mass matrix\index{mass matrix} $\mathbf{M}$, where the Hamiltonian dynamics are used to update the position $\st$ and momentum components $\bp$. In practice, the HMC\index{Hamiltonian Monte Carlo} algorithm uses the leapfrog numerical integration scheme to minimise numerical errors; however, for the purpose of illustration, we shall consider the simpler Euler integration scheme for creating a stochastic gradient\index{stochastic gradient Hamiltonian Monte Carlo} HMC\index{Hamiltonian Monte Carlo} algorithm, and return to leapfrog-type\index{leapfrog dynamics} schemes after \eqref{eq:sghmc}. 

The Euler-discretised Hamiltonian dynamics for the state $\bvartheta=(\st,\bp)$ are

\begin{equation}
\label{eq:hmc-naive}
\begin{pmatrix}
\st_{k+1} \\
\bp_{k+1}
\end{pmatrix} = \begin{pmatrix}
    \st_k \\
    \bp_k
\end{pmatrix} + \frac{\delta}{2}\begin{bmatrix}
    \mathbf{M}^{-1} \bp_k \\
    \nabla \log\pi(\st_k)
\end{bmatrix}, \quad k \geq 0,  
\end{equation}
which fits into the general SDE\index{stochastic differential equation}  framework of \eqref{eq:sde} and \eqref{eq:sde-f} by setting 
\begin{align*}
H(\bvartheta)=-\log\pi(\st) + \frac{1}{2}\bp^{\top}\mathbf{M}^{-1}\bp, \ \ \mathbf{D}(\bvartheta)=\mathbf{0} \ \ \mbox{and} \ \ \bQ(\bvartheta) = \left( \begin{array}{cc} 0 & -\mathbf{I} \\ \mathbf{I} & 0 \end{array}\right).    
\end{align*}

 If the gradient $\nabla\log\pi(\st)$ in \eqref{eq:hmc-naive} is replaced by a stochastic gradient\index{stochastic gradient} $\hat{\nabla}\log\pi(\st) = \nabla\log\pi(\st) + \mathsf{N}(0,\mathbf{V}(\st)),$ where $\mathbf{V}(\st)$ is the variance of the stochastic gradient\index{stochastic gradient}, then under this stochastic setting the dynamics in \eqref{eq:hmc-naive} will become 
$$
\bp_{k+1} = \bp_k + \frac{\delta}{2} \nabla\log\pi(\st_k) + \mathsf{N}\left(0,\tfrac{\delta^2}{4}\mathbf{V}(\st_k)\right),
$$
which is known as the \textit{naive} stochastic gradient HMC (SGHMC) algorithm \citep{chen2014stochastic,ma2015complete}. It was proved by \cite{chen2014stochastic} that the naive SGHMC\index{stochastic gradient Hamiltonian Monte Carlo} algorithm does not work well, as the error from estimating the gradient will accumulate over iterations and cannot be controlled. To overcome this, the authors suggest adding a \textit{friction term} to the velocity; this is equivalent to using a stochastic gradient\index{stochastic gradient} version of the \textit{underdamped Langevin dynamics}\index{Langevin diffusion!underdamped} of Section \ref{ch1:sec:underdamped}. The intuition is that the introduction of friction means that errors from previous iterations will decay geometrically so that the overall error from using a stochastic gradient\index{stochastic gradient} can be controlled. We can further improve accuracy by using the idea of correcting the variance of the injected noise that is introduced at each iteration.

Returning to the general SDE\index{stochastic differential equation}  framework \eqref{eq:sde}, the corrected stochastic gradient\index{stochastic gradient Hamiltonian Monte Carlo} HMC\index{Hamiltonian Monte Carlo} algorithm follows by setting
\begin{align*}
     \bvartheta &= (\st,\bp), \quad H(\bvartheta) = -\log\pi(\st) + \frac{1}{2}\bp^\top \mathbf{M}^{-1}\bp, \\ \bD(\bvartheta) &= \left( \begin{array}{cc} 0 & 0 \\ 0 & \mathbf{C} \end{array}\right), \quad \bQ(\bvartheta) = \left( \begin{array}{cc} 0 & -\mathbf{I} \\ \mathbf{I} & 0 \end{array}\right),
\end{align*}
where $\mathbf{C}$ is a generic matrix, sometimes known as the \textit{friction} term, and is chosen such that $\mathbf{C} - \frac{\delta}{4}\bV(\st)$ is a positive semi-definite matrix. Discretising this general form SDE\index{stochastic differential equation}  with this particular $\mathbf{D}(\bvartheta)$ and $\mathbf{Q}(\bvartheta)$ leads to the dynamics,

\begin{equation}
\label{eq:sghmc}
    \begin{pmatrix}
    \st_{k+1} \\
    \bp_{k+1}
\end{pmatrix} = \begin{pmatrix}
\st_{k} \\
\bp_{k}  
\end{pmatrix}  + \frac{\delta}{2}\begin{bmatrix}
    \mathbf{M}^{-1}\bp_k\\ 
    \hat{\nabla}\log\pi(\st_k) - \mathbf{C}\mathbf{M}^{-1}\bp_k
\end{bmatrix}   + \begin{bmatrix}
    0 \\
    \sqrt{\delta} \bZ
\end{bmatrix}, \quad k \geq 0.  
\end{equation}

The gradient $\nabla\log\pi(\st_k)$ is replaced by a  stochastic estimator $\hat{\nabla}\log\pi(\st_k)$  and $\bZ \sim \mathsf{N}(0,\mathbf{C}-\delta\hat{\mathbf{B}})$, where $\hat{\mathbf{B}}$ is an estimate of $\frac{1}{4}\mathbf{V}(\st_k)$, the non-zero block of $\mathbf{B}(\bvartheta_k)$ in \eqref{eq:sde-var-inflate}.

In the case where the gradient is exact rather than approximate, $\mathbf{B}(\bvartheta_k)=\mathbf{0}$ and \eqref{eq:sghmc} is the Euler--Maruyama\index{Euler-Maruyama} discretisation of the underdamped Langevin dynamics\index{Langevin diffusion!underdamped} in \eqref{eqn.underdampedA} and \eqref{eqn.underdampedB} with a time step of $\epsilon=\delta/2$. Specifically, in \eqref{eqn.underdampedB}, set $f=\pi$ and $c=1$ and in \eqref{eq:sghmc}, set $\mathbf{M}=\mathbf{I}$ and $\mathbf{C}=\gamma \mathbf{I}$.

The update \eqref{eq:sghmc} is the Euler discretisation \eqref{eq:sgmcmc-update} of the general SDE\index{stochastic differential equation}. In practice, more accurate and more stable algorithms are obtained by splitting the SDE into three parts, each of which can be simulated exactly: the position update $\md\st=\tfrac{1}{2}\mathbf{M}^{-1}\bp\,\md t$ (A), the momentum update $\md\bp=\tfrac{1}{2}\nabla\log\pi(\st)\,\md t$ (B), and the Ornstein--Uhlenbeck (OU)\index{Ornstein--Uhlenbeck process} process $\md\bp=-\tfrac{1}{2}\mathbf{C}\mathbf{M}^{-1}\bp\,\md t+\mathbf{C}^{1/2}\md W_t$ (O). With $\mathbf{C}=\gamma\mathbf{M}$ for some $\gamma>0$, one iteration of the BAOAB\index{BAOAB} scheme \citep{leimkuhler2013rational} with step size\index{step size} $\delta$ is
\begin{align*}
\bp &\leftarrow \bp+\tfrac{\delta}{4}\hat{\nabla}\log\pi(\st), \qquad \st \leftarrow \st+\tfrac{\delta}{4}\mathbf{M}^{-1}\bp, \\
\bp &\leftarrow e^{-\gamma\delta/2}\,\bp+\left(1-e^{-\gamma\delta}\right)^{1/2}\mathbf{M}^{1/2}\bZ, \\
\st &\leftarrow \st+\tfrac{\delta}{4}\mathbf{M}^{-1}\bp, \qquad \bp \leftarrow \bp+\tfrac{\delta}{4}\hat{\nabla}\log\pi(\st),
\end{align*}
where now $\bZ\sim\mathsf{N}(0,\mathbf{I})$, and the middle line is the exact transition of the OU process over a time $\delta$ (see ($*$) in Section \ref{sec.DefineOU}). Without the middle line, and with exact gradients, this is the leapfrog\index{leapfrog dynamics} step of Section \ref{sec:ch2-HMC} with $\epsilon=\delta/2$. The last update of one iteration and the first update of the next are evaluated at the same value of $\st$ and can share the same stochastic gradient\index{stochastic gradient}, so each iteration requires only one new gradient estimate.

BAOAB still has a discretisation error, so its stationary distribution\index{stationary distribution} for $\st$ is not exactly $\pi$ in general. With exact gradients, however, the bias in the distribution of $\st$ is of order $\delta^2$, compared with order $\delta$ for \eqref{eq:sghmc}, and it vanishes for a Gaussian target at any step size\index{step size} for which the scheme is stable. In addition, \eqref{eq:sghmc} is unstable unless $\delta$ is small relative to the friction. With stochastic gradients\index{stochastic gradient}, the noise in the gradient enters through the (B) updates and still inflates the variance of the stationary distribution\index{stationary distribution} at order $\delta$, as for SGLD\index{stochastic gradient Langevin dynamics}. As in \eqref{eq:sghmc}, this can be offset, to leading order, by keeping the friction in (O) but replacing $\mathbf{C}^{1/2}$ by $(\mathbf{C}-\delta\hat{\mathbf{B}})^{1/2}$. Since $\st$ does not change during (O), this step can still be simulated exactly, and the middle line becomes $\bp \leftarrow e^{-\gamma\delta/2}\,\bp+\{(1-e^{-\gamma\delta})(\mathbf{M}-\tfrac{\delta}{\gamma}\hat{\mathbf{B}})\}^{1/2}\bZ$.

The efficiency with which an MCMC\index{Markov chain Monte Carlo} algorithm can explore a posterior distribution\index{posterior distribution} is heavily tied to the geometry of the posterior. If the $\st$ components of the posterior distribution\index{posterior distribution} are strongly correlated, then the step size\index{step size} will need to be optimised for the component with the smallest variability in order to ensure that the algorithm does not diverge (as illustrated in Section \ref{sec:ch3-theory-example} for the case of a Gaussian target distribution). This will significantly increase the mixing\index{mixing} time of the other components unless the posterior distribution\index{posterior distribution} is reparameterised so that the components of $\st$ are uncorrelated and have similar marginal distributions. Within the context of SGMCMC dynamics, it is possible to develop algorithms that incorporate reparameterisation by preconditioning\index{preconditioned} the gradients with a positive-definite matrix $\mathbf{G}(\st)$. If $\mathbf{G}(\st)$ is the expected Fisher information of the posterior distribution\index{posterior distribution}, then the SGMCMC\index{stochastic gradient Markov chain Monte Carlo} dynamics will be locally adapted to the posterior curvature by exploiting the Riemannian geometry of the posterior distribution\index{posterior distribution}. 

\subsubsection{Stochastic Gradient Riemannian Langevin Dynamics}\index{stochastic gradient Riemannian Langevin dynamics}

Riemannian versions of SGLD (\emph{stochastic gradient Riemannian Langevin dynamics}; SGRLD)\index{SGRLD|see{stochastic gradient Riemannian Langevin dynamics}} and SGHMC\index{stochastic gradient Hamiltonian Monte Carlo} (\emph{stochastic gradient Riemannian Hamiltonian Monte Carlo}; SGRHMC) also follow as special cases of the general-form SDE\index{stochastic differential equation}  \eqref{eq:sde-f}. We can derive SGRLD by setting 
\begin{equation*}
    \bvartheta = \st, \quad H(\bvartheta) = -\log\pi(\st), \quad \bD(\bvartheta)= \mathbf{G}(\st)^{-1}, \quad \bQ(\bvartheta) = \mathbf{0}
\end{equation*}
which leads to the discrete-time dynamics
\begin{equation*}
     \st_{k+1} = \st_k + \frac{\delta}{2} \left[\mathbf{G}(\st_k)^{-1} \hat{\nabla} \log\pi(\st_k) + \Gamma(\st_k) \right] + \sqrt{\delta} \bZ, \quad k \geq 0,  
\end{equation*}
with $\bZ \sim \mathsf{N}(0,\mathbf{G}(\st_k)^{-1})$ and $\hat{\nabla}\log\pi(\st_k)$ is a stochastic estimator for $\nabla\log\pi(\st_k)$.  The term $\mathbf{G}(\st_k)^{-1} \nabla \log\pi(\st_k)$ is the \emph{natural gradient} of the posterior distribution\index{posterior distribution}, which gives the direction of steepest ascent by taking into account the geometry implied by $\mathbf{G}(\st_k)$. If $\mathbf{G}(\st)=\mathbf{I},$ then the direction of steepest ascent would be given in the Euclidean space. The term $\Gamma_i(\st)=\sum_j \frac{\partial (\mathbf{G}(\st)^{-1})_{ij}}{\partial \st_j}$ accounts for the curvature of the manifold defined by $\mathbf{G}(\st)$. A stochastic gradient\index{stochastic gradient Hamiltonian Monte Carlo} HMC\index{Hamiltonian Monte Carlo} implementation that utilises the Riemannian geometry of the posterior distribution\index{posterior distribution} can also be derived within the general SDE\index{stochastic differential equation}  framework. Taking the curl matrix $\mathbf{Q}(\st)$ from SGHMC\index{stochastic gradient Hamiltonian Monte Carlo} and replacing the identity matrices with $\mathbf{G}(\st)^{-1/2}$ leads to an SGHMC\index{stochastic gradient Hamiltonian Monte Carlo} algorithm that is adaptive to the local posterior geometry.

\subsubsection{Theoretical Comparison of SGMCMC Algorithms}
When comparing the variety of SGMCMC\index{stochastic gradient Markov chain Monte Carlo} algorithms, it is natural to consider which algorithm is the most accurate and computationally efficient.  It is possible to compare the theoretical convergence rates for some of these algorithms. In the case of SGHMC\index{stochastic gradient Hamiltonian Monte Carlo} and SGLD\index{stochastic gradient Langevin dynamics}, and in the context of smooth and strongly log-concave posteriors, it is possible to derive bounds on the Wasserstein-2 distance\index{Wasserstein metric} between the posterior and the SGMCMC\index{stochastic gradient Markov chain Monte Carlo} sample distribution. 
These results are non-asymptotic and bound the Wasserstein-2 error for some $k$ iterations of the SGMCMC\index{stochastic gradient Markov chain Monte Carlo} algorithm using optimally tuned step sizes. These results show that if the stochastic gradients have low variance, then SGLD\index{stochastic gradient Langevin dynamics} requires $O(d/\epsilon^2)$ iterations for a given accuracy $\epsilon$, while SGHMC\index{stochastic gradient Hamiltonian Monte Carlo} needs only $O(\sqrt{d}/\epsilon)$ \citep{Cheng:2017}. These rates suppress the dependence on the condition number and logarithmic factors, and the second is proved for a particular discretisation of the underdamped Langevin diffusion\index{Langevin diffusion!underdamped} rather than for the Euler scheme \eqref{eq:sghmc}. This suggests that SGHMC\index{stochastic gradient Hamiltonian Monte Carlo} is to be preferred, with increasing benefits in higher dimensions (see Figure \ref{fig:logistic-regression-error} for a numerical illustration). However, there is a phase transition in SGHMC\index{stochastic gradient Hamiltonian Monte Carlo} as the gradient noise variance grows; above a threshold SGHMC\index{stochastic gradient Hamiltonian Monte Carlo} behaves like SGLD\index{stochastic gradient Langevin dynamics}, requiring $O(d/\epsilon^2)$ iterations.

\section{Guidance for Efficient Scalable Bayesian Learning}
\label{ch3:implementation-guidance}

Each of the SGMCMC\index{stochastic gradient Markov chain Monte Carlo} algorithms introduced has tuning parameters that need to be chosen by the user when running these algorithms. Some SGMCMC\index{stochastic gradient Markov chain Monte Carlo} algorithms, such as SGLD\index{stochastic gradient Langevin dynamics} with control variates\index{control variate}, require additional tuning parameters, for example, choosing a control variate\index{control variate}, but common to all SGMCMC\index{stochastic gradient Markov chain Monte Carlo} algorithms are the step size\index{step size} parameter $\delta$, the subsample\index{subsample} size $m$, and the number of Monte Carlo iterations $n$. 

For MCMC\index{Markov chain Monte Carlo} algorithms with Metropolis--Hastings\index{Metropolis--Hastings} corrections, such as MALA\index{Metropolis--adjusted Langevin algorithm} (Section \ref{sec:ch2-MALA}) and HMC\index{Hamiltonian Monte Carlo} (Section \ref{sec:ch2-HMC}), and which do not utilise data subsampling, the main tuning parameter is the step size\index{step size} $\delta$. Existing theoretical results have established that the optimal $\delta$ should be chosen such that the Metropolis--Hastings\index{Metropolis--Hastings} acceptance rate is $57.4\%$ (in the case of MALA) or at least $65\%$ (in the case of HMC). Under certain assumptions on the posterior distribution\index{posterior distribution}, this choice minimises the integrated auto-correlation time\index{integrated auto-correlation time} of the Markov chain produced by these algorithms (see Section \ref{sec.MCcvgESS} for further details). 

In the case of SGMCMC\index{stochastic gradient Markov chain Monte Carlo} algorithms, which do not include a Metropolis--Hastings\index{Metropolis--Hastings} acceptance step, minimising the auto-correlation\index{auto-correlation} of the Markov chain is not an appropriate objective for optimising $\delta$. Intuitively, this is because the auto-correlation\index{auto-correlation} can be reduced by simply letting $\delta \rightarrow \infty$, but in the case of the ULA\index{unadjusted Langevin algorithm} and SGMCMC\index{stochastic gradient Markov chain Monte Carlo} algorithms, this will increase the bias in the discretised diffusion process and lead to poor posterior approximations. This is not an issue for MALA\index{Metropolis--adjusted Langevin algorithm} and HMC\index{Hamiltonian Monte Carlo} as the Metropolis--Hastings\index{Metropolis--Hastings} acceptance rate will tend to zero as $\delta \rightarrow \infty$, and so in this setting, minimising the auto-correlation\index{auto-correlation} time will balance between making large jumps in the posterior space (i.e., large $\delta$) and having a reasonable acceptance rate. Alternative metrics are required to minimise the variance in the Markov chain and account for the asymptotic bias present in SGMCMC\index{stochastic gradient Markov chain Monte Carlo} algorithms. A popular choice is the kernel Stein discrepancy\index{kernel Stein discrepancy} (KSD) metric, a kernelised version of the Stein discrepancy discussed in detail in Chapter \ref{chap:stein}. The KSD provides a measure of discrepancy between the true posterior distribution\index{posterior distribution} and the Monte Carlo approximation generated by an SGMCMC\index{stochastic gradient Markov chain Monte Carlo} or other MCMC\index{Markov chain Monte Carlo} algorithm. Using the KSD metric, it is possible to optimise $\delta$ (and other SGMCMC\index{stochastic gradient Markov chain Monte Carlo} tuning parameters) by assessing various tuning parameters over a grid of possible values and selecting the tuning parameters which minimise the KSD \citep{coullon2023efficient}. Related ideas, based on Stein's method\index{Stein's method} and explored in Chapter \ref{chap:stein}, can be used to optimally thin the Markov chain to maximise the information about the posterior contained by a smaller set of Monte Carlo samples.

The step size\index{step size} parameter $\delta$ in SGLD\index{stochastic gradient Langevin dynamics} can be chosen using the convergence results in Section \ref{sec:ch3-sgld-theory}. The upper bound on the Wasserstein-2 distance\index{Wasserstein metric}  from \cite{Dalalyan:2017}, under assumptions \eqref{eq:ass-lipschitz} and \eqref{eq:ass-convex}, requires $\delta \leq 4/(l+L)$ in order to establish convergence for the SGLD\index{stochastic gradient Langevin dynamics} algorithm. Setting $\delta$ to be too small leads to slow convergence of the Markov chain, but a step size\index{step size} that is too large can cause the SGLD\index{stochastic gradient Langevin dynamics} algorithm to diverge. If $l$ and $L$ are known, then this information can be used to set $\delta$. For example, returning to the running Gaussian example \eqref{eq:running-gaussian} where the posterior distribution\index{posterior distribution} is $\st|\by \sim \mathsf{N}(\bmu_N,\bSigma_N)$, we know that $\nabla \log\pi(\st)=-\bSigma_N^{-1}(\st-\bmu_N)$ and that the Lipschitz constant $L$ measures the largest change in the gradient. 
Taking the Hessian\index{Hessian}, $\nabla\nabla\log\pi(\st)=-\bSigma_N^{-1}$, the Lipschitz constant is the spectral norm of the inverse covariance matrix, $L=\|\nabla\nabla\log\pi(\st)\|=\|\bSigma_N^{-1}\|=1/\lambda_\mathrm{min}(\bSigma_N)$, which is the reciprocal of the smallest eigenvalue of the covariance matrix. The strong-concavity parameter $l$ is the smallest eigenvalue of $\bSigma_N^{-1}$, i.e., $l=1/\lambda_{\mathrm{max}}(\bSigma_N)$. The condition number is therefore $\kappa=L/l=\lambda_{\mathrm{max}}(\bSigma_N)/\lambda_{\mathrm{min}}(\bSigma_N)$, and the step size\index{step size} $\delta=1/L=\lambda_{\mathrm{min}}(\bSigma_N)$ satisfies $\delta \leq 4/(l+L)$.

 We can assess the relationship between the step size\index{step size} parameter and the properties of the Gaussian model by considering two choices of the data covariance matrix $\bV$ in the Gaussian model \eqref{eq:running-gaussian}, $\bV^{(i)} = \begin{pmatrix}
1 & 0 \\
0 & 10
\end{pmatrix}$ and $\bV^{(ii)} = \begin{pmatrix}
1 & 3 \\
3 & 10
\end{pmatrix}$. Under covariance $\bV^{(i)}$, the variables $\st$ are uncorrelated whereas for $\bV^{(ii)}$ there is imposed correlation between the components of $\st$. This leads to $1/L^{(i)}\approx 10^{-3}$ and $1/L^{(ii)} \approx 10^{-4}$ for $\bV^{(i)}$ and $\bV^{(ii)}$, respectively (the condition numbers are $\kappa^{(i)}\approx 10$ and $\kappa^{(ii)} \approx 118$). In Figure \ref{fig:w2_dist_stepsizes}, we plot the Wasserstein-2 distance\index{Wasserstein metric} between the true Gaussian posterior distribution\index{posterior distribution} with $N=1000$ data points with the approximation generated from $n=10000$ iterations of the SGLD\index{stochastic gradient Langevin dynamics} algorithm, where the step size\index{step size} parameter $\delta$ is varied over a grid of values  $\delta \in \{10^{-5},\ldots,10^{-1}\}$. The left panel of Figure \ref{fig:w2_dist_stepsizes} is for the model with uncorrelated covariance matrix $\bV^{(i)}$ and the right panel uses $\bV^{(ii)}$. For both experiments, there is a value for $\delta$ which minimises the Wasserstein-2 distance\index{Wasserstein metric}, and in the case of uncorrelated variables, i.e., $\bV^{(i)}$, the optimal $\delta$ is larger than in the correlated case $\bV^{(ii)}.$ The dot indicates the step size\index{step size} recommended by the theoretical results, i.e., $\delta=1/L$, which closely aligns with the optimal grid search result. 

In general settings, however, it is not possible to calculate the optimal step size\index{step size} parameter as this depends on properties of the unknown posterior distribution\index{posterior distribution}. Under certain conditions on the posterior distribution\index{posterior distribution}, and assuming that $N$ is sufficiently large, then by the Bernstein--von Mises theorem, the variance of the posterior distribution\index{posterior distribution} will be of order $O(d/N)$. Therefore, setting the step size\index{step size} parameter to be proportional to $1/N$, i.e., $\delta \propto1/N,$ gives a simple heuristic step size\index{step size} which is often used by practitioners. Note that, for the Gaussian posterior example given above, this heuristic would lead to a step size\index{step size} $\delta=1/N=10^{-3},$ which matches the optimal step size\index{step size} parameter for the setting with uncorrelated $\st$ components, that is, $\bV^{(i)}$ (see the left panel in Figure~\ref{fig:w2_dist_stepsizes}). However, this step size would be too large for the correlated setting with covariance $\bV^{(ii)}$ (see the right panel in Figure \ref{fig:w2_dist_stepsizes}). An alternative perspective, discussed in Section \ref{sec:ch3-sgld-grad-var}, is that the gradient noise contributes a variance of $O(\delta^2)$ per iteration, compared with the $O(\delta)$ variance of the injected noise. For the simple unbiased gradient estimator, whose variance is $O(N^2/m)$, the choice $\delta=1/N$ gives a gradient-noise variance of $O(1/m)$ against an injected-noise variance of $O(1/N)$. The subsample\index{subsample} size must then grow with $N$, or a control variate\index{control variate} must be used, for the injected noise to dominate.

The step size\index{step size} choices above require either knowledge of the posterior distribution\index{posterior distribution} or a search over a grid of values, with the algorithm run once for each. An alternative is to adapt the step size\index{step size} as the algorithm runs. \cite{sharrock2025tuning} view ULA\index{unadjusted Langevin algorithm} and SGLD\index{stochastic gradient Langevin dynamics} as optimisation algorithms on the space of probability distributions, and carry over ideas from learning-rate-free stochastic optimisation. Their Fuse\index{Fuse} schedule runs $P$ copies of SGLD\index{stochastic gradient Langevin dynamics} in parallel, $\st_k^{(1)},\ldots,\st_k^{(P)}$, which share the step size\index{step size}
\begin{equation*}
\delta_k = \frac{2\max\Big\{r_{\varepsilon},\ \max_{1\leq s\leq k}\big(\frac{1}{P}\sum_{p=1}^{P}\|\st_{s-1/2}^{(p)}-\st_{1/2}^{(p)}\|^2\big)^{1/2}\Big\}}{\big(\sum_{s=1}^{k}\frac{1}{P}\sum_{p=1}^{P}\|\nabla^{(m)}\log\pi(\st_s^{(p)})\|^2\big)^{1/2}}, \quad k\geq 1,
\end{equation*}
with $\delta_0=2r_{\varepsilon}$. Here $\st_{s+1/2}^{(p)}=\st_s^{(p)}+\frac{\delta_s}{2}\nabla^{(m)}\log\pi(\st_s^{(p)})$ is the state of chain $p$ after the gradient\index{gradient} step and before the noise is added, and $r_{\varepsilon}>0$ is a small constant. The numerator estimates how far the chains have moved from their starting points and the denominator accumulates the squared norms of the gradient estimates, so the step size\index{step size} eventually decreases as the algorithm proceeds. Since the pseudo-variance of the gradient estimator contributes to the denominator, the step size\index{step size} also tends to be smaller when the subsample\index{subsample} size $m$ is smaller. For log-concave posterior distributions\index{posterior distribution}, and under a boundedness condition on the stochastic gradients\index{stochastic gradient}, an idealised version of this schedule converges in Kullback--Leibler divergence, up to a logarithmic factor, at the rate achieved by the best fixed step size\index{step size} for the given number of iterations. In practice, the results are insensitive to $r_{\varepsilon}$ over several orders of magnitude. The approach requires several chains to be run, and has not yet been extended to momentum-based algorithms such as SGHMC\index{stochastic gradient Hamiltonian Monte Carlo}.

\begin{figure}
    \centering
       \includegraphics[width=\textwidth]{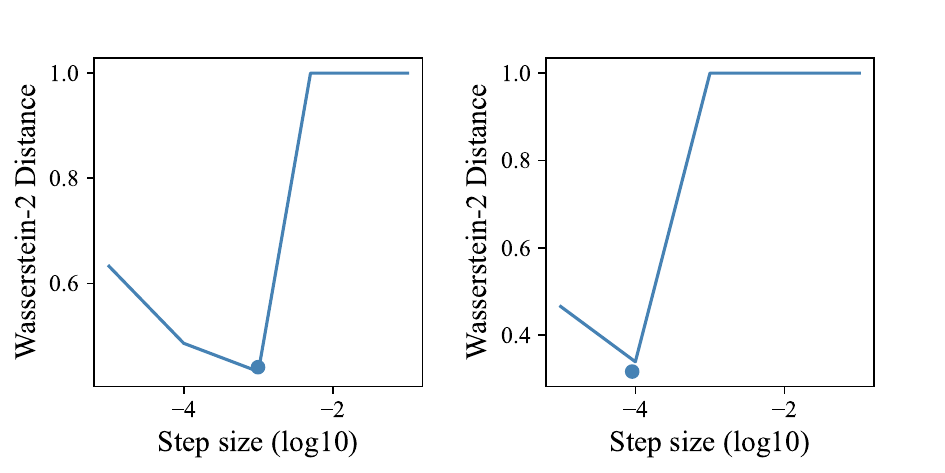}
    \caption{The Wasserstein-2 distance between the true and approximate posterior distributions when varying the step size parameter $\delta$. Left panel is for the model with covariance $\bV^{(i)}$. Right panel is for the model with covariance $\bV^{(ii)}$.}
    \label{fig:w2_dist_stepsizes}
\end{figure}

\subsection{Experiments on a Logistic Regression Model}
\label{sec:ch3-logistic-regression}

The logistic regression\index{logistic regression} model, first introduced in Section \ref{sec:ch1-logistic}, is used to predict the probability of binary outcomes $y_j \in \{0,1\}$ given covariates $\bx_j \in \mathbb{R}^{d}$. Assuming a Gaussian prior for the regression coefficients $\st \sim \mathsf{N}(\mathbf{0},\bSigma_{\btheta})$, the posterior distribution\index{posterior distribution}, conditional on the data $\mathcal{D} = \{y_j,\bx_j\}_{j=1}^N,$ is given by the unnormalised density 
\[
\pi(\st) := \pi(\btheta|\mathcal{D}) \propto \exp\left\{-\frac{1}{2}\btheta^{\top}\bSigma_{\btheta}^{-1}\btheta \right\} \prod_{j=1}^N \frac{\exp\{y_j\bx_j^{\top}\btheta\}}{1+\exp\{\bx_j^{\top}\btheta\}}.
\]

Metropolis--Hastings-based MCMC\index{Markov chain Monte Carlo} algorithms, such as MALA\index{Metropolis--adjusted Langevin algorithm} (Section \ref{sec:ch2-MALA}) and HMC\index{Hamiltonian Monte Carlo} (Section \ref{sec:ch2-HMC}), can be used to sample from the posterior distribution\index{posterior distribution} of the logistic regression\index{logistic regression} model. However, in the large data setting, these algorithms will converge slowly due to the higher computational cost of evaluating the posterior density and its gradient on the full dataset, whereas SGMCMC\index{stochastic gradient Markov chain Monte Carlo} algorithms are faster, per Monte Carlo iteration, but introduce an asymptotic bias into the posterior approximation. 

To assess the statistical accuracy of SGMCMC\index{stochastic gradient Markov chain Monte Carlo} against exact MCMC\index{Markov chain Monte Carlo} approaches, consider a logistic regression\index{logistic regression} model with $N=10000$ observations, which is small enough to allow MALA\index{Metropolis--adjusted Langevin algorithm} and HMC\index{Hamiltonian Monte Carlo} approaches to be computationally feasible. The dataset is split into a training and test dataset with a $80/20$ split. Data are simulated from the logistic regression\index{logistic regression} model where the dimension of the parameter vector $\st \in \mathbb{R}^d$ is varied, $d \in \{100,200,300,400,500\}$. 
The statistical accuracy of the posterior approximation of the SGLD\index{stochastic gradient Langevin dynamics} algorithm (Alg. \ref{alg:SGLD}) and SGHMC\index{stochastic gradient Hamiltonian Monte Carlo} algorithm \eqref{eq:sghmc}, with and without control variates\index{control variate} \eqref{eq:cv-grad-est}, is compared against a long Monte Carlo run of the NUTS algorithm \citep{hoffman2014no} using the Python package BlackJax \citep{cabezas2024blackjax}, which provides a ground-truth Monte Carlo approximation to the posterior distribution\index{posterior distribution}. As noted in Section \ref{ch3:implementation-guidance}, standard MCMC\index{Markov chain Monte Carlo} diagnostics\index{diagnostic} are not applicable for assessing the convergence of SGMCMC\index{stochastic gradient Markov chain Monte Carlo} algorithms, 
and so as a proxy for posterior accuracy, the MSE in the estimate of the posterior mean and variance given by the SGMCMC\index{stochastic gradient Markov chain Monte Carlo} algorithms is compared against the posterior mean and variance taken from the NUTS samples. The SGMCMC\index{stochastic gradient Markov chain Monte Carlo} algorithms are each run for $n=10000$ iterations and the NUTS algorithm for $5000$ iterations after warm-up; for the SGMCMC\index{stochastic gradient Markov chain Monte Carlo} algorithms, a subsample\index{subsample} size of $10\%$ is used. The step size\index{step size} of SGLD\index{stochastic gradient Langevin dynamics} is fixed using the heuristic $\delta=1/N$, and SGHMC\index{stochastic gradient Hamiltonian Monte Carlo} uses a step size of $10^{-5}$, with each iteration consisting of a momentum refresh followed by 15 steps of \eqref{eq:sghmc}. Note that improved numerical results could be achieved by optimising the step size\index{step size} parameter for each dimension $d$.

Figure \ref{fig:logistic-regression-error} shows the MSE in the estimate of $\Expects{\pi}{\st}$ and $\Vars{\pi}{\st}$, where the true expectation and variance are given by the NUTS sampler. The plotted results are presented as the MSE averaged over the parameter dimension. The results show that the SGHMC\index{stochastic gradient Hamiltonian Monte Carlo} algorithms are more robust to higher dimensions than the SGLD-type algorithms. This coincides with the established theory that, compared to SGHMC\index{stochastic gradient Hamiltonian Monte Carlo}, SGLD\index{stochastic gradient Langevin dynamics} requires more iterations to achieve a similar level of accuracy. Note that for higher-dimensional problems, it may be necessary to run the SGD optimiser for longer to find the mode of the distribution that is used in control variate-based SGMCMC algorithms. 

\begin{figure}
    \centering
    \includegraphics[width=\textwidth]{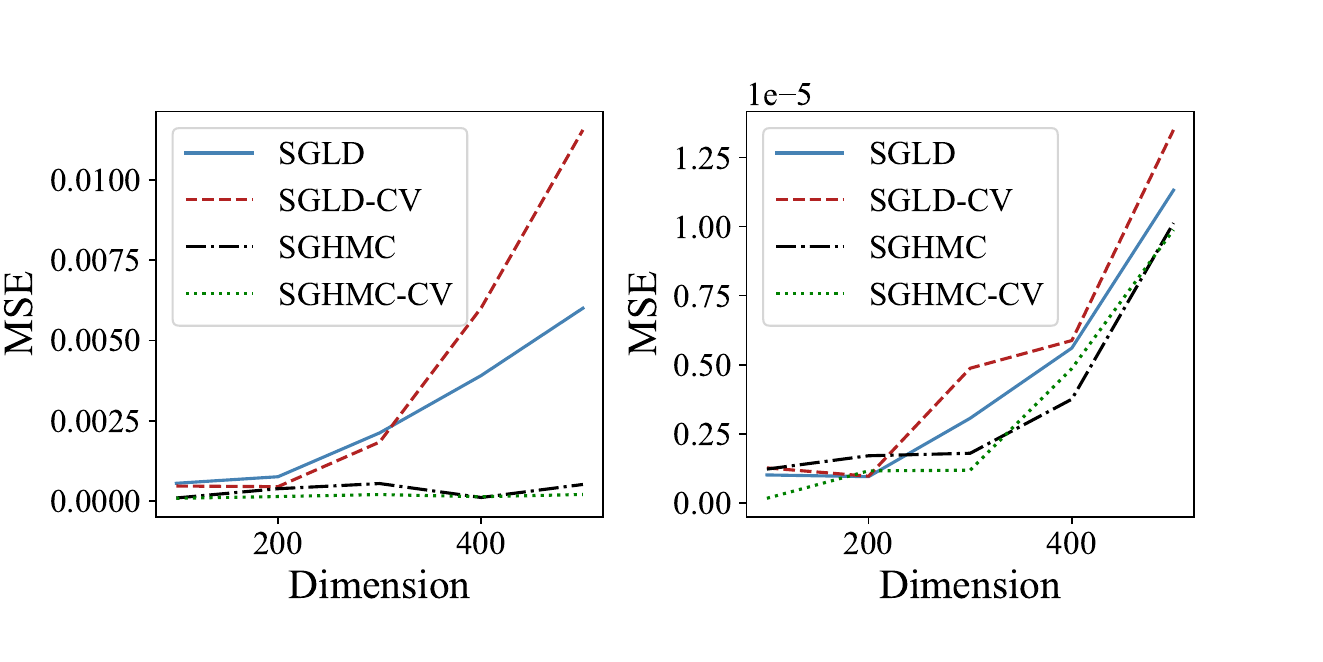}
    \caption{The mean squared error (MSE) of $\Expects{\pi}{\st}$ (left panel) and $\Vars{\pi}{\st}$ (right panel) compared to the same moments calculated from the NUTS posterior samples, which are treated as the ground-truth. 
    The results plotted are for the average MSE taken over all dimensions of the mean and marginal variance of $\pi(\st)$.}
    \label{fig:logistic-regression-error}
\end{figure}

The posterior accuracy results in Figure \ref{fig:logistic-regression-error} suggest that without increasing the number of Monte Carlo iterations, the SGMCMC\index{stochastic gradient Markov chain Monte Carlo} algorithms will produce poorer posterior approximations with increasing parameter dimension. The results show that SGMCMC\index{stochastic gradient Markov chain Monte Carlo} algorithms can produce highly accurate approximations for the first and second moments of the posterior distribution\index{posterior distribution}, and in the case of SGHMC\index{stochastic gradient Hamiltonian Monte Carlo}, the first moment is very similar to the first moment given by the NUTS sampler. As illustrated previously in Figure \ref{fig:stoc_grads}, SGLD\index{stochastic gradient Langevin dynamics} can produce good approximations to the first posterior moment, but for small data subsamples it tends to produce overestimates of the second posterior moment. The reason for the poorer posterior approximation of the SGLD-based samplers compared to the SGHMC-based samplers can be seen in the Monte Carlo trace plots in Figure \ref{fig:logistic-regression-trace-plots}. For the higher-dimensional setting ($d=500$), we can see that for posterior components $\theta_1$ and $\theta_2$, the mixing\index{mixing} is worse for SGLD and SGLD-CV compared to SGHMC and SGHMC-CV. This is then reflected in the Monte Carlo approximation for the posterior mean and variance (Figure \ref{fig:logistic-regression-error}), where SGLD and SGLD-CV are not as accurate as SGHMC and SGHMC-CV when $d=500$, but display similar levels of accuracy for $d=100$ and $d=200$. The mixing of the SGLD-based samplers could be improved by hand-tuning the step-size parameter $\delta$, or preconditioning the gradients to account for the correlation structure of the posterior distribution. 

\begin{figure}
    \centering
    \includegraphics[width=0.8\textwidth]{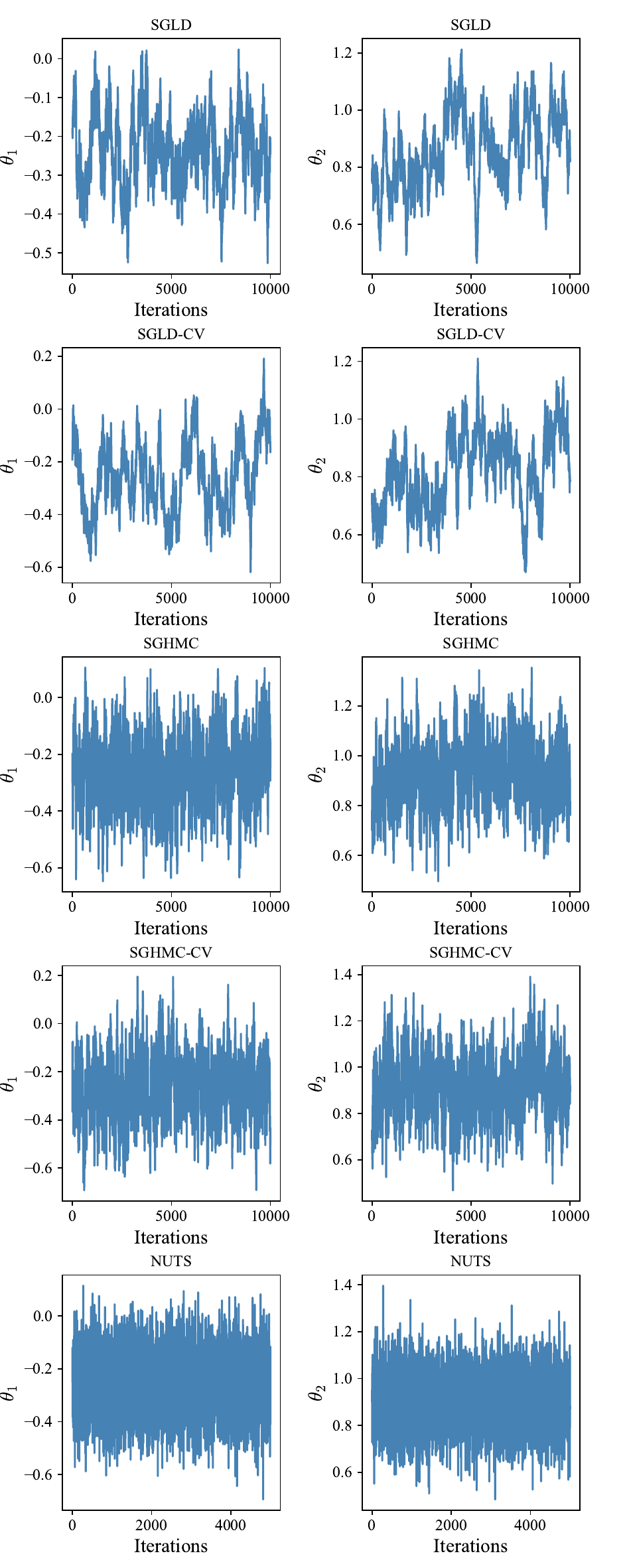}
    \caption{Trace plots for the first two components of $\st$ with $d=500$.}
    \label{fig:logistic-regression-trace-plots}
\end{figure}

Beyond posterior accuracy, we can also assess the predictive accuracy of SGMCMC\index{stochastic gradient Markov chain Monte Carlo} algorithms against exact full-data MCMC\index{Markov chain Monte Carlo} algorithms, in this case using the NUTS sampler. 
Figure \ref{fig:logistic-log-pred} illustrates that SGMCMC\index{stochastic gradient Markov chain Monte Carlo} algorithms are competitive against slower Metropolis--Hastings-based algorithms when assessed against predictive accuracy. Figure \ref{fig:logistic-log-pred} plots the percentage improvement in log-posterior predictive accuracy for each SGMCMC\index{stochastic gradient Markov chain Monte Carlo} algorithm over the log-posterior predictive accuracy of the NUTS sampler, which is treated as the gold standard approach. The results are given for the logistic regression\index{logistic regression} model on a test dataset using posterior samples over a range of parameter dimensions $d$. The results highlight that SGMCMC\index{stochastic gradient Markov chain Monte Carlo} algorithms are competitive and potentially superior to slower, full-data, MCMC\index{Markov chain Monte Carlo} algorithms in terms of predictive accuracy, displaying only a small decrease in accuracy but with a significant computational advantage. 

\begin{figure}
    \centering
    \includegraphics[width=0.7\textwidth]{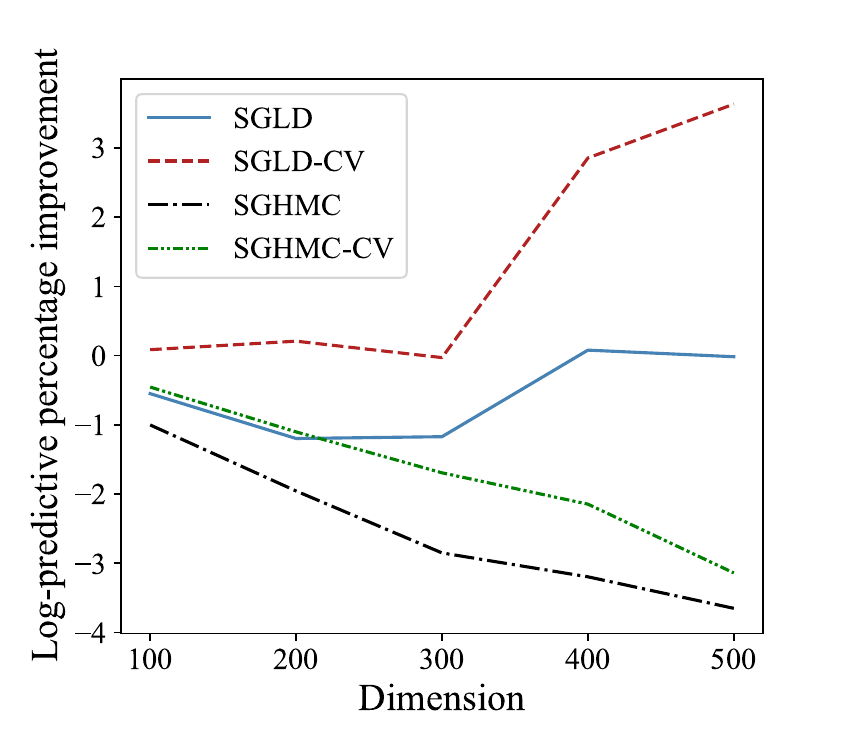}
    \caption{Percentage improvement in the log-predictive density of each SGMCMC\index{stochastic gradient Markov chain Monte Carlo} algorithm relative to the log predictive density of the NUTS sampler on the logistic regression model calculated on a test data set.
    The dimension of the parameter of interest is $d \in \{100,200,300,400,500\}$. 
    }
    \label{fig:logistic-log-pred}
\end{figure}

\subsection{Experiments on a Bayesian Neural Network Model}
\label{sec:ch3-bnn}

A Bayesian neural network\index{Bayesian neural network} model for multi-class classification was introduced in Section \ref{sec:ch1-BNN}, where the dataset $\mathcal{D}=\{y_j,\bx_j\}_{j=1}^N$ is comprised of a collection of $G$ classes $y_j \in \{0,\ldots,G-1\}$ and covariates $\bx_j \in \mathbb{R}^{d_x}$. The unnormalised posterior density is 

\begin{equation}
    \label{eq:ch3-bnn-post-dens}
    \pi(\btheta):=\pi(\btheta|\mathcal{D}) \propto \pi_0(\btheta)\prod_{j=1}^N \frac{\exp(\mathbf{A}^\top_{y_j+1} \sigma(\mathbf{B}^\top \bx_j + \mathbf{b}) + a_{y_j+1})}{\sum_{g=1}^{G}\exp(\mathbf{A}^\top_{g} \sigma(\mathbf{B}^\top \bx_j + \mathbf{b}) + a_{g})},
\end{equation}
where $\sigma(\cdot)$ is a softmax activation function (see Section  \ref{sec:ch1-BNN} for details). The model parameters $\btheta=\mbox{vec}(\mathbf{A},\mathbf{B},\mathbf{a},\mathbf{b})$ are the weights $\mathbf{A},\mathbf{B}$ and biases $\mathbf{a},\mathbf{b}$ of the network model. We shall assume independent standard Gaussian priors for each parameter, i.e. $\btheta \sim \mathsf{N}(0,\mathbf{I}).$

Neural networks are commonly used for image classification tasks. One of the most fundamental and widely used datasets in image classification is the MNIST handwritten digit dataset. The MNIST dataset consists of images of handwritten digits, ranging from zero to nine, that is, $y_j \in \{0,1,2,3,4,5,6,7,8,9\}$ (see Figure \ref{fig:mnist-digits} for a subsample of the dataset). Each image is represented by a small square of 28 pixels by 28 pixels which are treated as covariates. Each $\bx_j \in \mathbb{R}^{784}$ is a vectorisation of a matrix made up of 28 rows and 28 columns, with each pixel containing grayscale information representing the darkness of that specific point in the image. A brighter pixel would have a higher value, while a darker one would have a lower value.

\begin{figure}
    \centering
    \includegraphics[width=0.8\textwidth]{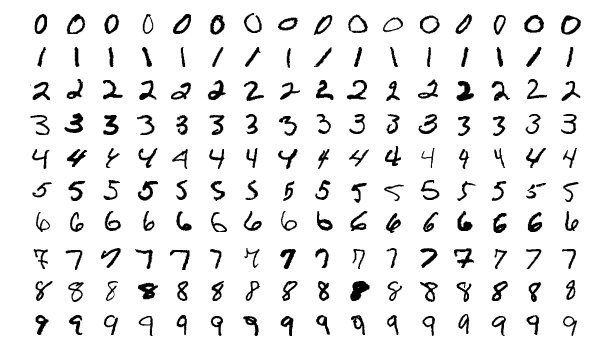}
    \caption{A selection of digits from the MNIST dataset. Image source - Wikipedia. }
    \label{fig:mnist-digits}
\end{figure}

The Bayesian neural network\index{Bayesian neural network} for this example \eqref{eq:ch3-bnn-post-dens} has two layers: an input layer that receives the information from the $28 \times 28$ image, and a hidden layer containing 100 hidden variables that act as intermediate processing units. The parameters of the neural network $\btheta$ are of the form $\mathbf{A}\in \mathbb{R}^{100 \times 10},$ $\mathbf{B}\in \mathbb{R}^{784 \times 100},$ $\mathbf{a}\in \mathbb{R}^{1 \times 10},$ and $\mathbf{b}\in \mathbb{R}^{1 \times 100}$. 

 The MNIST dataset contains a large collection of $60000$ images in the training set. Each image has a corresponding label, indicating which digit $(0-9)$ it represents. Using SGMCMC\index{stochastic gradient Markov chain Monte Carlo} algorithms, we can approximate the posterior distribution\index{posterior distribution} of the Bayesian neural network\index{Bayesian neural network} using subsamples of the labelled images and pixel values to train the Bayesian neural network\index{Bayesian neural network} to recognise patterns and relationships between the pixels and the corresponding digits. 

 We use the SGLD\index{stochastic gradient Langevin dynamics} and SGHMC\index{stochastic gradient Hamiltonian Monte Carlo} algorithms from the Python package SGMCMCJax \citep{coullon2022sgmcmcjax}, and their control-variate counterparts, to draw samples from the Bayesian neural network\index{Bayesian neural network} posterior \eqref{eq:ch3-bnn-post-dens}. We run each algorithm for $2000$ iterations and retain every $10$th iteration of the Markov chain. A subsample size of $1\%$ of the full dataset is used for all SGMCMC\index{stochastic gradient Markov chain Monte Carlo} samplers. For the control-variate-based algorithms, a stochastic gradient descent algorithm is used to find the posterior mode and the Markov chain is initialised at the posterior mode.

\begin{figure}
    \centering
    \includegraphics[width=0.8\textwidth]{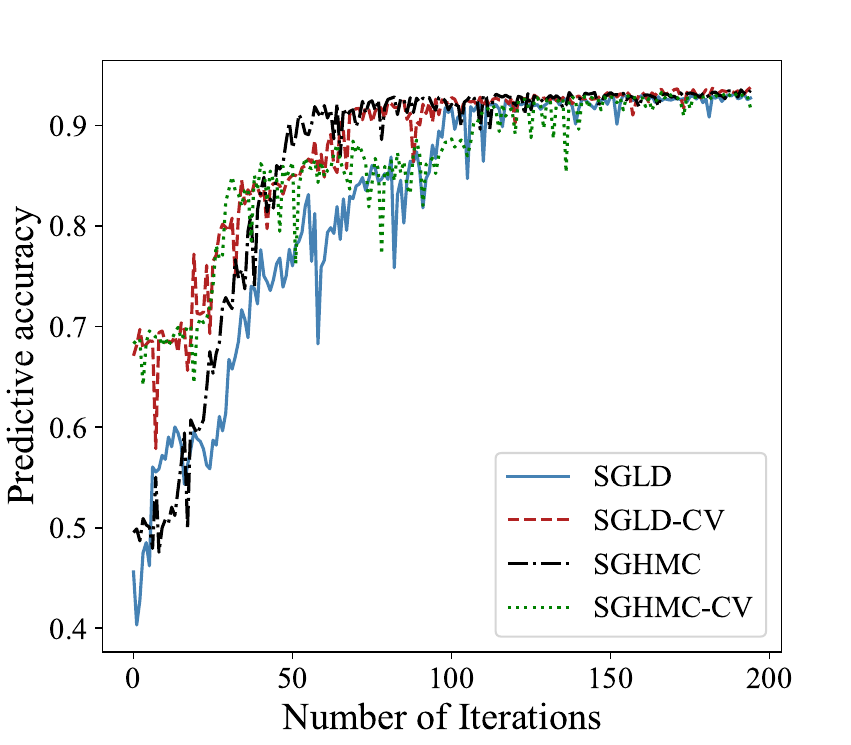}
    \caption{Average posterior predictive accuracy over all classes for each of the SGMCMC samplers.}
    \label{fig:ch3-bnn-predictive}
\end{figure}

 There is a separate set of $10000$ unseen images, also with corresponding labels but hidden from the training process. This test set allows us to evaluate how well our neural network performs on new data. By feeding these new images into the network, we can see if the posterior network is able to accurately classify the images into one of the ten digit categories (0-9). Figure \ref{fig:ch3-bnn-predictive} shows the posterior predictive accuracy for each SGMCMC\index{stochastic gradient Markov chain Monte Carlo} sampler over the number of Monte Carlo iterations (storing every $10$th posterior sample). The results show that all of the samplers converge to approximately $93\%$ accuracy in classifying the MNIST test set digits. The SGLD\index{stochastic gradient Langevin dynamics} and SGHMC\index{stochastic gradient Hamiltonian Monte Carlo} samplers converge at a similar rate (in terms of predictive accuracy). The control variate-based SGLD\index{stochastic gradient Langevin dynamics} and SGHMC\index{stochastic gradient Hamiltonian Monte Carlo} samplers also converge at a similar rate to each other but achieve higher predictive accuracy with fewer Monte Carlo iterations as these samplers are initialised at the posterior mode and thus remove the burn-in\index{burn-in} phase of the Monte Carlo sampler.

\section{Generalisations and Extensions}
\label{sec:ch3-extensions}

The SGMCMC\index{stochastic gradient Markov chain Monte Carlo} framework outlined in Section \ref{ch3:sgmcmc-general} can be extended beyond the SGLD\index{stochastic gradient Langevin dynamics} algorithm to improve Markov chain mixing\index{mixing}. However, two key assumptions limit the applicability of current SGMCMC\index{stochastic gradient Markov chain Monte Carlo} algorithms: (i) the parameters $\st \in \mathbb{R}^d$ exist in an unconstrained space, and (ii) the log-posterior density $\log\pi(\st)$ is a sum over conditionally independent terms.

Assumption (i) precludes estimating $\st$ on a constrained space (e.g. $\st \in [0,1]^d$). Assumption (ii) requires data $\by_1,\ldots,\by_N$ to be independent or have limited dependence, restricting the applicability of SGMCMC\index{stochastic gradient Markov chain Monte Carlo} for time series or spatial models. 

Ongoing research aims to relax these assumptions and expand SGMCMC\index{stochastic gradient Markov chain Monte Carlo} to broader model classes. Some promising directions include:

\begin{itemize}
    \item Transformation techniques to enable sampling on constrained parameter spaces \citep{brosse2017sampling, bubeck2018sampling, hsieh2018mirrored}. 
    \item Exploiting short-range dependencies and other model structures to allow subsampling for time series and network data \citep{li2016scalable, ma2017stochastic, aicher2023stochastic}.
    \item  Leveraging alternative stochastic processes with desired invariant distributions as samplers for models with complex data and parameter structures \citep{baker2018large}.
\end{itemize}

By developing specialised subsampling schemes and transformations, it is possible to make SGMCMC\index{stochastic gradient Markov chain Monte Carlo} algorithms more applicable to a wider range of Bayesian models while retaining computational efficiency.

\subsection{Scalable Inference for Models in Constrained Spaces}

Many statistical models contain parameters with inherent constraints, such as the variance parameter $\tau^2$ in a Gaussian distribution ($\tau^2 \in \mathbb{R}^+$) or the success probability $p$ in a Bernoulli model ($p \in [0,1]$). Simulating these constrained parameters using standard Langevin dynamics\index{Langevin dynamics} \eqref{eq:Euler} will often produce samples violating the constraints. For instance, if at iteration $k$ of the SGLD\index{stochastic gradient Langevin dynamics} algorithm $\st_k = \tau_k^2$ is close to zero, then with high probability, the next iterate $\st_{k+1}$ is likely to be negative, breaking the positivity constraint. 
One solution is to shrink the step size\index{step size} $\delta \rightarrow 0$ as $\tau^2 \rightarrow 0$, but this leads to poor mixing\index{mixing} near the boundary. 

A natural approach is to transform the Langevin dynamics\index{Langevin dynamics} so sampling occurs in an unconstrained space, but the choice of transformation greatly impacts mixing\index{mixing} near the boundary. Alternatively, we can project the dynamics into the constrained space; however, this yields poorer convergence compared to the unconstrained setting. The \emph{mirrored Langevin algorithm}\index{mirrored Langevin algorithm} \citep{hsieh2018mirrored} was proposed to address this issue. It builds on the mirror descent algorithm \citep{beck2003mirror} from the optimisation literature to transform constrained sampling into an unconstrained problem using a \textit{mirror map}. Compared to a generic mapping function, mirror maps have additional properties, such as strict convexity, differentiability and diverging gradients at the boundary of the domain, which makes mirror map algorithms well-suited to constrained sampling and optimisation problems.

If we assume that $\pi(\st)$ is the density of a constrained distribution, namely that $\log\pi(\st)$ has a bounded convex domain, then assuming that there exists a mirror map $\psi(\cdot)$ which is closed and proper, we can map the variable $\st \sim \pi$ from the constrained space (primal space) to the unconstrained space (dual space), where $\bvartheta:=\nabla \psi(\st)$ and $\bvartheta \sim \nu$.
Under the assumption that $\psi$ is a convex function that is closed, proper, and twice continuously differentiable, with Fenchel dual denoted by $\psi^*$, then Theorem 1 of \cite{hsieh2018mirrored} shows that  $\nabla\psi^*(\boldsymbol{\vartheta})\sim\pi$. This result implies that it is possible to transform the problem of sampling from a constrained distribution $\pi$ to the simpler problem of sampling from an unconstrained distribution $\nu$. 
Using the algorithms highlighted in this chapter, we can simulate a Markov chain $\boldsymbol{\vartheta} \sim \nu$ and apply the mapping $\nabla\psi^*(\boldsymbol{\vartheta})$ to produce samples from the desired posterior distribution\index{posterior distribution} $\pi$. In the case of the Langevin sampler, this is achieved by modifying the original Langevin diffusion \eqref{eq:LangevinSDE} to a \textit{mirrored Langevin diffusion},
\begin{align}
    \label{eq:mirrored-langevinSDE}
    \mathrm{d}\boldsymbol{\vartheta} &= \frac{1}{2}(\nabla\log\nu \circ \nabla\psi)(\st)\mathrm{d}t + \mathrm{d}W_t \\
    \st &= \nabla\psi^*(\boldsymbol{\vartheta}).
\end{align}

In practice, as noted earlier in this chapter, it is not possible to simulate exactly from the mirrored Langevin diffusion. Using the same Euler--Maruyama discretisation scheme,\index{Euler-Maruyama} it is possible to create a practical discrete-time algorithm. Stochastic gradient implementations of the mirrored Langevin dynamics are easily derived in the dual space and follow directly from the SGLD\index{stochastic gradient Langevin dynamics} algorithm. 

One popular model that requires sampling from a constrained domain is the latent Dirichlet allocation (LDA) model \citep{blei2003latent}, which is used for topic modelling. Here, the model parameters are constrained to the probability simplex, that is, $\theta_{i} \geq 0, i=1,\ldots,d$ and $\sum_{i=1}^d \theta_{i} \leq 1$. Mirrored Langevin dynamics\index{mirrored Langevin algorithm} \eqref{eq:mirrored-langevinSDE} can be used to sample from the simplex distribution by mapping the parameters to $\mathbb{R}^d$ and running the Langevin dynamics\index{Langevin dynamics} algorithm on the unconstrained space. The \emph{entropic} mirror map \citep{beck2003mirror} satisfies the required assumptions for a valid map function under the mirrored Langevin dynamics\index{mirrored Langevin algorithm},
\begin{equation}
\label{eq:entropic-map}
    \psi(\st)= \sum_{i=1}^d \theta_{i}\log\theta_{i}+ \left(1-\sum_{i=1}^d\theta_{i}\right)\log\left(1-\sum_{i=1}^d\theta_{i}\right),
\end{equation}
where $0\log0:=0$. The transformed log-posterior density is then given by 
\begin{equation}
\label{eq:dual-topic-model}
    \log\nu(\boldsymbol{\vartheta}) = (\log\pi \circ \nabla\psi^*)(\boldsymbol{\vartheta}) + \sum_{i=1}^d \vartheta_{i} -(d+1)\psi^*(\boldsymbol{\vartheta}),
\end{equation}
where $\psi^*(\boldsymbol{\vartheta})=\log(1+\sum_{i=1}^d \exp(\vartheta_{i}))$ is the Fenchel dual of $\psi(\cdot)$ and is strictly convex with Lipschitz gradients. 
Aside from transforming a constrained sampling problem into an unconstrained sampling problem, mirror maps can also lead to simpler posterior distributions in the dual space. For example, the Dirichlet posterior distribution\index{posterior distribution} introduced above leads to a posterior distribution\index{posterior distribution} on the dual space \eqref{eq:dual-topic-model} which is strictly log-concave.

Instead of using mirror maps to sample from the posterior distribution\index{posterior distribution} of the LDA model, we could instead use the stochastic gradient Riemannian Langevin dynamics\index{stochastic gradient Riemannian Langevin dynamics} algorithm (see Section \ref{ch3:sgmcmc-general}). Under the SGRLD\index{stochastic gradient Riemannian Langevin dynamics} algorithm, the constrained parameters $\st$ can be transformed to $\mathbb{R}^d$ via several alternative reparameterisations (see \cite{patterson2013stochastic} for a list). However, this can induce asymptotic biases dominating the boundary regions. An alternative approach is to recognise that the LDA posterior can be expressed as a transformation of independent gamma random variables. Therefore, rather than simulating from the simplex distribution via the Langevin diffusion, one could instead utilise the Cox--Ingersoll--Ross (CIR) process \citep{cox1985theory}, which is invariant with respect to the gamma distribution. This CIR-based approach \citep{baker2018large} avoids boundary biases. More broadly, leveraging alternative stochastic processes with desired invariant distributions can enable specialised samplers for models with complex structures.

\subsection{Scalable Inference with Time Series Data} 

A key requirement for developing stochastic gradient algorithms is the ability to generate unbiased estimates of $\nabla \log\pi(\st)$ using data subsampling, as in \eqref{eq:U-hat-simple}. This is straightforward when the log-posterior density and its gradient are expressed as a sum of $\log \pi_j(\btheta)$ and $\nabla \log \pi_j(\btheta)$ terms, respectively, i.e.  $\log\pi(\st)=\sum_{j=1}^N \log \pi_j(\btheta).$
Randomly subsampling these terms provides an unbiased log density and gradient estimate.

However, for models where data are not conditionally independent, for example, network data, time series, or spatial data, the log-posterior density cannot be expressed as a simple sum. Naively subsampling will yield biased estimates of $\log\pi(\st)$ and $\nabla \log\pi(\st)$. Capturing both short- and long-range dependencies in spatial data with subsampling remains an open challenge for SGMCMC\index{stochastic gradient Markov chain Monte Carlo}. For network data, \citet{li2016scalable} developed an SGMCMC\index{stochastic gradient Markov chain Monte Carlo} algorithm for the mixed-membership stochastic block model using block structure and stratified subsampling to obtain unbiased gradient estimates. 

Recent work in SGMCMC\index{stochastic gradient Markov chain Monte Carlo} for temporally correlated data has focused on hidden Markov models\index{Hidden Markov model} \citep{ma2017stochastic} with finite states, linear dynamical systems \citep{aicher2019stochastic} and general nonlinear hidden Markov models\index{Hidden Markov model} \citep{aicher2023stochastic}. Under this modelling framework, the hidden Markov model\index{Hidden Markov model} consists of two stochastic processes: (i) a latent state process $\{\bX_t\}_{t=0}^T$, which is a Markov chain that evolves over time $t=1,\ldots,T$ from an initial state $\bX_0$ whose distribution does not depend on $\btheta$, with $\bX_t$ depending only on $\bX_{t-1}$ and $\btheta$, with transition density given by $p(\bx_t | \bx_{t-1}, \st);$ and (ii) an observed process $\{\bY_t\}_{t=1}^T$ that is conditionally independent given the latent states and $\btheta$, which are observed with probability density  $p(\by_t | \bx_t, \st)$. Assuming model parameters $\st$, the full generative model (Figure \ref{fig:hmm-model}) is 
\begin{align}
    \label{eq:state-space}
    \bX_t | (\bX_{t-1} = \bx_{t-1}, \st) &\sim p(\bx_t | \bx_{t-1}, \st) \\
    \bY_t | (\bX_t = \bx_t, \st) &\sim p(\by_t | \bx_t, \st).
\end{align}
The latent Markov chain $\bX_t$ captures the dynamics and temporal dependence, while the observations $\bY_t$ depend only on the current state of the latent process. Hidden Markov models\index{Hidden Markov model} are useful for modelling complex time series data by augmenting the observables with latent states. It is common to distinguish between hidden Markov models\index{Hidden Markov model} and state-space models, where in the case of the former, the latent process $\bX_t$ is discrete and is continuous in the case of the latter. However, for the sake of convenience, we shall use the term \textit{hidden Markov model}\index{Hidden Markov model} to cover both model types. 

Using SGMCMC\index{stochastic gradient Markov chain Monte Carlo} methods, it is possible to estimate the model parameters $\st$ for general hidden Markov models\index{Hidden Markov model} which exhibit temporal dependency in the observations. As before, the goal is to sample from the posterior distribution\index{posterior distribution} $\pi(\st):=p(\st|\by)$ where $\by = \{\by_1, \ldots, \by_T\}$ is the observed data sequence, which is proportional to the product of the likelihood $p(\by|\st)$ and the prior $\pi_0(\st)$. The likelihood $p(\by|\st)$ typically cannot be evaluated exactly, as it requires summing (discrete setting) or integrating (continuous setting) over the latent states $\bX$. Focusing on the continuous variable setting, the latent states can be integrated out numerically using particle filtering techniques \citep{doucet2009tutorial}. Using a particle approximation of Fisher's identity \citep{nemeth2016particle},
\begin{align}
    \label{eq:fisher-identity}
    \nabla_{\st} \log p(\by_{1:T}|\st) &= \Expects{\bX|\bY,\st}{\nabla_{\st}\log p(\bX_{0:T},\by_{1:T}|\st)} \\
    &= \sum_{t=1}^T \Expects{\bX|\bY,\st}{\nabla_{\st}\log p(\bX_t,\by_t|\bX_{t-1},\st)} \nonumber
\end{align}
it is possible to approximate $\nabla_{\st}\log p(\by_{1:T}|\st)$, and hence $\nabla_{\st}\log\pi(\st)=\nabla_{\st}\log p(\by_{1:T}|\st)+\nabla_{\st}\log\pi_0(\st)$, by replacing the posterior distribution\index{posterior distribution} of the latent states $p(\bx_{0:T}|\by_{1:T},\st)$ with a numerical approximation represented by a set of $P$ particles $\{\bX_{0:T}^{(p)}\}_{p=1}^P$.

\begin{figure}
\centering
\begin{tikzpicture}[auto=false, node distance=1.7cm]
\tikzstyle{latent} = [circle, draw, fill=white, text width=0.6cm, text centered];
\tikzstyle{observed} = [rectangle, draw, fill=white, text width=0.6cm, text centered];
\tikzstyle{arrow} = [->, >=stealth];

\node[latent] (X1) {$\bX_0$};
\node[latent, right of=X1] (X2) {$\bX_{t-2}$};
\node[latent, right of=X2] (X3) {$\bX_{t-1}$};
\node[latent, right of=X3] (X4) {$\bX_t$};
\node[latent, right of=X4] (X5) {$\bX_{t+1}$};
\node[latent, right of=X5] (X6) {$\bX_{t+2}$};
\node[latent, right of=X6] (XT) {$\bX_T$};

\node[observed, below of=X2] (O2) {$\by_{t-2}$};
\node[observed, below of=X3] (O3) {$\by_{t-1}$};
\node[observed, below of=X4] (O4) {$\by_t$};
\node[observed, below of=X5, label=south east:$\mathcal{S}$] (O5) {$\by_{t+1}$};
\node[observed, below of=X6, label=south east:$\mathcal{S}^*$] (O6) {$\by_{t+2}$};
\node[observed, below of=XT] (OT) {$\by_T$};

\draw[arrow,dotted] (X1) -- (X2);
\draw[arrow] (X2) -- (X3);
\draw[arrow] (X3) -- (X4);
\draw[arrow] (X4) -- (X5);
\draw[arrow] (X5) -- (X6);
\draw[arrow,dotted] (X6) -- (XT);

\draw[arrow] (X2) -- (O2);
\draw[arrow] (X3) -- (O3);
\draw[arrow] (X4) -- (O4);
\draw[arrow] (X5) -- (O5);
\draw[arrow] (X6) -- (O6);
\draw[arrow] (XT) -- (OT);

\draw[thick] ($(X3.north west)+(-0.4,0.4)$)  rectangle ($(O5.south east)+(0.4,-0.4)$);
\draw[thick,dotted] ($(X2.north west)+(-0.5,0.5)$)  rectangle ($(O6.south east)+(0.5,-0.5)$);

\end{tikzpicture}
\caption{Graphical representation of the hidden Markov model with latent variables $\bX_t$ and observations $\by_t$. The subsequence is captured in the solid box $\mathcal{S}$ and the buffer region is highlighted by the dotted box $\mathcal{S}^*$.}
\label{fig:hmm-model}
\end{figure}
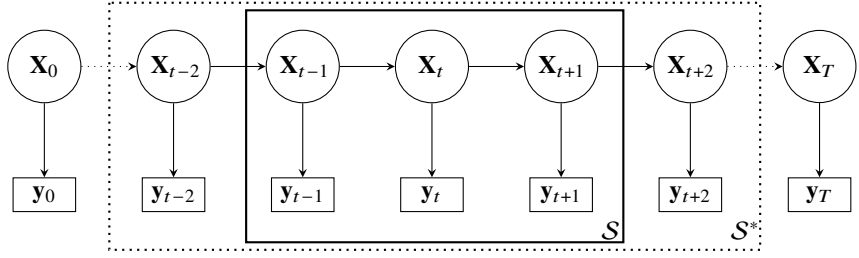

Calculating Fisher's identity can be computationally expensive for large $T$ and so an SGMCMC\index{stochastic gradient Markov chain Monte Carlo} approximation to Fisher's identity can be used to replace the full data gradient \eqref{eq:fisher-identity} with a stochastic approximation estimated from a random subset of the data. This allows each gradient evaluation to be cheaper, enabling more MCMC\index{Markov chain Monte Carlo} iterations and better convergence for large $T$. However, naively subsampling the data randomly \eqref{eq:U-hat-simple} induces bias since it breaks temporal dependencies. To address this issue, \cite{aicher2019stochastic,aicher2023stochastic} propose to subsample\index{subsample} contiguous subsequences. Figure \ref{fig:hmm-model} provides a graphical representation of the hidden Markov model\index{Hidden Markov model} which also highlights the subsampled data as a contiguous subsequence $\mathcal{S}$ of size $m$. The stochastic gradient\index{stochastic gradient} following Fisher's identity \eqref{eq:fisher-identity} is then 
\begin{equation}
    \label{eq:fisher-stoch}
     \nabla^{(m)}_{\st} \log p(\by_{1:T}|\st) = \sum_{t \in \mathcal{S}} \mathrm{Pr}(t\in\mathcal{S})^{-1} \cdot \Expects{\bX|\by_{\mathcal{S}},\st}{\nabla_{\st}\log p(\bX_t,\by_t|\bX_{t-1},\st)}.
\end{equation}
However, the stochastic gradient\index{stochastic gradient} estimator 
\eqref{eq:fisher-stoch} is biased because the expectation is taken over the latent state which is dependent only on $\by_{\mathcal{S}}$ and not $\by_{1:T}$. This bias can be reduced by extending the subsequence $\mathcal{S}$ to include a \textit{buffered} region $\mathcal{S}^*$. The expectation in \eqref{eq:fisher-stoch} is then taken over the wider buffered range $p(\bx_{t-1},\bx_t|\by_{\mathcal{S}^*},\st)$ rather than $p(\bx_{t-1},\bx_t|\by_\mathcal{S},\st)$. Results from \cite{aicher2019stochastic} show that under Lipschitz assumptions on the model \eqref{eq:state-space} and its gradients, the error in the stochastic gradients decays geometrically with the length of the buffer on either side of $\mathcal{S}$.

\section{Chapter Notes}
\label{sec:ch3-chapter-notes}

 MCMC algorithms have been the cornerstone of Bayesian inference since the 1990s. However, as the size of datasets has grown, the requirement that each MCMC\index{Markov chain Monte Carlo} update uses all of the data to approximate the posterior distribution\index{posterior distribution} has created a computational challenge. There has been significant recent interest in the Statistics and Machine Learning communities to create new Monte Carlo-based algorithms for scalable Bayesian inference in the presence of large datasets \citep{bardenet2014towards}. Broadly speaking, scalable MCMC\index{Markov chain Monte Carlo} algorithms tend to fall into two categories: either (i) subsample the data \citep{Welling:2011,chen2014stochastic,nemeth2021stochastic}, as covered in this chapter, or (ii) use parallel computing to distribute the computational cost across multiple CPUs \citep{scott2016bayes,nemeth2018merging,vyner2023swiss}. 

In the context of data subsampling, the per-iteration cost of the Monte Carlo algorithm is reduced. However, this reduction in computational cost is only beneficial if the subsampling scheme leads to posterior approximations with high statistical accuracy. \cite{Chatterji:2018} give results on the number of iterations and resulting computational cost required for different stochastic gradient algorithms to produce samples from a distribution that is within a specified `distance' of $\pi$. Other works \citep{Bierkens:2019,Huggins:2016,Pollock:2016} have studied the computational and statistical trade-offs that result from approximate and scalable MCMC\index{Markov chain Monte Carlo} schemes, which are often referred to as \textit{exact-approximate algorithms}. It is often the case with scalable MCMC\index{Markov chain Monte Carlo} algorithms that \textit{there is no free lunch} and simple naive subsampling alone does not lead to statistically efficient algorithms; see for example, \cite{johndrow2020no}.  In the case of control-variates for subsampling, a number of theoretical results \cite[e.g.][]{Nagapetyan:2017,Baker:2017,Brosse:2018} show that if we ignore the pre-processing cost of finding $\widehat{\st}$, the computational cost per-effective sample of SGLD with control variates is $O(1)$, rather than the $O(N)$ cost for SGLD\index{stochastic gradient Langevin dynamics} with the simple gradient estimator \eqref{eq:U-hat-simple}.


\chapter{Non-Reversible MCMC}
\label{chap:non-reversible}

A reversible Markov chain is any Markov chain that satisfies detailed balance\index{detailed balance}; see Section \ref{sec:ch1-MC}. Remember, this condition states that, at stationarity, the probability of the chain starting in a set $\cB$ and moving to set $\cC$ is equal to the probability of it starting in the set $\cC$ and moving to $\cB$. This means that the dynamics of the process are the same backwards in time as forwards in time. One consequence of reversibility\index{reversibility} is that the Markov chain can exhibit random-walk behaviour, where it can return to states or regions of the state space where it has recently been. 

A non-reversible Markov chain is any Markov chain that does not satisfy detailed balance. As we will see, the potential benefit of non-reversibility is that the Markov chain can more quickly explore the state space as it can suppress the random walk behaviour. However, designing non-reversible Markov chains with the required stationary distribution, or even determining the stationary distribution of a non-reversible Markov chain, is much more challenging than for a reversible chain. This chapter will describe one approach to designing non-reversible MCMC samplers based on the idea of {\em lifting} -- which involves taking a reversible MCMC scheme and then lifting it to a higher-dimensional state-space to enable the use of non-reversible moves. These ideas naturally motivate the non-reversible continuous-time MCMC samplers of Chapter \ref{chap:continuous-time}.

\section{The Benefits of Non-Reversibility}
\label{sec:ch5:benefits}

To see the benefits of non-reversible Markov chains\index{non-reversible MCMC|(}\index{Markov chain!discrete-time}, we will consider the following simple example.

\begin{example}
\label{example.rw.on.ring}
Let $X_0=0$ and 
\[
X_k= (X_{k-1}+J_k) (\mbox{mod } S),
\]
so that $X_k$ follows a random walk on $\{0,\ldots,S-1\}$. Here
$J_k$ is the jump size, which can have any distribution on $\{-h,\dots,h\}$ that is symmetric about $0$, and $(\mbox{mod } S)$ means we take the remainder after dividing by $S$. We include $(\mbox{mod } S)$, so a random walk that moves to negative values, or values equal to or above $S$, gets mapped back to $\{0,\ldots,S-1\}$, with $S$ mapping to 0 and $-1$ to $S-1$ and so on. For simplicity, we consider $J_k$ to have a uniform distribution on $\{-h,-h+1,\ldots,h\}$. It is straightforward to show that the resulting Markov chain is reversible and has the uniform distribution on $\{0,\ldots,S-1\}$ as its stationary distribution\index{stationary distribution}.
\end{example}

First, consider $h=1$ and look at the behaviour of the chain as we increase $S$. In the top row of Figure \ref{fig:ch5_nonrev_trace}, we show trace plots of the chain for $S=100$, $S=200$ and $S=400$. In each case, we show the path of the chain over $40S$ iterations. What we observe is that as $S$ increases, the chain becomes much slower at exploring the state-space. We can see this more clearly if we look at the empirical marginal distribution of the chain after $n$ time steps. Let $b$ be a chosen burn-in\index{burn-in} period, then for $n>b$ the empirical marginal distribution, which is our natural estimator of the stationary distribution of the chain, is
\begin{equation} \label{eq:ch5-empiricalmarginaldistribution}
\widehat{\pi}_n(i) = \frac{1}{n-b} \sum_{k=b+1}^n \Indicator{X_k=i}.
\end{equation}
This is just the proportion of time, after burn-in, that the chain was in state $i$. We can then compare this estimate with the true stationary distribution by calculating the total variation distance between the two. This is just $\sum_{i=0}^{S-1}|\widehat{\pi}_n(i)-1/S|$, and is shown in the top-row of Figure \ref{fig:ch5_TVD_nonrev}, where we show the total variation distance against $n/S$ and against $n/S^2$ for $S=100$, $S=200$ and $S=400$. We see evidence that, as we increase $S$, the time we need to run our Markov chain must increase proportionally with $S^2$ to have the same degree of accuracy.

It is interesting to compare this with the performance of the following non-reversible Markov chain\index{Markov chain!discrete-time}, which we construct by making the random walk biased; that is,  choosing a jump distribution $J_k$ whose expectation is non-zero. 
\begin{example}
\label{example.brw.on.ring}
In Example \ref{example.rw.on.ring}, keep $J_k$ taking values in $\{-1,0,1\}$, but set the jump probabilities to be $\{2/9,1/3,4/9\}$ respectively so that a positive jump is twice as likely as a negative one. The stationary distribution is still uniform  on $\{0,1,\dots,S-1\}$.
\end{example}

We show the resulting trace plots of the Markov chain in the bottom row of Figure \ref{fig:ch5_nonrev_trace}. Qualitatively, the trace plots look very different to those in the top row, as the paths tend to move upwards at each iteration. This is linked to the non-reversibility of the chain, as a realisation of the chain forward in time will now look very different from a realisation backwards in time. Furthermore, the realisations for different $S$ look similar. That is, once we scale the number of iterations by $S$, chains with different $S$ mix similarly. This is shown quantitatively in the bottom left plot of Figure \ref{fig:ch5_TVD_nonrev}, where we plot the total variation distance of the empirical marginal distribution of our chain (\ref{eq:ch5-empiricalmarginaldistribution}) from the uniform distribution against $n/S$: this is almost identical for the three values of $S$.

\begin{figure}
    \centering
    \includegraphics[width=\textwidth]{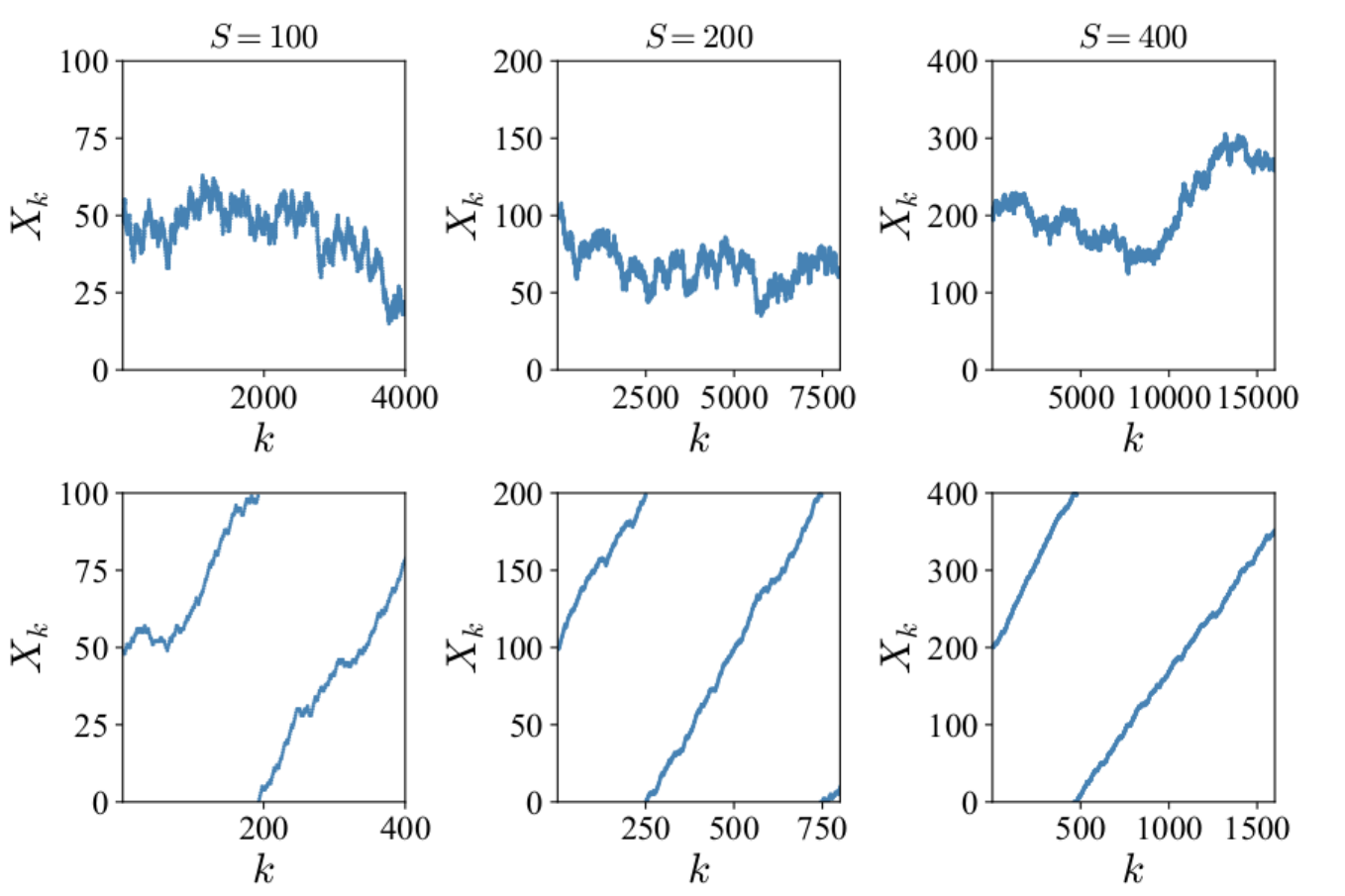}
    \caption{Trace plots of two MCMC algorithms for sampling from a uniform distribution on $\{0,\ldots,S-1\}$ for different values of $S$. A random walk of step size uniform on $\{-1,0,1\}$ (Example \ref{example.rw.on.ring}, top), and a non-reversible (biased) random walk which is twice as likely to have a step of 1 than a step of -1 (Example \ref{example.brw.on.ring}, bottom). Columns are for $S=100$ (left), $S=200$ (middle) and $S=400$ (right). We show the state of the MCMC algorithm for $40S$ iterations (top row) or $4S$ iterations (bottom row). For this scaling of iterations, we see the reversible MCMC algorithm mixes more slowly as $S$ increases, whereas qualitatively the mixing of the non-reversible algorithm remains similar.}
    \label{fig:ch5_nonrev_trace}
\end{figure}

Why does the non-reversible chain have better mixing\index{mixing} properties? Intuitively,  the poor performance of the reversible chain is because it has random-walk behaviour: it will often move up on one iteration and then move down on the next. The non-reversible chain suppresses this random-walk behaviour as its bias means it will tend to move in the same direction. The qualitative difference between these reversible and non-reversible chains that we have demonstrated empirically, can be shown theoretically \citep{diaconis2000analysis}. 

 For a simple heuristic of this behaviour, consider $X_0=\lfloor S/2\rfloor$ and $n\le \lfloor S/2 \rfloor$. For the symmetric random walk, $\Expect{X_n-X_0}=0$ and, since the total movement by iteration $n$ is a sum of $n$ independent moves, $\Var{X_n-X_0}\propto n$, so the typical amount of movement in the first $n$ iterations is proportional to $\sqrt{n}$. However, for the biased random walk, $\Expect{X_n-X_0}\propto n$ and so the amount of movement is roughly proportional to $n$.

Importantly, when we measured the performance of the non-reversible Markov chain we looked at the accuracy of \eqref{eq:ch5-empiricalmarginaldistribution}, which is the proportion of time it spends in each state averaged over time. If, instead, we look at the convergence of $\Prob{X_n=i}$ to $1/S$, we would obtain a very different result to that shown in Figure \ref{fig:ch5_TVD_nonrev}, as the bias of the random walk in Example \ref{example.brw.on.ring} means that the centre of the distribution of $X_n$ changes with $n$: the benefit of the non-reversible chain is only realised as we take the ergodic average over different time-points. This is most easily seen for the extreme case where $J_k=1$ with probability 1. In that case, we find that the distribution of $X_n$ is a point mass at a single value for each $n$, but by averaging over time, we still have that \eqref{eq:ch5-empiricalmarginaldistribution} converges to the uniform distribution at a rate of $1/n$.

Finally, as an aside, we observe that the poor performance of the reversible chain is also linked to the fact that the size of the moves at each iteration is small -- this can be quantified in terms of the variance of $X_k-X_{k-1}$ relative to the variance of the stationary distribution. And it is the fact that this ratio increased as we increased $S$ that meant that the reversible chain performed relatively poorly for larger $S$. To see this, we can implement the reversible random walk, but set the maximum step size\index{step size} $h=S/100$ so it is proportional to $S$. The total variation distance between (\ref{eq:ch5-empiricalmarginaldistribution}) and the uniform distribution for such a chain is shown in the bottom right plot of Figure \ref{fig:ch5_TVD_nonrev}, and demonstrates better scaling\index{scaling} with $S$: in fact, like the non-reversible chain, as $S$ increases, we now obtain the same accuracy providing we scale $n$ to be proportional to $S$.

\begin{figure}
    \centering
    \includegraphics[width=\textwidth]{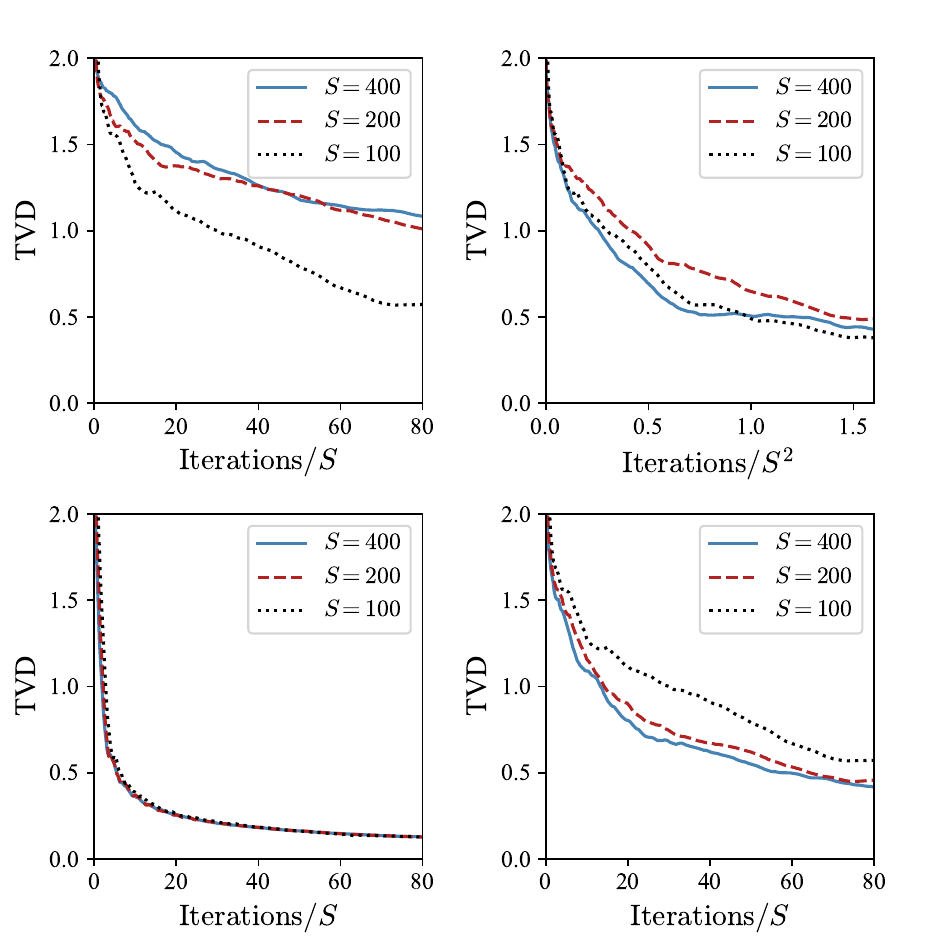}
    \caption{Empirical total variation distance (TVD) between  (\ref{eq:ch5-empiricalmarginaldistribution}) and the uniform distribution against the number of iterations and different values of $S$: $S=100$ (blue), $S=200$ (red) and $S=400$ (black). TVD for the random walk of Example \ref{example.rw.on.ring} with step size $1$ (top row) with $x$-axis scaled by $S$ (top left) or $S^2$ (top right). TVD for the non-reversible Example \ref{example.brw.on.ring} (bottom left), and for Example \ref{example.rw.on.ring} with step size scaled by $S$ (bottom right); in both cases the $x$-axis is scaled by $S$. All results are averaged over 10 realisations of the chains.  
    }
    \label{fig:ch5_TVD_nonrev}
\end{figure}

\index{non-reversible MCMC|)}
\section{Hamiltonian Monte Carlo Revisited}
\index{Hamiltonian Monte Carlo:(}
The Hamiltonian Monte Carlo, or HMC, algorithm of Section \ref{sec:ch2-HMC} is often viewed as a non-reversible MCMC algorithm. However, strictly it is a reversible algorithm. 

Remember that the HMC algorithm for sampling from a density $\pi(\btheta)$, involves augmenting the state of our MCMC\index{MCMC|see{Markov chain Monte Carlo}} algorithm with a momentum variable of the same dimension as $\btheta$. Denote the momentum variable by $\bp$, and the state of our Markov chain by $(\btheta,\bp)$. In the following, for simplicity, we will assume that the mass matrix\index{mass matrix} is the identity. We introduce a target density for $(\btheta,\bp)$,
\[
\tilde{\pi}(\btheta,\bp) \propto \pi(\btheta)\exp\left\{-\frac{1}{2}\bp^\top\bp \right\},
\]
which has independent components, with $\btheta$ from $\pi$ and $\bp$ having a standard Gaussian distribution. An HMC algorithm that has $\tilde{\pi}(\btheta,\bp)$ as its stationary density interleaves the following three moves at each iteration. Given current state $(\btheta_k,\bp_k)$ \index{Markov chain Monte Carlo}
\begin{itemize}
    \item[(1)] Simulate a new momentum $\bp$ from a standard Gaussian distribution.
    \item[(2)] \begin{itemize}
        \item[(i)] Obtain the proposed state $(\btheta',\bp')$ by running the leapfrog\index{leapfrog dynamics} dynamics, for some number, $L$, of steps with some step length, $\epsilon$, starting from $(\btheta_k,\bp)$, and flip the final momentum. 
        \item[(ii)] Accept the proposed state, $(\btheta_{k+1},\bp'')=(\btheta',\bp')$ with probability
        \[
        \min\left(1, \frac{\tilde{\pi}(\btheta',\bp')}{\tilde{\pi}(\btheta_k,\bp)} \right),
        \]
        otherwise $(\btheta_{k+1},\bp'')=(\btheta_k,\bp)$.
         \end{itemize}
    \item[(3)] Flip the momentum, $(\btheta_{k+1},\bp_{k+1})=(\btheta_k,-\bp'')$.
\end{itemize}
In the notation of Section \ref{sec:ch2-HMC}, the proposal in step (2i) is denoted by
$\Leap^L_-(\btheta_k,\bp)$. Move (2)  involves simulating non-reversible Hamiltonian dynamics to produce a potentially large proposed move, but the move itself is reversible. As, trivially, are moves (1) and (3). The move (3) has no net effect on the algorithm since the momentum is discarded at the next iteration, and it is not included in the description in Section \ref{sec:ch2-HMC}. It does, however, ensure that, if the proposal is accepted, the final momentum is the same as that which was obtained through the leapfrog approximation to the Hamiltonian dynamics rather than its opposite. This will be helpful in the next section.

Interleaving three different reversible moves does not necessarily result in a reversible Markov chain (see the next section, for example). But in this case, it is straightforward to show that the marginal process for $\btheta$ is a reversible Markov chain.  The benefit of HMC is in the use of the non-reversible deterministic leap-frog dynamics to produce large proposed moves that have a high chance of being accepted. As we saw in the previous section, the issues with reversible MCMC algorithms occur when the step size is small -- and HMC gets around this by being able to propose large moves rather than being non-reversible. This fact is seen in the scaling results for HMC as the dimension of the state-space increases (see Section \ref{sec:ch2-HMC} and the literature in Section \ref{sec.Ch2.ChapNotes}). 
\index{Hamiltonian Monte Carlo|)}

\section{Lifting Schemes for MCMC} \label{sec:ch5_lifting}
\index{lifting schemes|(}
Whilst HMC is a reversible algorithm when viewed as a Markov chain on $\btheta$, some of the ideas behind the Hamiltonian dynamics are common to non-reversible MCMC algorithms. Furthermore, it is possible to adapt the HMC algorithm so that it is non-reversible.

Most non-reversible MCMC algorithms involve the idea of ``lifting'', that is, defining the Markov chain on a higher-dimensional state space than you wish to sample from. Moreover, they tend to do this in a way similar to HMC, in that if you wish to sample from some target distribution, $\pi(\btheta)$, you work with a Markov chain with state  $(\btheta,\bp)$, where $\bp$ is of the same dimension, $d$, as $\btheta$ and can be interpreted as the momentum or velocity that is describing the direction and speed with which the Markov chain is currently moving through $\btheta$ space. The non-reversibility of the algorithm is based on trying to encourage Markov chain moves that continue in roughly the same direction over successive iterations. 

\subsection{Non-Reversible HMC}
\index{Hamiltonian Monte Carlo|(}
\index{non-reversible MCMC} 
One of the first truly non-reversible algorithms is an extension of HMC due to \cite{horowitz1991generalized}. The idea is to adapt how the momentum is refreshed at each iteration so that the momentum at the current iteration will be similar to that at the previous iteration. This can be achieved by replacing step (1) of the HMC algorithm with the update
\[
\bp'=\gamma \bp + (1-\gamma^2)^{1/2} \bzeta,
\]
where $\bzeta$ is a realisation of a $d$-dimensional standard Gaussian random variable, and $0<\gamma<1$. If $\bp$ has a standard Gaussian distribution, then so does $\bp'$: it is Gaussian as it is a linear combination of two independent Gaussian random variables, the expectation of the right-hand side is 0, and the variance is $\gamma^2\bI_d + (1-\gamma^2)\bI_d=\bI_d$. If $\gamma$ is close to 1 then $\bp'$ will be close to $\bp$. 
The overall algorithm is given in Algorithm \ref{alg:horowitz}.
\begin{algorithm}
    \caption{Non-reversible HMC of \cite{horowitz1991generalized}}
    \KwIn{Density $\pi(\btheta)$, initial value $\btheta_0$ and momentum $\bp_0$ sampled from $\Normal(\bzero,\bI_d)$, and refresh rate $\gamma\in[0,1)$.}
    \For{$k \in 0, \dots, n-1$} {
        Simulate $\bzeta\sim\Normal(\bzero,\bI_d)$ and set $\bp'=\gamma \bp_k + (1-\gamma^2)^{1/2} \bzeta$.\\
        $(\btheta',\bp'')\gets\Leap^L_-(\btheta_k,\bp')$. \\
        Calculate the acceptance probability\index{acceptance probability}:
    \[\alpha(\btheta_k,\bp';\btheta',\bp'')
    :=\min\left(1,\frac{\pitilde(\btheta',\bp'')}{\pitilde(\btheta_k,\bp')}\right).\]\\
    With a probability of $\alpha(\btheta_k,\bp';\btheta',\bp'')$ accept the proposal\index{proposal}, $(\btheta_{k+1},\bp''')\gets(\btheta',\bp'')$; otherwise reject it, $(\btheta_{k+1},\bp''')\gets (\btheta_k,\bp')$. \\
    Flip the velocity: $(\btheta_{k+1},\bp_{k+1})\gets (\btheta_{k+1},-\bp''')$.
        }
    \label{alg:horowitz}
\end{algorithm}

This algorithm has the required stationary distribution\index{stationary distribution}, as each step satisfies detailed balance\index{detailed balance} with respect to the extended posterior\index{extended posterior distribution}, $\pitilde$. However, whilst each step of the algorithm is a reversible Markov chain move, by interleaving the moves, we obtain a non-reversible algorithm. 

Whilst this algorithm generalises HMC, in practice, it is rarely noticeably more efficient \cite[see for example][]{Neal:2011}. The reason is the momentum flip that occurs if we reject our proposal, as, if $\gamma$ is large this sends the chain back in the direction from whence it came. Thus, we need to either choose the leapfrog step size\index{step size}, $\epsilon$, to be small enough that the probability of acceptance after $L$ steps is usually very close to $1$, or we need to choose $\gamma$ to be small so that the new momentum is relatively unrelated to the old momentum. The first comes at a high computational cost while the second leads to an algorithm that is very similar to standard HMC (which corresponds to $\gamma=0$). Later, in Section \ref{sec.DBPS}, we will present some alternative ideas that can alleviate the problems of momentum flips in algorithms such as this.
\index{Hamiltonian Monte Carlo|)}

\subsection{Gustafson's Algorithm and Multi-dimensional Generalisations}
\index{Gustafson's algorithm|see{guided random walk}}
\index{non-reversible MCMC}

It is possible to use similar ideas to obtain non-reversible versions of a random walk Metropolis--Hastings\index{Metropolis--Hastings}\index{random walk Metropolis} algorithm, known as \emph{guided random walks}\index{guided random walk|(}. This was first proposed by \cite{gustafson1998guided} for component-wise updates\index{component-wise updates}, though it has a natural extension to multivariate updates, which we will also describe. First, consider the univariate case and a Gaussian random-walk proposal\index{proposal} with a variance of $\delta^2$ and $\zeta\sim \Normal(0,\delta^2)$. Given a current value, $\theta$, we can write this proposal as
\[
\theta'= \theta + p |\zeta|,
\]
where $p$ is uniformly sampled from $\{-1,1\}$. We can think of $p$ as the direction of the move and $|\zeta|$ as the size of the move. The idea of \cite{gustafson1998guided} is to augment the state of the chain with $p$ and to simulate a chain such that the direction of the move is likely to be the same over successive time steps. This can be achieved by iterating the following two steps:
\begin{itemize}
    \item[(1)] Simulate $\zeta$, a realisation of a Gaussian random variable, and set $(\theta',p')=(\theta_k+p_k|\zeta|,-p_k)$ with probability
    \[
    \min\left\{1, \frac{\pi(\theta')}{\pi(\theta_k)} \right\}
    \]
    otherwise set $(\theta',p')=(\theta_k,p_k)$.
    \item[(2)] Flip the direction, so $(\theta_{k+1},p_{k+1})=(\theta',-p')$.
\end{itemize}
The stationary distribution\index{stationary distribution} of this algorithm has independent distributions for $\theta$ and $p$, with $\theta$ from the distribution whose density is proportional to $\pi(\theta)$, and $p$ having a uniform distribution on $\{-1,1\}$. This follows as both steps (1) and (2) are reversible moves that keep such a distribution invariant.

\begin{figure}
    \centering
    \includegraphics[width=\textwidth]{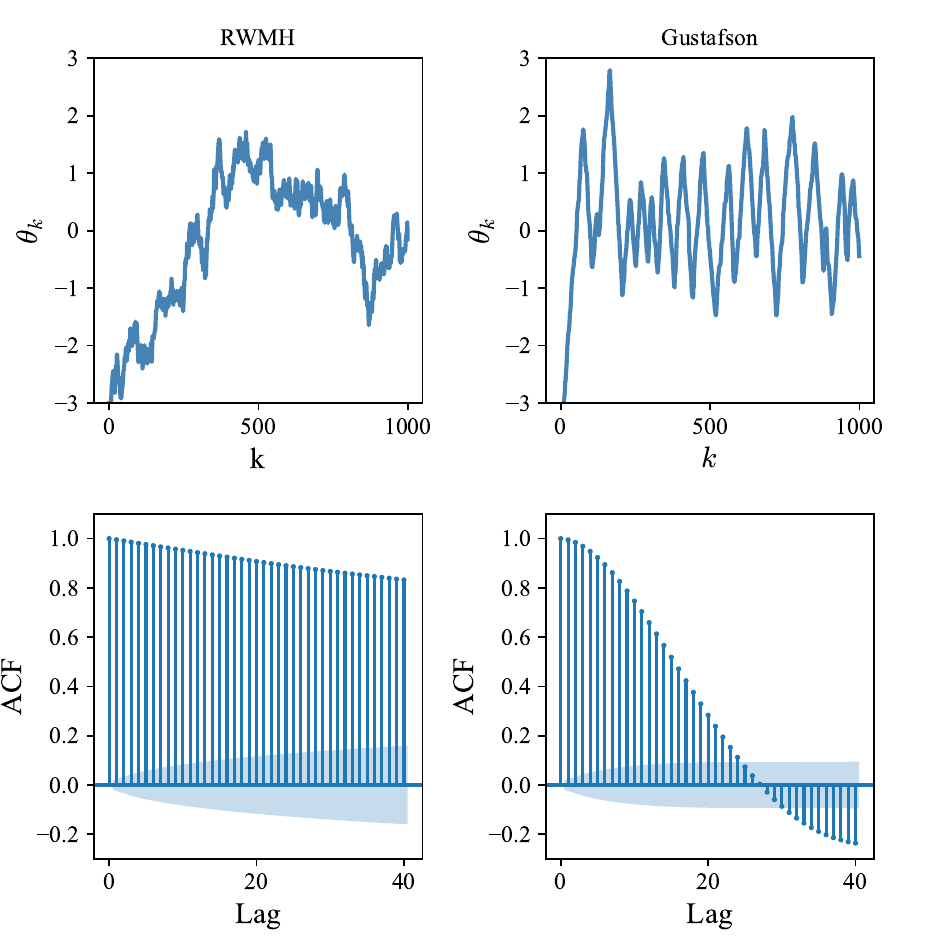}
    \caption{Comparison of random walk Metropolis--Hastings (left column) and the guided random walk (right column) when sampling from a 1-d Gaussian. The proposal is Gaussian with a standard deviation of 0.1 in both cases. Top row shows trace plots, and bottom row shows the estimated auto-correlation of the chains.
    }
    \label{fig:ch5_1dcomparison_small}
\end{figure}
\begin{figure}
    \centering
    \includegraphics[width=\textwidth]{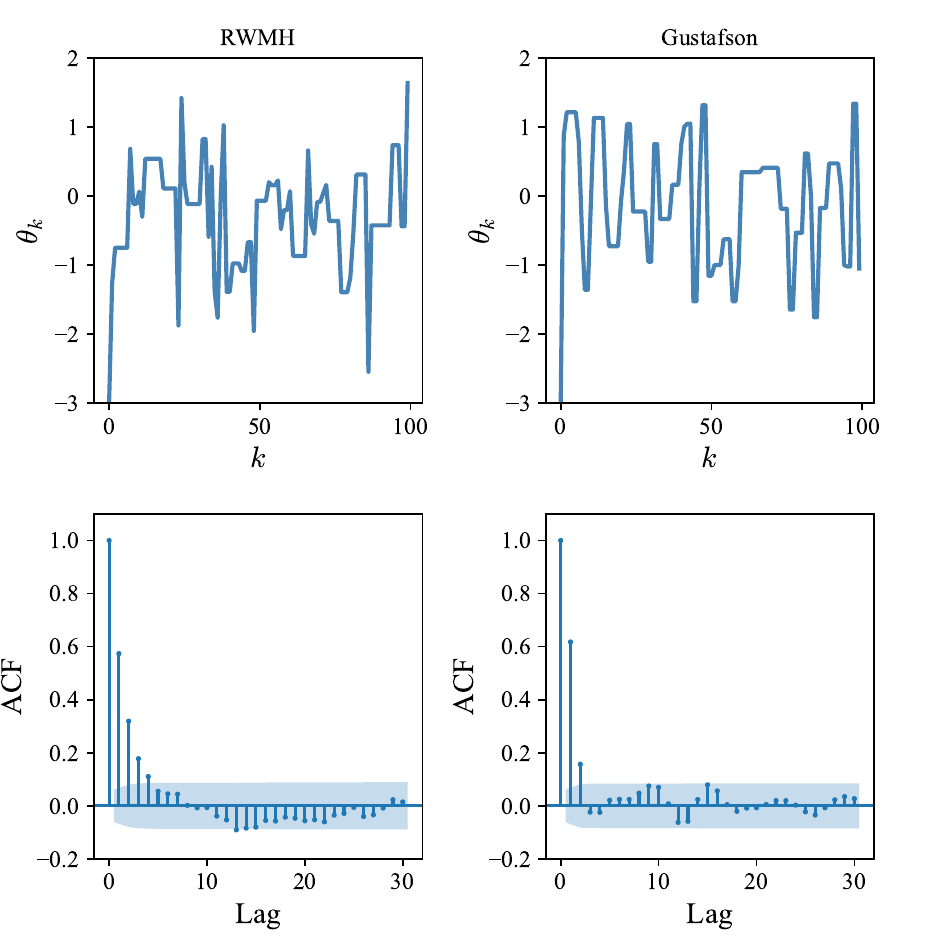}
    \caption{Comparison of random walk Metropolis--Hastings (left column) and the guided random walk algorithm (right column) when sampling from a 1-d Gaussian. The proposal is Gaussian with a standard deviation of 2.38 in both cases. Top row shows trace plots, and bottom row shows the estimated auto-correlation of the chains.}
    \label{fig:ch5_1dcomparison_large}
\end{figure}

To see the behaviour of this algorithm, and compare it to a standard random walk Metropolis--Hastings \index{random walk Metropolis}\index{Metropolis--Hastings} algorithm, we show the results of both algorithms when sampling from a Gaussian target distribution in Figures \ref{fig:ch5_1dcomparison_small} and \ref{fig:ch5_1dcomparison_large}. For Figure \ref{fig:ch5_1dcomparison_small}, we implement both algorithms using a small proposal variance, $\delta^2$. Here we see the poor mixing\index{mixing} of the random walk Metropolis \index{random walk Metropolis} due to its reversibility and the random walk behaviour of its output. This means that it takes nearly 1000 steps to reach the mode of the target distribution. By comparison, Gustafson's algorithm suppresses this random walk behaviour, with large periods of time spent moving in the same direction. As we are sampling from a uni-modal target, the qualitative dynamics of the algorithm are as follows: when it is moving towards the mode the acceptance probability\index{acceptance probability} is 1 and the algorithm keeps moving in that direction. It is only when it is moving away from the mode that it has any chance of rejecting a proposal and switching the direction of its move. 

The behaviour of the two algorithms when we use a larger step size, as shown in Figure \ref{fig:ch5_1dcomparison_large}, is very different. In this case, both trace plots look qualitatively similar, and both samplers have similarly good performance as shown by the auto-correlation \index{autocorrelation} plots. The reason is that for a large step size, the chances of rejecting the proposal and switching the direction are now much higher, closer to 0.5 on average. This also ties in with the observation from Section \ref{sec:ch5:benefits} that reversibility is only an issue when the step sizes are small.

If we wish to sample from a multivariate target $\pi(\btheta)$, we can do so by applying this update component-wise\index{component-wise updates}. That is, we augment the state to $(\btheta,\bp)$ where $\bp$ is $d$-dimensional and each entry of $\bp$ is either $1$ or $-1$ and specifies the direction of the next update of the corresponding component of $\btheta$. Then, we have $d$ parts to each iteration of the algorithm where each part uses the above algorithm to update a different component of $\btheta$.

One can see the similarity with the algorithm of \cite{horowitz1991generalized}. We have augmented the state to include a component, of the same dimension as $\btheta$, that governs the direction of the update of the Markov chain. Our Markov chain update is based on interleaving two reversible moves, but with the resulting Markov chain being non-reversible. Finally, each reversible move involves a flip of direction; providing the acceptance probability\index{acceptance rate|see{acceptance probability}}\index{acceptance probability} in step (1) is high, then these cancel out and the chain is likely to move in the same direction over multiple time-steps.

Finally, one can implement the idea of \cite{gustafson1998guided} in a way that jointly updates the $d$-dimensional state. This can be achieved by letting $\bp$ be a $d$-dimensional unit vector. Define a target distribution on $(\btheta,\bp)$ that is the product of $\pi(\btheta)$ and the uniform density for $\bp$ on the $d$-dimensional sphere; we denote the density by $\mathsf{U}_d(\bp)$ and the surface of the sphere as $\cS_d$. This target will be kept invariant by the following algorithm, where to aid the presentation of the algorithm in Section \ref{sec.DBPS} we replace the random $|\zeta|$ with a fixed, user-prescribed $\delta>0$. Here and for the remainder of this chapter, the proposal is a deterministic, 1-1 map, rather than a density, and we use the symbol $\qmap$ rather than $q$.

\index{non-reversible MCMC}\index{Markov chain Monte Carlo}

\begin{algorithm}
    \caption{Multi-dimensional guided random walk,  with fixed direction. 
    }
    \KwIn{Density $\pi(\btheta)$, initial value $\btheta_0$ and speed $\delta>0$, unit vector $\bp_0$ sampled from $\mathsf{U}_d$.}
    \For{$k \in 0, \dots, n-1$} {
        Propose $(\btheta',\bp')=\qmap_1(\btheta_k,\bp_k)=(\btheta_k+\delta \bp_k,-\bp_k)$.\\
        With probability
        \[\alpha_1(\btheta_k,\bp_k;\btheta',\bp'):=\min\left\{1,\frac{\pi(\btheta')}{\pi(\btheta_k)}\right\}\]\\
         accept the proposal\index{proposal}, $(\btheta_{k+1},\bp'')\gets(\btheta',\bp')$; otherwise reject it, $(\btheta_{k+1},\bp'')\gets (\btheta_k,\bp_k)$.\\
        Flip the velocity: $(\btheta_{k+1},\bp_{k+1})=(\btheta_{k+1},-\bp'')$.}
    \label{alg:MDGustafsenFD}
\end{algorithm}

The velocity flip does not update $\btheta$, and $\bp\in \cS_d \Leftrightarrow -\bp \in \cS_d$, so the flip is reversible  with respect to  $\pi(\btheta)\mathsf{U}_d(\bp)$, which must, therefore, be the stationary density. 
It is helpful to explore exactly why the first step is also reversible with respect to this density. Let the current state be $(\btheta,\bp)$ and let the state after the first step be $(\btheta',\bp')$. Suppose that $(\btheta,\bp)$ has a density of $\pi(\btheta)\mathsf{U}_d(\bp)$ and let $\cB$ and $\cC$ be disjoint sets in $\mathbb{R}^d\times \cS_d$. Then
$    \Prob{(\btheta,\bp)\in\cB, (\btheta',\bp')\in \cC}$ is
\begin{align*}
    & \int \pi(\btheta)\mathsf{U}_d(\bp)\min\left(1,\frac{\pi(\btheta')}{\pi(\btheta)}\right) 1_{\cB}(\btheta,\bp)1_{\cC}(\btheta',\bp')~
    \md \btheta \md \bp 
    \\
    &=
      \int 
\pi(\btheta')\mathsf{U}_d(\bp')\min\left(1,\frac{\pi(\btheta)}{\pi(\btheta')}\right) 1_{\cB}(\btheta,\bp)1_{\cC}(\btheta',\bp')~
    \md \btheta \md \bp \\
    &=
         \int    \pi(\btheta')\mathsf{U}_d(\bp')\min\left(1,\frac{\pi(\btheta)}{\pi(\btheta')}\right)~ 1_{\cB}(\btheta,\bp)1_{\cC}(\btheta',\bp')~
    \md \btheta' \md \bp', 
\end{align*}
which is $\Prob{(\btheta,\bp)\in\cC, (\btheta',\bp')\in \cB}$,  
as required.
Here, the second line follows by rearrangement and because the algorithm forces $\bp \in \cS_d \Leftrightarrow \bp' \in \cS_d$. The third line follows from the unit Jacobian\index{Jacobian} of the transformation $\qmap_1$. 

While this algorithm keeps the target on $(\btheta,\bp)$ invariant, the algorithm is reducible\index{reducible}; $\bp$ can only take two values: $\bp_0$ and $-\bp_0$, whilst $\btheta$ is confined to a discrete grid on the vector $\btheta_0+\delta \bp_0$. It is straightforward to make the algorithm irreducible\index{irreducibility} in dimension $2$ and above by adding occasional moves that update $\bp$, for example, that sample a new value of $\bp$ from $\mathsf{U}_d$.

\subsubsection{Comparison on a Ring-Shaped Target}

\begin{figure}
\centering
\includegraphics[width=0.7\textwidth]{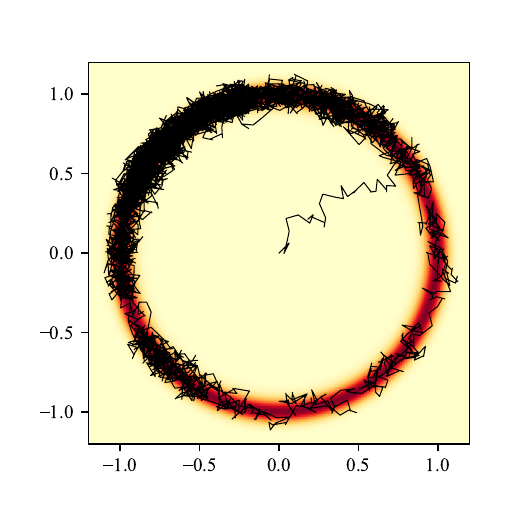}
    \caption{Trace plot for 50000 iterations of the algorithm of \cite{gustafson1998guided} sampling from a 2-$d$ density concentrated on the unit circle. The heat map shows the target density, and the line shows the trace of the algorithm. 
    }
    \label{fig:circle_gustafson}
\end{figure}
\begin{figure}
    \centering
    \includegraphics[width=0.7\textwidth]{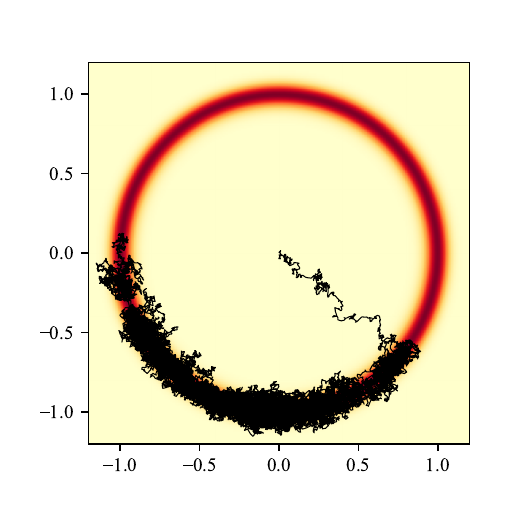}
    \caption{Trace plot for 50000 iterations of the random-walk Metropolis algorithm sampling from a 2-$d$ density concentrated on the unit circle. The heat map shows the target density, and the line is the trace of the algorithm. 
    }
    \label{fig:circle_RWMH}
\end{figure}

To see the benefits of these non-reversible algorithms, we compared the guided random walk algorithm with that of a standard random walk Metropolis algorithm for sampling from a density that concentrates on the perimeter of a circle in 2-$d$. This mimics a common challenge of sampling from a density that concentrates near a lower-dimensional manifold within a higher-dimensional space. We implemented Algorithm \ref{alg:MDGustafsenFD} with a refresh of the direction every 10 iterations. 

Trace plots for the non-reversible and reversible algorithms are shown in Figures \ref{fig:circle_gustafson} and \ref{fig:circle_RWMH}, respectively. Each algorithm uses the same average step size\index{step size}. For good mixing\index{mixing} in the manifold scenario, the step size needs to be of the order of the width of the density orthogonal\index{orthogonal} to the circle -- and thus is small relative to the size of the circle itself. As we have seen previously, using such a small step size leads to random-walk behaviour for the reversible algorithm. By comparison, the non-reversible algorithm is able to suppress this. The overall effect is much better mixing\index{mixing} for the non-reversible algorithm, which takes 50000 iterations to explore the full extent of the ring. By comparison, over the same number of iterations, the reversible algorithm has only explored the bottom half of the ring; in fact, it takes six times this number of iterations to explore the entire ring.\index{guided random walk|)}
\index{lifting schemes|)}
\section{Improving Non-Reversibility: Delayed Rejection}
\label{sec.DBPS}

\index{delayed rejection|(}
 The major source of inefficiency in Algorithm \ref{alg:MDGustafsenFD} is the net reversal of momentum whenever the proposed new position and momentum $(\btheta',\bp')$ are rejected. The subsequent momentum flip is designed to keep the process moving in the same direction as it was at the start of the iteration; however, it has the opposite effect when there is a rejection, causing the chain to retrace its steps. The \emph{delayed-rejection} method of \cite{green2001delayed} can be applied to any reversible step and provides a natural mechanism for reducing the probability of rejection as follows: whenever the original step would have rejected the proposal\index{proposal}, a new value is proposed; the acceptance probability\index{acceptance probability} for this new proposal is designed exactly so that the combination of the rejected existing step and the potential new step satisfies detailed balance\index{detailed balance} with respect to the intended posterior\index{posterior distribution}.

The validity of the propose-accept/reject step of Algorithm \ref{alg:MDGustafsenFD} relies on the Jacobian of $\qmap_1$ being $1$ and the fact that  $\qmap_1(\btheta',\bp')\equiv \qmap_1(\btheta_k+\delta \bp_k, -\bp_k)=(\btheta_k,\bp_k)$. Following analogous constraints, we incorporate a delayed-rejection move as follows:
If $(\btheta',\bp')$ is rejected, we can make a \emph{new} proposal\index{proposal} based not only on the current state, $(\btheta_k,\bp_k)$ but also on the initial, rejected proposal, $(\btheta',\bp')$. In this case, we now consider the current state to be the original current state and the initial, rejected proposal: $(\btheta_k,\bp_k,\btheta',\bp')$; we call this the \emph{full state}. We then make a proposal for this full state; \ie, we propose a new original state and a new initial, rejected proposal: $(\btheta'',\bp'', \btheta''',\bp''')=\qmap_2(\btheta_k,\bp_k,\btheta',\bp')$, where $(\btheta''',\bp''')=(\btheta''+\delta \bp'', -\bp'')$. To simplify the notation in the following, we define 
\[\bz_k=(\btheta_k,\bp_k), ~~\bz'=(\btheta',\bp'),~ ~\bz''=(\btheta'',\bp'')~~\mbox{and}~~ \bz'''=(\btheta''',\bp'''). 
\]
We will either accept or reject the proposal $(\bz'',\bz''')=\qmap_2(\bz_k,\bz')$ for the full state, and we choose the probability of acceptance exactly so that the move satisfies detailed balance with respect to the density of the current state (which implies the initial proposal), including the fact that it was rejected: 
\[
\pitiltil(\bz_k,\bz'):=
\pi(\btheta_k)\mathsf{U}_d(\bp_k)\mathbb{I}\left[\bz'=\qmap_1(\bz_k)\right]\left\{1-\alpha_1(\bz_k;\bz')\right\}.
\]
By analogy with the Metropolis--Hastings algorithm, we might expect this acceptance probability\index{acceptance probability} to be 
\[
\alpha_2(\bz_k,\bz';\bz'', \bz''')
:=
\min\left[1,\frac{\{1-\alpha_1(\bz''; \bz''')\}\pi(\btheta'')}{\{1-\alpha_1(\bz_k;\bz')\}\pi(\btheta_k)}\right],
\]
where for simplicity of presentation we have not included the indicator functions $\mathbb{I}\left[\bz'''=\qmap_1(\bz'')\right]$ and $\mathbb{I}\left[\bz'=\qmap_1(\bz_k)\right]$ in the numerator and denominator of the fraction, respectively, since by design these are both $1$.
Indeed, subject to conditions on $\qmap_2$, $\alpha_2$ is correct. For the proposal and accept/reject step to be valid, $\qmap_2$ must have a Jacobian of $1$, and must satisfy $\qmap_2(\bz'', \bz''')=(\bz_k,\bz')$. 
The probability of being at $\bz_k$ with a deterministic proposal of $\bz'$ that has been rejected, and then moving to this new full state is, therefore,
\[
\pitiltil(\bz_k,\bz') \alpha_2(\bz_k,\bz';\bz'', \bz'''),
\]
which, by design, is equal to 
\[
\pitiltil(\bz'',\bz''') \alpha_2(\bz'',\bz''';\bz_k,\bz').
\]
Using the two constraints on $\qmap_2$ and an analogous argument to that used for Algorithm \ref{alg:MDGustafsenFD}, detailed balance\index{detailed balance} is, therefore, satisfied.
\index{delayed rejection|)}
\subsection{The Discrete Bouncy Particle Sampler}
\index{non-reversible MCMC}
\index{discrete BPS|see{discrete bouncy particle sampler}}
\index{discrete bouncy particle sampler|(}

The current state is $\bz_k$, given this, $\bz'$ is fixed via $\qmap_1$. The 1-1 map $\qmap_1$ similarly fixes the relationship between $\bz''$ and $\bz'''$ so the additional flexibility this delayed-rejection formulation allows is the freedom to choose $\bz''$ or, equivalently, $\bz'''$.

To make $\alpha_2$ close to $1$, a sensible aim is to choose $(\bz'', \bz''')$ such that $\pi(\btheta'')\approx \pi(\btheta_k)$ and $\alpha_1(\bz'';\bz''')\approx \alpha_1(\bz_k;\bz')$. If we set $\btheta'''=\btheta'$, both of these conditions will be satisfied if $\pi(\btheta')/\pi(\btheta)\approx \pi(\btheta')/\pi(\btheta'')$; that is, 
\[
\log \pi(\btheta')-\log \pi(\btheta'')
\approx 
\log \pi(\btheta')-\log \pi(\btheta_k).
\]
Taylor expanding about $\btheta'$, we require $\delta \bg\cdot \bp
\approx
\delta \bg \cdot \bp''$, 
where $\bg=\nabla \log \pi(\btheta')$; that is, both $\bp''$ and $\bp$ should have approximately the same component in the gradient\index{gradient} direction.
Perhaps the most natural way to achieve this is by setting
\[
\bp'' = \Psi_{\bg}(\bp_k):= -\bp_k + 2(\bp_k \cdot \bghat)~\bghat,
\]
where $\bghat=\bg/\|\bg\|$, is the direction of the gradient of $\log\pi$ at $\btheta'$. In this case,
\[
\Psi_{\bg}(\bp'')=
\bp_k-2(\bp_k\cdot \bghat)~\bghat
+2[\left\{-\bp_k+2(\bp_k\cdot \bghat)~\bghat\right\}\cdot \bghat]~\bghat
=
\bp_k.
\]
Further, since $\Psi_{\bg}(\cdot)$ is self-inverse, the absolute value of its Jacobian must be $1$. 
The full proposal\index{proposal} is, therefore,
\[
(\bz'',\bz''')
=
\qmap_2(\bz_k,\bz')
=
\left(\btheta'-\Psi_{\bg}(\bp_k),\Psi_{\bg}(\bp_k),\btheta',-\Psi_{\bg}(\bp_k)\right).
\]
However, since $\bz'''$ plays no part in any subsequent movement and since momenta do not figure in the acceptance probabilities, it can simplify notation to think of the proposal as 
\[(\btheta'',\bp'')=\qmap_2^*(\btheta_k,\bp_k,\btheta')
:=
\left(\btheta'-\Psi_{\bg}(\bp_k),\Psi_{\bg}(\bp_k)\right),
\]
the initial acceptance probability\index{acceptance probability} as
\[
\alpha_1^*(\btheta_k,\btheta')
:=
\min\left\{1,\frac{\pi(\btheta')}{\pi(\btheta_k)}\right\},
\]
and the second acceptance probability\index{acceptance probability} as
\[
\alpha_2^*(\btheta,\btheta',\btheta'')
=
\min\left[1,\frac{\{1-\alpha_1^*(\btheta'', \btheta')\}\pi(\btheta'')}{\{1-\alpha_1^*(\btheta_k,\btheta')\}\pi(\btheta_k)}\right].
\]
The full algorithm is given in Algorithm \ref{alg:DBPS} and illustrated in Figure \ref{fig:DBPSbounce}.
\index{non-reversible MCMC}
\begin{algorithm}
    \caption{Discrete Bouncy Particle Sampler}
    \KwIn{Density $\pi(\btheta)$, initial value $\btheta_0$ and speed $\delta$, unit vector $\bp_0$ sampled from $\mathsf{U}_d$.}
    \For{$t \in 0, \dots, T-1$} {
        Propose $(\btheta',\bp')=\qmap_1(\btheta_k,\bp_k)$.\\
        With a probability of $\alpha_1^*(\btheta_k,\btheta')$
                accept the proposal\index{proposal}: $(\btheta_{k+1},\bp_{k+1})\gets(\btheta',\bp')$.\\
         If the proposal is not accepted then propose $(\btheta'',\bp'')=\qmap_2^*(\btheta_k,\bp_k,\btheta')$.\\
         With a probability of $\alpha_2^*(\btheta_k,\btheta',\btheta'')$ accept this proposal\index{proposal}: $(\btheta_{k+1},\bp_{k+1})\gets (\btheta'',\bp'')$, otherwise $(\btheta_{k+1},\bp_{k+1})\gets(\btheta_k,\bp_k)$.\\
        Flip the velocity: $(\btheta_{k+1},\bp_{k+1})\gets (\btheta_{k+1},-\bp_{k+1})$.}
    \label{alg:DBPS}
\end{algorithm}

\begin{figure}
\centering
    \includegraphics[width=0.7\textwidth]{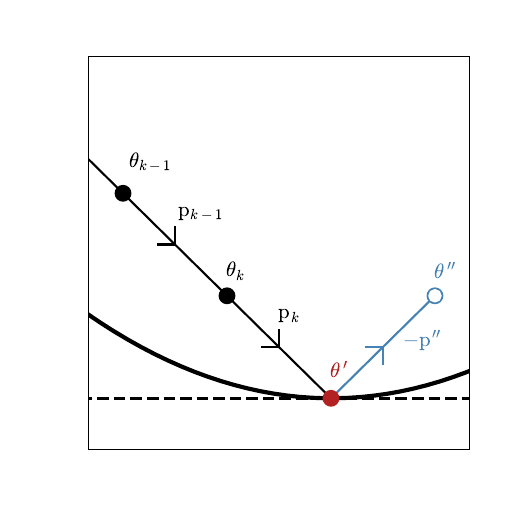}
    \vspace{-1cm}
    \caption{Heuristic of consecutive steps from the discrete bouncy particle sampler. From $(\btheta_{k-1},\bp_{k-1})$ the initial proposal is accepted and the momentum is flipped, leading to $(\btheta_k,\bp_k)$. The initial proposal from this point ($\btheta'$ is shown but $\bp'=-\bp_k$ is not) is rejected. The thick solid line shows a contour of $\pi$ at $\btheta'$ with the tangent line at $\btheta'$ shown dashed. 
 The full proposal\index{proposal} includes $(\btheta'',\bp'')$; the figure shows $-\bp''$ to emphasise how, if the proposal is accepted, with the subsequent momentum flipped the movement is analogous to a bounce off the tangent hyperplane.     \label{fig:DBPSbounce}}
\end{figure}

As with Algorithm \ref{alg:MDGustafsenFD}, we can ensure that Algorithm \ref{alg:DBPS} is irreducible\index{irreducibility} by refreshing the unit vector $\bp$. This could involve sampling $\bp_k\sim \mathsf{U}_d$ every $m$ iterations for some integer $m$, or with probability $p$ on any iteration; however, this still allows for rejections causing the sampler to exactly retrace the recent past. More usually, therefore, after every velocity flip step, the following update is made:
\[
\bp_{k+1}=\frac{\gamma \bp_{k+1}+\sqrt{1-\gamma^2}\bxi}{\|\gamma \bp_{k+1}+\sqrt{1-\gamma^2}\bxi\|},
\]
where $\bxi\sim \Normal(\bzero,\frac{1}{d}\bI_d)$ and $0\le \gamma < 1$. 

Considering the combined effect of the initial proposal, the delayed-rejection step and the momentum flip on a hypothetical particle at $\btheta_k$ with a velocity of $\bp_k$ moving along a level, frictionless surface provides a useful insight into the behaviour of Algorithm \ref{alg:DBPS}. In the following, we define $\cR_{\bg}(\bp):=-\Psi_{\bg}(\bp)=\bp-2(\bp\cdot \bghat)~\bghat$, which is a reflection\index{reflection} of $\bp$ in the hyperplane perpendicular to $\bg$. 
\begin{itemize}
\item If the initial proposal\index{proposal} is accepted then the net effect is $(\btheta_{k+1},\bp_{k+1})=(\btheta_k+\delta \bp_k,\bp_k)$; the particle moves exactly as it should over a time $\delta$.
\item If the initial proposal is rejected, but the delayed-rejection proposal is accepted, then the net effect is
$(\btheta_{k+1},\bp_{k+1})=\left(\btheta_k+\delta \bp_k+\delta \cR_{\bg}(\bp_k),\cR_{\bg}(\bp_k)\right)$; the particle moves forward for a time $\delta$ to $\btheta'$ then reflects off the hyperplane perpendicular to $\bg$ and moves forward for another time $\delta$.
\item If both the initial step and the delayed-rejection step are rejected then $(\btheta_{k+1},\bp_{k+1})=(\btheta_k,-\bp_k)$; the particle reverses direction.
\end{itemize}

If it were not for the occasional full rejection, the path of the points outputted from the algorithm would resemble a time discretisation of the smooth path of a particle moving along a frictionless surface and occasionally bouncing off a hard barrier in the hyperplane perpendicular to the local gradient\index{gradient}. For this reason, the algorithm is known as the \emph{Discrete Bouncy Particle Sampler}. In the limit, as $\delta\downarrow 0$ and the number of steps is increased in proportion to $1/\delta$, this becomes a continuous-time algorithm known as the \emph{Bouncy Particle Sampler}, which we shall meet in Section \ref{sec:ch6-BPS}.

\subsubsection{Performance on the Ring-Shaped Target}

We exemplify the improved trajectories of the discrete bouncy particle sampler by implementing it on the same two-dimensional ring-shaped target for which the guided random walk \index{guided random walk}with a directional update every 10 iterations took 50000 iterations to explore (Figure \ref{fig:circle_gustafson}) and the random walk Metropolis\index{random walk Metropolis} with the same step size\index{step size} took 300000 iterations to explore (Figure \ref{fig:circle_RWMH} shows the first 50000 iterations). 

\begin{figure}
\centering
 \includegraphics[width=\textwidth]{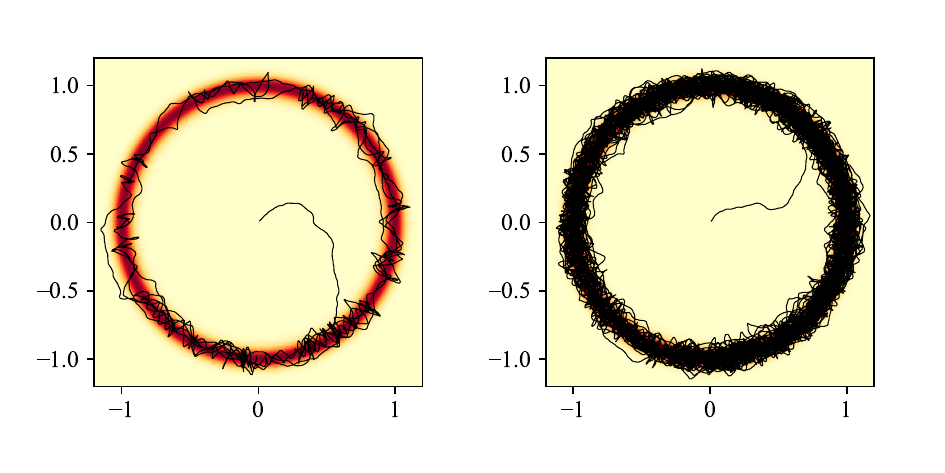}
    \caption{Trajectories after 5000 iterations (left) and 50000 iterations (right) of the discrete bouncy particle sampler on the two-dimensional ring target of Figures \ref{fig:circle_gustafson} and \ref{fig:circle_RWMH}, and using the same step size.
     \label{fig:DBPScirc}}
\end{figure}

Our discrete bouncy particle sampler uses the same step size as the earlier algorithms and sets $\gamma$ so that the direction of travel is only partially forgotten from just after one bounce to just before the next \cite[see][]{sherlock2022discrete}.
Figure \ref{fig:DBPScirc} shows the path of the discrete bouncy particle sampler after 5,000 (left) and 50,000 (right) iterations, on the two-dimensional ring target. Exploration of the full circle takes only a tenth of the number of iterations that the guided random walk\index{guided random walk} requires and a sixtieth of the number of iterations required by the RWM. The coverage after the full 50000 iterations is also more complete than when using the guided random walk.
\index{discrete bouncy particle sampler|)}

\section{Chapter Notes}

Perhaps the earliest theoretical results showing that non-reversible chains  \index{non-reversible MCMC} have better mixing\index{mixing} properties were given in  \cite{diaconis2000analysis}. The results from a pre-print of this work were extended in \cite{chen1999lifting}. See also \cite{neal2004improving} and \cite{sun2010improving}, which show that non-reversible samplers reduce asymptotic variance. More recent results are given in \cite{bierkens2016non}.

The discrete bouncy particle sampler, described in Section \ref{sec.DBPS}, is given in \cite{sherlock2022discrete}. Similar algorithms appear in \cite{neal2003slice} and \cite{vanetti2017piecewise}. The key difference is the use of ``external bounces" rather than ``internal bounces": reflection\index{reflection} happens in the current point rather than the proposed point, so that: $(\btheta'',\bp'',\btheta''',\bp''')=(\btheta_k,\Psi_{\bg}(\bp_k), \btheta_k+\Psi_{\bg}(\bp_k),-\Psi_{\bg}(\bp_k))$, where $\bg=\nabla \log \pi(\btheta_k)$. Use of the bounce move that is key to the success of the discrete bouncy particle sampler will be seen within the continuous-time bouncy particle sampler described in the next chapter. It has also been used in other discrete-time MCMC\index{Markov chain Monte Carlo}  algorithms. For example, the Hug move in \cite{LudShe2022} uses repeated bounces in the same way that HMC\index{HMC|see{Hamiltonian Monte Carlo}}\index{Hamiltonian Monte Carlo} uses repeated leapfrog\index{leapfrog dynamics} steps, with the result that the path to the proposal stays close to the contour of $\pi$ on which it started, and just as with HMC, for a given integration time, the acceptance probability can be made as close to $1$ as desired by reducing the step size. 

An alternative to the lifting schemes\index{lifting schemes} described in this chapter are methods that adapt a reversible Markov chain to a non-reversible one without adding additional states. The general approach is to find a loop of states, and then adapt the probability flow around this loop such that the net flow of probability at each state is preserved. For example, if we have states $i$, $j$ and $k$ then we can reduce each of the three probabilities of moving from $i$ to $i$, $j$ to $j$ and $k$ to $k$ by the minimum of the three and increase each of the three probability transitions from $i$ to $j$, $j$ to $k$ and $k$ to $i$ by the same amount, so that the invariant distribution of the chain is unchanged. Such changes have been described as adding {\em{vortices}}. These ideas are described in \cite{sun2010improving} and \cite{turitsyn2011irreversible}. See \cite{suwa2010markov} for related ideas. These constructions are hard to adapt to general Markov chains, particularly chains without discrete states, though see \cite{bierkens2016non} for an approach to adapt the Metropolis--Hastings\index{Metropolis--Hastings}\index{MH|see{Metropolis--Hastings}} acceptance probability to introduce non-reversibility.


\chapter{Continuous-Time MCMC}
\label{chap:continuous-time}

\index{PDMP|see{piecewise deterministic Markov process}}

The previous chapter introduced the idea of non-reversible MCMC and demonstrated that non-reversibility may help improve the Markov chain's mixing by suppressing random-walk behaviour. In this chapter, we now present a class of non-reversible MCMC algorithms that can target a general distribution, $\pi(\btheta)$. These algorithms are different from standard MCMC algorithms in that they are based on simulating a continuous-time Markov chain. Furthermore, specific examples of these continuous-time MCMC algorithms can be derived as continuous-time limits of the non-reversible algorithms we introduced in the previous chapter. 

As a way of motivating these algorithms, we will first look at this continuous-time limit. The resulting continuous-time process is from a class of processes called \emph{piecewise deterministic Markov processes}(PDMPs). We will introduce some background on this class of processes, including some details around how we can simulate their continuous-time trajectories, before introducing example MCMC algorithms and various extensions. These algorithms use gradient information, and unless stated otherwise, we will assume that the target distribution $\pi(\btheta)$ is differentiable everywhere (in practice being continuous and differentiable almost everywhere is sufficient).

Within this chapter, we will need to refer to both components of a vector and possibly the state of the vector at different times. To distinguish between these, we will use the convention that when both time and component are needed, we subscript by time and superscript by the component. So, for example, $\theta^{(i)}_t$, will be the $i$th component of the state, $\btheta$, at time $t$.

\section{Continuous-Time MCMC as the Limit of Non-Reversible MCMC} \label{sec:ch6-limitexample}
\index{non-reversible MCMC|(}
To help build intuition for continuous-time MCMC,\index{Markov chain Monte Carlo} and to see how it links to discrete-time MCMC algorithms, we will show we can derive the continuous-time algorithm as the limit of a discrete-time algorithm as we let the step size in the discrete-time algorithm to tend to 0. Here, we will consider this limiting argument at an informal level, before presenting a more formal justification for continuous-time MCMC methods.

We will consider sampling from a 1-dimensional target distribution using a simplified version of the non-reversible guided random walk\index{guided random walk}\index{lifting schemes} algorithm that was introduced in Section \ref{sec:ch5_lifting}. Our simplification is to assume that the step size at each iteration is fixed. Our state will still be $(\theta,p)$, with $p$ either $1$ or $-1$ and specifying the direction of the next proposed move, and we will let the size of the move be equal to some fixed value $\delta$. Such an MCMC algorithm would not be irreducible\index{irreducibility}, as it could only explore values of the state that are integer multiples of $\delta$ away from its initial value. 
However, the algorithm will still have the correct invariant distribution; that is if we simulate the initial state from the target distribution for $\theta$ and a uniform distribution for $p$, then this will also be the distribution of the state at any future iteration. The issue of lack of irreducibility vanishes in the limit as $\delta\rightarrow 0$.

As a reminder, if the target distribution of interest is $\pi(\theta)$ then the MCMC algorithm will iterate the following moves
\begin{itemize}
\item[(1a)] Propose a move from $(\theta,p)$ to $(\theta+\delta p,-p)$. Accept this with the standard Metropolis--Hastings acceptance probability, which simplifies to
\[
\min\left\{ 1, \frac{\pi(\theta+\delta p)}{\pi(\theta)}
\right\}
\]
\item[(1b)] Move from $(\theta',p)$ to $(\theta',-p)$.
\end{itemize}
In step (1a), we propose a move of size $\delta$ in direction $p$. If we accept this move, then we will flip the direction $p$ in both steps (1a) and (1b) -- so the direction will be the same at the next iteration. If we reject the move, then we will only flip $p$ in step (1b) and thus the direction of the move will be flipped for the next iteration.

Under this framework, we can then consider letting $\delta\rightarrow 0$ while increasing the number of iterations $n$. That is we fix a value $s$ and for any number of iterations, $n$ will set $\delta=s/n$. We will scale time so that the $i$th MCMC transition will occur at time $i\delta$, and define $(\theta_t,p_t)$ to be the value of the state after the $i$th MCMC transition for $i\delta\leq t < (i+1)\delta$. 

Now, for each move in step (1a), the rejection probability for small $\delta$ is
\begin{eqnarray*}
\lefteqn{\max\left\{ 
0, 1-\exp[\log\pi(\theta+\delta p)-\log\pi(\theta)]
\right\} } \\ 
&=& \max\left\{ 
0, 1-\exp\left[\frac{\mbox{d}\log \pi(\theta)}{\mbox{d}\theta} \delta p+o(\delta)\right]
\right\} \\ &=&
\max\left\{ 
0, -p\frac{\mbox{d}\log \pi(\theta)}{\mbox{d}\theta}
\right\}\delta+o(\delta),
\end{eqnarray*}
assuming that, for example, the derivative of $\pi(\theta)$ is continuous. 

Thus in our limit as $\delta\rightarrow0$, rejections in step (1a) will occur as events in a Poisson process of rate 
\[
\lambda(\theta_t,p_t)=\max\left\{0,-p_t\frac{\mbox{d}\log \pi(\theta_t)}{\mbox{d}\theta}\right\}.
\]
The dynamics between these events will be deterministic, with $p_t$ being constant and $\theta_t$ changing as per a constant velocity model with velocity $p_t$. At each event, the velocity will just flip. While
the process is moving to areas of higher probability density, as defined by $\pi(\theta)$, the rate of the Poisson process will be 0. Thus events will only occur if the process is moving to areas of lower probability mass.

How does this limiting, continuous-time algorithm compare to the algorithm of \cite{gustafson1998guided}\index{guided random walk}? We will compare with the standard, irreducible version of \cite{gustafson1998guided} where the step size\index{step size} at each iteration is random. We show trace plots for both algorithms for sampling from a mixture of two Gaussians in Figure \ref{fig:ch6_1dcomparison}. The target distribution is an equal mix of a Gaussian with mean 2 and variance 1 and with mean 0 and variance $0.1^2$. This was chosen so that we have two modes where different step sizes would be optimal, whilst still allowing for sufficient overlap of the modes that the chains would mix between them.

We can see that qualitatively the two trace-plots are similar. Both methods produce zig-zag-like traces, as they will continue to move in the same direction when moving to areas of higher probability density. However, the continuous-time process has a number of potential advantages:
\begin{itemize}
    \item  If we could simulate the continuous-time trajectory directly, then it has the benefit of having fewer events where the velocity changes (and where the state needs to be updated) than iterations of Gustafson's algorithm -- in our example, there are around 200 events in the continuous-time algorithm as compared to 1000 iterations of Gustafson's algorithm. 
    \item It also has the benefit of less tuning, as no step size needs to be specified. 
    \item Finally, it has the advantage that the full continuous-time sample path\index{sample path} can be used to calculate Monte Carlo averages. 
\end{itemize}
However, directly simulating the continuous-time trajectory is not normally possible -- and it is both the difficulty with simulating the continuous-time process and the additional computational overhead of doing so that are the main disadvantages. Furthermore, accounting for the very different computational costs, per iteration versus per unit of continuous-time, makes the algorithms hard to compare theoretically. We will return to issues around simulating continuous-time Markov processes\index{Markov chain!continuous-time} like this below.
\begin{figure}
    \centering
    \includegraphics[width=\textwidth]{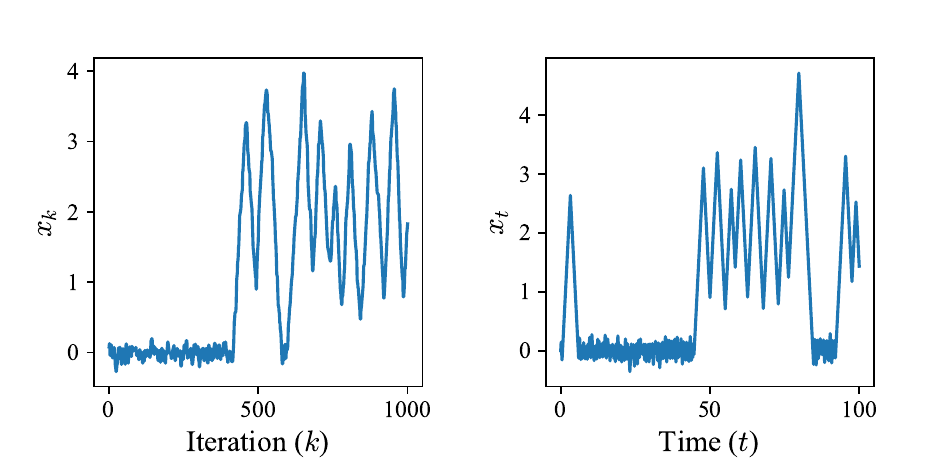}
    \caption{Trace plots of Gustafson's algorithm (left) and the continuous-time limit (right) for sampling from a mixture of Gaussians. }
    \label{fig:ch6_1dcomparison}
\end{figure}
\index{non-reversible MCMC|)}

\section{Piecewise Deterministic Markov Processes} \index{Markov chain!continuous-time}\index{piecewise deterministic Markov processes|(}

The continuous-time limiting process we derived in the previous section is an example of a piecewise deterministic Markov process (PDMP)\index{piecewise deterministic Markov process}. These are Markov processes that evolve deterministically between a set of random events. We will now introduce these processes before considering how they can be designed and used to sample from a target distribution of interest.

\subsection{What is a PDMP?}

We will denote the state of a PDMP at time $t$ by $\bz_t$. A PDMP is defined by the following properties:
\begin{itemize}
\item[(i)] {\em The deterministic dynamics.} The PDMP evolves deterministically between a set of event times. We consider PDMPs where the deterministic dynamics are specified through an ordinary differential equation
\begin{equation} \label{eq:PDMPODE}
\frac{\mbox{d}\bz_t}{\mbox{d}t}=\bphi(\bz_t),
\end{equation}
for some gradient field $\bphi(\cdot)$. We will also define $\bPhi$ to be the transition function, or flow map, for these dynamics, so for $s>0$ the solution to the differential equation satisfies $\bz_{s+t}=\bPhi(\bz_t,s)$. 
\item[(ii)] {\em The event rate.} Random events occur at a rate $\lambda(\bz_t)$ that depends on the current state.
\item[(iii)] {\em The transition kernel at events.} At each event, at time $\tau$, the state changes according to some Markov transition kernel\index{kernel!Markov transition} $\Ker{\bZ_\tau \in \cdot |\bz_{\tau-}}$, where 
\[
\bz_{\tau-} = \lim_{\epsilon \downarrow 0} \bz_{\tau-\epsilon}
\] 
is the value of the state immediately before the event. That is $\Ker{\bZ_\tau \in \cdot |\bz_{\tau-}}$ defines the conditional probability of moving to set $\cdot$, given the state immediately before the event. 

\end{itemize}
To simplify exposition, and because this is consistent with the PDMPs we will use for MCMC,\index{Markov chain Monte Carlo} we will primarily focus on discrete transition kernels at events, and let $q(\cdot|\bz)$ denote the probability mass function associated with the transition kernel. Though the ideas below extend naturally to continuous transition kernels.

As the deterministic dynamics are Markov, and the event rate and transition density depend just on the current state, then the resulting process will be Markov.
\index{piecewise deterministic Markov processes|)}
\subsection{Simulating PDMPs} \label{sec:ch6-simulatingPDMPs}
\index{piecewise deterministic Markov processes!simulation|(}
To be able to use a PDMP as the basis of a sampling algorithm, we will need to be able to simulate and store realisations of the PDMP. First, we can store a realisation of the continuous-time path of a PDMP by storing just the initial state of the PDMP, and the time and event immediately after each event. We will call these points the {\em skeleton} of the realisation. If we simulate a PDMP up to time $T$ then its skeleton will be of the form $\{t_k,\bz_{t_k}\}_{k=0}^n$, where $t_0=0$, and $t_1,\ldots,t_n$ are the event times of the PDMP prior to $T$. Given this skeleton, we can fill in the continuous-time path using the transition function for the deterministic dynamics:
\[
\bz_t = \bPhi(\bz_{t_k},t-t_k), 
\]
where $t_k$ is the largest skeleton time less than $t$. Thus simulating a PDMP will just require the ability to simulate the skeleton points. If we assume that simulating from the transition density at events is straightforward, then the only challenge will be simulating the event times themselves. We can do this by using the following argument that shows that the time until the next event can be recast as the time until the first event in a time inhomogeneous Poisson process. 

To simplify notation we will consider simulating the time of the first event, and denote the initial state as $\bZ_0=\bz$. By the Markov property, the same approach can then be applied to simulating subsequent events: as if the current state is $\bZ_t=\bz$ then the subsequent time {\em until} the next event is the same as the time until the first event if we start the process at state $\bZ_0=\bz$. 

If there has not been an event by time $t$, then due to the deterministic dynamics between events we know that the state at time $t$ will be $\bz_t=\bPhi(\bz,t)$. The instantaneous rate of an event at time $t$, if there has been no event before $t$, is thus $\lambda(\bz_t)=\lambda(\bPhi(\bz,t))$. Thus, the rate of the first event of the PDMP is equal to the rate of the first event in an inhomogeneous Poisson process with rate
\[
\tilde{\lambda}_{\bz}(t)=\lambda(\bPhi(\bz,t)).
\]
Here, we use the tilde symbol to distinguish this rate from the rate function, $\lambda$, that depends on the state. We also subscript the rate by $\bz$, the initial state. As described above, by the Markov property, if we have simulated the PDMP until time $s$ and the current state is $\bz_s$, then the rate of the next event as a function of the further time, $t$, until the next event will be $\tilde{\lambda}_{\bz_s}(t)$.

As a result, we have transformed the problem of simulating a PDMP to that of simulating the first event of a time inhomogeneous Poisson process. There are several methods for simulating such an event time \cite[see e.g.][]{lewis1979simulation,bouchard2018bouncy} and we will outline three general strategies for doing so. Whether these can be implemented in practice depends on the form of $\tilde{\lambda}_{\bz}$, and we will return to this with some examples later.

\subsubsection{Direct Simulation by Inversion}

In theory, one can simulate directly the time to the next event. Let $\tau$ be the random variable that is the time until the first (or next) event. Standard properties of a Poisson process give that the probability of no event by time $t$ is 
\[
\Prob{\tau>t}=\exp\left\{ - \int_0^t \tilde{\lambda}_\bz(s)\mbox{d}s \right\}.
\]
For a continuous random variable, $X$, with distribution function $F_X(\cdot)$, we have that the transformed random variable $F_X(X)$ has a standard uniform distribution. From this, we can simulate $X$ by simulating $u$, a realisation of a standard uniform distribution and setting $x=F_X^{-1}(u)$. The distribution function for $\tau$ is $1-\Prob{\tau>t}$, thus we can simulate $\tau$ from $u$ by solving
\[
u=1-\exp\left\{ - \int_0^\tau \tilde{\lambda}_\bz(s)\mbox{d}s \right\}.
\]
This can be rearranged to 
\[
- \int_0^\tau \tilde{\lambda}_\bz(s)\mbox{d}s = \log(1-u).
\]
In practice, it is common to further simplify this expression using that if $U$ is a standard uniform random variable then $-\log(1-U)$ has a standard exponential distribution. Thus if $w$ is a realisation of a standard exponential random variable then $\tau$ can be simulated as the solution of
\begin{equation} \label{eq:inversion}
    \int_0^\tau \tilde{\lambda}_\bz(s)\mbox{d}s = w.
\end{equation}
A summary of this approach is given in Algorithm \ref{alg:direct}.
   
\begin{algorithm}[h]
    \caption{Direct simulation of event time}
    \KwIn{Rate function $\tilde{\lambda}_{\bz}.$}
    Simulate $w$, from a standard exponential distribution.\\
    Find $\tau\geq0$ the (smallest) solution to
    \[
    \int_0^\tau \tilde{\lambda}_\bz(s)\mbox{d}s = w.
    \] \\
    \KwOut{$\tau$.}
    \label{alg:direct}
\end{algorithm}

Whether we can implement direct simulation depends on whether we can solve (\ref{eq:inversion}). This is possible if $\tilde{\lambda}_{\bz}(s)$ is constant, linear, or piecewise linear in $s$. It is also possible if it is some other simple function, such as proportional to the exponential of a linear function of $s$.  

\subsubsection{Poisson Thinning}
\index{Poisson thinning|(}

What if we cannot simulate the event directly by inversion? In this case, the most common approach to simulation is based on {\em Poisson thinning}. This approach is based on the following property of Poisson processes: If we have a Poisson process with rate $\lambda^+(s)$, and we simulate points from this Poisson process and then accept each point with probability $\alpha(s)$, then the accepted points have the same distribution as points from a Poisson process with rate $\lambda^+(s)\alpha(s)$. As we are only keeping, that is, accepting, some of the simulated points, this is called a \emph{thinned} point process.

Poisson thinning inverts this property. For any rate function $\lambda^+(s)$ that upper bounds $\tilde{\lambda}_{\bz}(s)$, that is, where $\lambda^+(s)\geq \tilde{\lambda}_{\bz}(s)$ for all $s\geq0$, we can simulate the time until the next event as the first event of a thinned Poisson process. This leads to Algorithm \ref{alg:thinning}

\begin{algorithm}
    \caption{Simulation of event time via Poisson thinning}
    \KwIn{Rate functions $\lambda^+$ and $\tilde{\lambda}_{\bz}$, with $\lambda^+(s)\geq \tilde{\lambda}_{\bz}(s)$ for all $s\geq0$.}
    Set $s=0$ and $\tau=0$ \\
    \While{$\tau=0$}{
    Simulate $t$, the time of the first event after $s$ in a Poisson process with rate $\lambda^+$.\\
    Set $s=t$.\\
    With probability $\tilde{\lambda}_{\bz}(t)/\lambda^+(t)$, set $\tau=t$.\\
    }
    \KwOut{$\tau$, the first event in a Poisson process with rate $\tilde{\lambda}_{\bz}$.}
    \label{alg:thinning}
\end{algorithm}

For this to work, we need to be able to upper bound $\tilde{\lambda}_{\bz}$ by a simple rate function, for which we can directly simulate events. In practice, this would often be a linear or piecewise linear rate function. The efficiency of Poisson thinning depends on how close the upper bound rate is to $\tilde{\lambda}_{\bz}$. In practice, we can improve on the simple Poisson thinning algorithm with adaptive thinning. That is, if we simulate an event and then reject it, we can use the information from evaluating $\tilde{\lambda}_{\bz}$ to improve our upper bound. We will see some examples of such adaptive thinning below.
\index{Poisson thinning|)}

\subsubsection{Superposition}
\index{superpositon|(}
The final property of Poisson processes that can help with simulating their events is that of \emph{superposition}. This says that if we have two Poisson processes, one with rate $\lambda^{(1)}(s)$ and one with rate $\lambda^{(2)}(s)$, and we independently simulate events from each process and take the union of events, then these have the same distribution as events in a Poisson process with rate $\lambda^{(1)}(s)+\lambda^{(2)}(s)$. 

In terms of simulating the first event in a Poisson process, this can be re-expressed as if $\tau^{(1)}$ is the first event time for a Poisson process with rate $\lambda^{(1)}(s)$, and $\tau^{(2)}$ is the first event time for a Poisson process with rate $\lambda^{(2)}(s)$, then $\min\{\tau^{(1)},\tau^{(2)}\}$ is distributed as the first event time in a Poisson process with rate $\lambda^{(1)}(s)+\lambda^{(2)}(s)$. 

By induction, superposition trivially applies if we consider more than two independent Poisson processes. That is, the time of the first event in any of the processes is distributed as the time of the first event of a Poisson process whose rate is the sum of the rates. Superposition can be useful as it allows us to split a complex rate function into a sum of simpler rate functions. If we can simulate events from processes with each of the simpler rates, then it allows us to simulate events from the process with the complex rate function.
\index{piecewise deterministic Markov processes|)}\index{superpositon|)}

\subsection{The Generator and Invariant Distribution of a PDMP}
\label{sec:ch6-generator}
\index{generator!of PDMP|(}
\index{piecewise deterministic Markov processes|(}
In order to use a PDMP to sample from a target distribution, we first need to be able to determine what the stationary distribution\index{stationary distribution} of a given PDMP is. Here, we give an informal derivation of how to calculate the invariant distribution of a PDMP. (Assuming the PDMP satisfies some regularity conditions, and in particular is irreducible, then this invariant distribution will be its stationary distribution.) In the next section, we will invert this characterisation of the invariant distribution to construct simple recipes for the dynamics of a PDMP to have a given target distribution as its stationary distribution\index{stationary distribution}. In terms of understanding the development of PDMP samplers in the rest of this chapter, the key result is (\ref{eq:PDMPinvariant}) below -- and those not interested in the intuition behind this result could skip the intervening material in this subsection.

First, we need to consider the generator of our PDMP. This is rigorously derived in \cite{davis1984piecewise}. Remember from Section \ref{sec:generator} that the generator of a continuous-time Markov process\index{Markov chain!continuous-time} is an operator that gives the time derivative of the expected value of a function of the state, as a function of its current value. If $\mathcal{L}$ is the generator, and $h$ a suitable test function from the domain of the generator, then
\[
(\mathcal{L}h)(\bz)=\lim_{t \downarrow 0} \frac{\Expect{h(\bZ_{ t})|\bZ_0=\bz}-h(\bz)}{t}.
\]

Informally, we can calculate the generator by considering the two types of dynamics of the PDMP. First, if we consider solely the deterministic dynamics (\ref{eq:PDMPODE}), then the change in $h(\bz_t)$ is deterministic and the contribution to the generator is just the time-derivative of $h(\bz_t)$ at $t=0$, which by the product rule is
\[
\left.
\frac{\mbox{d} h(\bz_t)}{\mbox{d} t}\right|_{t=0} = \bphi(\bz) \cdot \nabla h(\bz).
\]
Second, we have the contribution from the random events. The probability of an event in $[0, t]$ is $\lambda(\bz) t + o( t)$. If an event occurs, the change in $h(\bz)$ is $\Expects{\Ker{\cdot|\bz}}{h(\bZ')}-h(\bz)+o(t)$, where the expectation is with respect to $\bZ' \sim \mathbb{Q}(\cdot | \bz)$, the transition kernel\index{kernel!Markov transition} at an event. 
This gives a contribution to the generator that is 
\[
\lambda(\bz) \left[ \Expects{\Ker{\cdot|\bz}}{h(\bZ')}-h(\bz)\right].
\]
Thus the generator is
\[
(\mathcal{L}h)(\bz) = \bphi(\bz) \cdot \nabla h(\bz) + \lambda(\bz) \left[ \Expects{\Ker{\cdot|\bz}}{h(\bZ')}-h(\bz)\right].
\]
The domain of the generator is given in \cite{davis1984piecewise}. One extension of PDMPs that will be relevant later is that we can introduce boundaries, with additional specified, possibly random, behaviour if the PDMP hits a boundary. In such a case, the behaviour at the boundaries affects the generator solely through its domain. Essentially, the domain is reduced to include only those functions whose expectation is unaffected by the dynamics on the boundary.

If we start the PDMP with an initial distribution $\pi(\bz)$ for $\bZ_0$, then the derivative of $\Expect{h(\bZ_t)}$ at $t=0$ is equal to the average of the generator applied to $h(\bz)$ with respect to $\pi(\bz)$. This is equal to
\[
\int (\mathcal{L}h)(\bz)\pi(\bz)\mbox{d}\bz.
\]
For any $h$ in the domain of the generator, if $\pi(\bz)$ is sufficiently well-behaved, then this integral will be equal to
\begin{equation} \label{eq:changeinmean}
\int h(\bz)(\mathcal{L}^*\pi)(\bz)\mbox{d}\bz,
\end{equation}
where $\mathcal{L}^*$ is a different operator, called the adjoint of the generator. 

We can attempt to define the invariant distribution\index{invariant distribution} of the PDMP by the property that, if we draw $\bZ_0$ from this invariant distribution, then the expectation of any function of the state will be constant over time -- as if started from the invariant distribution, the marginal distribution of the PDMP will not change. Thus  (\ref{eq:changeinmean}) must be equal to 0 if $\pi$ is the invariant distribution. As this must happen for any function $h$ of the state, for which the expectation exists, this suggests that $\pi$ must satisfy $(\mathcal{L}^*\pi)(\bz)=0$. 

It is possible to derive the adjoint $\mathcal{L}^*$ using integration by parts. From this, we have that $(\mathcal{L}^*\pi)(\bz)=0$ implies that the invariant distribution must satisfy
\begin{equation} \label{eq:PDMPinvariant}
-\sum_{i=1}^d \frac{\partial \phi_{i}(\bz)\pi(\bz)}{\partial z_{i}}
+ \sum_{\bz'} \pi(\bz')\lambda(\bz')q(\bz|\bz')
-\pi(\bz)\lambda(\bz)=0,
\end{equation}
where $q(\bz|\bz')$ is the probability mass function associated with the transition kernel $\Ker{\cdot|\bz'}$. 
This equation has a natural interpretation. The first term on the left-hand side quantifies the change in probability mass due to the deterministic dynamics, the second term is the change due to events moving into state $\bz$ and the last is the change due to events that move the state out of $\bz$. If $\pi$ is an invariant distribution, then the net change in probability mass is zero.

In the following, we will use (\ref{eq:PDMPinvariant}) to define the invariant distribution of our PDMP. Though this requires inverting the informal argument we have made -- see \cite{chevallier2021pdmp} for results that give relatively weak conditions under which you can invert this argument and where (\ref{eq:PDMPinvariant}) implies that $\pi(\bz)$ is the invariant distribution\index{invariant distribution} of our PDMP.
\index{generator!of PDMP|)}
\index{piecewise deterministic Markov processes|)}
\subsection{The Limiting Process of Section \ref{sec:ch6-limitexample} as a PDMP} \label{sec:ch6-limitasPDMP}

We can now recognise the limiting process derived in Section \ref{sec:ch6-limitexample} as a PDMP. Remember that we want to sample from a distribution $\pi(\theta)$ for some scalar $\theta$. To do this, we introduced a velocity or momentum, $p$, and defined a state $\bz=(\theta,p)$ -- strictly this is  $\bz=(\theta,p)^{\top}$, but we will use the shorthand $(\theta,p)$ in the following. Henceforth, we will use the notation $\bz$, or $\bz_t$, and $(\theta,p)$, or $(\theta_t,p_t)$ interchangeably -- as viewed most helpful for the given context. For reasons that will become apparent, we will call $\theta$ the position component of the state, and $p$ the velocity component. 

The limiting process of Section \ref{sec:ch6-limitexample} was a PDMP\index{piecewise deterministic Markov processes} with state $\bz=(\theta,p)$, with $\theta\in \mathbb{R}$ and $p\in\{-1,1\}$, defined by the following properties:
\begin{itemize}
\item[(U1)] {\em Deterministic dynamics.} The deterministic dynamics are a constant velocity model:
\[
\frac{\mbox{d}\theta_t}{\mbox{d}t} = p_t, ~~ \frac{\mbox{d}p_t}{\mbox{d}t} = 0.
\]
\item[(U2)] {\em Event rate.} The rate of events is 
\[
\lambda(\theta,p)= \max\left\{0, -p \frac{\mbox{d}\log \pi(\theta)}{\mbox{d}\theta}  \right\}.
\]
\item[(U3)] {\em Transition kernel at events.}\index{kernel!Markov transition} At an event, the velocity component of the state flips, that is, $p_\tau=-p_{\tau-}$.
\end{itemize}
We will call this PDMP the \emph{univariate} PDMP.

It is possible to show, using (\ref{eq:PDMPinvariant}), that the invariant distribution of this PDMP is $\tilde{\pi}(\bz)=\pi(\theta)\pi_p(p)$, where $\pi_p$ is the probability mass function for a uniform distribution on $\{-1,1\}$. 
That is the $\theta$-marginal is $\pi(\theta)$, the distribution we wish to sample from. Furthermore, under the invariant distribution, $p$ is independent of $\theta$ and has a uniform distribution. For this simple example, the invariant distribution will also be the stationary distribution\index{stationary distribution}, unless we have reducibility caused by a region where $\pi(\theta)=0$ that separates two regions with positive probability under $\pi$.

As we will be considering more general PDMP samplers, it is helpful to consider a slightly different question. For PDMPs with the given deterministic dynamics and transitions at events, how would we calculate event rates that would give an invariant distribution whose $\theta$ marginal is $\pi(\theta)$? In answering this question, we will cover the steps for showing that the event rate given above leads to the stated invariant distribution.

To do this we will use (\ref{eq:PDMPinvariant}). If we substitute in a distribution $\tilde{\pi}(\theta,p)$,  the deterministic dynamics and transition kernel at events, this illustrates that for $\tilde{\pi}(\theta,p)$ to be an invariant distribution, it must satisfy
\begin{equation} \label{eq:ch6-simple_ex_invariant}
-p\frac{\mbox{d}\tilde{\pi}(\theta,p)}{\mbox{d}\theta}
+\lambda(\theta,-p)\tilde{\pi}(\theta,-p) 
-\lambda(\theta,p)\tilde{\pi}(\theta,p)=0.
\end{equation}
Here, the first term comes from the constant velocity deterministic dynamics, and the second comes from there only being one possible transition to state $(\theta,p)$ at an event, and this is from a state $(\theta,-p)$.

So what event rate would lead to an invariant distribution with the correction $\theta$-marginal? 
In answering this question, we first see that there may be a range of different event rates that would lead to a valid invariant distribution. Not least because many possible invariant distributions would have a $\theta$-marginal as $\pi(\theta)$. So, our first step is to attempt to find a rate such that the invariant distribution has $\theta$ independent of $p$. Denote such a distribution by $\tilde{\pi}(\bz)=\pi(\theta)\pi_p(p)$, where $\pi_p$ can be any distribution on $\{-1,1\}$. Then substituting this into (\ref{eq:ch6-simple_ex_invariant}), and using 
\[
\frac{\mbox{d}  \pi(\theta)}{\mbox{d}\theta} =
\pi(\theta)\frac{\mbox{d} \log \pi(\theta)}{\mbox{d}\theta}
\]
gives
\[
-p\frac{\mbox{d}\log{\pi}(\theta)}{\mbox{d}\theta} \pi(\theta)\pi_p(p) 
+ \lambda(\theta,-p) \pi(\theta)\pi_p(-p)
-\lambda(\theta,p)\pi(\theta)\pi_p(p)=0.
\]
If we consider $\theta$ for which $\pi(\theta)>0$ then this states
\begin{equation} \label{eq:ch6-1b}
\lambda(\theta,-p) \pi_p(-p) - \lambda(\theta,p)\pi_p(p)=p\frac{\mbox{d}\log{\pi}(\theta)}{\mbox{d}\theta}\pi_p(p) .
\end{equation}
If we consider the same argument, but for the state $(\theta,-p)$, we get
\begin{equation} \label{eq:ch6-1}
\lambda(\theta,p) \pi_p(p) - \lambda(\theta,-p)\pi_p(-p)=-p\frac{\mbox{d}\log{\pi}(\theta)}{\mbox{d}\theta}\pi_p(-p).
\end{equation}
Adding (\ref{eq:ch6-1b}) to (\ref{eq:ch6-1}) gives
\[
0=\frac{\mbox{d}\log{\pi}(\theta)}{\mbox{d}\theta} \left(p\pi_p(p)-p\pi_p(-p) \right).
\]
From this, we can conclude that the distribution $\pi_p$ must satisfy $\pi_p(p)=\pi_p(-p)$, that is, be the uniform distribution on $\{-1,1\}$. This makes sense intuitively. The transition at the events only changes the velocity. Thus invariance for the $\theta$-component comes from averaging out the dynamics for different $p$, and this requires that the invariant distribution for $p$ has a mean of zero. As $p$ can only take two values, this means it must be the uniform distribution. For the PDMPs that we consider later, which only change the velocity at events and have constant velocity dynamics\index{constant velocity dynamics}, a similar argument holds that the invariant distribution for the velocity component must have zero as the mean.

If we now return to our question of what rates will lead to an invariant\index{invariant distribution} distribution with $\theta$-marginal equal to $\pi(\theta)$, and substitute in (\ref{eq:ch6-1}) that $\pi_p(p)=\pi_p(-p)=1/2$ for $p\in\{-1,1\}$ then, by removing this common factor, we get
\[
\lambda(\theta,p)  - \lambda(\theta,-p)=-p\frac{\mbox{d}\log{\pi}(\theta)}{\mbox{d}\theta}.
\]
A solution to this equation is the rate we specified above, 
\begin{equation} \label{eq:ch6-canonical_rate}
\lambda(\theta,p)= \max\left\{0, -p \frac{\mbox{d}\log \pi(\theta)}{\mbox{d}\theta}  \right\}.
\end{equation}
However, this is not the only solution. In fact, for any positive function $\gamma(\theta)\geq0$, the rate
\[
\max\left\{0, -p \frac{\mbox{d}\log \pi(\theta)}{\mbox{d}\theta}  \right\} +\gamma(\theta),
\]
will also lead to the same invariant distribution. \index{invariant distribution}

The rate in (\ref{eq:ch6-canonical_rate}) is the smallest rate that will lead to the required invariant distribution. This is often called the {\em canonical rate}. Intuitively, there are two advantages of using the canonical rate, as opposed to a larger rate. The first is that a larger rate will lead to more events, and thus is likely to have a larger computational cost for simulating the resulting PDMP. The second is that a larger rate will lead to more changes in velocity and will re-introduce the random walk behaviour that we were trying to avoid with non-reversible MCMC\index{Markov chain Monte Carlo}\index{non-reversible MCMC}. Thus, we would expect that using the canonical rate will lead to better mixing\index{mixing}.

A final comment on the rates we are required to use, such as the canonical rate, is that they depend on the target distribution $\pi(\theta)$ only through the derivative of $\nabla \log \pi(\theta)$. 
This is important as it means that $\pi(\theta)$ only needs to be known up to a constant of proportionality, as is commonly the case for sampling from the posterior distribution\index{posterior distribution} in Bayesian statistics, see Section \ref{sec.MCBayes}.

\section{Continuous-time MCMC via PDMPs} \index{Markov chain Monte Carlo|(}
\index{piecewise deterministic Markov processes|(}
In practice, we will want to use MCMC to sample a target density in $\mathbb{R}^d$. Various PDMPs generalise the process introduced in the previous section to $d>1$. We will describe three such families of PDMPs, all of which reduce to the univariate PDMP of Section \ref{sec:ch6-limitasPDMP}  if $d=1$ but differ for $d>1$. They each share some common features, which we will describe first.

Assume we wish to sample from $\pi(\btheta)$ where $\btheta$ is $d$-dimensional. The state of our PDMP will be $\bz_t=(\btheta_t^{\top},\bp_t^{\top})^{\top}$, where $\bp_t$ is also $d$-dimensional. As we use the convention that vectors are column vectors, when defining $\bz_t$ we have had to transpose these vectors to concatenate $\btheta_t$ and $\bp_t$. In the following, to simplify notation, we will abuse this and just write  $\bz_t=(\btheta_t,\bp_t)$. As before, we can think of $\btheta$ as the position component of the state, and $\bp$ as the velocity. 

The deterministic dynamics of the three families of PDMPs will be the same:
\begin{itemize}
\item[(CV)] {\em Deterministic dynamics.} 
The process evolves according to a {\emph{Constant Velocity (CV)}} model.
\begin{equation} \label{eq:ch6-constant velocity}
\frac{\mbox{d}\btheta_t}{\mbox{d}t} = \bp_t,~~~
\frac{\mbox{d}\bp_t}{\mbox{d}t} = \bf{0},
\end{equation}
or in the notation we used to define PDMPs, $\bphi=(\bp,\bf{0})$. The solution of the deterministic dynamics are
\begin{equation} \label{eq:ch6-Phi constant velocity}
\bPhi(\bz,t)=\bPhi( (\btheta,\bp),t ) = (\btheta+t \bp, \bp).
\end{equation}
\end{itemize}
Furthermore, they also only allow the velocity component to change at events.

The three families of PDMPs will differ in terms of the possible values for the velocity component, the event rate, and the transition kernel for the velocity at events. In each case, these are chosen so that the invariant distribution of the PDMP will have a $\btheta$-marginal that is $\pi(\btheta)$, and for which the velocity component, $\bp$, is independent of the position $\btheta$. Throughout, we will denote the invariant distribution\index{invariant distribution} as $\tilde{\pi}(\btheta,\bp)=\pi(\btheta)\pi_{\bp}(\bp)$, though as mentioned, the form of $\pi_{\bp}$ will differ between different families of PDMPs.

Before we describe the three families in detail, it is helpful to introduce some notation. We will use $\nabla_{\btheta}$ to denote the gradient\index{gradient} vector with respect to $\btheta$ only. This is the $d$-dimensional column vector whose entries are the partial derivatives with respect to the components of $\btheta$. The first term in the equation for the invariant distribution of the PDMP (\ref{eq:PDMPinvariant}) will be the same for all three families as they share the same deterministic dynamics and form of the invariant distribution\index{invariant distribution}. Ignoring the minus sign, this term is
\begin{eqnarray}
\sum_{i=1}^{2d} \frac{\partial \phi_{i}(\bz)\tilde{\pi}(\bz)}{\partial z_{i}}
&=& \sum_{i=1}^d \frac{\partial p_{i}\tilde{\pi}(\btheta,\bp)}{\partial \theta_{i}}
= \pi_{\bp}(\bp) \sum_{i=1}^d p_{i} \frac{\partial \pi(\btheta)} {\partial \theta_{i}}  \nonumber \\
&=& \pi_{\bp}(\bp) \sum_{i=1}^d \pi(\btheta) p_{i}  \frac{\partial \log \pi(\btheta)} {\partial \theta_{i}} \nonumber \\
&=&  \tilde{\pi}(\btheta,\bp) ~\left( \bp \cdot \nabla_{\btheta}\log \pi(\btheta)\right).
\end{eqnarray}
Here, we have first used that only the $\theta_{i}$ components are changing, and then that $\tilde{\pi}(\btheta,\bp)=\pi(\btheta)\pi_{\bp}(\bp)$. The third step comes from the definition of the derivative of $\log \pi(\btheta)$ in terms of the derivative of $\pi(\btheta)$, and the final step from using $\pi(\btheta)\pi_{\bp}(\bp)=\tilde{\pi}(\btheta,\bp)$.
\index{Markov chain Monte Carlo|)}
\subsection{Different Samplers}

\subsubsection{The Coordinate Sampler}
\index{coordinate sampler|(}
Possibly the simplest extension of our univariate PDMP to one that samples from a multi-dimensional distribution is the \emph{Coordinate Sampler} of \cite{wu2020coordinate}. For this sampler, the set of possible velocities is $\mathcal{V}_{\text{cs}}=\{\pm \be_i\}_{i=1}^d$ where $\be_i$ is the $i$th unit vector.
That is, $\be_i$ is the unit vector whose $i$th component is 1, and all other components are 0. Thus, the possible velocities correspond to moving in either a positive or negative direction along one of the coordinate axes in $\mathbb{R}^d$. It can be viewed as a sampler that applies the univariate PDMP dynamics along each coordinate in turn -- though the order in which different coordinate directions are chosen is random. Introducing a constant \emph{refresh rate} $\refr\geq0$\index{refresh event/rate}, the dynamics of the coordinate sampler involve the \emph{constant velocity} (CV) deterministic dynamics together with
\begin{itemize}
\item[(CS1)] {\em Event rate.} Events occur with the rate 
\[
\lambda_{\text{cs}}(\btheta,\bp) = \max\left\{0, - \bp\cdot  \nabla_{\btheta} \log \pi(\btheta) \right\} +\refr.
\]
\item[(CS2)] {\em Transition kernel at events.}\index{kernel!Markov transition} At an event, the probability of switching to a new velocity $\bp'\in\mathcal{V}_{\text{cs}}$ is
\[
q_{\text{cs}}( (\btheta,\bp') | (\btheta,\bp) ) = \frac{1}{C(\btheta)} \left( \max\left\{0,  \bp' \cdot \nabla_{\btheta} \log \pi(\btheta) \right\} +\refr\right), 
\]
where the normalising constant is
\[
C(\btheta) =\sum_{\bp'\in \mathcal{V}_{\text{cs}}} \left( \max\left\{0,  \bp' \cdot \nabla_{\btheta} \log \pi(\btheta) \right\} +\refr\right) =
2d\refr + \sum_{i=1}^d \left| \frac{\partial \log \pi(\btheta)}{\partial \theta_{i}}\right|. 
\]
\end{itemize}
The refresh rate\index{refresh event/rate} $\refr$ introduces additional random velocity switches. As discussed above, intuitively a larger refresh rate will lead to more random walk behaviour and thus worse mixing\index{mixing}. However, choosing $\refr>0$ allows for stronger theoretical results about the sampler, including that the sampler will be irreducible\index{irreducibility} unless, for example, $\pi(\btheta)$ has disconnected regions where there is positive probability.

The invariant distribution\index{invariant distribution} of the Coordinate Sampler is given by the following result.
\begin{theorem} For any $\refr\geq0$, the Coordinate Sampler, whose dynamics are defined by (CV), (CS1) and (CS2), has an invariant distribution\index{invariant distribution} $\tilde{\pi}(\btheta,\bp)=\pi(\btheta)\pi_{\bp}(\bp)$   where $\pi_{\bp}$ is the uniform distribution over $\mathcal{V}_{\text{cs}}$.  \label{thm:CS}
\end{theorem}
\begin{proof} We show this result by showing that (\ref{eq:PDMPinvariant}) holds. To simplify expressions slightly, we will use the notation that for any scalar $a$ we have $\{a\}_+ = \max\{0,a\}$. 

Substituting in the event rate and transition kernel and distribution $\tilde{\pi}$, for $\bp \in \mathcal{V}_{\text{cs}}$, the left-hand side of (\ref{eq:PDMPinvariant}) is
\begin{eqnarray*}
\lefteqn{\tilde{\pi}(\btheta,\bp) \left( -\bp \cdot \nabla_{\btheta}\log \pi(\btheta)\right) - \tilde{\pi}(\btheta,\bp) (\left\{ - \bp \cdot \nabla_{\btheta} \log \pi(\btheta) \right\}_+ +\refr )} \\ &+& \sum_{\bp'\in\mathcal{V}_{\text{cs}}} (\left\{- \bp' \cdot \nabla_{\btheta} \log \pi(\btheta) \right\}_+ +\refr)\tilde{\pi}(\btheta,\bp')q_{\text{cs}}((\btheta,\bp) | (\btheta,\bp')), 
\end{eqnarray*}
where we have used (\ref{eq:ch6-constant velocity}) to simplify the first term.
By the definition of $\tilde{\pi}$, we have $\tilde{\pi}(\btheta,\bp)=\tilde{\pi}(\btheta,\bp')$ for all $\bp,\bp' \in \mathcal{V}_{\text{cs}}$. Thus this simplifies to
\begin{eqnarray*}
\lefteqn{ - \bp \cdot \nabla_{\btheta}\log \pi(\btheta) -  (\left\{- \bp \cdot \nabla_{\btheta} \log \pi(\btheta) \right\}_+ +\refr )} \\ &+& \sum_{\bp'\in\mathcal{V}_{\text{cs}}} (\left\{ - \bp' \cdot \nabla_{\btheta} \log \pi(\btheta) \right\}_+ +\refr)q_{\text{cs}}((\btheta,\bp) | (\btheta,\bp')). 
\end{eqnarray*}
Now using the definition of $q_{\text{cs}}$, we get that the final term in this expression is
\begin{eqnarray*}
\lefteqn{\sum_{\bp'\in\mathcal{V}_{\text{cs}}} (\left\{- \bp' \cdot \nabla_{\btheta} \log \pi(\btheta) \right\}_+ +\refr)q_{\text{cs}}((\btheta,\bp) | (\btheta,\bp'))}\\
&= & \sum_{\bp'\in\mathcal{V}_{\text{cs}}} (\left\{ - \bp' \cdot \nabla_{\btheta} \log \pi(\btheta) \right\}_+ +\refr) \frac{1}{C(\btheta)} \left(\left\{  \bp \cdot \nabla_{\btheta} \log \pi(\btheta) \right\}_+ +\refr\right) \\
&=& \left( \left\{ \bp \cdot \nabla_{\btheta} \log \pi(\btheta) \right\}_+ +\refr\right) \frac{1}{C(\btheta)}\sum_{\bp'\in\mathcal{V}_{\text{cs}}} (\left\{ - \bp' \cdot \nabla_{\btheta} \log \pi(\btheta) \right\}_+ +\refr) \\
&=& \left( \left\{ \bp \cdot \nabla_{\btheta} \log \pi(\btheta) \right\}_+ +\refr\right) ,
\end{eqnarray*}
where for the last line we use
\[
C(\btheta)= \sum_{\bp'\in\mathcal{V}_{\text{cs}}} (\left\{ - \bp' \cdot \nabla_{\btheta} \log \pi(\btheta) \right\}_+ +\refr).
\]
Substituting into our expression for the left-hand side of (\ref{eq:PDMPinvariant}) gives
\begin{eqnarray*}
\lefteqn{ -\bp \cdot \nabla_{\btheta}\log \pi(\btheta) -  (\left\{ - \bp \cdot \nabla_{\btheta} \log \pi(\btheta) \right\}_+ +\refr ) + 
\left( \left\{  \bp \cdot \nabla_{\btheta} \log \pi(\btheta) \right\}_+ +\refr\right) }\\ &=&  
- \bp \cdot \nabla_{\btheta}\log \pi(\btheta)  -  (\left\{ - \bp \cdot \nabla_{\btheta} \log \pi(\btheta) \right\}_+) + \left( \left\{  \bp \cdot \nabla_{\btheta} \log \pi(\btheta) \right\}_+ \right).
\end{eqnarray*}
By considering separately the cases where $\bp \cdot \nabla_{\btheta} \log \pi(\btheta) \geq0$ and $\bp \cdot \nabla_{\btheta} \log \pi(\btheta) <0$, it is simple to see that this expression for the left-hand side of (\ref{eq:PDMPinvariant})  is 0, as required.
\end{proof}
\index{coordinate sampler|)}

\subsubsection{The Zig--Zag Sampler} \label{sec:ch6-ZZ}
\index{Zig-Zag sampler|(}
We now present the \emph{Zig--Zag} algorithm of \cite{Bierkens:2019}. This algorithm has velocities in $\mathcal{V}_{\text{zz}}=\{\pm 1\}^d$, and the state moves simultaneously along each coordinate axis, and the velocity determines which direction it moves for each axis. There are $2^d$ possible velocities, and if for example $d=2$, these will be $(1,1)$, $(1,-1)$, $(-1,1)$ and $(-1,-1)$. At an event, one component of the velocity will change signs. The sampler gets its name from the resulting dynamics consisting of zig-zagging lines.

To define the dynamics of the Zig--Zag Sampler it is helpful to introduce coordinate-specific rates
\[
\lambda_{i}(\btheta,\bp)= \max\left\{0, -p_{i} \frac{\partial \log \pi(\btheta)}{\partial \theta_{i}} \right\},
\]
which is of the same form as the canonical rate of the univariate PDMP if we just vary that $i$th component of $\btheta$. We also introduce the functions $F_i$, for $i=1,\ldots,d$, which flips the sign of the $i$th component of a vector. So if $\bp'=F_i(\bp)$ then $p'_{i}=-p_{i}$ and, for $j\neq i $, $p'_{j}=p_{j}$.

The PDMP process is defined by CV dynamics together with
\begin{itemize}
\item[(ZZ1)] {\em Event rate.} Events occur with the rate
\[
\lambda_{\text{zz}}(\btheta,\bp) = \sum_{i=1}^{d} \lambda_{i}(\btheta,\bp).
\]
\item[(ZZ2)] {\em Transition kernel at events.}\index{kernel!Markov transition} At an event, the probability mass function of the transition is
\[
q_{\text{zz}}(\btheta,\bp'|\btheta,\bp) = \frac{\lambda_i(\btheta,\bp)}{\lambda_{\text{zz}}{(\btheta,\bp)}}, ~~ \mbox{ for } \bp'=F_i(\bp).
\]
Thus, the position is unchanged and we flip component $i$ of the velocity with probability proportional to $\lambda_i(\btheta,\bp)$.
\end{itemize}

Here, we have presented the dynamics in terms of the rate of an event and a transition probability for that event. However, by the superposition\index{superpositon}
 property of Poisson processes that was discussed above, one can equivalently represent the dynamics in terms of $d$ possible event types. Event type $i$ corresponds to flipping the $i$th component of the velocity, and this event occurs, independently of other events, with rate $\lambda_i$. This view of the dynamics of the Zig--Zag algorithm is often used in algorithmic implementations to sample realisations of the process.

We can relate the Zig--Zag Sampler to a limiting version of the guided random walk algorithm\index{guided random walk} of \cite{gustafson1998guided} in a similar way to the argument presented in Section \ref{sec:ch6-limitexample}. The Zig--Zag Sampler is the limit of an MCMC algorithm that repeatedly applies one iteration of the guided random walk algorithm to each component of $\btheta$ in turn. 

The following result gives the invariant distribution\index{invariant distribution} of the PDMP process.
\begin{theorem}
The Zig--Zag Sampler, defined by (CV), (ZZ1) and (ZZ2) has invariant distribution $\tilde{\pi}(\btheta,\bp)=\pi(\btheta)\pi_{\bp}(\bp)$   where $\pi_{\bp}$ is the uniform distribution over $\mathcal{V}_{\text{zz}}$.   \label{thm:ZZ}
\end{theorem}
\begin{proof}
Again, we show this result by showing that (\ref{eq:PDMPinvariant}) holds. By the same argument as in the first step of the proof of Theorem \ref{thm:CS}, if we substitute the form of $\tilde{\pi}$ and the definition of the dynamics of the Zig--Zag Sampler into the left-hand side of (\ref{eq:PDMPinvariant}) we get that this is proportional to
\begin{equation} \label{eq:ch6-ZZ1}
 - \bp \cdot \nabla_{\btheta}\log \pi(\btheta) - \sum_{i=1}^d \lambda_i(\btheta,\bp) + \sum_{i=1}^d \lambda_{i}(\btheta,F_i(\bp)).    
\end{equation}
The first term relates to the change in probability mass due to the deterministic dynamics and is the same term as appeared in the calculations for the Coordinate Sampler. The second term is the rate of leaving the state $(\btheta,\bp)$, and the third is the rate of moving to the state $(\btheta,\bp)$ which has to be from a state of the form $(\btheta,F_i(\bp))$. As in the argument for the Coordinate Sampler, we have removed the $\tilde{\pi}$ terms as these are the same for $(\btheta,\bp)$ and $(\btheta, F_i(\bp))$ for all $i$.

To simplify this expression we use 
\[
-\lambda_i(\btheta,\bp)+\lambda(\btheta,F_i(\bp))=p_{i}\frac{\partial \log \pi(\btheta)}{\partial \theta_{i}},
\]
and
\[
\bp \cdot \nabla_{\btheta}\log \pi(\btheta) = \sum_{i=1}^d p_{i}\frac{\partial \log \pi(\btheta)}{\partial \theta_{i}}.
\]
Thus (\ref{eq:ch6-ZZ1}) becomes
\[
\sum_{i=1}^d \left(-p_{i}\frac{\partial \log \pi(\btheta)}{\partial \theta_{i}} + p_{i}\frac{\partial \log \pi(\btheta)}{\partial \theta_{i}} \right)=0,
\]
as required.
\end{proof}
\index{Zig-Zag sampler|)}

\subsubsection{Bouncy Particle Sampler} \label{sec:ch6-BPS}
\index{bouncy particle sampler|(}\index{BPS|see{bouncy particle sampler}}
The third sampler that we introduce is the \emph{Bouncy Particle Sampler}, which was first introduced as a way of simulating particle systems in statistical mechanics \cite[]{peters2012rejection}, but was then proposed as a general sampling algorithm by \cite{bouchard2018bouncy}. It can be derived as a continuous-time limit of the Discrete Bouncy Particle Sampler\index{discrete bouncy particle sampler} that was introduced in Section \ref{sec.DBPS}. Like that algorithm, the transitions at events are reflections\index{reflection} of the velocity in the contours of $\log \pi(\btheta)$. 

For a $d$-dimensional vector $\bg$,  let $\bghat= \bg/(\bg\cdot \bg)^{1/2}$ be the unit vector in the direction of $\bg$. As in Section \ref{sec.DBPS}, define the function $\cR_{\bg}(\bp)=\bp-2(\bp\cdot \bghat)~\bghat$, to be the reflection of $\bp$ in the hyperplane perpendicular to $\bg$.  An important property of a reflection, that we will use below, is that it preserves the size of the vector. That is
\begin{eqnarray*}
\cR_{\bg}(\bp)\cdot \cR_{\bg}(\bp) &=& \left(\bp-2(\bp\cdot \bghat)~\bghat \right)\cdot\left(\bp-2(\bp\cdot \bghat)~\bghat \right) \\
&=& \bp\cdot\bp - 4(\bp\cdot \bghat)~\bp\cdot\bghat + 4(\bp\cdot \bghat)^2 \bghat\cdot\bghat \\
&=& \bp\cdot\bp -4(\bp\cdot \bghat)^2 + 4(\bp\cdot \bghat)^2  = \bp\cdot\bp,
\end{eqnarray*}
where for the penultimate equality we have used that $\bghat$ is a unit vector. Also, reflection is an involution, that is $\cR_{\bg}(\cR_{\bg}(\bp))=\bp$. To see this
\begin{eqnarray*}
\cR_{\bg}(\cR_{\bg}(\bp)) &=& \cR_{\bg}\left(\bp-2(\bp\cdot \bghat)~\bghat \right) \\
&=& \bp-2(\bp\cdot \bghat)~\bghat  - 2 \{(\bp-2(\bp\cdot \bghat)~\bghat)\cdot \bghat\}~\bghat \\
&=& \bp-2(\bp\cdot \bghat)~\bghat - 2 \{\bp\cdot\bghat-2(\bp\cdot \bghat)\}~\bghat \\
&=& \bp-2(\bp\cdot \bghat)~\bghat  + 2 (\bp\cdot \bghat)~\bghat =\bp. 
\end{eqnarray*}

There are two versions of the Bouncy Particle Sampler, that differ only in the set of possible velocities. We will mainly work with the version where the velocities take values in $\mathbb{R}^d$, and have an invariant distribution\index{invariant distribution} that is standard normal. For any refresh rate\index{refresh event/rate} $\refr\geq0$, the Bouncy Particle Sampler in this case is a PDMP with constant velocity dynamics (CV) \index{constant velocity dynamics} and
\begin{itemize}
\item[(BPS1)] {\em Event rate.} Events occur at a rate
\[
\lambda_{\text{BPS}}(\btheta,\bp)=\max\{0, -\bp\cdot\nabla_{\btheta}\log \pi(\btheta)\}+\refr.
\]
\item[(BPS2)] {\em Transition at events.} At an event with probability $1-\refr/\lambda_{\text{BPS}}(\btheta,\bp)$, reflect the velocity in the hyperplane perpendicular to $\nabla_{\btheta}\log \pi(\btheta)$, that is the new velocity is
\[
\bp'=\cR_{\bg}(\bp),\mbox{ with }\bg=\nabla_{\btheta}\log \pi(\btheta);
\]
otherwise sample a new velocity, $\bp'$ from a standard normal distribution. 
The position is unchanged at an event.
\end{itemize}

As with the Zig--Zag Sampler, we can interpret the dynamics in terms of events of different types. In this case, we have reflection\index{reflection} events that occur with rate $\max\{0, -\bp\cdot\nabla_{\btheta}\log \pi(\btheta)\}$, and refresh events\index{refresh event/rate} that occur with rate $\refr$. If $\refr>0$ then \cite{bouchard2018bouncy} prove that the resulting process is irreducible\index{irreducibility}, assuming weak conditions on $\pi(\btheta)$. Furthermore, they give an example where if $\refr=0$, the sampler will be reducible\index{reducible}, and this occurs if we were to use the Bouncy Particle Sampler to sample from a Gaussian distribution. As many target distributions can be close to Gaussian, we may have a reducible sampler, or one which mixes slowly if $\refr=0$. 
In practice, tuning $\refr$ is important, as not only can the sampler mix poorly for $\refr\approx0$, but if we choose $\refr$ too large it will introduce random walk behaviour which will also lead to poor mixing\index{mixing}. We will return to this issue later in this section and in Section \ref{sec:ch6-comparison}

The alternative version of the algorithm has velocities that lie on the unit $d$-dimensional hypersphere. The only difference in terms of the dynamics is that at a refresh event, we sample a new velocity from the uniform distribution on the sphere rather than from a standard normal distribution. Furthermore, there are extensions of the Bouncy Particle Sampler that only partially refresh the velocity. That is at a refresh event we sample a new velocity from a Markov kernel which has a standard normal distribution (or for the alternative version a uniform distribution on the sphere) as its stationary distribution\index{stationary distribution}.

The following result gives the invariant distribution\index{invariant distribution} of the Bouncy Particle Sampler.
\begin{theorem} \label{thm:BPS}
For any $\refr\geq0$, the Bouncy Particle Sampler, whose dynamics are defined by (CV), (BPS1) and (BPS2), has an invariant distribution\index{invariant distribution} $\tilde{\pi}(\btheta,\bp)=\pi(\btheta)\pi_{\bp}(\bp)$   where $\pi_{\bp}$ is the density of a $d$-dimensional standard normal distribution.
\end{theorem}
\begin{proof}
As our presentation of PDMPs has focussed on discrete transitions at events, we will prove the result for $\refr=0$ only. The extension to $\refr>0$ is straightforward, as the additional refresh rates\index{refresh event/rate} trivially keep $\tilde{\pi}$ invariant.

As shown above, a reflection\index{reflection} does not change the length of a vector. As for the proposed invariant distribution\index{invariant distribution}, $\pi_{\bp}(\bp)$ depends on $\bp$ only through its length, we have $\pi_{\bp}(\bp)=\pi_{\bp}(\cR_{\bg}(\bp) )$, for any direction of reflection $\bg$. Thus a reflection event does not change the value of $\tilde{\pi}$. This means we can use the same argument as at the start of the proofs of Theorems \ref{thm:CS} and \ref{thm:ZZ} to get that the left-hand side of (\ref{eq:PDMPinvariant}) is proportional to
\begin{equation} \label{eq:BPSthm1}
- \bp \cdot \nabla_{\btheta}\log \pi(\btheta) + \lambda_{\text{BPS}}(\btheta,\cR_{\bg}(\bp)) - \lambda_{\text{BPS}}(\btheta,\bp),
\end{equation}
where to simplify notation we have used $\bg=\nabla_{\btheta}\log \pi(\btheta)$. 
The middle term is the rate of transitioning to state $(\btheta,\bp)$ and uses, as shown above, that reflections are involutions, so it is the state $(\btheta,\cR_{\bg}(\bp))$ that will transition to $(\btheta,\bp)$ at a reflection event. Substituting in the form of $\lambda_{\text{BPS}}$, and using the definition of $\bg$, we have
\[
\lambda_{\text{BPS}}(\btheta,\cR_{\bg}(\bp)) - \lambda_{\text{BPS}}(\btheta,\bp) =  \max\{0, \cR_{\bg}(\bp) \cdot\bg\}- \max\{0, \bp\cdot\bg\}.
\]
Now
\[
\cR_{\bg}(\bp) \cdot \bg = (\bp-2(\bp\cdot \bghat)~\bghat )\cdot\bg=\bp\cdot\bg-2(\bp\cdot \bghat)(\bghat\cdot\bg)=\bp\cdot\bg-2(\bp\cdot \bg)(\bghat\cdot\bghat), 
\]
where the last equality follows as $\bg$ is proportional to $\bghat$. As $\bghat$ is a unit vector, we have $\cR_{\bg}(\bp) \cdot \bg= - \bp\cdot\bg$. Substituting gives
\[
\lambda_{\text{BPS}}(\btheta,\cR_{\bg}(\bp)) - \lambda_{\text{BPS}}(\btheta,\bp) =  \max\{0, -\bp \cdot\bg\}- \max\{0, \bp\cdot\bg\} = -\bp\cdot\bg. 
\]
By definition of $\bg$ this is just $\bp\cdot\nabla_{\btheta} \log \pi(\btheta)$. Substituting this into (\ref{eq:BPSthm1}) we see that the left-hand side of (\ref{eq:PDMPinvariant}) is 0 as required.

\end{proof}
\index{bouncy particle sampler|)}
\subsubsection{Example: Sampling from a Gaussian Target}
\index{piecewise deterministic Markov processes!simulation|(}
To gain an initial understanding of these algorithms in practice, how they differ from each other and how they compare to HMC, we will consider their implementation for sampling from a Gaussian distribution. This is an example where all methods can be implemented exactly and enables a simple comparison of the dynamics of the different samplers.

To simplify notation, we will assume our target Gaussian distribution has a mean zero. This can be assumed without loss of generality in terms of the behaviour of the samplers, as we can re-centre any Gaussian distribution with non-zero mean and this would not change the samplers' dynamics. A mean-zero Gaussian distribution is commonly parameterised by its covariance matrix, $\bSigma$ say, but in terms of its density function, it is easier to use the precision matrix, which is the inverse of the covariance matrix. We will denote the precision matrix by $\bQ=\bSigma^{-1}$.  We assume that $\bSigma$, and hence $\bQ$, is positive-definite.
Then up to an additive constant, we have
\[
\log \pi(\btheta) = -\frac{1}{2} \btheta^{\top} \bQ \btheta.
\]
The rates of the PDMP samplers depend on $\pi$ through $-\nabla_{\btheta} \log \pi(\btheta)$, which for the Gaussian target is $\bQ \btheta$. 

Before showing the output from the different PDMP samplers for this model, we will describe an approach to simulating the PDMPs. Each of the three PDMPs introduced in the previous section can be viewed as having multiple types of event, with each event having a deterministic transition. For the Zig--Zag Sampler\index{Zig-Zag sampler},
 we have one event associated with each component of $\btheta$, and that flips the associated component of the velocity. For the Bouncy Particle Sampler\index{bouncy particle sampler}, and the Coordinate Sampler\index{coordinate sampler},
 there are two events, one of which is a refresh\index{refresh event/rate} of the velocity. Our approach to simulating these PDMPs is to use the idea of superposition\index{superpositon}, that is we will simulate the time for each of the possible events, find which occurs first, and then this type of event with its associated time is the next event for our PDMP.

As described in Section \ref{sec:ch6-simulatingPDMPs}, to simulate times of events we need, for each type of event, to calculate the rate of the time until the next event, given the current state. We will describe how to calculate this for the bounce event of the Bouncy Particle Sampler, and for a flip event in the Zig--Zag Sampler\index{Zig-Zag sampler}. Simulating the refresh events is trivial, and simulating the events of the Coordinate Sampler\index{coordinate sampler} follows by similar arguments.

Let the current state of our PDMP be $\bz=(\btheta,\bp)$, and let $\tilde{\lambda}_{\bz}(t)$ be the rate at which an event occurs in terms of the future time $t$. We will first consider the bounce event of the Bouncy Particle Sampler. Let the current time be $s$, so $\bz_s=\bz$, then
\begin{eqnarray*}
\tilde{\lambda}_{\bz}(t) & =& \max\{0, \bp_{s+t} \cdot (-\nabla_{\btheta} \log \pi(\btheta_{t+s}) \} = \max\{0, \bp_{s+t}^{\top} \bQ \btheta_{s+t} \} \\
&=&    \max\{0, \bp^{\top} \bQ (\btheta+t\bp) \} = \max\{0, \bp^{\top}\bQ\btheta + t \bp^{\top}\bQ\bp\}.
\end{eqnarray*}
Here we have used the definition of $\tilde{\lambda}_{\bz}(t)$, substituted in $\log \pi(\btheta_{t+s})$ for our target and then used the fact that up until the next event, $\bp_{s+t}=\bp_s=\bp$ and $\btheta_{s+t}=\btheta_{s}+t\bp_s=\btheta+t\bp$.

We are viewing $\tilde{\lambda}_{\bz}(t)$ as a function of the further time until the event, $t$. We can see that $\tilde{\lambda}$ is the maximum of zero and a linear function of $t$, and defining $a=\bp^{\top}\bQ\btheta $ and $b=\bp^{\top}\bQ\bp$, the linear function is equal to $a+bt$. Furthermore, as $\bQ$ is positive-definite we have $b>0$. For this rate, we can simulate event times directly using Algorithm \ref{alg:direct}. To do this, we need to solve $\int_0^t \tilde{\lambda}_{\bz}(u)\mbox{d}u=w$. There are two cases for the integral. First, if $a>0$, then $\tilde{\lambda}_{\bz}(u)=a+bu$ for all $u>0$, so 
\[
\int_0^t \tilde{\lambda}_{\bz}(u)\mbox{d}u = \int_0^t (a+bu) \mbox{d}u=at+ \frac{bt^2}{2},
\]
and this is equal to $w>0$ if $t=-a/b+\sqrt{a^2+2wb}/b$. If $a<0$ then $a+bt$ is only positive for $t>|a|/b$, thus for such $t$
\[
\int_0^t \tilde{\lambda}_{\bz}(u)\mbox{d}u = \int_{|a|/b}^t (a+bu) \mbox{d}u = \int_{0}^{t-|a|/b} bu'\mbox{d}u'= \frac{b(t-|a|/b)^2}{2}.
\]
This is equal to $w>0$ when $t=|a|/b+\sqrt{2w/b}$.

The resulting algorithm for one iteration of the Bouncy Particle Sampler\index{bouncy particle sampler} is given in Algorithm \ref{alg:BPS_Gaussian}. The algorithm simulates the time of a refresh\index{refresh event/rate} event and a bounce event. It sets the time of the next event to be the smaller of these two times and updates the position of the state. Then, depending on which type of event occurred first, it updates the velocity. For a refresh event, this involves simulating from the sampler's \index{invariant distribution}invariant distribution for the velocity, which depending on the type of Bouncy Particle Sampler can be either a standard Gaussian distribution or the uniform distribution on the unit hyper-sphere. For a bounce event, the velocity is reflected in the hyperplane perpendicular to $-\nabla \log \pi(\btheta)$ at the current position $\btheta'$, which for our model is $\bQ \btheta'$.

\begin{algorithm}[h]
    \caption{Bouncy Particle Sampler: Gaussian Target}
    \KwIn{Precision Matrix, $\bQ$, current state $(\btheta,\bp)$, refresh rate $\refr>0$}
    Calculate $a=\bp^{\top}\bQ\btheta $ and $b=\bp^{\top}\bQ\bp$.\\
    Simulate $w_1$ and $w_2$, independent realisations of a standard exponential random variable.\\
    Calculate time until a refresh event $\tau_1=w_1/\refr$.\\
    Calculate time until a bounce event: if $a<0$ $\tau_2=\sqrt{2w_2/b}+|a|/b$, otherwise $\tau_2=-a/b+\sqrt{a^2+2w_2b}/b$.\\
    Calculate event time $t=\min\{\tau_1,\tau_2\}$.\\
    Update position $\btheta'=\btheta+t\bp$.\\
    Decided on event type and update velocity:\\
    \eIf{$\tau_1<\tau_2$}{
    Refresh event. Simulate $\bp'$ from its invariant distribution.\index{refresh event/rate}
    }{Bounce event. Set
    \[
    \bp'=\bp-2(\bp^{\top}\bQ\btheta') \frac{\bQ\btheta'}{{\btheta'^{\top}\bQ^{\top}\bQ\btheta'}}.
    \]}
    \KwOut{Time to next event $t$, and new state $(\btheta',\bp')$}
    \label{alg:BPS_Gaussian}
\end{algorithm}

A similar derivation is possible for the Zig--Zag Sampler\index{Zig-Zag sampler}. The only difference is that the flip rate of the $i$th component of $p_{i}$ is
\[
\max\{0, -p_{i} (\nabla_{\btheta} \log \pi(\btheta))_i\}=\max\{0,p_{i} (\bQ\btheta)_{i} \},
\]
where we write, for example, $(\bQ\btheta)_{i}$ to denote the $i$th component of $\bQ\btheta$.
Thus, the associated rate as a function of time until this event, if the current state of $\bz=(\btheta,\bp)$, is
\[
\tilde{\lambda}_{\bz}^{(i)}(t) = \max\{0, p_{i}[\bQ(\btheta+t\bp)]_{i}\} = \max\{0, p_{i}(\bQ\btheta)_{i}+tp_{i}(\bQ\bp)_{i} \}.
\]
This again is the maximum of 0 and a linear function of $t$ and can be simulated as described above. As above, we can set $\tilde{\lambda}_{\bz}^{(i)}(t)=\max\{0,a+bt\}$ but now with $a=p_{i}(\bQ\btheta)_{i}$
and $b=p_{i}(\bQ\bp)_{i} $. The only difference is that in this case, $b$ can be negative. If $b<0$ and $a<0$, then this event can never happen. If $a>0$, then an event can happen for $t<a/|b|$. In this case
\[
w=\int_0^t (a+bu)\mbox{d}u \Rightarrow w=at+ \frac{bt^2}{2}. 
\]
This has a solution for $t>0$ only if $w\leq a^2/(2|b|)$, in which case $t=-a/b+\sqrt{a^2+2wb}/b$. If $w>a^2/(2|b|)$ then this event does not happen. 
If an event cannot or does not happen, then algorithmically we set the associated event time to infinity.

An algorithm for one iteration of the Zig--Zag Sampler\index{Zig-Zag sampler} is given in Algorithm \ref{alg:ZZ_Gaussian}. It has a similar form as the Bouncy Particle Sampler\index{bouncy particle sampler}. We calculate the event time for each type of event -- though due to the four possible cases we have not given the formulae for calculating the event times within the algorithm. We then set the event time to the smallest of these times and update the position. Finally, the event type is calculated and we apply the appropriate transition to the velocity, remembering that $F_i(\bp)$ flips the $i$th component of the vector $\bp$.
\begin{algorithm}[h]
    \caption{Zig--Zag Sampler: Gaussian Target}
    \KwIn{Precision Matrix, $\bQ$, current state $(\btheta,\bp)$, refresh rate $\refr>0$}\index{refresh event/rate}
    \For{$i=1,\ldots,d$}{
        Calculate $a_i=p_{i}(\bQ\btheta)_{i}$ and $b_i=p_{i}(\bQ\bp)_{i}$.\\
        Simulate $w_i$, a realisation of a standard exponential random variable.\\
        Calculate $\tau_i$, the event time for a flip of the $i$th component of $\bp$.\\
    }
    Calculate event time $t=\min_{i=1,\ldots,d}\{\tau_i\}$.\\
    Update position $\btheta'=\btheta + t\bp$.\\
    Decide on event type, $i^*=\arg\min \{\tau_i\}$.\\
    Update velocity $\bp'=F_{i^*}(\bp)$.\\
    \KwOut{Time to next event $t$, and new state $(\btheta',\bp')$}
    \label{alg:ZZ_Gaussian}
\end{algorithm}

To gain some intuition of the properties of these PDMP algorithms, we will first look qualitatively at the output of running the algorithms for a bivariate Gaussian -- as we can plot the realisations from the paths generated by the position component of the PDMPs. First, we look at the importance of the refresh rate with the Bouncy Particle Sampler\index{bouncy particle sampler}. To most clearly see the potential issues for this sampler, it is helpful to observe it sampling from a target with uncorrelated, equal variance components. See Figure \ref{fig:ch6_BPS_refresh} for output from the Bouncy Particle Sampler for different refresh rates. 

\begin{figure}[!h]
    \centering
    \includegraphics[width=\textwidth]{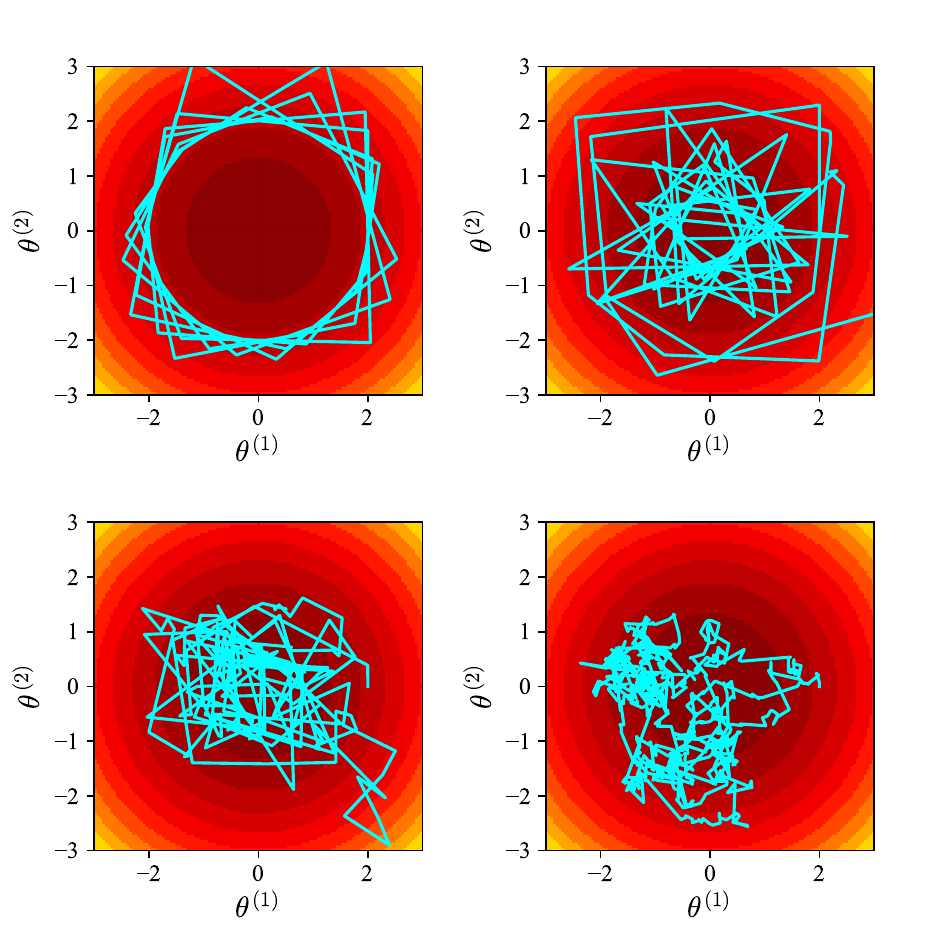}
    \caption{Plots of realisations of the trajectory or path of the Bouncy Particle Sampler when sampling from a standard (i.e. uncorrelated, equal variance) bivariate Gaussian distribution. Trajectories shown for no refresh events (top left), $\refr=0.1$ (top right), $\refr=1$ (bottom left) and $\refr=10$ (bottom right). The heat map shows the log posterior density of the target in each case.
    }
    \label{fig:ch6_BPS_refresh}
\end{figure}

The top-left plot shows the trajectory if we do not use any refresh events\index{refresh event/rate}. We can clearly see evidence of the sampler being reducible\index{reducible} -- as the sampler does not enter a large region around the mode. \cite{bouchard2018bouncy} prove that the Bouncy Particle Sampler\index{bouncy particle sampler} is in fact reducible\index{reducible} for this example. Furthermore, they show that this is avoided if we use any non-zero refresh rate, and if we do so the sampler will converge to the target distribution. However, we get very different behaviour for different values of the refresh rate, as shown in the remaining plots of Figure \ref{fig:ch6_BPS_refresh}. A small rate produces a sampler, which whilst irreducible, has poor mixing \index{mixing}properties (see top-right plot). The sampler has long periods between refresh events, and for each of these periods, there are regions of the state space that the sampler cannot reach. Too large a refresh rate means that the sampler shows random-walk behaviour, which can also adversely affect its mixing (see bottom-right plot). Tuning of the refresh rate to a good intermediate value can result in a sampler that mixes well and avoids this random-walk behaviour (see bottom-left plot). We will return to how to tune the refresh rate later. Finally, whilst the Coordinate Sampler\index{coordinate sampler} also has refresh events, it does not suffer from the same problems as the Bouncy Particle Sampler and will be irreducible\index{irreducibility} even if $\refr=0$.

\begin{figure}[!t]
    \centering
    \includegraphics[width=0.8\textwidth]{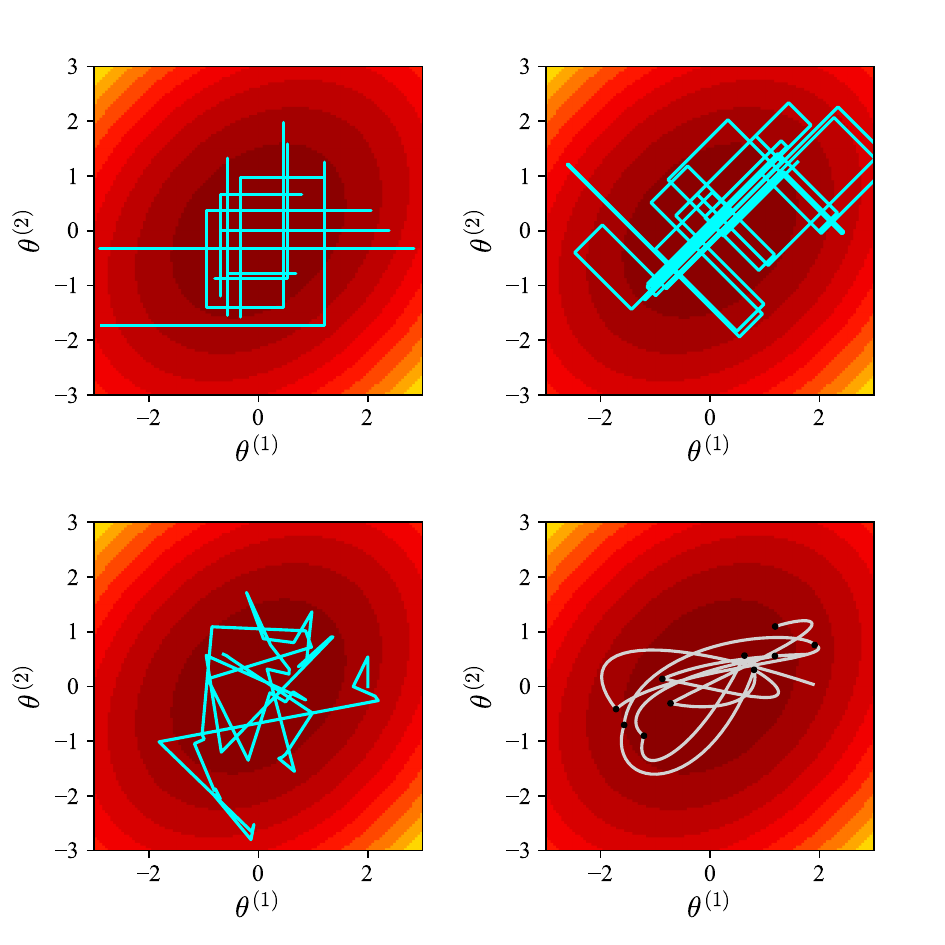}
    \caption{Comparison of three PDMP algorithms and HMC for sampling from a bivariate Gaussian with unit marginal variances and correlation of  0.5.
    Realisations of the trajectories (lines) of the Coordinate Sampler (top left), the Zig--Zag Sampler (top right) and the Bouncy Particle Sampler with $\refr=1$ (bottom left). Trajectories of Hamiltonian dynamics (line) and the sampled points (dots) from HMC (bottom right).
    The heat map shows the log posterior density of the target in each case.
    }
    \label{fig:ch6_PDMPs}
\end{figure}

We now compare the outputs of the three different PDMP samplers, with the Bouncy Particle Sampler\index{bouncy particle sampler} using an appropriately tuned refresh rate\index{refresh event/rate}, $\refr=1$. These are shown in Figure \ref{fig:ch6_PDMPs} together with the output from HMC.  At this stage, there are two main points we wish to make. First, one can see the qualitative similarities and differences between the different PDMP samplers. Each sampler explores $\btheta$ space with straight-line trajectories, and they all have the property that they continue in the same direction whilst moving to areas of higher probability density -- though for the Zig--Zag Sampler\index{Zig-Zag sampler}, this has to be interpreted separately for each axis component. However, the trajectories themselves are very different due to the different possible velocities and transitions. The Coordinate Sampler\index{coordinate sampler} explores the posterior by exploring a single component of $\btheta$ at a time. This has some similarities with a Gibbs sampler\index{Gibbs moves} (see Section \ref{ch2:sec:Gibbs}), and intuition from results on mixing of Gibbs samplers suggests that this sampler will perform best when there is no strong correlation between the components of $\btheta$. 
The Zig--Zag Sampler has trajectories consisting of diagonal lines, and the transitions that flip a single component of the velocity lead to trajectories that look like zig-zags, which gives the sampler its name. Finally, the Bouncy Particle Sampler can have trajectories that explore the space in any direction.

Second, it is interesting to compare PDMP samplers with HMC (see bottom-right plot). For this model, we can solve the Hamiltonian dynamics\index{Hamiltonian Monte Carlo} for each proposal\index{proposal} of the HMC algorithm exactly, and thus we always will accept a proposal. The trajectories of the Hamiltonian dynamics are elliptical, as compared to the straight-line segments of our PDMP samplers. However, the main difference is that the output of a PDMP sampler is a continuous path, whilst for HMC, we obtain a set of points. For non-Gaussian target distributions, the HMC sampler will not always accept the proposal which can lead to worse mixing. In such cases, whilst we cannot directly sample  trajectories of the PDMP samplers, this only impacts the computational cost of simulation, rather than the output or the mixing properties of the sampler. 

\index{piecewise deterministic Markov processes|)}
\index{piecewise deterministic Markov processes!simulation|)}

\subsection{Use of PDMP Output}
\index{piecewise deterministic Markov processes!output|(}
Simulating a PDMP, as described in Section \ref{sec:ch6-simulatingPDMPs}, produces a skeleton of the sample path\index{sample path} of the process. Assume that we have simulated the process up to some time $T$.  We will denote this skeleton by the set $\{\tau_i,(\btheta_{\tau_i},\bp_{\tau_i})\}_{k=0}^{n+1}$ that gives the initial state of the process, with $\tau_0=0$, the time and state after each event, for $k=1,\ldots,n$, and the final state of the process at time $\tau_{n+1}=T$. How do we use this output to approximate the target distribution?

For any Monte Carlo method, an approximation to the target distribution, $\pi(\btheta)$ comes from the ability to estimate the expectation for arbitrary functions of $\btheta$ with respect to $\pi$. Thus consider estimating $\Expects{\pi}{h(\btheta)}$ for some function $h$ for which this expectation exists. First, we describe how we do not estimate this expectation! We cannot just use the sample average of $h(\cdot)$ at the skeleton points for $\btheta$, even after allowing some burn-in\index{burn-in}. In general, the skeleton points will not be sampled from $\pi$ at stationarity, as they will be biased towards values where there is a large average rate of an event occurring.

Instead, we need to estimate the expectation with respect to the continuous-time sample path\index{sample path}, after allowing for a suitable burn-in. There are two ways of doing this. Assume we choose the burn-in to be some time $S$. 
Then one approach is to calculate the average value of the integral of $h(\btheta_t)$ for our sample path for $S<t\leq T$. This can be calculated from the skeleton as follows. First, we work out the value of $\btheta_S$ for our sample path\index{sample path}. This is possible by finding $l$ such that $\tau_l\leq S < \tau_{l+1}$, and using linear interpolation
\[
\btheta_S= \frac{\tau_{l+1}-S}{\tau_{l+1}-\tau_l} \btheta_l +  \frac{S-\tau_l}{\tau_{l+1}-\tau_l} \btheta_{l+1}.
\]
Then our estimator can be calculated as
\begin{eqnarray*}
\hExpects{\pi}{h(\btheta)} &=& \frac{1}{T-S}\left( 
\int_{S}^{\tau_{l+1}} h\left( \btheta_S + (t-S)\frac{\btheta_{\tau_{l+1}}-\btheta_S}{\tau_{l+1}-S}\right)\mbox{d}t
 \right. \\ & &
\left. +\sum_{k=l+1}^K \int_{\tau_k}^{\tau_{k+1}} h\left( \btheta_{\tau_k} + (t-\tau_k)\frac{\btheta_{\tau_{k+1}}-\btheta_{\tau_k}}{\tau_{k+1}-\tau_k}\right)\mbox{d}t
\right)
\end{eqnarray*}
This estimator is only practical if we can analytically calculate the integrals along the linear segments of the path. A more general, and arguably simpler approach is to evaluate  $\btheta_t$ at $N$ evenly spaced points between $S$ and $T$ and then use standard Monte Carlo averages with respect to this set of values. That is, let $\delta=(T-S)/N$ and using linear interpolation as above evaluate $\btheta_{S+j\delta}$ for $j=1,\ldots,N$.  
Then our estimator of $\Expects{\pi}{h(\btheta)}$ would now be
\[
\hExpects{\pi}{h(\btheta)}  = \frac{1}{N} \sum_{j=1}^N 
h(\btheta_{S+j\delta}).
\]
One advantage of this approach is the final output is similar to that for standard MCMC, which eases comparison and enables us to use methods for assessing the accuracy of standard MCMC estimators such as the integrated auto-correlation time and effective sample size\index{effective sample size} (see Section \ref{sec:ch1-IACFESS}).
\index{piecewise deterministic Markov processes!output|)}
\subsection{Comparison of Samplers} \label{sec:ch6-comparison}
\index{piecewise deterministic Markov processes!comparison of samplers|(}
\cite{bierkens2022high} examines the mixing\index{mixing|(} properties of the Bouncy Particle Sampler\index{bouncy particle sampler|(} and the Zig--Zag Sampler\index{Zig-Zag sampler} in the limit as the dimension of the space, $d\to \infty$, for the special case where the posterior of interest is a $d$-dimensional standard normal distribution. In related work, \cite{bierkens2023scaling} investigates finite-dimensional normal targets where some principal components have a much smaller length scale than others. We summarise the findings of these two papers and provide some intuition for them.

To compare the samplers, we need to consider both their computational cost for simulating a trajectory of fixed duration, and how the mixing of the process depends on time. First, consider the computational cost, and the intensity of bounce events on a $d$-dimensional standard normal target, so that at stationarity $\nabla \log \pi(\btheta)=-\btheta$. The momentum for the Zig--Zag Sampler\index{Zig-Zag sampler} is $\bp_{\text{zz}}\in \{-1,+1\}^d$. For the Bouncy Particle Sampler, we describe the case where $\bp$ is sampled from $\mathsf{U}_d$, the uniform distribution on the unit hypersphere. If, instead, $\bp_{\text{BPS}} \sim \Normal(\bzero,\frac{1}{d} \bI_d)$, for large $d$, $\|\bp_{\text{BPS}}\|^2\approx 1$ and the analysis is the same as for the setting where  $\bp_{\text{BPS}}\sim \mathsf{U}_d$. Speeding up time by a factor of $\sqrt{d}$ leads to the version of the BPS that samples $\bp \sim \mathsf{N}(\bzero, \bI_d)$, but makes no difference to the overall efficiency in terms of mixing per unit of computational effort.

 For either sampler, let $p_{i}$ be the $i$th component of its momentum. For the Bouncy Particle Sampler, the intensity is $\lambda_{\text{BPS}}(\btheta,\bp)=\max\{0,\sum_{i=1}^d p_{i} \theta_{i}\}=O(1)$, since each term in the sum has the same expectation of $0$ and a variance of $O(1/d)$. Thus, there are $O(1)$ events per unit of time. In contrast, for the Zig--Zag Sampler, the intensity is $\lambda_{\text{zz}}(\btheta,\bp)=\sum_{i=1}^d \max\{0, p_{i}\theta_{i}\}=O(d)$, since each term in the sum has the same positive expectation. Hence, there are $O(d)$ events per unit time. In general, for each sampler, the computational cost of performing a bounce is $O(d)$: for the Bouncy Particle Sampler, this is the order of the cost of calculating the gradient required for both the event rate and for calculating the new velocity at a bounce event; for Zig-Zag, after a bounce event we will need to update the rates for each of the $d$ possible events (though see Section \ref{sec.ZZsparse} for situations where this can be reduced).
Thus the total cost per unit time is $O(d)$ for the Bouncy Particle Sampler and $O\left(d^2\right)$ for the Zig--Zag Sampler.

The above costing generalises to any reasonably well-behaved target. 
However, because it only updates a component at a time, in cases where components have a sparse conditional dependence graph the cost of performing a Zig--Zag bounce can be reduced to $O(1)$ (see Section \ref{sec.ZZsparse}).

Now we turn to the mixing properties. In this special case of an isotropic target, the state of the Markov process can be encapsulated by the radial component, $\|\btheta\|$, and the angle between the radius and the momentum, which is proportional to $\btheta \cdot \bp$.

For Bouncy Particle Samplers, the radial component mixes in $O(d)$ time, whereas the angular component mixes in $O(1)$ time. To see why, for simplicity, we ignore any refresh events. Consider a single straight-line path between bounces, which we will call a segment: because the contours on a $\mathsf{N}(\bzero,\bI_d)$ target are spherical, $\log \pi$ increases monotonically to a maximum along the segment and then decreases monotonically until the next bounce. Let $E$ be the size of the drop in $\log \pi$ from its maximum until the next bounce. Start a clock at a time when $\log \pi$ is at a maximum and let $t$ be the time since the clock started. Since there are no refresh events, if there has  been no bounce since the clock started, the size of the total drop in $\log \pi$ by time $t$ is
\[
D(t)=-
\int_0^t \bp\cdot \nabla \log \pi(\btheta_s) \md s.
\]
 Now, for $t>0$,   $\pi(\btheta_t)$ 
 decreases until a bounce occurs,  so
\[
E>D(t) \Leftrightarrow \mbox{"No bounce by time $t$."}
\]
However, while $\log \pi$ is decreasing, bounce events follow an inhomogeneous Poisson process with a rate at time $s$ of $\lambda(s)=-\bp \cdot \nabla \log \pi(\btheta_s)$, so the probability that there has been no event by time $t$ is
\[
\exp\left[-\int_0^t \lambda(s) \md s\right]
=
\exp[-D(t)].
\]
Thus $\Prob{E>D(t)}=\exp[-D(t)]$, 
so $E$ has an exponential distribution with rate parameter 1. 

Over the $O(1)$ time between bounces, the angle $\btheta \cdot \bp$  moves monotonically from its most negative extent (just after a bounce) to its most positive extent (just before the next bounce). Thus, $\btheta \cdot \bp$ mixes in $O(1)$ time.
However,
\[
\log \pi(\btheta)
=
\mbox{constant}-\frac{1}{2}\|\btheta\|^2
=
\mbox{constant}-\frac{1}{2}\chi^2_d
\]
at stationarity, and a $\chi^2_d$ random variable has a standard deviation of $\sqrt{2d}$. So, to mix, the $\log \pi$ process needs to move by $O(d^{1/2})$, yet it only moves by $O(1)$ in $O(1)$ time. In an analogous manner to the limit for the random walk Metropolis in Section \ref{sec.rwm}, speeding up time by a factor of $d$ leads to a limiting diffusion for $\|\btheta\|$; thus $\|\btheta\|$ mixes in $O(d)$ time.

By maximising the speed of the limiting diffusion for $\|\btheta\|$, the same analysis advises on tuning the refresh rate\index{refresh event/rate} for the Bouncy Particle Sampler\index{bouncy particle sampler}, $\refr$. This suggests choosing $\refr$ so that the ratio of refresh to bounce events is $\approx 0.78$. 

In the case of an isotropic normal target, the Zig--Zag Sampler\index{Zig-Zag sampler} simplifies to $d$ independent one-dimensional instances of Gustafson's algorithm (Section \ref{sec:ch5_lifting}). Thus, each individual component mixes in $O(1)$ time, so both $\|\btheta\|$ and $\btheta \cdot \bp$ mix in $O(1)$ time.

Multiplying the mixing times by the computational costs, we can define mixing costs. For Bouncy Particle Samplers, these are $O(d)$ for the angular component and $O(d^2)$ for the radial component, whereas the costs for the Zig--Zag Sampler are $O(d^2)$ for both components. If one follows the adage that a sampler is only as good as its worst-mixing component, this suggests that, at least for well-behaved targets, both algorithms have a similar efficiency.

What about the mixing of the Bouncy Particle Sampler for other functions of the state? \cite{bierkens2022high} show that it has the same $O(d^2)$ rate for marginal components of the state -- that is, if we are interested in $\theta_{i}$ for some $i$. This compares to results in \cite{deligiannidis2021randomized} which suggest that mixing costs for marginal components are $O(d)$. The difference in results comes from different choices of how the refresh rate\index{refresh event/rate} depends on $d$. \cite{bierkens2022high} have a constant rate, whereas \cite{deligiannidis2021randomized} assume that the rate decays like $O(d^{-1/2})$. The latter choice improves mixing for marginal components but worsens the mixing of the radial component, so that as $d\rightarrow \infty$, the limiting process for the radial component is degenerate. 

\begin{figure}
    \centering
    \includegraphics[width=\textwidth]{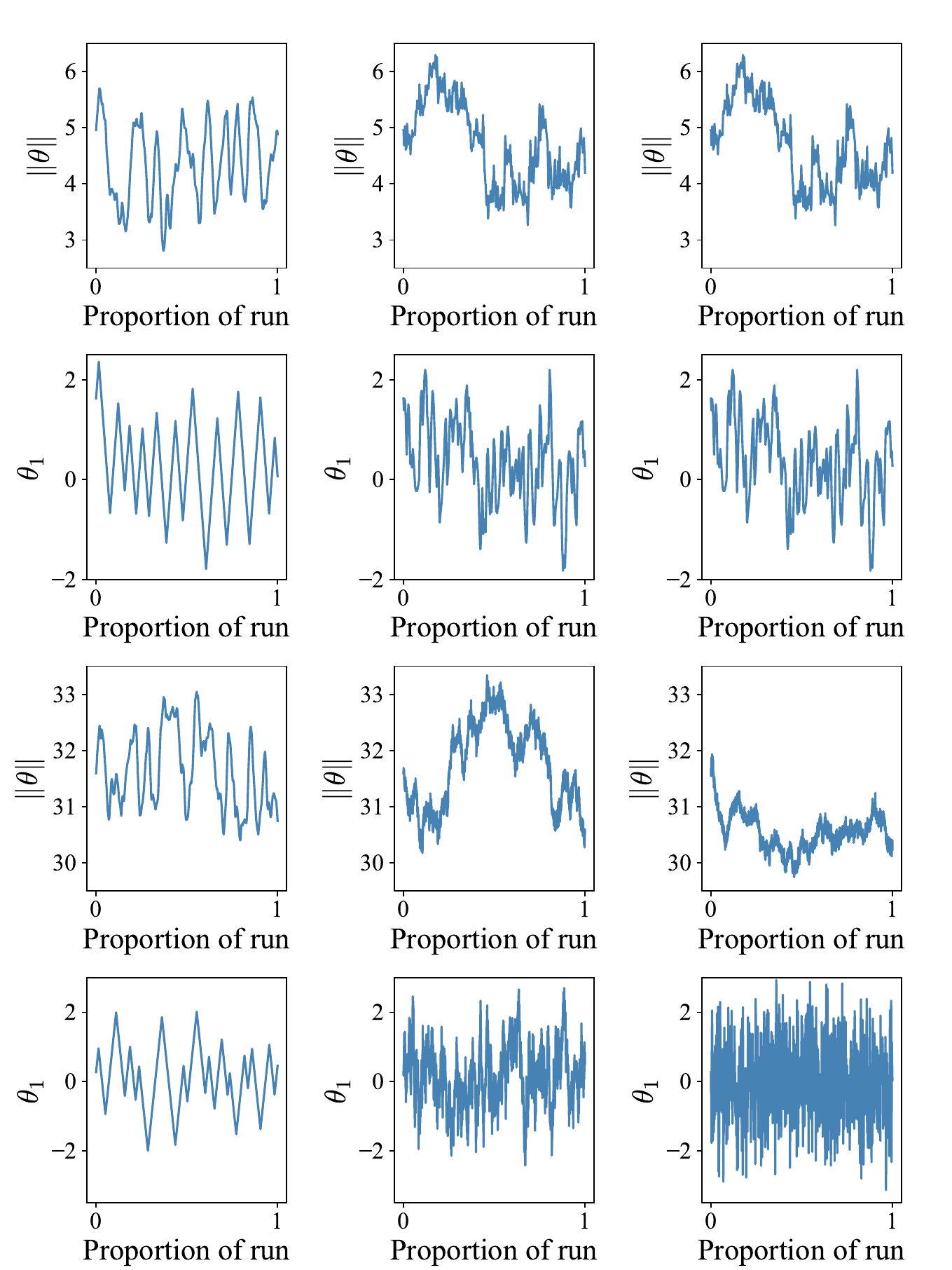}
    \caption{Trace plots for Zig--Zag (left-hand column), Bouncy Particle Sampler with  $\refr=1.5\sqrt{d/20}$ (middle column) and Bouncy Particle Sampler with  $\refr=1.5$ (right-hand column) for a Gaussian target with $d=20$ (top two rows) and $d=1000$ (bottom two rows). In each case we show trace plots for the radial component of the state, $\|\btheta\|$, and the first component of the state, $\theta_1$. We ran all samplers for $20d$ bounce events, and scaled the time axis by the proportion of the resulting simulation time. }
    \label{fig:ch6-scale}
\end{figure}

To see this in practice, we compared the performance of Zig--Zag\index{Zig-Zag sampler} and the Bouncy Particle Sampler at sampling from a Gaussian with $d=20$ and $d=1000$. Results are shown in Figure \ref{fig:ch6-scale}, where we scale the number of events simulated to be proportional to the dimension, $d$. The theory states that for this scaling\index{scaling} we would expect similar mixing for the Zig--Zag sampler over the length of the simulation. We observe this qualitatively for both $\|\btheta\|$ and $\theta_1$. The theory also suggests similar qualitative behaviour for the Bouncy Particle Sampler with a refresh rate that scales with $\sqrt{d}$, so that the proportion of refresh events is roughly similar for different $d$; we observe this in the middle column of the plot. By comparison, if we use a fixed refresh rate\index{refresh event/rate}, as in \cite{deligiannidis2021randomized}, then we observe better mixing for $\theta_1$ but worse mixing for $\|\btheta\|$ as we increase $d$ -- see the right-hand column of Figure \ref{fig:ch6-scale}.

One important consequence of the theoretical results, and that is seen in the results in Figure \ref{fig:ch6-scale}, is that care is needed when we assess mixing and convergence of, in particular, the Bouncy Particle Sampler -- as we can get substantially different measures of mixing, such as auto-correlation time\index{auto-correlation time|see{integrated auto-correlation time}}, \index{integrated auto-correlation time} depending on which function of the state we consider. Observing a fast mixing chain for one component may mask that other functions of the state are mixing very slowly (see Section \ref{subsec: conv control kernel} for more discussion on summarising convergence in multivariate settings and measures that can give differing importance to different coordinates). 

\cite{bierkens2023scaling} investigate fixed, finite-dimensional normal targets where between $1$ and $d-1$ principal components have a length scale of $\epsilon$, while the remainder have a length scale of $1$. Both algorithms have $O(\epsilon^{-1})$ events per unit time because the momentum is $O(1)$  but some length scales are $O(\epsilon)$. The Bouncy Particle Sampler mixes in $O(1)$ time; however, the alignment of the Zig--Zag Sampler is crucial to its mixing time: if the principal axes of the target are aligned with the $d$ Zig--Zag momentum components then it also mixes in $O(1)$ time, but in almost all other scenarios its mixing time is $O(\epsilon^{-1})$. Preconditioning is usually advised for any MCMC\index{Markov chain Monte Carlo} algorithm; this analysis highlights that the need for preconditioning the Zig--Zag Sampler\index{Zig-Zag sampler} is even more marked than it is for Bouncy Particle Samplers.
\index{bouncy particle sampler|)}
\index{mixing|)}
\section{Efficient Simulation of PDMP Samplers}

We now consider various approaches for simulating the PDMP samplers. 
As a running example that we will use to help explain some of the ideas, we will consider the logistic regression model with a Gaussian prior, which was introduced in Section \ref{sec:ch1-logistic}.
\index{piecewise deterministic Markov processes!output|)}

\subsection{Simulating PDMPs} \label{sec:ch6-simPDMPsamplers}
\index{piecewise deterministic Markov processes!simulation|(}

When sampling from the Gaussian target, the rates that determine the time until the next event were linear, and thus we could simulate the event times of the PDMP exactly. 
What happens for more complicated targets where this is not possible? Here, we describe three possible methods for simulating these events. 

The most common approaches for simulating events of a PDMP are based on the idea of Poisson thinning\index{Poisson thinning}
 (see Section \ref{sec:ch6-simulatingPDMPs}). Remember this involves upper-bounding the event rate, simulating events with this bounding rate, and then accepting the simulated events with the ratio of the true rate to the bounding rate. The challenge with Poisson thinning\index{Poisson thinning} is finding good upper bounds that are simple enough that we can simulate events analytically, and, ideally, close to the true rate, as the computational efficiency of Poisson thinning depends on how close the bounding rate is to the true rate. The first two methods we describe are based on this idea but differ in the assumptions they make on the target and how the bounding rates are constructed.

The first approach comes from \cite{Bierkens:2019} and assumes that the Hessian\index{Hessian} of the minus log-target is bounded. 
To simplify the notation, define the vector $\bU(\btheta)=-\nabla \log \pi(\btheta)$,
so $\bU=(U_1(\btheta),\ldots,U_d(\btheta))^{\top}$ with 
\[
U_j(\btheta)=-\frac{\partial \log \pi(\btheta)}{\partial \theta_j}.
\]
Now the Hessian\index{Hessian} of $-\log \pi(\btheta)$ is a $d\times d$ matrix $\bH(\btheta)$ defined as
\[
\bH(\btheta)=\left(\begin{array}{ccc} 
\nabla U_1(\btheta) &
\cdots &
\nabla U_d(\btheta)
\end{array}\right).
\]
Then we assume that there is some matrix $\bJ$ such that for any vector, $\bw$, and any $\btheta$,
\[
\bw^\top \bH(\btheta) \bw \leq \bw^\top\bJ \bw.
\]
This holds for Gaussian target distributions, as $\bH(\btheta)=\bH$ is a constant. More importantly, it holds for targets that are heavier-tailed than Gaussian, including the posterior distribution\index{posterior distribution} for logistic regression or many versions of robust regression if we have, for example, Gaussian priors on the parameters of the model.

Now, as introduced before, consider a rate for the next event, or the next specific type of event, in a PDMP, $\tilde{\lambda}_{\bz}(t)$, in terms of the further time until the event $t$. Assume that 
\[
\tilde{\lambda}_{\bz}(t) = \max\{0, \bw \cdot \bU(\btheta+\bp t)\},
\]
where $\bU=-\nabla \log \pi$ as defined above, the current state is $(\btheta,\bp)$, and $\bw$ is some vector that depends on the sampler. For the Bouncy Particle Sampler\index{bouncy particle sampler} or the Coordinate Sampler\index{coordinate sampler}, if we are considering time until the next non-refresh event, then $\bw=\bp$, whereas for the Zig--Zag Sampler\index{Zig-Zag sampler} if we are considering the next flip of component $j$ then $\bw$ will be either the unit vector with $1$ or $-1$ in the $j$the component and zero elsewhere. 

Now consider the term $\bw \cdot \bU(\btheta+\bp t)$.
This can be rewritten as
\begin{eqnarray*}
\bw \cdot \bU(\btheta+\bp t)  &=&\bw \cdot \bU(\btheta)+ \sum_{i=1}^d w_i \int_0^t \frac{\mbox{d} U_i(\btheta+\bp s)}{\mbox{d}s} \mbox{d}s\\
&=& \bw \cdot \bU(\btheta)+\sum_{i=1}^d w_i \int_0^t \bp \cdot \nabla U_i(\btheta+\bp s) \mbox{d}s\\ 
&=&\bw \cdot \bU(\btheta)+\int_0^t \bw^\top \bH(\btheta+\bp s)\bp \mbox{d}s.
\end{eqnarray*}
The first equality comes from writing the value of a function at time $t$ as its value at time $0$ plus the integral of its derivative from time $0$ to time $t$. We then use the chain rule to get the derivative of $U_i(\btheta+\bp s)$ with respect to $s$, and finally the definition of the Hessian matrix.

Now, by Cauchy--Schwarz\index{Cauchy-Schwarz} for vectors, 
$\bw^\top\bH\bp\leq \|\bw\|\|\bH\bp\|$, where $\|\bw\|=(\sum_{i=1}^d w_i)^{1/2}$ is the $L_2$ norm of the vector $\bw$. 
This, together with our assumption on the bound of the Hessian, gives 
\[\bw^\top\bH(\btheta+\bp s)\bp \leq \|\bw\|\|\bH(\btheta+\bp s)\bp\|
\leq \|\bw\|\|\bJ\bp\|
\] 
which is a constant. 
Thus, we get the linear bound
\[
\bw \cdot \bU(\btheta+\bp t) \leq \bw \cdot \bU(\btheta) + \|\bw\|\|\bJ\bp\| t,
\]
or equivalently the piecewise linear bound on the rate
\[
\tilde{\lambda}_{\bz}(t) \leq \max \{0, \bw \cdot \bU(\btheta) + \|\bw\|\|\bJ\bp\| t\}.
\]
We can simulate events from this upper-bounding rate analytically, in an equivalent way to which we simulate the events for the Gaussian target. If $\bw\neq\bp$, we can use instead $\bw^\top\bH\bp\leq \|\bp\|\|\bH\bw\|$ to get the alternative bound
\[
\tilde{\lambda}_{\bz}(t) \leq \max \{0, \bw \cdot \bU(\btheta) + \|\bp\|\|\bJ\bw\| t\}.
\]

Our second approach \cite[]{sutton2023concave} is based on a different assumption for the target, namely that we can decompose $-\nabla \log \pi(\btheta+\bp t)$, as a function of $t$ for any $\btheta$ and $\bp$ into the sum of convex and concave functions\index{concave-convex decomposition}, and assume the concave function is differentiable everywhere. A function $f_{\cup}(t)$ for $t\geq 0$ is a convex function if for any $0\leq r <s <t$, we have
\[
f_{\cup}(s) \leq \frac{t-s}{t-r}f_{\cup}(r)+  \frac{s-r}{t-r}f_{\cup}(t),
\]
while a function $f_{\cap}(t)$ for $t\geq 0$ is a concave function if $-f_{\cap}(t)$ is convex. 
That is if we pick any two points on the function and join them by a straight line, then the function lies below the line if it is convex, and above the line if it is concave. See Figure \ref{fig:ch6_CC} for an example.

We will assume that we can decompose the rate until the next event as
\[
\tilde{\lambda}_{\bz}(t) = \max\{0,f_{\cup}(t)+f_{\cap}(t)\},
\]
that is, in terms of a single convex and single concave function -- though the ideas below apply trivially if our decomposition involves multiple concave and convex functions. 
(In fact, the sum of convex functions is convex, and the sum of concave functions is concave, so we can immediately simplify the decomposition to the case we are considering.)  Our starting point is the bound
\begin{equation} \label{eq:ch6-CC}
\tilde{\lambda}_{\bz}(t) = \max\{0,f_{\cup}(t)+f_{\cap}(t)\} \leq \max\{0,f_{\cup}(t)\} + \max\{0,f_{\cap}(t)\}.
\end{equation}
\cite{sutton2023concave} then use the fact that we can bound $f_{\cup}(t)$ and $f_{\cap}(t)$ by piecewise linear functions just by evaluating the functions at a set of grid points. Once we have these piecewise linear bounds, they immediately give us a piecewise linear bound on $\tilde{\lambda}_{\bz}(t)$ by substituting them into (\ref{eq:ch6-CC}).

How do we get piecewise linear bounds on $f_{\cup}(t)$ and $f_{\cap}(t)$? It is simplest to see this through a picture -- see Figure \ref{fig:ch6_CC}. For the convex function, the bound comes immediately from the definition: if we evaluate $f_{\cup}(t)$ at $t=0$ and $t=t_1$, then the straight line that joins these points gives an upper bound on $[0,t_1]$. For a concave function, we need to evaluate $f_{\cap}(t)$ and its derivative at $t=0$ and $t=t_1$. We then construct the tangents to $f_{\cap}(t)$ at $t=0$ and $t=t_1$, and the function lies below both tangents. 

\begin{figure}
    \centering
    \includegraphics[width=\textwidth]{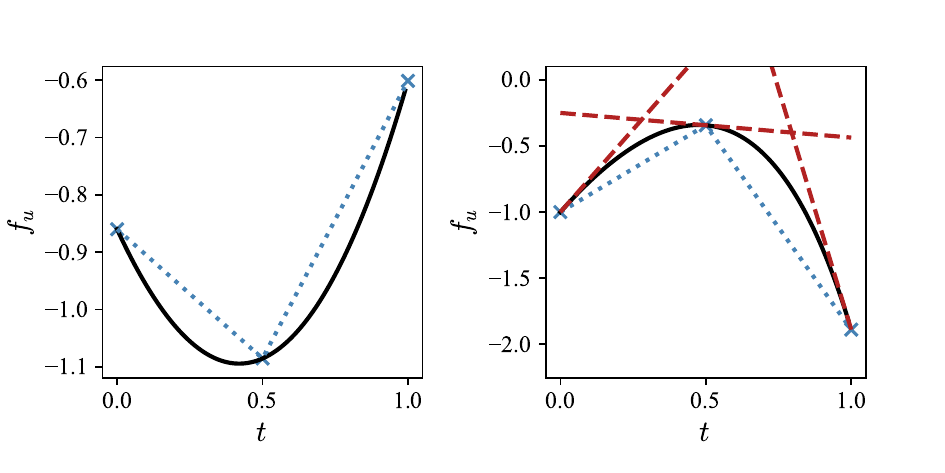}
    \caption{Example of a convex function (left) and a concave function (right). For a convex function, the straight line (dotted line) that joins any two points will upper-bound the function between those two points. For a concave function, the straight line (dotted line) that joins any two points will lower-bound the function between those two points. For a concave function, we can upper-bound the function by any tangent to the function (dashed line). Example bounds for $[0,0.5]$ and $[0.5.1]$ are shown by the dotted line (left-hand plot) and minimum of the dashed lines (right-hand plot).}
    \label{fig:ch6_CC}
\end{figure}

The above approach gives upper bounds on some interval $[0,t_1]$, and we can proceed by using these upper bounds to simulate events, if any, on $[0,t_1]$. If there are no events, we then choose some $t_2>t_1$ and calculate an upper bound on $[t_1,t_2]$ and repeat the process. In practice, \cite{sutton2023concave} suggest choosing $t_1,t_2,\ldots$ to be equally spaced and suggest ways of choosing the spacing in an adaptive way that balances the number of times we do not simulate an event from the bounding process on an interval, against the number of times we simulate events from the bounding process that are not accepted. The idea is that if the intervals are too low, we waste time by having too small an interval and thus having to calculate the upper bound, whereas if the interval is too large then the upper bound can become loose and we waste time by simulating lots of events that are rejected. It is also possible to recycle calculations used to decide whether to accept an event to improve the bounds. 

The final set of methods we will overview is based on simulating the events using numerical methods\index{numerical integration}. There have been two distinct approaches that have been suggested. The first is based on the equation for directly simulating the event times. Remember that we can simulate the next event time for a process with rate $\tilde{\lambda}_{\bz}(t)$ by simulating $u$, a realisation of a standard uniform random variable, and then finding $\tau$ the solution to
\[
- \int_0^\tau \tilde{\lambda}_\bz(s)\mbox{d}s = \log(1-u).
\]
The smallest solution, $\tau$, of this equation, is the time until the next event. \cite{pagani2020nuzz} suggest solving this equation numerically using Brent's method \cite[]{press2007numerical}. The other approach is to use numerical methods to find an upper bound. \cite{corbella2022automatic} suggest such a method, where they use Brent's algorithm to find the maximum of $\tilde{\lambda}_{\bz}(t)$ on some interval $[0,t_1]$, then propose points from a constant rate set to this maximum and use Poisson thinning\index{Poisson thinning}. If no event is simulated over the interval $[0,t_1]$ they repeat the process on the next interval.

The advantage of using such numerical methods is that they are fully general, that is they can be applied to any model in an automatic manner. The disadvantages are twofold. First, computationally they can be slow, depending on the numerical methods used. The second is that numerical errors may lead to errors in the simulation of the dynamics of the PDMP, so that the resulting PDMP may no longer target the correct distribution. The hope is that any numerical error is small so that the PDMP will have a stationary distribution\index{stationary distribution} that is still close to the target distribution. \cite{pagani2020nuzz} give results on how numerical error will impact the distribution that the PDMP is sampling from.

\paragraph{{\em Example: Bounded Hessian for Logistic Regression}}\index{logistic regression}

Logistic regression is one example where we have a bounded Hessian\index{Hessian}. To see this, let $\pi(\btheta)$ be the posterior for the logistic regression model of Section \ref{sec:ch1-logistic} with a Gaussian prior. Then differentiating (\ref{eq:ch1-logisticgradient}) gives
\[
-\frac{\partial^2 \log \pi(\btheta)}{\partial \theta_i\partial \theta_l} = \left[\bSigma_{\btheta}^{-1}\right]_{i,l} +\sum_{j=1}^N x_j^{(i)}x_j^{(l)}\left\{\frac{\exp\{\bx_j^\top\btheta\}}{1+\exp\{\bx_j^\top\btheta\}}\right\}  \left\{ \frac{1}{1+\exp\{\bx_j^\top\btheta\}}\right\},
\]
where the subscripts $i,l$ denote the $(i,l)$th element of the corresponding matrix. 
Now, for any probability $q$ we have $q(1-q)\leq 1/4$, so 
\[
\frac{\partial^2 \log \pi(\btheta)}{\partial \theta_i\partial \theta_l} \leq \left[\bSigma_{\btheta}^{-1}\right]_{i,l} +\frac{1}{4} \sum_{j=1}^N x_j^{(i)}x_j^{(l)}.
\]
If we introduce a $N\times d$ matrix $\bX$ whose $(j,l)$th entry is $x_j^{(l)}$, then this gives a bound on the Hessian of $-\log \pi(\btheta)$ that is $ \bJ=\bSigma^{-1}+(1/4)\bX^\top\bX$.  \index{Hessian}

For the Bouncy Particle Sampler\index{bouncy particle sampler}, the rate of the next bounce event, if the current state is $(\btheta,\bp)$, is 
\[
\max\{0,\bp\cdot\nabla(-\log\pi(\btheta+t\bp)\} \leq \max\{0,-\bp\cdot \nabla(\log\pi(\btheta) + \|\bp\|\|\bJ\bp\|t\},
\]
by the above argument. For the Zig--Zag Sampler, the rate of the next flip of component $i$ of the velocity is bounded above by, for example,
\[
\max\left\{0,
-p_{i} \frac{\partial \log \pi(\btheta)}{\partial \theta_{i}} +\sqrt{d}\|\bJ\be_i\|t
\right\}.
\]

\paragraph{{\em Example: Concave--Convex\index{concave-convex decomposition} Sampling for Bayesian Matrix Factorisation}}

Consider the Bayesian matrix factorisation\index{Bayesian matrix factorisation} model of Section \ref{sec:ch1-BMF}. To simplify notation, we will assume an improper uniform prior for the parameters $\btheta=\{\bU,\bV\}$ and set $\sigma^2=1$. The resulting log-posterior is
\[
\log\pi(\bU,\bV|\bY)=
-\frac{1}{2} \left\{ \sum_{i=1}^n\sum_{j=1}^m \left(Y_{i,j}- \sum_{k=1}^d U_{i,k}V_{k,j}\right)^2
\right\}.
\]
As we will see, the event rates for this model for the Zig--Zag\index{Zig-Zag sampler} Sampler or the Bouncy Particle Sampler\index{bouncy particle sampler} are polynomials of the time to an event, and such events can be simulated by concave--convex sampling. We will show this for the Zig--Zag Sampler, but the extension to the Bouncy Particle Sampler is simple.

Consider the rate associated with $U_{i,l}$. This depends on the derivative of minus $\log \pi$, which is
\[
-\frac{\partial \log\pi(\bU,\bV|\bY)}{\partial U_{i,l}} = \sum_{j=1}^m \left(Y_{i,j}- \sum_{k=1}^d U_{i,k}V_{k,j}\right) V_{l,j}.
\]
If the current state is given by a position $(\bU,\bV)$ and a velocity $(\dot{\bU},\dot{\bV})$ then the rate of an event as a function of the time to the next event $t$ is
\begin{eqnarray*}
\lefteqn{-\dot{U}_{i,l} \frac{\partial \log\pi(\bU+t\dot{\bU},\bV+t\dot{\bV}|\bY)}{\partial U_{i,l}} }\\
&=& 
\dot{U}_{i,l}  \sum_{j=1}^m \left(Y_{i,j}- \sum_{k=1}^d (U_{i,k}+t\dot{U}_{i,k})(V_{k,j}+t\dot{V}_{k,j})\right) (V_{l,j}+t\dot{V}_{l,j}).
\end{eqnarray*}
This is cubic in $t$, and it is easy to obtain a concave--convex decomposition\index{concave-convex decomposition} of this rate. The convex function will be the sum of terms in the cubic expression that have positive coefficients, and the concave function will be the sum of terms with negative coefficients. 
\index{piecewise deterministic Markov processes!simulation|)}

\subsection{Exploiting Model Sparsity}
\label{sec.ZZsparse}
One advantage of PDMP samplers is that they can take advantage of a certain type of sparsity in the target distribution to speed up computation. This is most easily described for the Zig--Zag Sampler\index{Zig-Zag sampler}, though similar ideas can be used for an adapted version of the Bouncy Particle Sampler \cite[see the section on the Local Bouncy Particle Sampler in][]{bouchard2018bouncy}.

We have described how we can simulate the Zig--Zag Sampler by simulating $d$ event times, one for each possible transition of the velocity. We then implement the event that occurs first, and then resimulate the further times for each of the $d$ possible events. We repeat this process multiple times in order to simulate the skeleton of a realisation of the process. However, we can improve on this implementation if the event that occurs does not affect the rates of many of the other types of events. This can happen if the model has a form of sparsity in terms of
\begin{equation} \label{eq:ch6-ZZrate}
\frac{\partial \log \pi(\btheta)}{\partial \theta_{i}} 
\end{equation}
only depending on a small number of the components of $\btheta$. 

To describe the idea formally, it is helpful to introduce some notation. For $i=1,\ldots,d$ let $\mathcal{S}_i \subset \{1,\ldots,d\}$ denote the components of $\btheta$ that (\ref{eq:ch6-ZZrate}) depends on. So if $j$ is not in $\mathcal{S}_i$, then (\ref{eq:ch6-ZZrate}) will not change as we vary $\theta_j$ if we keep all other components of $\btheta$ fixed. This means that if we have an event that changes the $j$th component of $\btheta$, then this will only affect the future rate of events that change the $i$th component of $\btheta$ if $j\in\mathcal{S}_i$. By the Markov property of the PDMP, if the rates are unchanged, we can re-use the simulated event times if $j\notin \mathcal{S}_i$ and $j\neq i$. We obviously need to re-simulate the event that flips the $j$th component of the velocity even if $j\notin \mathcal{S}_j$. The resulting algorithm is shown in Algorithm \ref{alg:ZZ_sparse}, with the key part being that after each event we only re-simulate some of the event times, for other events we just update and re-use the previously simulated times. 

\begin{algorithm}[h]
    \caption{Zig--Zag Sampler: Exploiting Sparsity}\index{Zig--Zag sampler}
    \KwIn{Event rates for each type of event $\lambda_{i}$, initial state $(\btheta,\bp)$, simulation time $T$.
  
    }
    Set $s=0$ and $k=0$.\\
    \For{$i=1,\ldots,d$}{
       Simulate $t_i$ the time until the next event that flips the $i$th component of the velocity. 
    }
    \While{$s<T$}{
    Calculate further time to next event $t=\min_{i=1,\ldots,d}\{t_i\}$.\\
    Update position $\btheta_{s+t}=\btheta_s+t\bp$.\\
    Decide on event type, $i^*=\arg\min \{t_i\}$.\\
    Update velocity $\bp_{s+t}=F_{i^*}(\bp)$.\\
    Update time $s=s+t$. \\
    Store skeleton points: set $k=k+1$, $\tau_k=s$ and $\btheta_{\tau_k}=\btheta_s$.\\
    Update further time to events:\\
     \For{$i=1,\ldots,d$}{
        \If{$i^*\in \mathcal{S}_{i} \mbox{ or } i=i^*$}{
       Simulate $t_i$ the time until the next event that flips the $i$th component of the velocity.}
       \Else{
       Set $t_i=t_i-t$.
       }
    }
    }
    \KwOut{Skeleton of events $\{(\tau_k,\btheta_{\tau_k})\}_{k=0}^n$}
    \label{alg:ZZ_sparse}
\end{algorithm}

As one example of the potential advantage of this algorithm, consider simulating a Gaussian target where the precision matrix is tri-diagonal. That is, the  $(i,j)$ entry of the precision matrix is zero if $|i-j|>1$. Such a model occurs if there is some form of Markov or AR(1) structure to the components of $\btheta$. In this case, $\mathcal{S}_i=\{i-1,i,i+1\}$ for $i=2,\ldots,d-1$, with $\mathcal{S}_1=\{1,2\}$ and $\mathcal{S}_d=\{d-1,d\}$. To understand the idea of Algorithm \ref{alg:ZZ_sparse}, imagine $d=5$ say, and that we have simulated that the further time to the five possible events is $0.3$, $0.7$, $1.3$, $\infty$ and $0.5$. The event that occurs first corresponds to flipping the first component of the velocity. This change only affects the rate at which future events that affect the first and second components of the velocity occur. So we need to re-simulate the further time to the next event of these types. For the other three types of events, we just update the further time to take account of the fact that a time of length has passed. Thus, the events at which the third to fifth components of the velocity change will now occur after a further time period of $1.0$, $\infty$ and $0.2$, respectively.

In terms of the computational advantage of this scheme, in this example, if $d$ is large then the computational cost per iteration involves resampling at most 3 event times rather than $d$ event times. It is also possible to further improve on Algorithm \ref{alg:ZZ_sparse} by using the fact that the order of occurrence of events that we do not re-simulate will not change \cite[see][]{bouchard2018bouncy} -- and this can reduce the cost per iteration to be of the order of the number of event times that we need to re-simulate.

\paragraph{{\em Example: Logistic Regression}}\index{logistic regression}

When would this idea be useful for sampling from the posterior of our logistic regression model? The rate of an event that flips component $i$ of the velocity depends on the $\theta_i$ derivative of $\log \pi(\btheta)$, which from (\ref{eq:ch1-logisticgradient}), is 
\[
-\left[\btheta^\top\bSigma_{\btheta}^{-1}\right]_{i} + \sum_{j=1}^N x_j^{(i)}\left\{y_j - \frac{\exp\{\bx_j^\top\btheta\}}{1+\exp\{\bx_j^\top\btheta\}}\right\}.
\]
We need this to depend on only a small set of components of $\btheta$. This would require two things. First that $\bSigma_{\btheta}^{-1}$ is sparse so only a small number of entries of the $i$th row or column of $\bSigma_{\btheta}^{-1}$ are non-zero. Second, we would require that the observations for which $\theta_j^{(i)}\neq 0$ would, combined, only have a small number of components of the covariates that are non-zero. Formally, we can define the set of rates that we would need to update after a flip of component $i$ of $\btheta$ as
\[
\mathcal{S}_{i}=\left\{ 
k : (\bSigma^{-1})_{i,k}\neq 0,\mbox{ or } \exists ~ j \mbox{ such that } \bx_j^{(i)}\bx_j^{(k)}\neq0 
\right\}.
\]

This can happen for models with random effects which are included within $\btheta$. If we set the random effects to be $\theta_{1},\ldots,\theta_{N}$, then for $j\in \{1,\ldots,N\}$, $\bx_j^{(j)}=1$ and $\bx^{(i)}_j=0$ for $i\neq j$. If further, we have that the random effects are independent of each other and the other parameters, $\mathcal{S}_i$ will only include the parameters of the fixed effects.

\subsection{Data Subsampling Ideas} \label{ch5:sec-SS}

One potential advantage of PDMP samplers in Bayesian statistics is that they can use subsampling ideas to reduce the computational cost per iteration. This was first suggested for the Zig--Zag Sampler by \cite{Bierkens:2019},\index{Zig-Zag sampler} though the ideas apply more widely. 

The starting point is a more general observation that we can potentially simulate from a target distribution if we have an unbiased estimator of $\nabla \log \pi$. This is most easily seen for the Zig--Zag Sampler, and we will focus just on this case for simplicity.  See \cite{Fearnhead:2018}, and related ideas for the local Bouncy Particle Sampler in \cite{bouchard2018bouncy}, for how this is generalised to other PDMPs.

The Zig--Zag Sampler\index{Zig-Zag sampler} has $d$ possible types of event. Consider the $i$th such event.
If the current position is $\btheta$ and the $i$th component of the velocity is $p_{i}$, then this component flips with a rate
\[
\max\left\{0, -p_{i}\frac{\partial \log \pi(\btheta)}{\partial \theta_{i}}
\right\}.
\]
The key property of this rate that means that the sampler targets $\pi(\btheta)$ is that the difference in rate between a flip from $p_{i}$ to $-p_{i}$ and the rate of the reverse event is
\[
\max\left\{0, -p_{i}\frac{\partial \log \pi(\btheta)}{\partial \theta_{i}}
\right\} - \max\left\{0, p_{i}\frac{\partial \log \pi(\btheta)}{\partial \theta_{i}} \right\}= -p_{i}\frac{\partial \log \pi(\btheta)}{\partial \theta_{i}}.
\]
In practice, for Zig--Zag, one of the two rates is equal to 0, and this is the most efficient choice and corresponds to the Zig--Zag Sampler using the canonical rates (see Section \ref{sec:ch6-ZZ}). 

Now imagine we have a family of vector-valued random variables $\bG(\btheta)$, such that $\Expect{\bG(\btheta)}=-\nabla \log \pi(\btheta)$ for all $\btheta$. Then if we implement the same dynamics as the Zig--Zag\index{Zig-Zag sampler} Sampler, but with the rate of flipping the $i$th component of the velocity equal to
\[
\Expect{\max\{0,p_{i}G_{i}(\btheta)\}},
\]
then this will also produce a PDMP sampler that targets $\pi$. To see this, as above, consider the difference in the rate of flipping $p_{i}$ to $-p_{i}$ and the rate of the reverse event. This is
\begin{eqnarray*}
\lefteqn{
\Expect{\max\{0,p_{i}G_{i}(\btheta)\}} - \Expect{\max\{0,-p_{i}G_{i}(\btheta)\} } } \\ &=&
\Expect{\max\{0,p_{i}G_{i}(\btheta)\}-\max\{0,-p_{i}G_{i}(\btheta)\} } \\
&=& \Expect{p_{i}G_{i}(\btheta)} = p_{i}\Expect{G_{i}(\btheta)} = -p_{i}\frac{\partial \log \pi(\btheta)}{\partial \theta_{i}},
\end{eqnarray*}
where we have used the standard result $\max\{0,x\}-\max\{0,-x\}=x$, linearity of expectation, and the definition of the expectation of $\bG$. This is precisely the condition we need on the rates for a PDMP Sampler with the Zig--Zag dynamics to target $\pi$, the only difference is the rates being used are no longer the canonical rates.

In order to use the Zig--Zag Sampler with these rates we need a way of simulating the events. This is more challenging than for standard Zig--Zag\index{Zig-Zag sampler} as we need to deal with the rates being defined implicitly by an expectation. To do this, the standard approach is to find a bounding $b_i(\btheta)$ such that for any realisation of $\bG(\btheta)$ we have $p_{i}G_{i}(\btheta)\leq b_i(\btheta)$. 
If we can find such a bound then we can still use Poisson thinning\index{Poisson thinning} to simulate the events. Let the time until the next event which flips $p_{i}$ be
\[
\tilde{\lambda}^{(i)}_{\bz}(t)=\Expect{\max\{0,p_{i}G_{i}(\btheta+t\bp)\}},
\] 
and the corresponding bounding rate, which is used to simulate potential events, be  $b_i(\btheta+t \bp)$. 
Then Poisson thinning\index{Poisson thinning} for events of rate $\tilde{\lambda}^{(i)}_{\bz}(t)$ is possible by using the following steps:
\begin{itemize}
\item[(T0)] Set current time to $s=0$
\item[(T1)] Simulate the time $\tau>s$ of the next event in a process with rate $\bar{\lambda}(t)=b_i(\btheta+t \bp)$.
\item[(T2)] Simulate $g_{i}$, a realisation of $G_{i}(\btheta+\tau \bp)$.
\item[(T3)] Accept the event time with probability $\max\{0,p_{i}g_{i}\}/b_i(\btheta+t\bp)$. Otherwise set $s=\tau$ and return to (T1).
\end{itemize}
To see that this is a valid Poisson thinning\index{Poisson thinning} algorithm to simulate events with rate $\tilde{\lambda}^{(i)}_{\bz}(t)$, we just need to calculate the probability of accepting an event in step (T3). By averaging over the possible realisation of $g_i$ in step (T2) and using the fact that by definition for any $g_i$, the probability in step (T3) is less than or equal to 1, this is
\[
\Expect{\max\{0,p_{i}G_{i}\}/b_i(\btheta+t\bp)}=
\frac{\Expect{\max\{0,p_{i}G_{i}\}}}{b_i(\btheta+t\bp)} =
\frac{\tilde{\lambda}^{(i)}_{\bz}(t)}{b_i(\btheta+t\bp)},
\]
as required. 

How does this idea relate to the use of subsampling\index{subsampling}? Consider $\pi$ being a posterior distribution\index{posterior distribution}, and suppose that $\log \pi$ can be written as a sum
\[
\log \pi (\btheta) =\sum_{j=1}^N \log \pi_j(\btheta),
\]
where $\log \pi_j$ for $j=1,\ldots,N$ is $1/N$ times the log-prior plus the log-likelihood contributions from the $j$th data point. Then this gives a simple way of constructing an unbiased estimator\index{gradient!subsampling estimator} of $-\nabla \log \pi(\btheta)$, by simulating $I$ uniformly on $\{1,\ldots,N\}$ and setting $\bG(\btheta)$ to $-N\nabla \log \pi_I(\btheta)$. 

The advantage of using such an unbiased estimator within the Zig--Zag\index{Zig-Zag sampler} Sampler is that at each iteration of the Poisson thinning\index{Poisson thinning} algorithm used to simulate an event, that is, step (T2) and (T3) above, we need to process only one data point. This gives a per-iteration saving of a factor of $N$ over the standard Zig--Zag Sampler which requires calculating derivatives of $\log \pi$. However, there are additional costs to using this subsampling idea. First, often the bounds that we use for  Poisson thinning\index{Poisson thinning} will be larger if we use subsampling -- as they have to bound the rate for all possible realisations of $g_{i}$. This will lead to more iterations of the Poisson thinning\index{Poisson thinning} algorithm to simulate the PDMP for the same amount of (stochastic process) time. Second, as we are no longer using the canonical rates we will introduce more events, and this will lead to more random-walk-like behaviour and slower mixing\index{mixing}. Empirical results in \cite{Bierkens:2019} suggest that the overall effect of these is to counteract the factor of $N$ improvement in per-iteration cost.

So can subsampling ideas within PDMPs be beneficial? It turns out they can, but we must use a better, that is, lower variance, estimator for $\bG$. This can be done using control variate\index{control variate}\index{gradient!control variate estimator} ideas that are common in SGLD\index{stochastic gradient Langevin dynamics} (see Section \ref{sec:ch3-sgld-grad-var}). 
We first run an optimisation algorithm, such as SGD, to find the mode or a value close to the mode of $\log \pi$. Denote this value by $\widehat{\btheta}$. Then we can write
\[
\log \pi (\btheta) = \log\pi(\widehat{\btheta})+ \sum_{j=1}^N \{\log \pi_j(\btheta) -\log \pi_j(\widehat{\btheta})\}.
\]
So an unbiased estimator can be obtained by first sampling $I$ uniformly on $\{1,\ldots,N\}$ and then setting $\bG(\btheta)$ to 
\[
-\log\pi(\widehat{\btheta})-N\{\nabla \log \pi_I(\btheta) - \nabla \log \pi_I(\widehat{\btheta})\}.
\]
Importantly the term $-\log\pi(\widehat{\btheta})$ is a constant, so requires a single up-front $O(N)$ cost to calculate it. Then evaluating a realisation of this estimator has $O(1)$ cost. If the \index{Hessian}Hessian of $-\log \pi$ is bounded, then \cite{Bierkens:2019} show that we can obtain the bounds needed to simulate the Zig--Zag Sampler using this unbiased estimator using the linear bounds described for bounded Hessian targets in Section \ref{sec:ch6-simPDMPsamplers}. In this case, for a fixed accuracy of the final Monte Carlo sample, we can obtain a speed-up by a factor of $N$, after finding $\widehat{\btheta}$  relative to the standard Zig--Zag Sampler\index{Zig-Zag sampler}.

\paragraph{\em{Example: Subsampling for Logistic Regression}} \index{logistic regression}

To further explain how to implement the Zig--Zag Sampler with subsampling, we will consider the logistic regression model. To simplify exposition we will assume that we have an improper flat prior, so in the notation above $\pi_0(\btheta)\propto 1$. Thus, we can drop the contribution of the prior to $\pi$ and we have $\pi(\btheta)\propto \prod_{j=1}^N \pi_j(\btheta)$ with
\[
\pi_j(\btheta) = \left(  \frac{\exp\{y_j\bx_j^\top\btheta\}}{1+\exp\{\bx_j^\top\btheta\}} \right),
\]
where $y_j$ is the binary response and $\bx_j$ is the vector of covariates for the $j$th observation.

Taking the first derivatives of $-\log \pi(\btheta)$ gives
\[
-\frac{\partial \log \pi_j(\btheta)}{\partial \theta_i} =
x_j^{(i)}\left\{ \frac{\exp\{\bx_j^\top\btheta\}}{1+\exp\{\bx_j^\top\btheta\}} - y_j \right\}. 
\]
Using $0\leq\exp(a)/(1+\exp(a))\leq1$, we have that the modulus of this derivative is bounded by $x_j^{(i)}$. Thus, if we use the unbiased estimator
\[
G_{i}(\btheta)= -
N\frac{\partial \log \pi_I(\btheta)}{\partial \theta_i}, ~~ \mbox{$I$ uniformly distributed on $\{1,\ldots,N\}$}, 
\]
then we can bound the rate of an event by 
\[
b_i(\btheta)=N\max_{j=1,\ldots,N} |x_j^{(i)}|.
\]
In the following, we will call the resulting sampler Zig--Zag with subsampling.\index{Zig-Zag sampler}

How about if we use control variates\index{control variate}? Fix $\widehat{\btheta}$, and consider the estimator of the gradient\index{gradient} \index{gradient!control variate estimator}
\[
G^{(CV)}_{i}(\btheta)= -\frac{\partial \log \pi(\widehat{\btheta})}{\partial \theta_i} - 
N\left(\frac{\partial \log \pi_I(\btheta)}{\partial \theta_i}-
\frac{\partial \log \pi_I(\widehat{\btheta})}{\partial \theta_i}
\right),
\]
with again, $I$ is uniformly distributed on $\{1,\ldots,N\}$. To obtain appropriate bounds for $\bp^{(j)}\bG^{(j)}_{CV}(\btheta)$, consider the second derivatives
\[
-\frac{\partial^2 \log \pi_j(\btheta)}{\partial \theta_i\partial \theta_l} =
x_j^{(i)}\bx_j^{(l)}\left\{ \frac{\exp\{\bx_j^\top\btheta\}}{(1+\exp\{\bx_j^\top\btheta)\})^2} \right\}. 
\]
As before using that $0\leq\exp(a)/(1+\exp(a))^2\leq 1/4$, we can bound the modulus of this second derivative by $(1/4)|\bx_j^{(i)}\bx_j^{(l)}|$. This gives the following bound 
\begin{eqnarray}
\lefteqn{\left|\frac{\partial \log \pi_j(\btheta+t\bp)}{\partial \theta_i} - \frac{\partial \log \pi_j(\widehat{\btheta})}{\partial \theta_i} \right|} \nonumber \\
&\leq& 
\left|\frac{\partial \log \pi_j(\btheta+t\bp)}{\partial \theta_i} -\frac{\partial \log \pi_j(\btheta)}{\partial \theta_i}
\right| + 
\left|\frac{\partial \log \pi_j(\btheta)}{\partial \theta_i} -\frac{\partial \log \pi_j(\widehat{\btheta})}{\partial \theta_i}
\right|  \nonumber \\
& \leq &
\frac{1}{4}|x_j^{(i)}|\lVert\bx_i\rVert
\left(\lVert \btheta-\widehat{\btheta}\rVert + t\lVert \bp \rVert \right), \label{eq:ch6-subsamplingbound}
\end{eqnarray}
where $\lVert\cdot\rVert$ denotes the Euclidean norm. The latter inequality comes from (i) bounding the change in a function by the size of the change in the argument times a bound on the gradient\index{gradient} in the direction of the change; and (ii) as the gradient\index{gradient}  $\partial \log \pi_j/\partial \theta_i$ is bounded by $(1/4)|\bx_j^{(i)}\bx_j^{(l)}|$ in the $l$th direction, we  can bound the dot-product of the gradient with a vector $\bv$ by
\begin{eqnarray*}
\sum_{l=1}^d \frac{1}{4}|x_j^{(i)}x_j^{(l)}v_l| &=&
\frac{1}{4} |x_j^{(i)}| \sum_{l=1}^d  |x_j^{(l)}v_{l}| 
\leq \frac{1}{4} |x_j^{(i)}| \lVert\bx_i\rVert \lVert\bv\rVert,
\end{eqnarray*}
with the last step using Cauchy--Schwarz\index{Cauchy-Schwarz}.

Using \ref{eq:ch6-subsamplingbound}, we get the following linear bound on the rate of events when we use control variates\index{control variate}\index{gradient!control variate estimator}. 
\begin{eqnarray*}
\lefteqn{\max\{0,p_{i}G_{i}^{(CV)}(\btheta+t\bp)\} }
\\ &\leq&  
\max\left\{0,-p_{i} \frac{\partial \log \pi(\widehat{\btheta})}{\partial \theta_i} + \frac{1}{4} N M_i 
\left(\lVert \btheta-\widehat{\btheta}\rVert + t\lVert \bp \rVert \right)\right\},
\end{eqnarray*}
where $M_i=\max_{j=1,\ldots,N} (|x_j^{(i)}| \lVert\bx_j\rVert )$.  In the following, we will call the resulting sampler Zig--Zag with control variates.\index{Zig-Zag sampler}

\begin{figure}
    \centering
    \includegraphics[width=\textwidth]{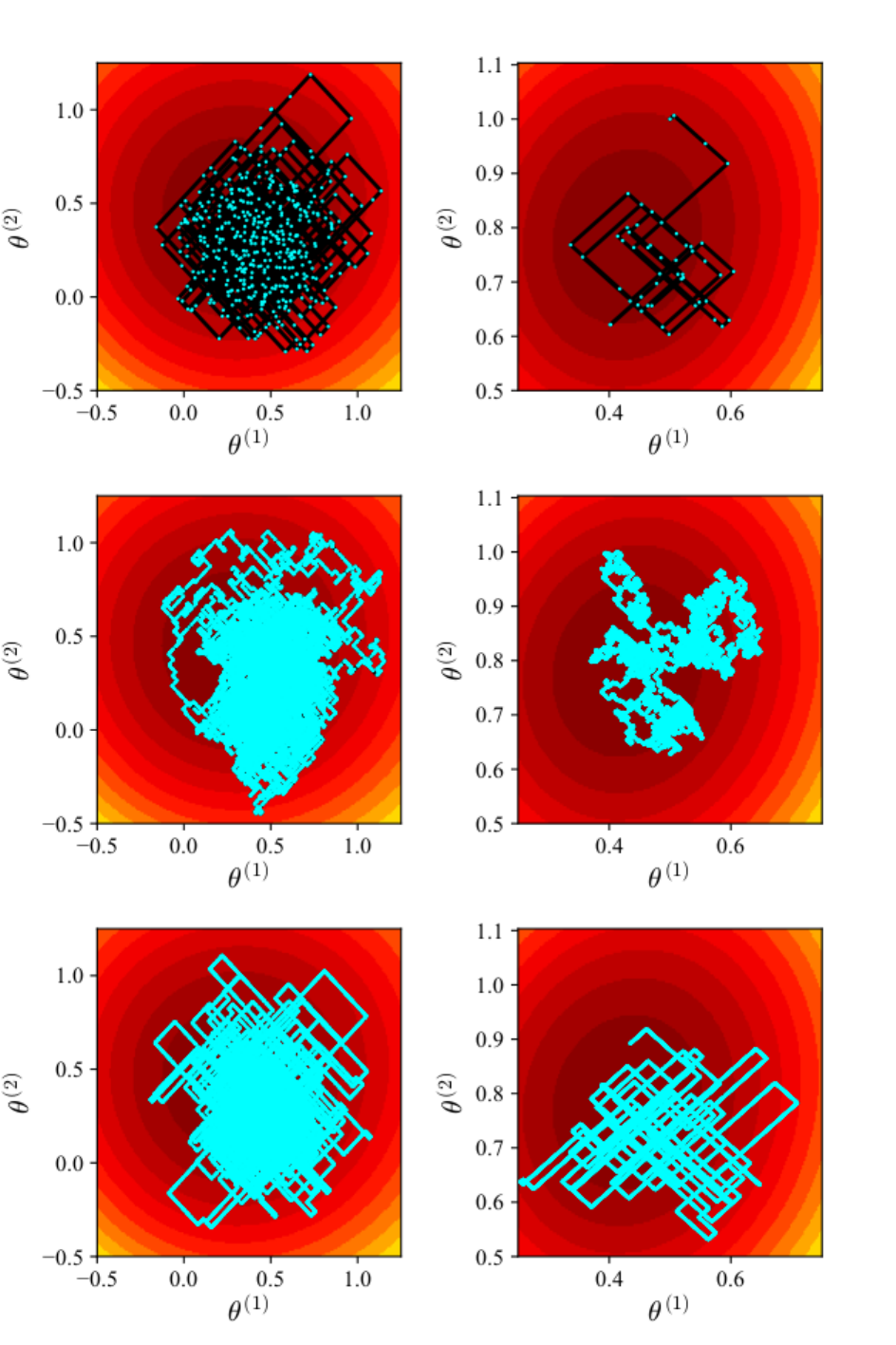}
    \caption{Example output of Zig--Zag Samplers for logistic regression: standard Zig--Zag (top row), Zig--Zag with subsampling (middle row) and Zig--Zag with control variates (bottom row). For each plot, the heat map shows the contours of $\log\pi$, the line shows the trajectory of Zig--Zag and the lighter dots show the points where we propose a possible event. In all cases, we ran Zig--Zag for the same number of data-point gradient calculations. As the top row does not involve subsampling, this means we propose $N$-times as many events for the middle and bottom rows as for the top row. Results for $N=100$ (left-hand column) and $N=900$ (right-hand column).}
    \label{fig:ch6_ZZSS}
\end{figure}

To demonstrate how subsampling works in practice and the scaling\index{scaling} of the methods with sample size we will compare three PDMP samplers for logistic regression\index{logistic regression} with sample sizes of $N=100$ and $N=900$. These are the standard Zig--Zag Sampler, Zig--Zag with subsampling, and Zig--Zag\index{Zig-Zag sampler} with control variates\index{control variate}. In particular, we want to give intuition about the different impacts of the computational cost for estimating gradients\index{gradient}, the efficiency of the Poisson thinning\index{Poisson thinning} bounds on simulating events, and the mixing\index{mixing} of the PDMP on the three algorithms.

Results are shown in Figure \ref{fig:ch6_ZZSS}. There are a number of points to draw out. First, we see the potential advantage of subsampling in reducing the cost per iteration of the samplers. If this is dominated by accessing and calculating gradients\index{gradient} for each data point, then the subsampling versions of Zig--Zag are able to propose many more event times (as seen by the larger number of blue dots for the middle and bottom rows). However, this in itself does not lead to a more accurate MCMC\index{Markov chain Monte Carlo} method for the same computational cost -- as the drawback of subsampling is that the bounds used for Poisson thinning\index{Poisson thinning} are worse. About two proposed events are needed for one actual event with standard Zig--Zag, but this reduces to over 10 proposed events for the two samplers that use subsampling.

One impact of this is, if we consider $N=100$ then the standard Zig--Zag process is simulated for stochastic process time of 80 time units, and this is only increased to 130 time units for Zig--Zag with subsampling and 160 time units for Zig--Zag with control variates. Moreover, the Zig--Zag processes that use subsampling have a higher overall rate of events, particularly when control variates are not used. This leads to more random-walk behaviour for Zig--Zag with subsampling (see middle row) which worsens the mixing\index{mixing} of the sampling. This is less of an issue when we use control variates.

Next, consider the scaling\index{scaling} of the algorithms by comparing $N=100$ with $N=900$. First, the posterior for $N=900$ is more concentrated, but if we re-scale axes (as has been done in the plots), then the posterior contours and the dynamics of standard Zig--Zag are similar for the two cases. However, if we have fixed the computational resource, then the impact on standard Zig--Zag is that we can only simulate the process for a shorter period -- in this case simulating nine times fewer actual events. For Zig--Zag with subsampling, we are able to propose the same number of events, but the higher rate of events and the looser bound for Poisson thinning\index{Poisson thinning} means that we are only able to simulate a trajectory that is only 4 times as long as for the standard Zig--Zag, despite proposing 900 times as many events. Moreover, we see that the trajectory is increasingly diffusive, and thus the overall exploration of the state-space of this sampler qualitatively looks no better than for standard Zig--Zag. This property is shown more rigorously by \cite{Bierkens:2019}, who show that the scaling with $N$
of the standard Zig--Zag and Zig--Zag with subsampling is similar. Moreover, their empirical results suggest that for the same computational cost, standard Zig--Zag is more accurate.

By comparison, we see better mixing\index{mixing} behaviour for Zig--Zag with control variates. The sampler is able to simulate a trajectory that is approximately ten times as long as the standard Zig--Zag. Moreover, the mixing looks qualitatively similar to that of the algorithm for $N=100$ and to that of the standard Zig--Zag for $N=900$. This is again shown more rigorously in \cite{Bierkens:2019}, where results suggest that the accuracy as measured, say, by effective sample size\index{effective sample size} per CPU cost scales by a factor of $N$ better for Zig--Zag with control variates than for the other two samplers. However, Zig--Zag\index{Zig-Zag sampler} with control variates does need an additional pre-processing step to find the mode, or a value near the mode, of the posterior and to calculate the gradient\index{gradient} of the log posterior at this estimate of the mode. This improved scaling\index{scaling} has been termed \textit{super-efficiency}, and the fact that we can only achieve super-efficiency after a pre-processing step is shown theoretically by \cite{johndrow2020no}.

\section{Extensions}

\index{piecewise deterministic Markov processes!simulation|(}

The PDMP samplers we have considered so far are appropriate for sampling from continuous densities that are differentiable almost everywhere and are based on specific constant velocity dynamics. We now describe recent work at generalising the samplers: to allow sampling from discontinuous targets; introduce reversible-jump moves to allow sampling from targets defined across spaces of different dimension; and generalise the velocity space and the constant velocity dynamics.

\subsection{Discontinuous Target Distribution} \index{discontinuous target}

The PDMP samplers we have described can sample from target densities that are differentiable everywhere. It is also easy to see that they are suitable for densities that are non-differentiable, providing the set of points where the density is not differentiable is a null set. However, as described, they cannot be used for densities that are not continuous everywhere.

It is possible to extend PDMP samplers so they are suitable for many discontinuous target\index{discontinuous target} densities. The idea is to use standard PDMP dynamics in regions where the target is continuous, and then add additional dynamics whenever the PDMP sampler reaches a point of discontinuity. This approach has been suggested by \cite{bierkens2018piecewise} for continuous densities but defined only on a compact region, and by \cite{chevallier2021pdmp} in more generality. We will outline the basic idea and give some examples of appropriate dynamics at points of discontinuity for different samplers.

Assume our target density $\pi(\btheta)$ can be defined in terms of a set of continuous densities, $\pi^{(i)}$, each constrained to a set of open regions $E_i$ of $\mathbb{R}^d$, for $i=1,\ldots, K$ for some $K$. Let $\Gamma$ be the set of $\btheta$ points that lie on the boundary of one or more regions, and assume this is a null-set with respect to Lebesgue measure of $\mathbb{R}^d$. We assume that $E_i$ and $\Gamma$ partition $\mathbb{R}^d$, so that any $\btheta\in \mathbb{R}^d$ lies in precisely one of $E_1,\ldots,E_K,\Gamma$. For $\btheta\notin\Gamma$, our target is
\[
\pi(\btheta)= \pi^{(i)}(\btheta) ~~ \mbox{ for } \btheta\in E_i.
\]
The simplest example of such a density will be for $\btheta$ constrained to some compact region, $E_1$. In this case, we have $\Gamma$ as the boundary of $E_1$ and $E_2$ is the complement of $E_1 \cup \Gamma$, with $\pi^{(2)}(\btheta)=0$ for $\btheta\in E_2$. 

The idea is that we can define a PDMP sampler that is appropriate for each $\pi^{(i)}$, but need to now define what happens when the $\btheta$ component of the state tries to leave $E_i$. To explain this, whilst keeping the notation simple, we will consider the case of $K=2$ regions, and describe conditions for the PDMP sampler on the boundary between $E_1$ and $E_2$ that are sufficient for it to target $\pi(\btheta)$ as defined above. Extending this to $K>2$ is trivial as we just apply the same conditions to each boundary between two regions. (We do not need to consider behaviour at points that lie on the boundary between three or more regions as the sampler will not hit such points with probability 1.)

For $\btheta\in\Gamma$, let $\bn(\btheta)$ be the normal to the boundary and assume that the normal is defined to point out of region $E_1$. Assume we have a PDMP sampler with velocity set $\mathcal{V}$ and with stationary distribution\index{stationary distribution} for the velocity component that is independent of $\btheta$ and is denoted by $\pi_{\bp}$. Once we hit the boundary, the velocity will determine whether the state is moving out of $E_1$ and into $E_2$ or vice-versa. To distinguish these two possibilities, define for each $\btheta\in\Gamma$
\[
\mathcal{V}_{\btheta}^+=\left\{ \bp\in\mathcal{V}: \bn(\btheta)^\top\bp>0 \right\}, \mbox{ and }~~
\mathcal{V}_{\btheta}^-=\left\{ \bp\in\mathcal{V}: \bn(\btheta)^\top\bp<0 \right\}. 
\]
So, for example, if the state is $(\btheta,\bp)$ for $\btheta\in\Gamma$ and $\bp \in \mathcal{V}^+_{\btheta}$ then the sampler was in region $E_1$ and is moving into $E_2$.

Now define a family of probability density, or probability mass, functions for $\bp\in\mathcal{V}$, as
\[
\ell_{\btheta}(\bp)=\left\{
\begin{array}{cl}
\left|\bn(\btheta)^\top\bp \right| \pi_{\bp}(\bp) \pi^{(2)}(\btheta) & \mbox{ if } \bp\in\mathcal{V}^+_{\btheta} \\
\left|\bn(\btheta)^\top\bp \right| \pi_{\bp}(\bp) \pi^{(1)}(\btheta) & \mbox{ otherwise}
\end{array}
\right.
\]
This is just proportional to the density $\pi_{\bp}(\bp)$ weighted by the size of the velocity in the direction of the normal $\bn(\btheta)$ and weighted by the density at $\btheta$ in the region that the sampler is moving to.

Finally, define a family of transition kernels\index{kernel!Markov transition} for the velocity component, $\mathbb{Q}_{\btheta}^b(\bp'\in \cdot |\bp)$ for each $\btheta\in\Gamma$. The following theorem, taken from \cite{chevallier2021pdmp}, gives appropriate dynamics for our PDMP sampler on the boundary.
\begin{theorem} \label{thm:ch6-boundary}
Assume $\pi_{\bp}$ is symmetric, so $\pi_{\bp}(\bp)=\pi_{\bp}(-\bp)$, and that for each $\btheta\in\Gamma$ the transition kernel $\mathbb{Q}_{\btheta}^b$ has $\ell_{\btheta}$ as its invariant distribution\index{invariant distribution}. Then a PDMP sampler with:
\begin{itemize}
    \item[(i)] dynamics for $\btheta\in E_i$, for $i=1,2$, that have invariant distribution\index{invariant distribution} $\pi^{(i)}(\btheta)\pi_{\bp}(\bp)$; and
    \item[(ii)] for $\btheta\in\Gamma$ has a transition that keeps $\btheta$ unchanged but that:
    \begin{itemize}
    \item[(B1)] flips the velocity $\bp'=\bp$; 
    \item[(B2)] updates the velocity according to  $\mathbb{Q}_{\btheta}^b$, that is $\bp''\sim \mathbb{Q}_{\btheta}^b(\cdot,\bp')$; and
    \item[(B3)] updates the state of the PDMP to $(\btheta,\bp'')$;
    \end{itemize}
\end{itemize}
will have invariant distribution\index{invariant distribution} $\pi(\btheta)\pi_{\bp}(\bp)$. 
\end{theorem}
Steps (B1) -- (B3) of the theorem give appropriate dynamics for the velocity when we hit the boundary, with (B2) stated in terms of a transition kernel\index{kernel!Markov transition} that has $\ell_{\btheta}$ as its invariant distribution\index{invariant distribution}. 

There are various possible choices of dynamics due to different choices for the transition kernel. A trivial choice for $\mathbb{Q}_{\btheta}^b$ is the identity map. In this case, the transition at the boundary is to reverse the sign of the velocity, which means that the sampler will never leave the region that it starts with. Thus this choice is only suitable for the case where we start the sampler in $E_1$ and where $\pi^{(2)}(\btheta)=0$, that is, there is no probability mass in $E_2$. Even in this case, this choice may not be a good one, as it will force the sampler to retrace its steps once it hits the boundary, which will slow down mixing\index{mixing}.

An alternative choice is to define $\mathbb{Q}_{\btheta}^b$ to be independent of the current velocity, and just to involve sampling from $\ell_{\btheta}$. The problem with this is that sampling from $\ell_{\btheta}$ may be difficult. For both the Coordinate Sampler\index{coordinate sampler} and the Zig--Zag Sampler\index{Zig-Zag sampler}, $\ell_{\btheta}$ is a discrete distribution and can be calculated exactly. For the Coordinate Sampler\index{coordinate sampler}, $\bp$ can take $2d$ possible values, and this approach can be reasonable. However, for the Zig--Zag Sampler, $\bp$ can take $2^d$ possible values, and thus calculating and simulating from $\ell_{\btheta}$  is prohibitive unless $d$ is small. In cases where $\ell_{\btheta}$ is difficult to sample from one can instead use Metropolis--Hastings to define a transition kernel\index{kernel!Markov transition} that has the required invariant distribution\index{invariant distribution}. This involves proposing a new velocity from an arbitrary proposal\index{proposal} distribution and then accepting or rejecting that proposal. The problem with this is that if the acceptance probability\index{acceptance probability} is not high, then we are likely to reject the proposal, and our new velocity will just be minus the velocity with which we hit the boundary, from step (B1), which will inhibit mixing\index{mixing}. A partial solution is to define $\mathbb{Q}_{\btheta}^b$ to involve $L>1$ Metropolis--Hastings steps, though this comes with an increased computational cost of sampling from $\mathbb{Q}_{\btheta}^b$.

For the Bouncy Particle Sampler\index{bouncy particle sampler}, there is a simple and very natural choice of dynamics at the boundary which satisfies the condition of Theorem \ref{thm:ch6-boundary}. If $\bp\in\mathcal{V}^+_{\btheta}$, so we are currently moving out of $E_1$ then the dynamics are: 
\begin{itemize}
\item[(R1)] With probability $\min\{1,\pi^{(2)}(\btheta)/\pi^{(1)}(\btheta)\}$ the velocity is unchanged, that is, $\bp''=\bp$;
\item[(R2)] Otherwise, we reflect the velocity in the tangent to the boundary at $\btheta$, that is, using the notation introduced for reflection\index{reflection}s $\bp''=\cR_{\bn(\btheta)}(\bp)$.
\end{itemize}
If $\bp\in\mathcal{V}^-_{\btheta}$, then the dynamics are as above but with $\pi^{(1)}$ and $\pi^{(2)}$ interchanged in (R1). Under these dynamics, if the sampler is moving to a region of higher probability density, it continues. If not, then with some probability it continues, otherwise it reflects back. A special case where we have $\pi(\btheta)$ defined only on the compact region $E_1$, so that $\pi^{(2)}(\btheta)=0$, in which case the sampler will always reflect if it hits the boundary.

We cannot apply similar dynamics for the Zig--Zag Sampler unless the boundary aligns appropriately with the velocity axes, as the reflected velocity in step (R2), $\cR_{\bn(\btheta)}(\bp)$, may no longer be a valid velocity. However, \cite{chevallier2021pdmp} show that the above dynamics for the Bouncy Particle Sampler can be viewed as the behaviour of the Bouncy Particle Sampler if we approximate the discontinuous target\index{discontinuous target} by a continuous one, by smoothing out the discontinuity, but then consider the limit where we allow the transition between $\pi^{(1)}$ and $\pi^{(2)}$ to occur more quickly. By the same strategy, we can construct an appropriate transition kernel\index{kernel!Markov transition} for Zig--Zag\index{Zig-Zag sampler}, see \cite{chevallier2021pdmp} for more details.

\paragraph{\em{Example: Logistic Regression with Constraints}}\index{logistic regression}

As an example application for this algorithm, consider the logistic regression model but with constraints on the parameters. There are two natural constraints, one is that we may know that the effect of a covariate is such that larger values will increase, say, the probability of observing a response of 1. If the $i$th component of the $\bx_j$s is such a covariate, then $\theta_i\geq 0$. Similarly, if an increase in the $i$th component of the covariate vector is known to lead to a decrease in the probability, then $\theta_i\leq 0$. Alternatively, we may have some constraints on the size of the effect, for example, $\theta_i\geq \theta_l$.

For the Bouncy Particle Sampler\index{bouncy particle sampler}, the simplest way of including such a constraint is to calculate the time when the current deterministic dynamics will first violate the constraint. If an event does not happen by this time, then we just reflect the velocity off the boundary of $\btheta$ space implied by the constraint. For the constraint $\theta_i\geq 0$ or $\theta_i\leq 0$ this would involve flipping the $i$th component of the velocity. For the constraint $\theta_i\geq \theta_l$, this involves the reflection\index{reflection} $\cR_\bg$ with $\bg$ defined so $g_{i}=1$, $g_{l}=-1$ and all other entries set to zero. 

In this example, we can use the same adaption to the Zig--Zag Sampler\index{Zig-Zag sampler}, as for each constraint the reflection\index{reflection} at the boundary will produce a new velocity that is valid for the sampler. However, this would not be the case if, for example, we wanted to enforce a constraint such as $\theta_{1}>2\theta_{2}$.

\subsection{Reversible Jump PDMP Samplers}
\index{piecewise deterministic Markov processes!reversible jump|(}

Another class of distributions that the standard PDMP samplers are not suitable for are distributions defined on a set of spaces of different dimensions. In the following, we think of this in terms of a distribution over a discrete set of models, with a continuous density for each model. Whilst PDMPs can be used to explore the distribution defined for each model, we would need a way of moving between these models -- which would lead to a type of reversible jump version of PDMP samplers.

One approach to do this is to add additional events that introduce discrete jumps from one space to another. Such moves have been proposed for other continuous-time samplers, see \cite{grenander1994representations}, \cite{phillips1996bayesian}, \cite{stephens2000bayesian}. However, they can be hard to implement due to challenges with simulating these moves with the correct rate.

A computationally more efficient procedure exists if the different models are defined in terms of some common $d$-dimensional parameter $\btheta$, but each model fixes one or more components of $\btheta$ to a specific value. The most common example of such a distribution is the posterior distribution\index{posterior distribution} for the coefficients of a linear or generalised linear regression model, where different models correspond to including different sets of covariates in the model. Thus, we can define $\btheta$ to be the coefficients for the full set of covariates, and a given model will fix the coefficients of covariates not in the model to zero.

To motivate the form of the PDMP samplers we will introduce, consider first the case where we have a single covariate, so $d=1$. Also, ignore having any intercept in the linear or generalised linear model. We now have two models, depending on whether we include the covariate or not. If $\ell(\theta)$ denotes the likelihood as a function of $\theta$ and if we have a prior that is a mixture of point mass at $0$, which we will denote by $\delta_0(\theta)$, and a prior defined on $\mathbb{R}$, $\nu(\theta)$, then the posterior distribution\index{posterior distribution} is
\[
\pi(\theta) \propto w \delta_0(\theta)\ell(\theta)+ (1-w)\nu(\theta)\ell(\theta),
\]
where $w$ is the prior probability of excluding the covariate. Whilst we cannot use a standard PDMP to sample from this target, we can approximate the prior by replacing the point mass at zero with a distribution concentrated around zero. If we use a Gaussian density with variance $\sigma^2$ for $\sigma\approx 0$, then we have the target
\begin{equation} \label{eq:ch6-spikeslab}
\pi_{\sigma}(\theta) \propto w \mathsf{N}(\theta;0,\sigma^2)\ell(\theta)+ (1-w)\nu(\theta)\ell(\theta),
\end{equation}
where $\mathsf{N}(\theta;0,\sigma^2)$ denotes the density of a Gaussian with mean 0 and variance $\sigma^2$. We can now use a PDMP to sample from this posterior distribution.  Also by letting $\sigma\rightarrow0$ we have that this posterior will tend, in some sense, to the posterior with the point mass at 0. So a natural approach is to consider the dynamics of the PDMP targeting $\pi_{\sigma}$ and whether we get well-defined dynamics in the limit of $\sigma\rightarrow0$.

\begin{figure}
    \centering
    \includegraphics[width=\textwidth]{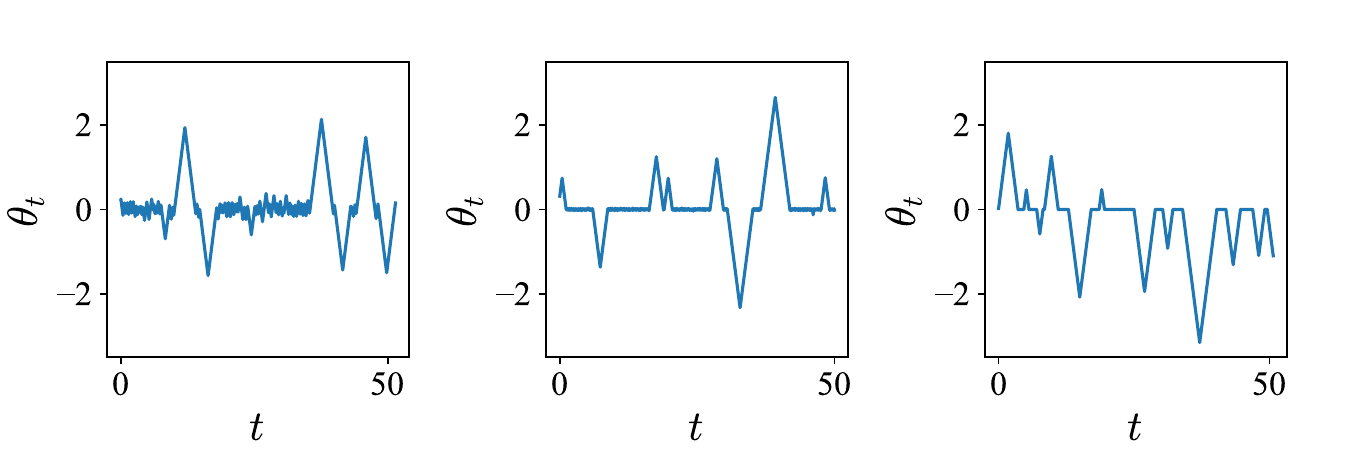}
    \caption{Example output from PDMP sampler for a target that is an equal mixture of a $\mathsf{N}(0,1)$ distribution and a $N(0,\sigma^2)$ distribution for $\sigma=0.1$ (left), $\sigma=0.01$ (middle) and $\sigma=0.001$ (right). As $\sigma\rightarrow0$ the trajectories qualitatively converge to periods at, or close to 0, and periods where the trajectory is governed by the  $\mathsf{N}(0,1)$ distribution. This limit is the form for the reversible jump PDMP sampler. The computational advantage of using the limiting dynamics is that it avoids simulating the larger number of events close to 0 -- which was of the order of 10000 for the right-hand plot.}
    \label{fig:ch6-rj1d}
\end{figure}

Figure \ref{fig:ch6-rj1d} shows the dynamics for three different values of $\sigma$. We see that in each case the sampler will spend periods of time in the neighbourhood of $0$. As we reduce $\sigma$, these neighbourhoods concentrate closer to 0 but the time the sampler spends in them seems to be similarly distributed. Away from the neighbourhood, the dynamics are those of sampling from a density proportional to $\nu(\theta)\ell(\theta)$. This suggests the limiting behaviour would be that of a PDMP targeting $\nu(\theta)\ell(\theta)$ but that if $\theta_t$ hits zero then the sampler stays at zero for a random amount of time. 

\cite{chevallier2022reversible} propose such a PDMP sampler and show that if we specify the dynamics as given below it will target the correct distribution across models. To describe the sampler for a $d$-dimensional parameter $\btheta$, assume that 
\[
\pi(\btheta) = \sum_{k=1}^{2^d} \pi_{k}(\btheta),
\]
where we assume that each model $k$ corresponds to a different set of components of $\btheta$ being set to zero. Here, $\pi_k(\btheta)$ is, up to proportionality, the density of $\btheta$ associated with the $k$th model. These densities must be defined up to a common constant of proportionality across all models.

For model $k$, let $\mathcal{S}_k$ be the active set of model $k$, that is the set of indices of components of $\btheta$ that are non-zero for model $k$. The idea of the PDMP is that if we are currently exploring model $k$ we will simulate a PDMP that targets $\pi_k(\btheta)$, and that has non-zero velocity only for components of $\btheta$ that are in $\mathcal{S}_k$. But in addition, we allow two further events. The first is that if a component of $\btheta$ hits zero, and we denote this component by $i$, then we move to the model with active set $\mathcal{S}_k \backslash \{i\}$. The second is that for each $j\notin \mathcal{S}_k$ we have a rate of moving to the model with active set $\mathcal{S}_k \cup \{j\}$. If we move to such a model we do not change $\btheta$, but simulate a new velocity $\bp$, which will have a non-zero velocity for component $j$.

To define such a sampler, we just need to specify the rate of adding a component to the model, and the distribution of the velocity after the transition. We will describe these for the Zig--Zag Sampler\index{Zig-Zag sampler} and the Bouncy Particle Sampler. For general results, and generalisations of this sampler which does not always move between models when a component of $\btheta$ hits zero, see \cite{chevallier2022reversible}.

For the Zig--Zag Sampler, assuming we are in model $k$ with current state $(\btheta,\bp)$, we have the following process for adding a variable to the model.
\begin{itemize}
\item[(Add--ZZ)] For each $j\notin \mathcal{S}_k$, move to model $k'$ with active set $\mathcal{S}_k\cup\{j\}$ at rate $\pi_{k'}(\btheta)/\pi_k(\btheta)$.
Set the new velocity, $\bp'$, such that $p'_{i}=p_{i}$ for $i\neq j$ and $p'_{j}$ is drawn uniformly at random from $\{-1,1\}$.
\end{itemize}
For the Bouncy Particle Sampler\index{bouncy particle sampler}, with standard Gaussian distribution for the velocity, again assuming we are in model $k$, with current state $(\btheta,\bp)$, we have the following 
\begin{itemize}
\item[(Add--BPS)] For each $j\notin \mathcal{S}_k$, move to model $k'$ with active set $\mathcal{S}_k\cup\{j\}$ at rate 
\[
\frac{2}{\sqrt{2\pi}}\frac{\pi_{k'}(\btheta)}{\pi_k(\btheta)}.
\]
Set the new velocity, $\bp'$, such that $p'_{i}=p_{i}$ for $i\neq j$ and $p'_{j}=x$ is simulated from a distribution with density function proportional to
\[
|x|\exp\left\{-\frac{1}{2}  x^2\right\},
\]
this is a standard normal density function scaled by $|x|$.
\end{itemize}
 The intuition for the density of the new velocity component for the Bouncy Particle Sampler is that it is skewed, relative to its invariant distribution\index{invariant distribution}, to larger absolute values of the velocity as these correspond to velocities that would hit zero more quickly. 

Importantly for both samplers, the rates at which we add components will often be simple. If our target distribution is defined as a posterior distribution\index{posterior distribution}, with common likelihood for each model, then the likelihood components of the posteriors will cancel and the rates will just depend on the ratio of priors. For many priors, the distribution of each component of $\btheta$ will be independent, in which case these rates become constant.

For the Zig--Zag Sampler\index{Zig-Zag sampler}, one can improve on this sampler by remembering the velocity of each inactive component prior to it becoming inactive. Then, when that component is re-introduced to the model, we re-use the same velocity. This is the Sticky Zig--Zag Sampler of \cite{bierkens2023sticky}. It can mix better as it ensures that the dynamics of each component of $\btheta$ reflects less often.

\paragraph{\em{Example: Logistic Regression with Model Choice}}\index{logistic regression}

An extension to the logistic regression model of Section \ref{sec:ch1-logistic} is to include a choice as to which covariates to include in the model. We will consider two example priors and calculate the rate of adding a covariate to the model for the Zig--Zag Sampler in each case.

A common prior would be to assume independence across covariates, so $\theta_i=0$ with probability $w_i$ and is drawn from a normal distribution with mean 0 and variance $\sigma^2_i$ with probability $1-w_i$. In this case, because of the independence, if covariate $i$ is not in the current model, the rate at which we add it will not depend on the value of $\theta_l$ for $l\neq i$. Thus, this rate will be constant and equal to 
\[
\frac{(1-w_i)}{w_i} \frac{1}{\sqrt{2\pi\sigma^2_i}},
\]
the ratio of the prior probability of a model, which includes covariate $i$ to the prior probability of the same model without covariate $i$, times the prior probability density of $\theta_i=0$ under the former model.

What about when we have a prior under which components of $\btheta$ are dependent? Assume we are currently in model $k$ which does not include the $i$th covariate, and that adding this covariate will produce model $k'$. Let $q_k$ and $q_{k'}$ be the prior probability of the two models and assume that they both have Gaussian priors on $\btheta$ with mean 0 and covariance matrices on the active components of $\btheta$ denoted by $\bSigma_k$ and $\bSigma_{k'}$ respectively. Let $a_k$ be the number of active components of model $k$, with $a_{k'}=a_k+1$. We can introduce matrices $\bA_k$ and $\bA_{k'}$ so that the prior for the active components of $\btheta$ under our prior for model $k$ is
\[
\left(\frac{1}{2\pi}\right)^{a_k/2} \mbox{det}(\bSigma_k)^{-1/2} \exp\left\{ -\frac{1}{2} \btheta^{\top}\bA_k\btheta\right\}.
\]
This is possible by padding $\bA$ with zeroes, so if covariate $l$ is not in the model then $A_{lj}=A_{jl}=0$, for $j=1,\ldots,d$. 

With this definition of the prior, the rate of moving from model $k$ to $k'$ as a function of $\btheta$ becomes
\[
\frac{q_{k'}}{q_k} \left(\frac{1}{2\pi}\right)^{1/2} \left(\frac{\mbox{det}(\bSigma_{k})}{\mbox{det}(\bSigma_{k'})}\right)^{1/2}
\exp\left\{ -\frac{1}{2}  
\btheta^\top(\bA_{k'}-\bA_k)\btheta
\right\}.
\]
Define $C$ to be the constant before the exponential term. If our current state is $(\btheta,\bp)$ then the rate until we move to model $k'$ is
\begin{eqnarray*}
\lefteqn{C\exp\left\{ -\frac{1}{2}  
(\btheta^\top+t\bp)(\bA_{k'}-\bA_k)(\btheta+t\bp)
\right\} = } \\ & & C\exp\left\{ -\frac{1}{2}  
\btheta^\top(\bA_{k'}-\bA_k)\btheta
\right\}\exp\left\{ -  \btheta^\top(\bA_{k'}-\bA_k)\bp t -\frac{1}{2}  \bp^\top(\bA_{k'}-\bA_k)\bp t^2
\right\}.
\end{eqnarray*}
This is of the form $a\exp\{bt+ct^2\}$ for some constants $a$, $b$ and $c$. We can simulate the time of the next event with this rate if $c\leq0$ as the integral of the rate is analytic for $c=0$, and can be expressed in terms of probabilities of a normal distribution for $c<0$. For $c>0$, we can use Poisson thinning \index{Poisson thinning}with, for example, bounds of the form $A\exp\{Bt\}$ over suitable intervals for $t$. 

\index{piecewise deterministic Markov processes!reversible jump|)}

\subsection{More General Velocity Models}

Another possible way of extending PDMP samplers is to consider more general models for the dynamics. There are two simple ways of doing this, the first is to alter the distribution of the velocities for the Bouncy Particle Sampler or the Zig--Zag Sampler\index{Zig-Zag sampler} so that they are not spherically symmetric. The other is to consider non-constant velocity models. We will briefly describe each of these in turn.

First, we will focus on the Zig--Zag Sampler. If we let $\be_i$ be the $i$th unit vector, that is, the vector whose $i$th component is 1 and all other components are 0, then the set of velocities of the Zig--Zig are the velocities of the form $\sum_{i=1}^d a_i \be_i$, where $a_i\in\{-1,1\}$ for $i=1,\ldots,d$. That is, they are the set of velocities that one obtains by adding plus or minus each unit vector. The rate of flipping $a_i$ is equal to the maximum of 0 and minus the dot product of $a_i\be_i$ with $\nabla \log \pi(\btheta)$.

To generalise this we just need to replace the unit vectors with another set of vectors that span $\mathbb{R}^d$. Denote this set by $\bp_1,\ldots,\bp_d$.  Importantly for Zig--Zag\index{Zig-Zag sampler}, the change to the process is trivial -- as the event rates are of a similar form but with $\be_1,\ldots,\be_d$ replaced with $\bp_1,\ldots,\bp_d$. There are two natural approaches to choosing $\bp_1,\ldots,\bp_d$. One is to just change the speed in each direction, so $\bp_i=c_i\be_i$ for some set of positive scalars $c_1,\ldots,c_d$. This can be helpful if different components of $\btheta$ under $\pi$ are on different scales. A natural choice is to set $c_i$ to be an estimate of the marginal standard deviation of $\btheta_i$ under $\pi$. The other approach is to also change the directions of the velocities as well. If we have an estimate of the variance-matrix of $\btheta$ under $\pi$, say $\bSigma$, then one choice is to choose the $\bp_i$s to be the eigenvectors of $\bSigma$. 

To see why this is a natural choice, consider $\pi(\btheta)$ being a Gaussian distribution with variance $\bSigma$. Centre this distribution so it has a mean of 0. If we use $\bp_1,\ldots,\bp_d$ as our basis for the velocities, and $\bp=\sum_{i=1}^d a_i\bp_i$, then the rate at which we flip $a_i$ is equal to
\[
\max\{0, -a_i\bp_i^\top\nabla \log \pi(\btheta+\bp t)\} =
\max\left\{0, -a_i\bp_i^\top\left(-\bSigma^{-1}(\btheta+\bp t)\right) \right\}.
\]
But using that $\bp_i$ is an eigenvector of $\bSigma$, and hence also of $\bSigma^{-1}$, and if we assume the corresponding eigenvalue of $\bSigma$ is $\gamma_i$, we have that
\[
-a_i\bp_i^\top\left(-\bSigma^{-1}(\btheta+\bp t)\right) =
a_i \frac{1}{\gamma_i} \bp_i^\top\left(\btheta+t\sum_{j=1}^d a_j \bp_j\right) t =
a_i \frac{1}{\gamma_i} \left(\bp_i^\top\btheta + a_i t\right).
\]
In this case, the event rate does not depend on the velocity in other components, and essentially Zig--Zag\index{Zig-Zag sampler} will reduce to independent processes along each component, $\bp_i^\top\btheta$, of $\btheta$. 

A similar idea can be used to generalise the Bouncy Particle Sampler\index{bouncy particle sampler}. It is simplest to describe this for the case where the invariant distribution for $\bp$ is Gaussian, as the generalisation is to allow a non-identity covariance matrix for this invariant distribution. In the following, we will assume the invariant distribution\index{invariant distribution} is Gaussian with mean 0 and variance $\bSigma$.

As we change $\bSigma$, we have to change the reflection\index{reflection} events of the sampler. To see why, note that a key property of the reflection\index{reflection} event of the standard Bouncy Particle Sampler, wherein step (BPS2)
\[
\bp'=\cR_{\bg}(\bp),\mbox{ with }\bg=\nabla_{\btheta}\log \pi(\btheta),
\]
was that $||\bp'||_2^2=||\bp||_2^2$, so this transition does not change the density of the state under $\pi_{\bp}$. This is no longer the case if the variance of $\bp$ under $\pi_{\bp}$ is not a multiple of the identity.

So we need to generalise the reflection so that it depends on $\bSigma$. It turns out that the appropriate reflection is
\[
\cR^{\bSigma}_{\bg}(\bp) = \bp - \frac{2\bg^{\top}\bp}{(\bg^{\top}\bSigma\bg)} \bSigma\bg.
\]
Importantly, if $\bp'=\cR^{\bSigma}_{\bg}(\bp)$ for any $\bg$, then
\[
\bp'^{\top}\bSigma^{-1}\bp'=\bp^{\top}\bSigma^{-1}\bp,
\]
so it does not change the density under a Gaussian with variance $\bSigma$. Furthermore, we still have that if $\bg=\nabla \log \pi(\btheta)$ then
\[
\bp' \cdot \nabla \log \pi(\btheta) = -\bp \cdot \nabla \log \pi(\btheta),
\]
which is the other key requirement of the transition needed for the validity of the sampler. Using these two properties, it is straightforward to show that the Bouncy Particle Sampler\index{bouncy particle sampler} with (BPS2) replaced by (BPS2') below will have $\pi(\btheta)\pi_{\bp}(\bp)$ as its invariant distribution\index{invariant distribution}, where $\pi_{\bp}(\bp)$ is the density of a Gaussian distribution with mean 0 and variance $\bSigma$.
\begin{itemize}
 \item[(BPS2')] {\em Transition at events.} At an event with probability $1-\refr/\lambda_{\text{BPS}}(\btheta,\bp)$, reflect the velocity 
\[
\bp'=\cR^{\bSigma}_{\bg}(\bp),\mbox{ with }\bg=\nabla_{\btheta}\log \pi(\btheta);
\]
otherwise sample a new velocity, $\bp'$ from a  normal distribution with mean 0 and variance $\bSigma$. The position is unchanged at an event.   
\end{itemize}

As above, a natural choice of $\bSigma$ to use in the distribution for the velocity is to choose it to be an estimate of the variance of $\btheta$ under $\pi$. Furthermore, for both the Bouncy Particle Sampler and Zig--Zag, one can relate the choice of distribution on the velocity to running the canonical version of the PDMP but after applying a linear reparameterisation to the random variable of the distribution we wish to sample from. We will describe the link for the Bouncy Particle Sampler, though a similar argument applies to other PDMP samplers.

Consider the Bouncy Particle Sampler\index{bouncy particle sampler} for $(\btheta,\bp)$ with target distribution $\pi(\btheta)$ and a standard Gaussian distribution for $\bp$. For some invertible matrix $\bL$ define $\bpsi=\bL\btheta$, and consider the dynamics of the PDMP but viewed in terms of $\bpsi$. If $\btheta$ is drawn from $\pi(\btheta)$, and $\bpsi=\bL\btheta$, then the density of $\bpsi$ is $\pi_{\bpsi}(\bpsi)\propto \pi(\bL^{-1}\bpsi)$, as the Jacobian\index{Jacobian} of the transformation is constant. If we consider derivatives then
\[
\frac{\partial \log \pi_{\bpsi}(\bpsi)}{\partial \psi_{i}}
=
\frac{\partial \log \pi(\bL^{-1}\bpsi)}{\partial \psi_{i}}
=
\sum_{j=1}^d \frac{\partial \log \pi(\bL^{-1}\bpsi)}{\partial \theta_{j}} \left(\bL^{-1}\right)_{ji}.
\]
This is just the $i$th entry of $\bL^{-\top}\nabla_{\btheta} \log \pi(\bL^{-1}\bpsi) $, which gives that 
\begin{equation} \label{eq:ch6-nablapsi}
\nabla_{\bpsi}\pi_{\bpsi}(\bpsi) =  \bL^{-\top}\nabla_{\btheta} \log \pi(\bL^{-1}\bpsi).
\end{equation}

Now let us consider the dynamics of the Bouncy Particle Sampler\index{bouncy particle sampler} in $\bpsi$ space. We will consider each aspect of the dynamics in turn:\\
{\it{Deterministic Dynamics:}} If we transform the constant velocity dynamics\index{constant velocity dynamics} into $\bpsi$ space, we have
\[
\frac{\mbox{d} \bpsi}{\mbox{d}t}= \bL\frac{\mbox{d} \btheta}{\mbox{d}t}=\bL\bp,
\]
so these are still constant velocity dynamics but with velocity $\bw=\bL\bp$. Furthermore, if $\bp$ has a Gaussian distribution with an identity covariance matrix, then $\bw$ is Gaussian with covariance $\bL\bL^{\top}$.

{\it{Rate of Bounce Events:}} If the current state is $(\btheta,\bp)$ then the rate of a bounce event is $\max\{0,\bp\cdot \nabla\log\pi(\btheta)\}$. Now 
\[
\bp\cdot \nabla\log\pi(\btheta) =  \bp^{\top} \nabla\log\pi(\btheta) = \bw^{\top}(\bL^{-1})^{\top}\nabla\log\pi(\bL^{-1}\bpsi) 
=\bw^{\top}\nabla_{\bpsi}\pi_{\bpsi}(\bpsi),
\]
where we have transformed $(\btheta,\bp)$ to $(\bpsi,\bw)$ and used (\ref{eq:ch6-nablapsi}). This is the rate for the Bouncy Particle Sampler targeting $\pi_{\bpsi}(\bpsi) $.

{\it{Reflection at Bounce Events:}}\index{reflection} If the current state is $(\btheta,\bp)$ then at a bounce event the new velocity is
\[
\bp'=\bp - 2(\bp\cdot\nabla_{\btheta} \log\pi(\btheta)) \frac{\nabla_{\btheta} \log\pi(\btheta)}{(\nabla_{\btheta} \log\pi(\btheta)^{\top}\nabla_{\btheta} \log\pi(\btheta))^{1/2}}.
\]
So if we consider the velocity for the $\bpsi$ process, $\bw'=\bL\bp'$, and use $\bSigma=\bL\bL^{\top}$, this is
\begin{eqnarray*}
\bw'&=&\bL\bp - 2(\bp\cdot\nabla_{\btheta} \log\pi(\btheta)) \frac{\bL \nabla_{\btheta} \log\pi(\btheta)}{(\nabla_{\btheta} \log\pi(\btheta)^{\top}\nabla_{\btheta} \log\pi(\btheta))^{1/2}} \\
&=& \bw - 2 ( \bw^{\top}\bL^{-\top}\nabla_{\btheta} \log \pi(\bL^{-1}\bpsi )) \\ & & \times\frac{\bSigma\bL^{-\top}\nabla_{\btheta} \log\pi(\bL^{-1}\bpsi )}{(\nabla_{\btheta} \log\pi(\bL^{-1}\bpsi )^{\top}\bL^{-1}\bSigma\bL^{-\top}\nabla_{\btheta} \log\pi(\bL^{-1}\bpsi ))^{1/2}} \\
&=& \bw -2(\bw^{\top}\nabla_{\bpsi} \log  \pi_{\bpsi}(\bpsi)) \frac{\bSigma\nabla_{\bpsi} \log  \pi_{\bpsi}(\bpsi)}{(\nabla_{\bpsi} \log  \pi_{\bpsi}(\bpsi)^{\top}\bSigma\nabla_{\bpsi} \log  \pi_{\bpsi}(\bpsi))^{1/2}}.
\end{eqnarray*}
This is just $\cR^{\bSigma}_{\bg}(\bw)$ with $\bg=\nabla \log_{\bpsi}(\bpsi)$, the reflection \index{reflection}of the Bouncy Particle Sampler with covariance matrix $\bSigma$.

{\it{Refresh Events:}}\index{refresh event/rate} These events occur at a constant rate, which is unaffected by the transformation to $\bpsi$. At a refresh event, we simulate $\bp$ from a standard Gaussian, which corresponds to simulating $\bw=\bL\bp$ from a Gaussian with covariance $\bSigma=\bL\bL^{\top}$.

Thus the process in $\bpsi$ space is a Bouncy Particle Sampler\index{bouncy particle sampler} with covariance matrix $\bSigma=\bL\bL^{\top}$ for the velocity.

A second generalisation is to alter the constant velocity dynamics. This has been suggested in particular as a way of generalising the Bouncy Particle Sampler with covariance matrix $\bSigma$ for the velocity, with the resulting algorithm called the Boomerang Sampler\index{boomerang sampler} \cite[]{bierkens2020boomerang}, though similar ideas also appear under the name of  Hamiltonian-BPS in \cite{vanetti2017piecewise}. 

Consider a velocity model with marginal distribution such that $\log \pi_{\bp}(\bp)=-(1/2)\bp^{\top} \bSigma^{-1}\bp$. Write $\log \pi(\btheta)= U(\btheta)-(1/2)(\btheta-\btheta^*)^{\top}\bSigma^{-1}(\btheta-\btheta^*)$  for some function $U(\btheta)$ and constant vector $\btheta^*$. The idea is to have deterministic dynamics that move along contours of $-(1/2)(\btheta-\btheta^*)^{\top}\bSigma^{-1}(\btheta-\btheta^*)-(1/2)\bp^{\top} \bSigma^{-1}\bp$ in $(\btheta,\bp)$ space. Such dynamics are given by Hamiltonian dynamics, which are tractable in this case, and are \index{Hamiltonian Monte Carlo}
\[
\frac{\mbox{d} \btheta}{\mbox{d}t} = \bp, ~~~\frac{\mbox{d} \bp}{\mbox{d}t} = \btheta-\btheta^*.
\]
The solution of these dynamics is $\btheta_t=\btheta^*+(\btheta_0-\btheta^*)\cos(t)+\bp_0\sin(t)$ and $\bp_t=\bp_0\cos(t)-(\btheta_0-\btheta^*)\sin(t)$. The rate of bounce events for the Boomerang Sampler\index{boomerang sampler} is just $\max\{0,\bp_t\cdot\nabla U(\btheta_t)\}$, with bounces as per the Bouncy Particle Sampler when the velocity has covariance matrix $\bSigma$. As before, we can also introduce refresh events.\index{refresh event/rate}

If $U(\btheta)=0$, so we are targeting a Gaussian distribution for $\btheta$ with mean $\btheta^*$ and covariance $\bSigma$, then this sampler just undergoes Hamiltonian dynamics. In this case, a strictly positive refresh rate is needed to avoid the sampler being reducible\index{reducible}, and the resulting process is a form of randomised  Hamiltonian dynamics, that is HMC but with a random refresh\index{refresh event/rate} time for the velocity. For non-Gaussian targets, this sampler will have additional bounce events, but the hope is that if the target is close to Gaussian with mean $\btheta^*$ and covariance $\bSigma$, then the rate of bounce events will be much lower than for the standard Bouncy Particle Sampler\index{bouncy particle sampler}.

Care is needed with one aspect of simulating the Boomerang Sampler\index{boomerang sampler}, as the different dynamics require slightly different approaches to simulate the event times. If the current state is $(\btheta_0,\bp_0)$ then the rate until the next event is now
\begin{eqnarray*}
\lefteqn{\tilde{\lambda}_{(\btheta_0,\bp_0)}(t)=\max\{0,\bp_t \cdot \nabla U(\btheta_t)\} = \max\{0,}\\
& &
(\bp_0\cos(t)+(\btheta_0-\btheta^*)\sin(t)) \cdot \nabla U(\btheta^*+(\btheta_0-\btheta^*)\cos(t)+\bp_0\sin(t))\},
\end{eqnarray*}
where we have substituted in the definitions of $\bp_t$ and $\btheta_t$. \cite{bierkens2020boomerang} give some general approaches to bounding this rate, which uses the property that the deterministic dynamics of the Boomerang Sampler\index{boomerang sampler} are such that $|\btheta_t-\btheta^*|^2+|\bp_t|^2$ is a constant. To keep the notation simple, we will show this for $\btheta^*=\bzero$, but the same argument applies more generally. In this case 
\begin{eqnarray*}
|\btheta_t|^2+|\bp_t|^2 &=&  (\btheta_0\cos(t)+\bp_0\sin(t))\cdot (\btheta_0\cos(t)+\bp_0\sin(t)) \\
& & + 
(\bp_0\cos(t)-\btheta_0\sin(t))\cdot (\bp_0\cos(t)-\btheta_0\sin(t)) \\
&=& |\btheta_0|^2\cos^2(t)+|\bp_0|^2\sin^2(t)+2\btheta_0\cdot\bp_0\sin(t)\cos(t) \\
& & + |\bp_0|\cos^2(t) + |\btheta_0|^2\sin^(t)-2\btheta_0\cdot\bp_0\sin(t)\cos(t) \\
&=& |\btheta_0|^2(\sin^2(t)+\cos^2(t)+ |\bp_0|^2(\sin^2(t)+\cos^2(t)) = |\btheta_0|^2+|\bp_0|^2.
\end{eqnarray*}

How is this useful? This property means that we can bound the distance from $\btheta^*$ that the current deterministic trajectory can take. Thus if we can bound the Hessian\index{Hessian} of $U$, which is the derivative of $\nabla U$, then this enables us to bound $\nabla U$ for the current trajectory based on the value of $\nabla U$ at $\btheta^*$ plus a term that depends on the bound on the Hessian and the distance the trajectory can be from $\btheta^*$. One such bound that we will use below is that if the spectral norm of the Hessian of $U$ is bounded by $M$, so that $\|\nabla^2U(\btheta)\bx\|\leq M\|\bx\|_2$ for any vector $\bx$ with $\|\cdot\|$ denoting Euclidean distance,
then for the current deterministic trajectory, we have a constant bound:
\begin{equation} \label{eq:ch6-boomerangbound}
\lambda(\btheta_t,\bp_t)\leq |\nabla U(\btheta^*)|(|\btheta_0-\btheta^*|^2+|\bp_0|^2)^{1/2}   + \frac{1}{2} M (|\btheta_0-\btheta^*|^2+|\bp_0|^2).
\end{equation}

\paragraph{\em{Example: Boomerang for Logistic Regression}}\index{logistic regression}

As an example, consider again the logistic regression model. There are two natural choices for the centring value and covariance of the Boomerang Sampler\index{boomerang sampler}. The first is to set them to the prior mean and covariance, $\btheta^*=\bzero$ and $\bSigma=\bSigma_{\btheta}$. The second is to estimate the mode of $\log \pi$, $\widehat{\btheta}_{MAP}$ say, and the inverse of the Hessian\index{Hessian} of $-\log \pi$ at $\widehat{\btheta}_{MAP}$. We will compare these two options.

If we set them to the prior values, then $U(\btheta)$ is minus the log-likelihood. As described in Section \ref{sec:ch6-simPDMPsamplers} we can bound the Hessian of minus the log-likelihood by $(1/4)\bX^{\top}\bX$, where $\bX$ is the $N\times d$ matrix of covariates. 

What about if we set $\btheta^*$ and $\bSigma$ based on the estimate of the mode of $\log \pi$ and the inverse of the Hessian of $-\log \pi$ at the mode? Denoting the log-likelihood of the logistic model by $\ell(\btheta; \mathcal{D})$, and the Hessian\index{Hessian} of minus the log-likelihood by $\bH(\btheta)$. This choice gives
\[
U(\btheta)= -\ell(\btheta; \mathcal{D}) - \frac{1}{2}(\btheta-\widehat{\btheta}_{MAP})^{\top} \bH(\widehat{\btheta}_{MAP}) (\btheta-\widehat{\btheta}_{MAP}),
\]
where we have used the fact that the contribution from the prior will cancel. Taking second derivatives, the Hessian of this at $\btheta$ will be the difference between two matrices, $\bH(\btheta)-\bH(\widehat{\btheta}_{MAP})$. These matrices are both positive semi-definite, with spectral norm bounded by $(1/4)\bX^{\top}\bX$, thus the spectral norm of the difference is also bounced by $(1/4)\bX^{\top}\bX$. This is because the eigenvalues of $\bH(\btheta)$ are bounded between $0$ and $M$ for some constant $M$, and the eigenvalues of $-\bH(\widehat{\btheta}_{MAP})$ are bounded between $-M$ and $0$, so the eigenvalues of $\bH(\btheta)-\bH(\widehat{\btheta}_{MAP})$ are between $-M$ and $M$.

\begin{figure}
    \centering
    \includegraphics[width=\textwidth]{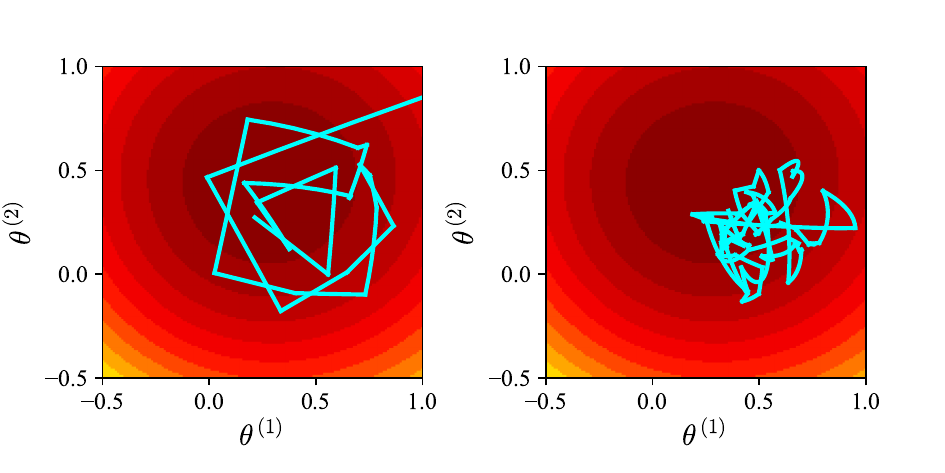}
    \caption{Trace plots of the Boomerang Sampler for the posterior of a logistic regression model. The heat plot shows the contours of the posterior. Example realisation when $\bSigma=\bSigma_{\btheta}$ and $\btheta^*$ is the prior mean (left) and when $\btheta^*$ is an estimate of the posterior mode and $\bSigma$ is based on the Hessian of $\log \pi$ at $\btheta^*$ (right).}
    \label{fig:ch6_Boomerang}
\end{figure}
Thus we can implement the Boomerang Sampler\index{boomerang sampler} for both choices of $\btheta^*$ and $\bSigma$ with the constant bound given by (\ref{eq:ch6-boomerangbound}) using the same value for $M$. The bounds will differ though due to the different values for $\btheta^*$ and $U$ and hence for $|\nabla U(\btheta^*)|$. In particular, this will be 0  for $\btheta^*=\widehat{\btheta}_{MAP}$, or at least close to 0 if we have a reasonable approximation for the mode of $\log\pi$. Example realisations for the two samplers are shown in Figure \ref{fig:ch6_Boomerang}, for data with $N=100$ and $d=2$ and with a prior covariance of 2 for each component of $\btheta$.

The main difference between the Boomerang Sampler\index{boomerang sampler} and the previous PDMP samplers is most obviously seen in the right-hand plot of Figure \ref{fig:ch6_Boomerang}, as we have elliptical trajectories between events. This is reminiscent of the trajectories for HMC. For the left-hand plot, where $\Sigma$ is large compared to the curvature of the posterior, the contours are elliptical but with larger radii and, thus the output looks more similar to our previous PDMP samplers. In this example, one of the main advantages of basing $\btheta^*$ and $\Sigma$ on the mode and Hessian at the mode is that the computational cost of simulating the Boomerang Sampler\index{boomerang sampler} is lower. Both have been simulated with roughly the same length of trajectory, but using the prior values has required proposing four times as many events. This is because of the looser bound on the event rate that we have in this case.

\index{piecewise deterministic Markov processes|)}

\section{Chapter Notes}

The initial idea of using PDMP processes for sampling comes from the physical sciences, see for example \cite{turitsyn2011irreversible}, \cite{peters2012rejection} and \cite{michel2014generalized}. These were introduced into statistics by \cite{bouchard2018bouncy} and \cite{bierkens2017piecewise}, and one of the early papers to describe the link to PDMPs was \cite{Fearnhead:2018}. The latter paper also gives an example where avoiding refresh events\index{refresh event/rate} in the Bouncy Particle Sampler\index{bouncy particle sampler} can lead to slow mixing\index{mixing}. How the subsampling ideas of Section \ref{ch5:sec-SS} extend to samplers other than the Zig--Zag Sampler is also discussed in \cite{Fearnhead:2018}, with related ideas for the local Bouncy Particle Sampler in \cite{bouchard2018bouncy}.

As well as the PDMP samplers mentioned in the chapter, there have been various papers suggesting extensions to PDMP methods. For example \cite{vanetti2017piecewise}, \cite{wu2017generalized}, \cite{michel2020forward}. Examples of adaptive methods for choosing the velocity process is given in \cite{bertazzi2022adaptive} and there are methods that allow the speed to depend on position \cite[]{vasdekis2023speed}. The continuous-time methods can be related to discrete-time MCMC\index{Markov chain Monte Carlo} methods such as reflective slice sampling \cite[]{neal2003slice} and the Discrete Bouncy Particle Sampler\index{discrete bouncy particle sampler} \cite[]{sherlock2022discrete}. 

The theoretical analysis of PDMP samplers has been active, including showing ergodicity  \cite[]{deligiannidis2019exponential, bierkens2019ergodicity}, obtaining rates of convergence \cite[]{andrieu2021subgeometric,bierkens2019ergodicity} and exploring limiting behaviour of the Bouncy Particle Sampler and the Zig Zag sampler \cite[]{deligiannidis2021randomized,bierkens2022high,andrieu2021hypocoercivity}. Particularly strong results are available for the one-dimensional case \cite[]{bierkens2017limit,bierkens2022spectral,vasdekis2022note}. 

Complementary results on scaling\index{scaling} of the Discrete Bouncy Particle Sampler\index{discrete bouncy particle sampler} to those shown in Section \ref{sec:ch6-comparison} are given in \cite{sherlock2022discrete}, which shows similar scaling to the Bouncy Particle Sampler and also provides supporting theory to help choose the refresh rate.\index{refresh event/rate}


\chapter{Assessing and Improving MCMC}
\label{chap:stein}

The development of more sophisticated and, especially, approximate sampling algorithms, aimed at improving scalability in one or more of the senses already discussed in this book, raises important considerations about how a suitable algorithm should be selected for a given task, how its tuning parameters should be determined, and how its convergence should be assessed. This chapter presents recent solutions to the above problems, whose starting point is to derive explicit upper bounds on an appropriate distance between the posterior and the approximation produced by MCMC. 
Further, we explain how these same tools can be adapted to provide powerful post-processing methods that can be used retrospectively to improve approximations produced using scalable MCMC.

\section{Diagnostics for MCMC}
\label{sec: diagnostics}

The approximations provided by MCMC are only useful if we can be confident that the samples collectively form a reasonable approximation to the intended target.
This, however, appears to be a circular requirement -- how can we verify that MCMC has worked without access to the limiting target to check?
Several \emph{diagnostics} have emerged as pragmatic solutions, enabling a practitioner to detect certain failure modes of MCMC.
In particular, we highlight \emph{convergence diagnostics}\index{convergence diagnostic}, which aim to determine whether the Markov chain has converged to \emph{some} stationary distribution, and \emph{bias diagnostics}, which aim to detect whether the stationary distribution of the Markov chain is indeed the target distribution of interest.
For context, both classes of diagnostic will be briefly discussed.
Throughout this chapter, we shall restrict attention to distributions defined on $\mathbb{R}^d$ for simplicity of presentation.

\subsection{Convergence Diagnostics}

To limit the scope, here we describe the convergence diagnostics \index{convergence diagnostic} that are most widely used. 
The \emph{Gelman--Rubin} diagnostic \index{Gelman--Rubin diagnostic} is based on realisations of $L$ independent Markov chains, each of length $n$, where practical considerations typically restrict $L$ to be a small number, such as 3, 4 or 5.
For a univariate target distribution, the Gelman--Rubin diagnostic is defined as the square root of the ratio of two estimators of the variance $\sigma^2$ of the target distribution
\begin{align}
	\widehat{R} := \sqrt{\frac{\widehat{\sigma}^2}{\widehat{s}^2}}, \label{eq: Rhat}
\end{align}
 is the (arithmetic) mean of the sample variances $s^2_l$ along the $L$ sample paths, 
\begin{align*}
	\widehat{s}^2 := \frac{1}{L} \sum_{l=1}^L s_l^2 ,
\end{align*}
which typically provides an underestimate of $\sigma^2$, since it is possible that one or more of the $L$ chains has not explored the posterior well, while $\widehat{\sigma}^2$ is constructed as 
\begin{align}
	\widehat{\sigma}^2 := \frac{n-1}{n} \widehat{s}^2 + \frac{1}{L-1} \sum_{l=1}^L \left( m_{l} - \frac{1}{L} \sum_{l'=1}^L m_{l'} \right)^2 ,    \label{eq: sigmahar R}
\end{align}
where $m_l$ is the sample mean from the $l$th sample path, which typically provides an overestimate of the target variance.
Indeed, the second term in \eqref{eq: sigmahar R} is an estimate of the asymptotic variance of the sample mean of the Markov chain, which is typically larger than the variance $\frac{\sigma^2}{n}$ we would obtain if our samples were truly independent; recall the discussion of effective sample sizes and \eqref{eqn.VarIhat.IACT} from Chapter \ref{chap:background}.
For an ergodic Markov chain, $\widehat{R}$ converges to 1 as $n \rightarrow\infty$.
In practice, it is common to discard a burn-in \index{burn-in} period of length $\frac{n}{2}$, where $n$ is the smallest integer for which $\widehat{R} < 1 + \delta$, where $\delta$ is a suitable threshold. 
The somewhat arbitrary choices of $\delta = 0.1$ and $\delta = 0.01$ are often used.

Convergence diagnostics \index{convergence diagnostic} are widely and successfully used.
However, it remains the case that the performance of $\widehat{R}$ and related convergence diagnostics depends strongly on how the independent realisations of MCMC are initialised.
Indeed, consider the task of generating  approximate samples from the mixture distribution
\begin{align}
    \pi(x) = \frac{1}{2} \mathsf{N}(x;-2,0.5^2) + \frac{1}{2} \mathsf{N}(x;2,0.5^2) . \label{eq: mixture example}
\end{align}
To avoid the situation where all chains are confined to the same local high-probability region due to chance, standard practice is to initialise the Markov chains by sampling their initial state from a distribution that is over-dispersed with respect to the target.
Thus, we may initialise Markov chains by sampling from, say, $\mathsf{N}(0,5^2)$.
Running $L = 3$ chains of length $n = 1000$ leads to the two sets of sample paths shown in Figure \ref{fig: convergence diagnostics}.
In both sets of sample paths\index{sample path}, the length $n$ of the sample paths was insufficient to enable the Markov chains to explore both components of $\pi$, and each of the chains remained in the component in which it was initialised.
On the left side of the figure, one of the sample paths is clearly qualitatively distinct from the other two, since the Markov chains explored different components of $\pi$, and the Gelman--Rubin diagnostic \index{Gelman--Rubin diagnostic} correctly detects that the Markov chains have not converged.
Unfortunately, on the right side of the figure, it so happened that each of the chains was initialised in the high probability region of the same component.
As a result, the Gelman--Rubin diagnostic \index{Gelman--Rubin diagnostic} appears to be converging to values below the commonly used thresholds $\delta = 0.1$ and $\delta = 0.01$, and fails to detect that the Markov chains have explored only one of the components of $\pi$.

What went wrong with the convergence diagnostic \index{convergence diagnostic} \eqref{eq: Rhat} in this example?
Well, the Markov processes exhibited a form of \emph{quasi-stationarity}; transitions from one component of the mixture to the other is a rare event, and conditional on such a transition not occurring the behaviour of the Markov chains is arguably excellent.
The rarity of transitions between components makes it fundamentally difficult to distinguish between quasi-stationarity and convergence of a Markov chain when the sample paths are confined to the same component; some knowledge of the invariant distribution $\pi$ is required.
This motivates the discussion of an alternative diagnostic which does indeed leverage information about $\pi$, a \emph{bias diagnostic}, which we describe next.

\begin{figure}
    \centering
    \includegraphics[width = \textwidth]{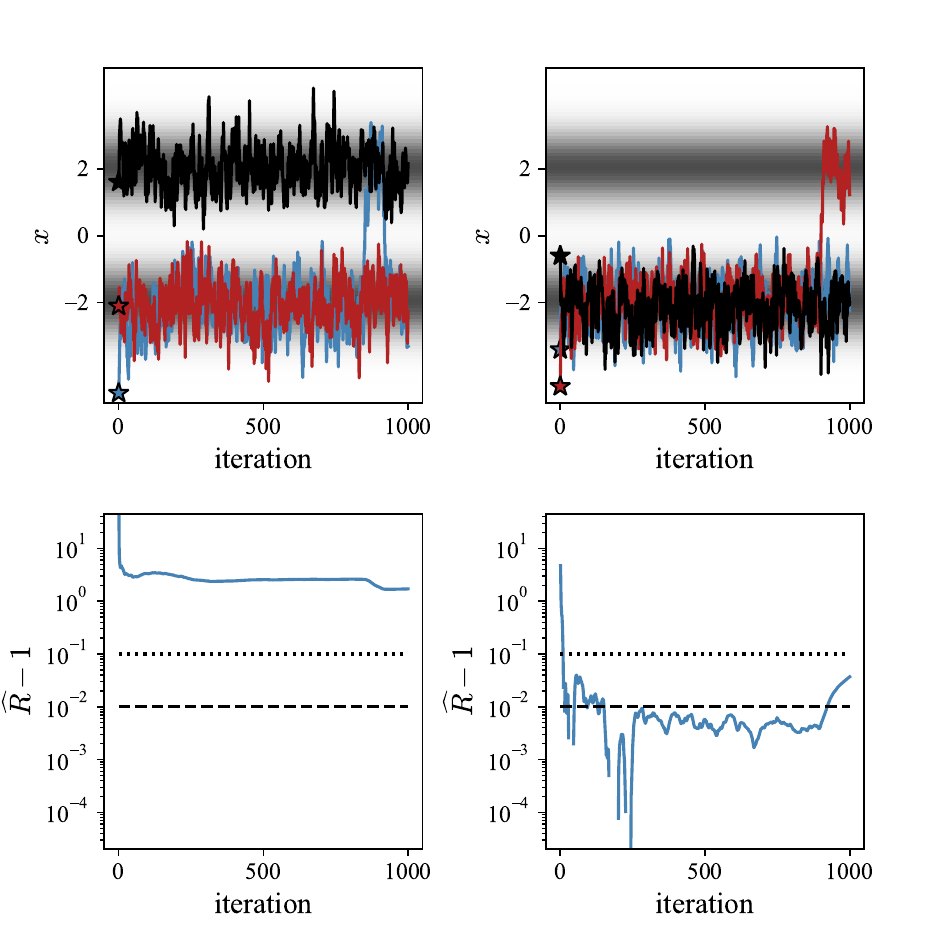}
    \caption{Convergence diagnostics for MCMC. 
 Three independent Markov chains were simulated to generate samples from the Gaussian mixture target $\pi$ in \eqref{eq: mixture example}. 
 In the first scenario (left panels) the chains explore different components of $\pi$, and the Gelman--Rubin diagnostic \index{Gelman--Rubin diagnostic} $\widehat{R}$ correctly detects that the Markov chains have not converged.
 In the second scenario (right panels) the chains explore the same component of $\pi$, and the Gelman--Rubin diagnostic does not detect that the Markov chains have not converged.
 [Stars indicate the initial state of each Markov chain.  
 The density $\pi$ is shaded.
 Dotted lines indicate the two commonly used thresholds $\delta = 0.1$ and $0.01$ for $\widehat{R}-1$.]
 }
    \label{fig: convergence diagnostics}
\end{figure}

\subsection{Bias Diagnostics} \label{subsec: bias diagn}

Even in favourable situations, such as the case of a uni-modal target, convergence diagnostics do not provide a guarantee that Markov chain samples constitute a faithful approximation of the target.
Indeed, convergence diagnostics are not capable of detecting \emph{bias} in sampler output, for example, as introduced when using stochastic gradient \index{stochastic gradient} methods (Chapter \ref{chap:sgld}), or introduced when a coding error has occurred.
Instead, \emph{bias diagnostics} \index{bias diagnostic} can be used to identify such situations, a simple example of which is
\begin{align}
\widehat{B} := \left\| \frac{1}{n} \sum_{k=1}^n (\nabla \log \pi)(X_k) \right\| . \label{eq: bias diagnostic}
\end{align}
Provided that $\nabla \log \pi \in \mathcal{L}^1(\pi)$, which we recall means that $\int \|\nabla \log \pi\| \; \mathrm{d}\pi < \infty$ from Section \ref{subsec: what is MC}, from the strong law of large numbers \index{strong law of large numbers} for Markov chains 
the series in \eqref{eq: bias diagnostic} almost surely converges to the limit
\begin{align*}
\int \nabla \log \pi \; \mathrm{d}\pi = \int \frac{(\nabla \pi)(x)}{\pi(x)} \pi(x) \; \mathrm{d}x = \int (\nabla \pi)(x) \; \mathrm{d}x = 0
\end{align*}
whenever the Markov chain is ergodic and $\pi$-invariant.
The final equality is integration by parts; a special case of Lemma \ref{lem: Stein op int by parts} in the sequel.
On the other hand, just as the passing of a convergence diagnostic \index{convergence diagnostic} test does not guarantee that the MCMC has converged, the convergence of \eqref{eq: bias diagnostic} to 0 does not guarantee that the Markov chain preserves the correct target distribution.
Surprisingly, bias diagnostics are not widely used, at least compared to convergence diagnostics, which may be due to (in our example) the requirement to compute a gradient, or may simply be because they have been historically overlooked.

Consider again the mixture distribution $\pi$ in \eqref{eq: mixture example} and suppose that, due to a coding error, we have implemented a Markov chain whose stationary distribution is $\mathsf{N}(\mu,0.5^2)$.
Running $L=3$ chains of length $n = 1000$ leads to the two sets of sample paths \index{sample path} shown in Figure \ref{fig: bias diagnostics} for $\mu = 2$ (a) and $\mu = 0$
In both sets of sample paths, the Gelman--Rubin convergence diagnostic \index{Gelman--Rubin diagnostic} test is passed, despite the Markov chains failing to be $\pi$-invariant.
On the left side of the figure, the bias diagnostic \eqref{eq: bias diagnostic} clearly does not converge to zero, and thus, the bias in the Markov chain output is detected.
Unfortunately, on the right side of the figure, the bias diagnostic appears to be decreasing for all chains as the number $n$ of samples is increased, and we do not diagnose the failure of MCMC.

What went wrong with the bias diagnostic \eqref{eq: bias diagnostic} in this example?
Well, information about a \emph{finite-dimensional} generalised moment \index{generalised moment} $\int \nabla \log \pi \; \mathrm{d}\nu \in \mathbb{R}^d$ is insufficient to characterise a probability distribution; there are an infinitude of distributions $\nu$ for which all $d$ components of this generalised moment are $0$.
This suggests a potential solution; 
find an \emph{infinite-dimensional} generalised moment that fully determines whether or not $\pi$ and $\nu$ are equal.
Surprisingly, this can be achieved without needing to explicitly deal with an infinite-dimensional generalised moment, due to the \emph{kernel trick} \index{kernel trick} from machine learning, which was introduced in Section \ref{sec.fdkerneltrick} for finite-dimensional inner-product spaces, and will now be explored in the infinite-dimensional setting.

\begin{figure}
    \centering
    \includegraphics[width=\textwidth]{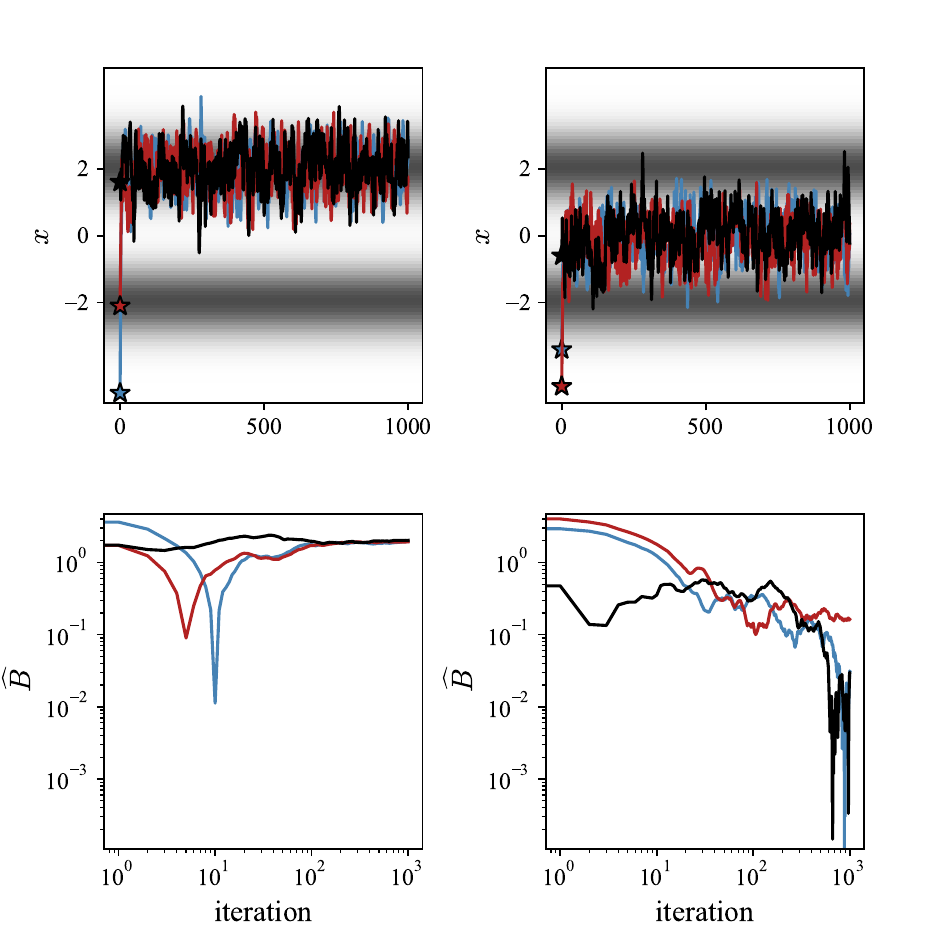}
    \caption{Bias diagnostics for MCMC.
    Three independent biased Markov chains were simulated, so that the invariant distribution differs from the Gaussian mixture target $\pi$ in \eqref{eq: mixture example}. 
    In the first scenario (left panels) the chains explore a Gaussian centred at $x=2$, and the bias diagnostic $\widehat{B}$ correctly detects that the Markov chains have not converged.
    In the second scenario (right panels) the chains explore a Gaussian target centred at $x=0$, and the bias diagnostic does not detect that the Markov chains have not converged.
    [Stars indicate the initial state of each Markov chain.
    The density $\pi$ is shaded.] 
    }
    \label{fig: bias diagnostics}
\end{figure}

\subsection{Improved Bias Diagnostics via the Kernel Trick}
\label{subsec: kernel trick}

Though finite-dimensional bias diagnostics can be misled, the same may not be true of a bias diagnostic that is infinite-dimensional.
The aim of this section is to indicate, at a high level, how such an infinite-dimensional bias diagnostic can be constructed.
A more rigorous mathematical treatment is then provided in Section \ref{sec: convergence bound}.

Suppose that we are able to write down a countable set $\{\phi_1,\phi_2,\dots\}$ of functions $\phi_j : \mathbb{R}^d \rightarrow \mathbb{R}$ such that each of the moments $\int \phi_j(\mathbf{x}) \pi(\mathbf{x}) \; \mathrm{d}\mathbf{x}$ of $\pi$ can be analytically evaluated; without loss of generality we may suppose that each generalised moment \index{generalised moment} of $\pi$ is equal to 0 (since if not, we may simply redefine $\phi_j$ as $\phi_j - \int \phi_j(\mathbf{x}) \pi(\mathbf{x}) \; \mathrm{d}\mathbf{x}$).
We have already seen examples of functions $\phi_j$ that can be used; we could take 
$$
\phi_j(\mathbf{x}) = \frac{\partial \log \pi(\mathbf{x})}{\partial x_j}
$$ 
for $j = 1,\dots,d$.
More generally, we can use the generator of any $\pi$-invariant Markov process to construct such functions; the details are deferred to Section \ref{sec: convergence bound}.
Assuming that the $\phi_j$ are linearly independent and appropriately normalised, we can construct a Hilbert space 
\begin{align*}
    \mathcal{H} = \left\{ h = \sum_{j=1}^\infty c_j \phi_j \; : \; \sum_{j=1}^\infty c_j^2 < \infty \right\}
\end{align*}
whose elements are functions $h : \mathbb{R}^d \rightarrow \mathbb{R}$, equipped with an inner product $\langle h , h' \rangle_{\mathcal{H}} = \sum_{j=1}^\infty c_j c_j'$, where here $h = \sum_{j=1}^\infty c_j \phi_j$ and $h' = \sum_{j=1}^\infty c_j' \phi_j$.
The induced norm is $\|h\|_{\mathcal{H}} := \langle h , h \rangle_{\mathcal{H}}^{1/2}$.
By picking elements from this Hilbert space \index{Hilbert space} we can construct an infinitude of bias diagnostics, and the question is then which diagnostic to pick?

One solution is to adopt an \emph{adversarial} perspective, where we select an element $h \in \mathcal{H}$ that maximally discriminates between $\pi$ and the empirical approximation to $\pi$ produced from MCMC.
The bias diagnostic obtained in this way can be written as
\begin{align*}
    \widetilde{B} := \sup_{\|h\|_{\mathcal{H}} \leq 1} \left| \frac{1}{n} \sum_{k=1}^n h(\mathbf{X}_k) \right| ,
\end{align*}
where the supremum is taken over the unit ball of $\mathcal{H}$, to ensure that the supremum is computed over a bounded set.
Further re-writing in terms of the basis functions, we have
\begin{align}
    \widetilde{B} := \sup \left\{ \left| \frac{1}{n} \sum_{k=1}^n \sum_{j=1}^\infty c_j \phi_j(\mathbf{X}_k) \right| \; : \; \sum_{j=1}^\infty c_j^2 \leq 1 \right\} , \label{eq: Btilde}
\end{align}
whence we see the maximisation is equivalent to finding the point $\mathbf{c}$ on the surface of the (infinite-dimensional) unit hypersphere that maximises the dot product with the (infinite-dimensional) vector $\mathbf{c}'$ with $c_j' = \frac{1}{n} \sum_{k=1}^n \phi_j(\mathbf{X}_k)$.
The solution of this maximisation problem is to align $\mathbf{c}$ to $\mathbf{c}'$, and upon properly normalising we obtain
\begin{align*}
    c_j = \frac{\frac{1}{n}\sum_{k=1}^n \phi_j(\mathbf{X}_k)}{ \sqrt{\sum_{j'=1}^\infty \left( \frac{1}{n} \sum_{k'=1}^n \phi_{j'}(\mathbf{X}_{k'}) \right)^2 } } .
\end{align*}
Inserting this expression back into \eqref{eq: Btilde} and rearranging, we obtain the explicit bias diagnostic
\begin{align*}
\widetilde{B} = \sqrt{ \frac{1}{n^2} \sum_{k=1}^n \sum_{k'=1}^n \left( \sum_{j=1}^\infty \phi_j(\mathbf{X}_k) \phi_j(\mathbf{X}_{k'}) \right) } .
\end{align*}
At this point, one can raise a reasonable objection that evaluating $\widetilde{B}$ appears to require an infinite computational budget, due to the summation over the functions $\phi_j$.
Remarkably, there are situations where such infinite series admit closed-form analytic expressions 
\begin{align*}
    \mathsf{k}_\pi(\mathbf{x},\mathbf{x}') := \sum_{j=1}^\infty \phi_j(\mathbf{x}) \phi_j(\mathbf{x}') ,
\end{align*}
a situation known in machine learning as the \emph{kernel trick}\index{kernel trick}.
Refer to Section \ref{sec.RKHS} for a primer on the kernel trick.
Provided that we have access to a kernel trick, which, of course, depends on the precise choice we make for the functions $\phi_j$ to ensure that $\int \phi_j(\mathbf{x}) \pi(\mathbf{x}) \; \mathrm{d}\mathbf{x} = 0$, we can hope to obtain a closed-form expression for the bias diagnostic \eqref{eq: Btilde}, namely
\begin{align}
    \widetilde{B} = \sqrt{ \frac{1}{n^2} \sum_{k=1}^n \sum_{k'=1}^n \mathsf{k}_\pi(\mathbf{X}_k,\mathbf{X}_{k'}) } . \label{eq: rkhs diagnostic}
\end{align}
One would hope that \eqref{eq: rkhs diagnostic} converges to 0 as $n \rightarrow \infty$ if and only if the Markov chain is $\pi$-invariant.
It turns out that such an idea can be made to work, as we will explain in Section \ref{sec: convergence bound}.

\smallskip

To summarise, we have seen that both convergence diagnostics \index{convergence diagnostic} and conventional finite-dimensional bias diagnostics can provide a useful practical tool to detect the failure of MCMC, but even taken together they are insufficient to guarantee that output from the sampler provides an accurate approximation of the intended target distribution.
In the next section, we turn our attention to the construction of infinite-dimensional bias diagnostics, of the form \eqref{eq: rkhs diagnostic}, which attempts to solve the problem of assessing MCMC output by establishing explicit upper bounds on an appropriate distance between the posterior and the approximation produced by MCMC in terms of diagnostics of the form \eqref{eq: rkhs diagnostic}.

\section{Convergence Bounds for MCMC}
\label{sec: convergence bound}

In contrast to convergence diagnostics and conventional finite-dimensional bias diagnostics, which may fail to detect instances where MCMC has failed, the aim of this section is to seek explicit and computable \emph{upper bounds} on the error between the MCMC output and the target distribution. 
This topic has received considerable recent interest following the pioneering work of \citet{gorham2015measuring}.
To set the scene, we first explain how the use of a suitable diffusion process enables explicit bounds on integral probability metrics, focusing on the Wasserstein-1 metric for clarity in Sections \ref{subsec: formal arg} and \ref{subsec: overdamped}.
However, the Wasserstein-1 metric is not favourable for computation in this context, and we instead consider integral probability metrics \index{integral probability metric} based on reproducing kernel Hilbert spaces in Section \ref{subsec: test functions}, noting that the associated \emph{kernel Stein discrepancies} can also provide valid convergence bounds in Section \ref{subsec: conv control kernel}.
Lastly, in Section \ref{sec: ssd} we connect kernel Stein discrepancies to the stochastic gradient methods from Chapter \ref{chap:sgld}.

\subsection{Bounds on Integral Probability Metrics}
\label{subsec: formal arg}

Our aim here is to arrive at an explicit and computable upper bound on an appropriate metric between the target distribution $\pi$ and the empirical distribution produced by MCMC.
Let $\mathcal{P}(\mathbb{R}^d)$ denote the set of probability distributions on $\mathbb{R}^d$ and consider a metric $\mathsf{d} : \mathcal{P}(\mathbb{R}^d) \times \mathcal{P}(\mathbb{R}^d) \rightarrow [0,\infty]$.
As a useful convention, we have extended the range of a metric to include the value $\infty$, to avoid the need to specify the subset of distributions on which the metric is defined.
Let $\pi \in \mathcal{P}(\mathbb{R}^d)$ be the distributional target of MCMC.
For our theoretical development, we will now introduce an auxiliary discrete time ergodic Markov chain, which need not be related to the Markov process(es) underpinning the MCMC method(s) being assessed.
The role of this auxiliary chain is limited to being a theoretical device in what follows, and we denote its transition kernel \index{kernel (Markov transition)} as $T_\pi$, meaning that $T_\pi \nu$ is the distribution after one step of the auxiliary Markov chain initialised from $\mathbf{X}_0 \sim \nu$.
This auxiliary chain is required to satisfy the \emph{contraction} property
\begin{align}
\mathsf{d}(T_\pi \pi,T_\pi \nu) \leq \rho \; \mathsf{d}(\pi,\nu)  \label{eq: contraction}
\end{align}
for some $\rho \in [0,1)$ and all $\nu \in \mathcal{P}(\mathbb{R}^d)$.
From the triangle inequality, $\mathsf{d}(\pi,\nu) \leq \mathsf{d}(\pi, T_\pi \nu) + \mathsf{d}(T_\pi \nu, \nu)$, combined with the contraction \index{contraction} $\mathsf{d}(\pi, T_\pi \nu) = \mathsf{d}(T_\pi \pi, T_\pi \nu) \leq \rho \; \mathsf{d}(\pi, \nu)$, it follows that $(1-\rho) \mathsf{d}(\pi,\nu) \leq \mathsf{d}(T_\pi \nu, \nu)$.
The \emph{discrepancy} \index{discrepancy}
$$
D_\pi(\nu) := \frac{1}{1-\rho} \; \mathsf{d}(T_\pi \nu, \nu)
$$ 
therefore constitutes an upper bound on the metric $\mathsf{d}(\pi,\nu)$, which could, in principle, be used to quantify how well a given distribution $\nu$ approximates $\pi$ in situations where we do not have direct access to $\pi$, but where the auxiliary Markov chain can be simulated.
Further, since $\mathsf{d}$ is a metric and the Markov chain has a unique invariant distribution, $D_\pi(\nu) = 0$ if and only if $\nu = \pi$.
An ideal scenario would be a fast mixing auxiliary Markov chain, so that $\rho \ll 1$, $T_\pi \nu \approx \pi$, and $D_\pi(\nu) \approx \mathsf{d}(\pi,\nu)$.
On the other hand, if the auxiliary Markov chain mixes slowly then the values taken by the discrepancy could fail to provide a meaningful indication of whether or not $\nu$ is an accurate approximation to $\pi$.
The utility of this upper bound therefore depends on the mixing \index{mixing} properties of the auxiliary Markov chain on which it is based.

To move towards a computable bound, let us suppose that $\mathsf{d}$ is an \emph{integral probability metric}\index{integral probability metric}, meaning that for any $\pi, \nu \in \mathcal{P}(\mathbb{R}^d)$
\begin{align}
\mathsf{d}(\pi,\nu) & = \sup_{g \in \mathcal{G}} \int g(\mathbf{x}) \; \pi(\mathrm{d}\mathbf{x}) - \int g(\mathbf{x}) \; \mathrm{d}\nu(\mathbf{x})    \label{eq: ipm}
\end{align}
for a suitable symmetric set\index{symmetric set}\footnote{The set $\mathcal{G}$ is symmetric if $-g \in \mathcal{G}$ whenever $g \in \mathcal{G}$; this allows us to avoid taking absolute values in the definition of the integral probability metric.} of test functions $\mathcal{G}$.
Introducing the linear operator
$$
(L_\pi g)(\cdot) = \int g(\mathbf{x}') T_\pi(\cdot, \mathrm{d}\mathbf{x}') - g(\cdot) , 
$$
and observing that
\begin{align*}
\int (L_\pi g)(\mathbf{x}) \; \mathrm{d} \nu(\mathbf{x}) & = \int g(\mathbf{x}') T_\pi (\mathbf{x}, \mathrm{d}\mathbf{x}') \; \mathrm{d}\nu(\mathbf{x}) - \int g(\mathbf{x}) \; \mathrm{d}\nu(\mathbf{x}) \\
& = \int g(\mathbf{x}) \; \mathrm{d}T_\pi \nu(\mathbf{x})  - \int g(\mathbf{x}) \; \mathrm{d}\nu(\mathbf{x}) ,
\end{align*}
the discrepancy \index{discrepancy} can be expressed as
\begin{align}
D_\pi(\nu) & = \frac{1}{1-\rho} \sup_{g \in \mathcal{G}} \int (L_\pi g)(\mathbf{x}) \; \mathrm{d}\nu(\mathbf{x}) .
\label{eq: general discrepancy}
\end{align}
However, to actually evaluate this discrepancy, we are required to compute expressions involving $L_\pi$, which in effect requires simulating all possible realisations of one step of the auxiliary Markov chain, and is therefore intractable in general.
To circumvent this issue, we will move from a discrete-time auxiliary Markov chain to a continuous-time auxiliary Markov process, with time $t$ transition kernel $T_\pi^t$ \index{kernel (Markov transition)} and associated linear operator $L_\pi^t$ and discrepancy $D_\pi^t$.
The contraction \index{contraction} property \eqref{eq: contraction}, in this case, reads
\begin{align}
    \mathsf{d}(T_\pi^t \pi , T_\pi^t \nu) \leq \rho_t \; \mathsf{d}(\pi,\nu). \label{eq: contraction property}
\end{align}
Considering the $t \downarrow 0$ limit we may, if the auxiliary Markov process mixes rapidly enough, obtain an expression for the discrepancy in terms of the generator $\mathcal{L}_\pi$ of the auxiliary Markov process, which we recall from Section \ref{sec:ch6-generator} is defined through its action on suitably regular test functions $g: \mathbb{R}^d \rightarrow \mathbb{R}$ as
$$
(\mathcal{L}_\pi g)(\cdot) := \lim_{t \downarrow 0} \frac{1}{t} L_\pi^t g(\cdot) .
$$
Indeed, assume that $\rho_t = e^{-ct}$ for some $c > 0$.
Then, in a purely formal manipulation, 
\begin{align}
    D_\pi(\nu) & := \lim_{t \downarrow 0} D_\pi^t(\nu) \label{eq: discrepancy} \\
    & = \sup_{g \in \mathcal{G}} \int \lim_{t \downarrow 0} \frac{1}{1-\rho_t} (L_\pi^t g)(\mathbf{x}) \; \mathrm{d}\nu(\mathbf{x})  
    = \frac{1}{c} \sup_{g \in \mathcal{G}} \int (\mathcal{L}_\pi g)(\mathbf{x}) \; \mathrm{d}\nu(\mathbf{x}) , \nonumber
\end{align}
where the final step uses the definition of the generator $\mathcal{L}_\pi$ and the fact that $e^{-ct} = 1 - ct + o(t)$ when $t$ is small.
Intriguingly, this form of discrepancy \index{discrepancy} may be computable, up to the rate constant $c$, which will be unknown in general.
The remaining challenges appear to be the selection of a suitable auxiliary Markov process, for which the contraction \index{contraction} property \eqref{eq: contraction property} is satisfied, and the solution of the optimisation problem over $\mathcal{G}$.
These issues are addressed, respectively, in Sections \ref{subsec: overdamped} and \ref{subsec: test functions}.

\begin{remark}[Stein discrepancies]
The discrepancy that we have introduced in \eqref{eq: discrepancy} is an instance of \emph{Stein discrepancy}\index{Stein discrepancy}, as defined in the pioneering work of \citet{gorham2015measuring}.
A \emph{Stein discrepancy} refers to any discrepancy of the form
\begin{align}
    \sup_{g \in \mathcal{G}'} \int (\mathcal{A}_\pi g)(\mathbf{x}) \; \mathrm{d}\nu \label{eq: stein discrepancy}
\end{align}
where the the \emph{Stein operator} \index{Stein operator} $\mathcal{A}_\pi$ and the \emph{Stein class} \index{Stein class} $\mathcal{G}'$ are selected in such a way that \eqref{eq: stein discrepancy} is zero if and only if $\pi$ and $\nu$ are equal.
The concept of a Stein discrepancy is more general than the discrepancy $D_\pi(\nu)$ we have constructed, since the Stein operator $\mathcal{A}_\pi$ need not arise from consideration of a continuous-time Markov process; see for example, the review of \citet{anastasiou2021stein}.
\end{remark}

\subsection{Choice of Auxiliary Markov Process} \label{subsec: overdamped}

To make this argument useful, we require a metric $\mathsf{d}$ and an auxiliary continuous-time Markov process for which the contraction \index{contraction} property \eqref{eq: contraction property} is satisfied.
For the auxiliary Markov process, we will consider the overdamped Langevin diffusion from Section \ref{sec.SDEs}:
\begin{align}
\mathrm{d}\mathbf{X}_t = \nabla \log \pi(\mathbf{X}_t) \; \mathrm{d}t + \sqrt{2} \; \mathrm{d}\mathbf{W}_t   \label{eq: overdamped lang diff}
\end{align}
whose infinitesimal generator \index{generator} is the second order differential operator $(\mathcal{L}_\pi g)(\mathbf{x}) = (\Delta g)(\mathbf{x}) + \langle (\nabla \log \pi)(\mathbf{x}) , (\nabla g)(\mathbf{x}) \rangle$, where $\Delta$ denotes the Laplacian differential operator for $\mathbb{R}^d$. 
To simplify the presentation, we initially make the assumption that $\pi$ is \emph{strongly log-concave}, meaning that
\eqref{eq:ass-convex} holds for some $l > 0$, and we recall that a sufficient condition for strong log-concavity \index{log-concave} is that $- \nabla \nabla \log \pi(\mathbf{x}) \succ \epsilon \mathbf{I}_d$ for some $\epsilon > 0$ and all $\mathbf{x} \in \mathbb{R}^d$, where $\nabla \nabla$ denotes the Hessian differential operator and the notation $\mathbf{A} \succ \mathbf{B}$ is used to mean that $\mathbf{A}-\mathbf{B}$ is a symmetric positive definite matrix.
This assumption will be relaxed in Section \ref{subsec: conv control kernel}.
Let $C^s(\mathbb{R}^d,\mathbb{R}^p)$ denote the set of functions $f : \mathbb{R}^d \rightarrow \mathbb{R}^p$ for which continuous derivatives exist of orders up to $s \in \{0,1,\dots\} \cup \{\infty\}$.
For $g \in C^0(\mathbb{R}^d,\mathbb{R}^p)$, let 
\begin{align}
M_1(g) := \sup_{\substack{\mathbf{x}, \mathbf{x}' \in \mathbb{R}^d \\ \mathbf{x} \neq \mathbf{x}' }} \frac{\|g(\mathbf{x})-g(\mathbf{x}')\|}{\|\mathbf{x}-\mathbf{x}'\|} \label{eq: def m1}
\end{align}
denote its (possibly infinite) Lipschitz constant.
Recall that a function $g$ is called \emph{Lipschitz} whenever $M_1(g) < \infty$.
The following is a well-known contraction \index{contraction} result for the overdamped Langevin diffusion, whose proof can be found in e.g. \citet{von2005transport}, or see Remark 1 in \citet{eberle2016reflection}:

\begin{theorem}[Contraction of the overdamped Langevin diffusion).] \label{thm: strong conc}
Let $\pi$ be strongly log-concave and let $\nabla \log \pi$ be Lipschitz.
Then the overdamped Langevin diffusion \index{Langevin diffusion} \eqref{eq: overdamped lang diff} satisfies the contraction \index{contraction} property \eqref{eq: contraction property} in the \emph{Wasserstein}-1 metric \index{Wasserstein metric}
\begin{align}
\mathsf{d}_{W_1}(\pi,\nu) := \sup_{\substack{ g \in C^0(\mathbb{R}^d,\mathbb{R}) \\ M_1(g) \leq 1 }} \int g(\mathbf{x}) \; \mathrm{d}\pi(\mathbf{x}) - \int g(\mathbf{x}) \; \mathrm{d}\nu(\mathbf{x}) \label{def: Wasserstein}
\end{align}
with $\rho_t = e^{- c t}$ for some $c > 0$ and all $t \in [0,\infty)$.
\end{theorem}

The infinitesimal generator \index{generator} $\mathcal{L}_\pi$ of the overdamped Langevin diffusion \index{Langevin diffusion} requires $\nabla g$ and $\Delta g$ to exist, but the Wasserstein-1 \index{Wasserstein metric} integral probability metric \index{integral probability metric} contains non-differentiable functions in the test function set $\mathcal{G}$.
This appears to prevent us from running the formal argument in \eqref{eq: discrepancy}.
However, it turns out that we may, without loss of generality, impose additional smoothness on the Wasserstein-1 test function set:

\begin{lemma}[(Smoother test functions for Wasserstein-1] \label{lem: smooth test functions}
For $\pi,\nu \in \mathcal{P}(\mathbb{R}^d)$,
\begin{align}
\mathsf{d}_{W_1}(\pi,\nu) = \sup_{\substack{g \in C^\infty(\mathbb{R}^d,\mathbb{R}) \\ M_1(g) \leq 1}} \int g(\mathbf{x}) \; \mathrm{d}\pi(\mathbf{x}) - \int g(\mathbf{x}) \; \mathrm{d}\nu(\mathbf{x}) . \label{eq: smooth Wass}
\end{align}
\end{lemma}
\begin{proof}
Since the supremum is being computed over a subset of the test functions in \eqref{def: Wasserstein}, it is immediate that the right-hand side of \eqref{eq: smooth Wass} is upper-bounded by $\mathsf{d}_{W_1}(\pi,\nu)$.
To prove the corresponding lower bound, let $\epsilon \in (0,1)$.
From the definition of $\mathsf{d}_{W_1}$ in \eqref{def: Wasserstein}, there exists $g_\epsilon$ with $M_1(g_\epsilon) \leq 1$ such that $\int g_\epsilon(\mathbf{x}) \mathrm{d}\pi(\mathbf{x}) - \int g_\epsilon(\mathbf{x}) \mathrm{d}\nu(\mathbf{x}) > \mathsf{d}_{W_1}(\pi,\nu) - \epsilon$.
Let $\delta > 0$ and $\mathbf{Z} \sim \mathsf{N}(\mathbf{0},\mathbf{I}_d)$.
Set $g_{\epsilon,\delta}(\mathbf{x}) = \mathbb{E}[g_\epsilon(\mathbf{x} + \delta \mathbf{Z})]$.
Then $g_{\epsilon,\delta} \in C^\infty(\mathbb{R}^d,\mathbb{R})$ and the Lipschitz constant of $g_{\epsilon,\delta}$ is not greater than that of $g_\epsilon$, since for all $\mathbf{x},\mathbf{x}' \in \mathbb{R}^d$,
\begin{align*}
| g_{\epsilon,\delta}(\mathbf{x}) - g_{\epsilon,\delta}(\mathbf{x}') | & 
= | \mathbb{E}[ g_\epsilon(\mathbf{x} + \delta \mathbf{Z}) - g_\epsilon(\mathbf{x}' + \delta \mathbf{Z}) ] | \\
& \leq M_1(g_\epsilon) \|\mathbf{x} - \mathbf{x}'\| .
\end{align*}
Thus $g_{\epsilon,\delta}$ is an element of the test set over which the supremum is computed on the right hand side of \eqref{eq: smooth Wass}.
From
\begin{align}
    |g_{\epsilon,\delta}(\mathbf{x}) - g_\epsilon(\mathbf{x})| 
    = | \mathbb{E}[ g_\epsilon(\mathbf{x} + \delta \mathbf{Z}) - g_\epsilon(\mathbf{x})] |
    & \leq \delta \mathbb{E}[\|\mathbf{Z}\|] M_1(g_\epsilon) ,
\end{align}
it follows that $g_{\epsilon,\delta}$ distinguishes between $\pi$ and $\nu$ almost as well as $g_\epsilon$, in the sense that
\begin{align*}
& \hspace{-10pt} \int g_{\epsilon,\delta}(\mathbf{x}) \; \mathrm{d}\pi(\mathbf{x}) - \int g_{\epsilon,\delta}(\mathbf{x}) \; \mathrm{d}\nu(\mathbf{x}) \\
& \geq \int g_\epsilon(\mathbf{x}) \; \mathrm{d}\pi(\mathbf{x}) - \int g_\epsilon(\mathbf{x}) \; \mathrm{d}\nu(\mathbf{x}) - 2 \delta \mathbb{E}[\|\mathbf{Z}\|] M_1(g_\epsilon) \\
& > \{ \mathsf{d}_{W_1}(\pi,\nu) - \epsilon - 2 \delta \mathbb{E}[\|\mathbf{Z}\|] \} ,
\end{align*}
which can be made arbitrarily close to $\mathsf{d}_{W_1}(\pi,\nu)$ by taking $\epsilon, \delta \rightarrow 0$.
Thus the supremum in \eqref{eq: smooth Wass} coincides with $\mathsf{d}_{W_1}(\pi,\nu)$, as claimed.
\end{proof}

\smallskip

To summarise, our formal argument has led to a bound 
\begin{align}
D_\pi(\nu) = \frac{1}{c} \sup_{\substack{g \in C^\infty(\mathbb{R}^d,\mathbb{R}) \\ M_1(g) \leq 1}} \int \Delta g(\mathbf{x}) + \langle \nabla \log \pi(\mathbf{x}) , \nabla g(\mathbf{x}) \rangle \; \mathrm{d}\nu(\mathbf{x})  \label{eq: smooth Wass ksd}
\end{align}
on the Wasserstein-1 \index{Wasserstein metric} distance between $\pi$ and $\nu$ that holds in the strongly log-concave \index{log-concave} setting of Theorem \ref{thm: strong conc}.
The route to obtaining this bound is instructive, and the lessons that we learned will be exploited in the subsequent sections, but unfortunately, the evaluation of this discrepancy requires a challenging optimisation problem to be solved.
In the case where $\nu$ has finite support, the objective function depends on $g$ only through its derivatives at the nodes in the support.
This observation enabled \citet{gorham2015measuring} to cast a closely related optimisation problem as a collection of linear programmes, which then can be numerically solved.
The interested reader is referred to \citet{gorham2015measuring} for further detail.
However, the reliance on numerical methods to evaluate \eqref{eq: smooth Wass ksd} limits the utility of \eqref{eq: smooth Wass ksd}.
Instead, we will proceed in Section \ref{subsec: test functions} to consider alternative sets of test functions for which the corresponding optimisation problem can be \emph{analytically} solved.

\subsection{Kernel Stein Discrepancy}
\label{subsec: test functions}

The aim of this section is to consider alternatives to the Wasserstein-1 distance, corresponding to alternative sets $\mathcal{G}$ of test functions defining the integral probability metric \eqref{eq: ipm}, for which the optimisation problem in \eqref{eq: general discrepancy} can be explicitly solved using the kernel trick advertised in Section \ref{subsec: kernel trick}.
However, the use of alternative metrics leads us to depart from the argument of Section \ref{subsec: formal arg}, which was based on the Wasserstein-1 contraction \index{contraction} result of Theorem \ref{thm: strong conc}, raising the question of whether the resulting discrepancy is still a meaningful convergence bound.
This question will be answered positively in Section \ref{subsec: conv control kernel}. 

To simplify the discussion, we start by considering vector fields as test functions, allowing us to reduce the order of the differential operators involved.
Thus, in the general notation of \eqref{eq: stein discrepancy}, we consider 
\begin{align}
(\mathcal{A}_\pi \mathbf{g})(\mathbf{x}) = (\nabla \cdot \mathbf{g})(\mathbf{x}) + \langle (\nabla \log \pi)(\mathbf{x}) , \mathbf{g}(\mathbf{x}) \rangle, \label{eq: curly Ap}
\end{align}
which is a \emph{first} order differential operator and the elements $\mathbf{g}$ are now vector fields $\mathbf{g} : \mathbb{R}^d \rightarrow \mathbb{R}^d$.
The discussion in Section \ref{subsec: overdamped} corresponds to $\mathbf{g}(\mathbf{x}) = (\nabla g)(\mathbf{x})$ for twice-differentiable $g : \mathbb{R}^d \rightarrow \mathbb{R}$.
Here, and in the sequel, for ease of presentation, we have subsumed the constant factor $1/c$ into the definition of the vector fields $\mathbf{g}$.
Now, if we are to consider alternative test functions $\mathbf{g}$, the minimum requirement on $\mathbf{g}$ is that $\mathcal{A}_\pi \mathbf{g}$ integrates to 0 with respect to $\pi$, to ensure that the discrepancy we construct vanishes when $\pi$ and $\nu$ are equal.
To this end, we have the following result:

\begin{lemma} \label{lem: Stein op int by parts}
Let $\mathbf{g} : \mathbb{R}^d \rightarrow \mathbb{R}^d$ satisfy $\mathbf{g} \in \mathcal{L}^1(\pi)$ and $\mathcal{A}_\pi \mathbf{g} \in \mathcal{L}^1(\pi)$, where $\mathcal{A}_\pi$ is defined in \eqref{eq: curly Ap}.
Then $\int (\mathcal{A}_\pi \mathbf{g})(\mathbf{x}) \; \mathrm{d}\pi(\mathbf{x}) = 0$. 
\end{lemma}
\begin{proof}
First notice that
\begin{align*}
\int (\mathcal{A}_\pi \mathbf{g})(\mathbf{x}) \; \mathrm{d}\pi(\mathbf{x}) 
 = \int \frac{1}{\pi(\mathbf{x})} (\nabla \cdot (\pi \mathbf{g}))(\mathbf{x}) \; \mathrm{d}\pi(\mathbf{x})
 = \int (\nabla \cdot (\pi \mathbf{g}))(\mathbf{x}) \; \mathrm{d}\mathbf{x} ,
\end{align*}
which suggests using the divergence theorem to calculate this integral.
To avoid the explicit calculation of surface integrals, which would otherwise be required when using the divergence theorem, we will first approximate the vector field $\pi \mathbf{g}$ using another continuously differentiable vector field with compact support.
Let $\varphi_m : \mathbb{R}^d \rightarrow \mathbb{R}$ denote the $m$th term in a sequence of compactly supported functions with $\varphi_m(\mathbf{x}) = 1$ for $\|\mathbf{x}\| \leq m$, $\sup_{\mathbf{x}} \|\nabla \varphi_m(\mathbf{x})\| < m^{-1}$ for each $m \in \mathbb{N}$, and $\varphi_m(\mathbf{x}) \uparrow 1$ for each $\mathbf{x} \in \mathbb{R}^d$.
From the divergence theorem \index{divergence theorem} applied to a continuously differentiable vector field with compact support,
\begin{align*}
0 & = \int (\nabla \cdot (\varphi_m \pi \mathbf{g}))(\mathbf{x}) \; \mathrm{d}\mathbf{x} \\
& = \int \langle \nabla \varphi_m(\mathbf{x}) , (\pi \mathbf{g})(\mathbf{x}) \rangle \; \mathrm{d}\mathbf{x} + \int \varphi_m(\mathbf{x}) (\nabla \cdot (\pi \mathbf{g}))(\mathbf{x}) \; \mathrm{d}\mathbf{x} .
\end{align*}
Since $\varphi_m \uparrow 1$ pointwise and $\nabla \cdot (\pi \mathbf{g}) \in \mathcal{L}^1(\mathbb{R}^d)$, from the dominated convergence theorem \index{dominated convergence theorem} 
\begin{align*}
\int \varphi_m(\mathbf{x}) (\nabla \cdot (\pi \mathbf{g}))(\mathbf{x}) \; \mathrm{d}\mathbf{x} \rightarrow \int (\nabla \cdot (\pi \mathbf{g}))(\mathbf{x}) \; \mathrm{d}\mathbf{x} .
\end{align*}
On the other hand, using Cauchy--Schwarz \index{Cauchy--Schwarz} and the assumption that $\pi \mathbf{g} \in \mathcal{L}^1(\mathbb{R}^d)$,
\begin{align*}
\left| \int \langle \nabla \varphi_m(\mathbf{x}) , (\pi \mathbf{g})(\mathbf{x}) \rangle \; \mathrm{d}\mathbf{x} \right| \leq \left( \sup_{\mathbf{x}} \|\nabla \varphi_m (\mathbf{x}) \| \right) \int \|(\pi \mathbf{g})(\mathbf{x})\| \; \mathrm{d}\mathbf{x} \rightarrow 0 .
\end{align*}
Thus we have shown that
\begin{align*}
\int (\mathcal{A}_\pi \mathbf{g})(\mathbf{x}) \; \mathrm{d}\pi(\mathbf{x})
= \int (\nabla \cdot (\pi \mathbf{g}))(\mathbf{x}) \; \mathrm{d}\mathbf{x} = 0,
\end{align*}
as claimed. 
\end{proof}

Our attention now turns to selecting a set of vector fields $\mathbf{g}$ for which Lemma \ref{lem: Stein op int by parts} holds and for which the optimisation problem in \eqref{eq: smooth Wass ksd} can be explicitly solved.
One approach to this task is to use a \emph{matrix-valued reproducing kernel}\index{kernel (matrix-valued)}, meaning a function $\mathsf{K} : \mathbb{R}^d \times \mathbb{R}^d \rightarrow \mathbb{R}^{d \times d}$ that is
\begin{enumerate}
    \item \emph{transpose-symmetric}; $\mathsf{K}(\mathbf{x},\mathbf{x}') = \mathsf{K}(\mathbf{x}',\mathbf{x})^\top$ for all $\mathbf{x},\mathbf{x}' \in \mathbb{R}^d$
    \item \emph{positive semi-definite;}; 
    $$
    \sum_{k=1}^n \sum_{k'=1}^n \langle \mathbf{c}_k , \mathsf{K}(\mathbf{x}_k,\mathbf{x}_{k'}) \mathbf{c}_{k'} \rangle \geq 0
    $$
    for all $\mathbf{x}_1,\dots,\mathbf{x}_n \in \mathbb{R}^d$, all $\mathbf{c}_1,\dots,\mathbf{c}_n \in \mathbb{R}^d$, and all $n \in \mathbb{N}$. 
\end{enumerate} 
For clarity, we emphasise that $\langle \mathbf{c},\mathbf{c}' \rangle = \mathbf{c}^\top \mathbf{c}'$ is the usual Euclidean inner product \index{inner product} on $\mathbb{R}^n$; in the sequel we will use subscripts to distinguish other inner products as they are introduced.
Let $\mathsf{K}_{\mathbf{x}} = \mathsf{K}(\cdot,\mathbf{x})$, so that $\mathsf{K}_{\mathbf{x}} : \mathbb{R}^d \rightarrow \mathbb{R}^{d \times d}$ is matrix-valued.
For vector-valued functions $\mathbf{g} = \sum_{k=1}^n \mathsf{K}_{\mathbf{x}_k} \mathbf{c}_k$ and $\mathbf{g}' = \sum_{l=1}^m \mathsf{K}_{\mathbf{x}_l'} \mathbf{c}_l'$, define an inner product
\begin{align}
    \langle \mathbf{g} , \mathbf{g}' \rangle_{\mathcal{H}(\mathsf{K})} = \sum_{k=1}^n \sum_{l=1}^m \langle \mathbf{c}_k , \mathsf{K}(\mathbf{x}_k , \mathbf{x}_l') \mathbf{c}_l' \rangle . \label{eq: HK inn prod}
\end{align}
There is a unique Hilbert space \index{reproducing kernel Hilbert space (vector-valued)} reproduced by $\mathsf{K}$, denoted $\mathcal{H}(\mathsf{K})$; see Proposition 2.1 of \citet{carmeli2006vector}.
This space is characterised as
\begin{align*}
    \mathcal{H}(\mathsf{K}) = \overline{\mathrm{span}}\{ \mathsf{K}_{\mathbf{x}} \mathbf{c} : \mathbf{x}, \mathbf{c} \in \mathbb{R}^d \}
\end{align*}
where here the closure is taken with respect to the inner product in \eqref{eq: HK inn prod}.
The resulting Hilbert space satisfies the \emph{reproducing property} \index{reproducing property}
\begin{align*}
    \langle \mathbf{g} , \mathsf{K}_{\mathbf{x}} \mathbf{c} \rangle_{\mathcal{H}(\mathsf{K})} = \langle \mathbf{g}(\mathbf{x}) , \mathbf{c} \rangle
\end{align*}
for all $\mathbf{g} \in \mathcal{H}(\mathsf{K})$ and $\mathbf{x}, \mathbf{c} \in \mathbb{R}^d$, which is a particular instance of the kernel trick \index{kernel trick} discussed in Section \ref{subsec: kernel trick}.
In what follows, it is convenient to overload notation, such that the reproducing property becomes $\langle \mathbf{g}, \mathsf{K}_{\mathbf{x}} \rangle_{\mathcal{H}(\mathsf{K})} = \mathbf{g}(\mathbf{x})$ in an informal shorthand.

Assuming sufficient regularity that 
\begin{eqnarray*}
F_\nu : \mathcal{H}(\mathsf{K}) & \rightarrow & \mathbb{R} \\
\mathbf{g} & \mapsto & \int (\mathcal{A}_\pi \mathbf{g})(\mathbf{x}) \; \mathrm{d}\nu(\mathbf{x})    
\end{eqnarray*}
is a bounded linear functional, the Riesz representation theorem \index{Riesz representation theorem} tells us that there is a unique element $\mu_\nu$ such that $F_\nu(\cdot) = \langle \mu_\nu , \cdot \rangle_{\mathcal{H}(\mathsf{K})}$.
Using our reproducing property shorthand,
\begin{align*}
    \mu_\nu(\mathbf{x}') & = \langle \mu_\nu , \mathsf{K}_{\mathbf{x}'} \rangle_{\mathcal{H}(\mathsf{K})} 
    = F_\nu(\mathsf{K}_{\mathbf{x}'})
    = \int \mathcal{A}_\pi^{\mathbf{x}} \mathsf{K}(\mathbf{x},\mathbf{x}') \mathrm{d}\nu(\mathbf{x}) ,
\end{align*}
where the superscript in $\mathcal{A}_\pi^{\mathbf{x}}$ indicates the action of $\mathcal{A}_\pi$ on the $\mathbf{x}$ argument, collapsing the matrix-valued function $\mathsf{K}_{\mathbf{x}}$ into the vector-valued function $\mathcal{A}_\pi^{\mathbf{x}} \mathsf{K}_{\mathbf{x}}$.
It follows that, if we consider the collection of vector fields $\mathbf{g}$ within the unit ball of $\mathcal{H}(\mathsf{K})$, our optimisation problem can be explicitly solved:
\begin{align}
\sup_{ \|\mathbf{g}\|_{\mathcal{H}(\mathsf{K})} \leq 1 } \int (\mathcal{A}_\pi \mathbf{g})(\mathbf{x}) \; \mathrm{d}\nu(\mathbf{x}) = \sup_{\|\mathbf{g}\|_{\mathcal{H}(\mathsf{K})} \leq 1} \langle \mathbf{g} , \mu_\nu \rangle_{\mathcal{H}(\mathsf{K})} = \|\mu_\nu\|_{\mathcal{H}(\mathsf{K})} , \label{eq: solve rkhs optimise}
\end{align}
where, again from the reproducing property and the assumption that $F_\nu$ is a bounded linear functional,
\begin{align}
\|\mu_\nu\|_{\mathcal{H}(\mathsf{K})}^2 & = \left\langle \int \mathcal{A}_\pi^{\mathbf{x}} \mathsf{K}_{\mathbf{x}} \; \mathrm{d}\nu(\mathbf{x}) , \int \mathcal{A}_\pi^{\mathbf{x}'} \mathsf{K}_{\mathbf{x}'} \; \mathrm{d}\nu(\mathbf{x}') \right\rangle_{\mathcal{H}(\mathsf{K})} \nonumber  \\
& = \iint \mathcal{A}_\pi^{\mathbf{x}} \mathcal{A}_\pi^{\mathbf{x}'} \langle  \mathsf{K}_{\mathbf{x}} , \mathsf{K}_{\mathbf{x}'} \rangle_{\mathcal{H}(\mathsf{K})} \; \mathrm{d}\nu(\mathbf{x}) \mathrm{d}\nu(\mathbf{x}') \nonumber \\
& = \iint \mathcal{A}_\pi^{\mathbf{x}} \mathcal{A}_\pi^{\mathbf{x}'} \mathsf{K}(\mathbf{x},\mathbf{x}') \; \mathrm{d}\nu(\mathbf{x}) \mathrm{d}\nu(\mathbf{x}') . \nonumber
\end{align}
It is convenient to introduce the shorthand $k_\pi(\mathbf{x},\mathbf{x}') := \mathcal{A}_\pi^{\mathbf{x}} \mathcal{A}_\pi^{\mathbf{x}'} \mathsf{K}(\mathbf{x},\mathbf{x}')$, whence we obtain the discrepancy
\begin{align}
\mathcal{D}_{k_\pi}(\nu) := \sqrt{ \iint k_\pi(\mathbf{x},\mathbf{x}') \; \mathrm{d}\nu(\mathbf{x}) \mathrm{d}\nu(\mathbf{x}') } ,
\label{eq: sd double integral}
\end{align}
which is exactly of the form we sought in \eqref{eq: rkhs diagnostic}.
This was termed a \emph{kernel Stein discrepancy} \index{kernel Stein discrepancy} in \citet{chwialkowski2016kernel,liu2016kernelized}, due to its dependence on a reproducing kernel and its characterisation as a Stein discrepancy \eqref{thm: strong conc}.
A second consequence of \eqref{eq: solve rkhs optimise} is that we can view the kernel Stein discrepancy as a generalised moment \index{generalised moment}
\begin{align*}
    \mathcal{D}_{k_\pi}(\nu_n) = \left\| \frac{1}{n} \sum_{k=1}^n \left. \mathcal{A}_\pi^{\mathbf{x}} \mathsf{K}_{\mathbf{x}} \right|_{\mathbf{x} = \mathbf{x}_k} \; \mathrm{d}\nu(\mathbf{x}) \right\|_{\mathcal{H}(\mathsf{K})} ,
\end{align*}
which takes a similar form to \eqref{eq: bias diagnostic} from Section \ref{subsec: bias diagn}, albeit the generalised moment can now be infinite-dimensional by virtue of taking values in $\mathcal{H}(\mathsf{K})$.
The function $\mathcal{A}_\pi^{\mathbf{x}} \mathsf{K}_{\mathbf{x}}$ is indeed a member of $\mathcal{H}(\mathsf{K})$ due to the \emph{differential reproducing property} \citep[see][Appendix C6]{barp2022targeted}.
The kernel Stein discrepancy has the potential to be a useful bias diagnostic, but first we need to establish that it has our basic desired functionality, such as being equal to 0 if and only if $\pi$ and $\nu$ are identical.
Clearly, then the choice of kernel $\mathsf{K}$ will be important, so we address this point next.

One of the simplest forms of matrix-valued reproducing kernel \index{kernel (matrix-valued)} is $\mathsf{K}(\mathbf{x},\mathbf{x}') = \mathsf{k}(\mathbf{x},\mathbf{x}') \mathbf{I}_d$, where $\mathsf{k}$ is a scalar-valued reproducing kernel\index{kernel (scalar-valued)}.
This choice leads to the explicit formula, due to \citet{oates2017control}:
\begin{align}
\mathsf{k}_\pi(\mathbf{x},\mathbf{x}') & = \nabla_{\mathbf{x}} \cdot \nabla_{\mathbf{x}'} \mathsf{k}(\mathbf{x},\mathbf{x}') + \langle \nabla_{\mathbf{x}} \log \pi(\mathbf{x}) , \nabla_{\mathbf{x}'} \mathsf{k}(\mathbf{x},\mathbf{x}') \rangle \nonumber \\
& \qquad + \langle \nabla_{\mathbf{x}'} \log \pi(\mathbf{x}') , \nabla_{\mathbf{x}} \mathsf{k}(\mathbf{x},\mathbf{x}') \rangle \nonumber \\
& \qquad + \langle \nabla_{\mathbf{x}} \log \pi(\mathbf{x}) , \nabla_{\mathbf{x}'} \log \pi(\mathbf{x}') \rangle \mathsf{k}(\mathbf{x},\mathbf{x}') \label{eq: AAk}
\end{align}
The function $\mathsf{k}_\pi$ is automatically a scalar-valued reproducing kernel \citep[see][Theorem 2.6]{barp2022targeted}, and $\mathsf{k}_\pi(\mathbf{x},\mathbf{x}') = \mathcal{A}_\pi^{\mathbf{x}} \mathbf{g}_{\mathbf{x}'}(\mathbf{x})$ where $\mathbf{g}_{\mathbf{x}'}(\mathbf{x}) := \mathcal{A}_\pi^{\mathbf{x}'} \mathsf{K}(\mathbf{x},\mathbf{x}') \in \mathcal{H}(\mathsf{K})$.
Thus, \emph{if} the matrix-valued reproducing kernel $\mathsf{K}$ is selected such that the conditions $\mathbf{g} \in \mathcal{L}^1(\pi)$ and $\mathcal{A}_\pi \mathbf{g} \in \mathcal{L}^1(\pi)$ of Lemma \ref{lem: Stein op int by parts} are satisfied for each $\mathbf{g} \in \mathcal{H}(\mathsf{K})$, it follows that $\int \mathsf{k}_\pi(\mathbf{x},\mathbf{x}') \mathrm{d}\pi(\mathbf{x}) = \int \mathcal{A}_\pi^{\mathbf{x}} \mathbf{g}_{\mathbf{x}'}(\mathbf{x}) \; \mathrm{d}\pi(\mathbf{x}) = 0$ for all $\mathbf{x}' \in \mathbb{R}^d$.
Sufficient conditions for satisfying the preconditions of Lemma \ref{lem: Stein op int by parts} will shortly be discussed.

In the case where $\nu = \sum_{k=1}^n w_k \delta_{\mathbf{x}_k}$ is a discrete distribution, \eqref{eq: sd double integral} reduces to the double sum
\begin{align}
\mathcal{D}_{\mathsf{k}_\pi}(\nu) = \sqrt{ \sum_{k=1}^n \sum_{k'=1}^n w_k w_{k'}  \mathsf{k}_\pi(\mathbf{x}_k,\mathbf{x}_{k'})  } .
\label{eq: sd double sum}
\end{align}
For the degenerate reproducing kernel with $\mathsf{k}(\mathbf{x},\mathbf{x}') = 1$ for all $\mathbf{x},\mathbf{x}' \in \mathbb{R}^d$ and uniform weights $w_k = n^{-1}$, the kernel Stein discrepancy \index{kernel Stein discrepancy} in \eqref{eq: sd double sum} reduces to the simple form
\begin{align*}
    \mathcal{D}_{\mathsf{k}_\pi}(\nu) = \left\| \frac{1}{n} \sum_{k=1}^n \nabla_{\mathbf{x}} \log \pi(\mathbf{x}_k) \right\|  ,
\end{align*}
which is precisely the bias diagnostic from \eqref{eq: bias diagnostic}.
In this case, $\mathcal{H}(\mathsf{K})$ is the Hilbert space of constant vector fields on $\mathbb{R}^d$ with norm $\|\mathbf{g}\|_{\mathcal{H}(\mathsf{K})} = \|\boldsymbol{\beta}\|$ where $\mathbf{g}(\mathbf{x}) = \boldsymbol{\beta}$ for all $\mathbf{x} \in \mathbb{R}^d$, which is insufficiently rich to determine whether or not $\pi$ and $\nu$ are close or equal.
However, with a suitable choice of reproducing kernel, the kernel Stein discrepancy can distinguish between different distributions, and indeed, provide a form of convergence control\index{convergence control}, as we explain in Section \ref{subsec: conv control kernel}. 

First, however, we must ensure the conditions of Lemma \ref{lem: Stein op int by parts} are satisfied, so that $\mathcal{D}_{\mathsf{k}_\pi}(\nu) = 0$ when $\pi$ and $\nu$ are equal.
This can be achieved using the following result:

\begin{lemma}
\label{lem: LP conditions on K}
If $\mathsf{K}(\mathbf{x},\mathbf{x}') = \mathsf{k}(\mathbf{x},\mathbf{x}') \mathbf{I}_d$ with $\mathsf{k}(\mathbf{x},\mathbf{x}') = \phi(\mathbf{x}-\mathbf{x}')$ for some $\phi \in C^2(\mathbb{R}^d,\mathbb{R})$, then $\nabla \log \pi \in \mathcal{L}^1(\pi)$ implies $\mathbf{g} \in \mathcal{L}^1(\pi)$ and $\mathcal{A}_\pi \mathbf{g} \in \mathcal{L}^1(\pi)$ for all $\mathbf{g} \in \mathcal{H}(\mathsf{K})$.
\end{lemma}
\begin{proof}
The reproducing property, followed by Cauchy--Schwarz\index{Cauchy--Schwarz}, gives
\begin{align*}
    \int \|\mathbf{g}(\mathbf{x})\| \; \mathrm{d}\pi(\mathbf{x}) & = \int \| \langle \mathbf{g} , \mathsf{K}_{\mathbf{x}} \rangle_{\mathcal{H}(\mathsf{K})} \| \; \mathrm{d}\pi(\mathbf{x}) \\
    & \leq \|\mathbf{g}\|_{\mathcal{H}(\mathsf{K})} \int \sqrt{ \mathrm{tr}( \mathsf{K}(\mathbf{x},\mathbf{x}) ) } \; \mathrm{d}\pi(\mathbf{x}) = \|\mathbf{g}\|_{\mathcal{H}(\mathsf{K})} \sqrt{d \phi(\mathbf{0})} ,
\end{align*}
which is finite for all $\mathbf{g} \in \mathcal{H}(\mathsf{K})$.
Similarly,
\begin{align*}
    \int |(\mathcal{A}_\pi \mathbf{g})(\mathbf{x})| \; \mathrm{d}\pi(\mathbf{x}) & = \int | \langle \mathbf{g} , \mathcal{A}_\pi^{\mathbf{x}} \mathsf{K}_{\mathbf{x}} \rangle_{\mathcal{H}(\mathsf{K})} | \; \mathrm{d}\pi(\mathbf{x}) \\
    & \leq \|\mathbf{g}\|_{\mathcal{H}(\mathsf{K})} \int \sqrt{  \mathcal{A}_\pi^{\mathbf{x}} \mathcal{A}_\pi^{\mathbf{x}'} \mathsf{K}(\mathbf{x},\mathbf{x}') |_{\mathbf{x}'=\mathbf{x}}  } \; \mathrm{d}\pi(\mathbf{x}) .
\end{align*}
Here the assumption $\phi \in C^2(\mathbb{R}^d,\mathbb{R})$ ensures the application of $\mathcal{A}_\pi^{\mathbf{x}} \mathcal{A}_\pi^{\mathbf{x}'}$ to $\mathsf{K}(\mathbf{x},\mathbf{x}')$ is well-defined.
Specialising to the translation-invariant reproducing kernel in the statement, we have 
\begin{align}
    \mathcal{A}_\pi^{\mathbf{x}} \mathcal{A}_\pi^{\mathbf{x}'} \mathsf{K}(\mathbf{x},\mathbf{x}') |_{\mathbf{x}'=\mathbf{x}} = - (\Delta \phi)(\mathbf{0}) + \phi(\mathbf{0}) \|(\nabla \log \pi)(\mathbf{x}) \|^2 \label{eq: form of stein kernel}
\end{align}
which shows that $\mathcal{A}_\pi \mathbf{g} \in \mathcal{L}^1(\pi)$ whenever $\nabla \log \pi \in \mathcal{L}^1(\pi)$, and completes the argument.
\end{proof}

The kernel Stein discrepancies we have just constructed are well-defined and computable, but we have not yet addressed the question of if and how the values of the discrepancy $\mathcal{D}_{\mathsf{k}_\pi}(\nu)$ are related to the closeness of $\pi$ and $\nu$.
Indeed, since we have used alternative test functions compared to \eqref{eq: smooth Wass ksd}, we cannot expect $\mathcal{D}_{\mathsf{k}_\pi}(\nu)$ to provide an upper bound on the Wasserstein-1 distance \index{Wasserstein metric} between $\pi$ and $\nu$.
The next section explains to what extent kernel Stein discrepancies relate to the closeness of $\pi$ and $\nu$.

\subsection{Convergence Control}
\label{subsec: conv control kernel}

The aim of this section is to establish whether kernel Stein discrepancies can provide control \index{convergence control} over integral probability metrics\index{integral probability metric}.
At the same time, we will weaken the strong log-concavity assumption from Section \ref{subsec: overdamped} to an assumption that $\pi$ is \emph{distantly dissipative}\index{distantly dissipative}, meaning that
$$
\liminf_{r \rightarrow \infty} \inf_{\substack{ \mathbf{x},\mathbf{x}' \in \mathbb{R}^d \\ \|\mathbf{x}-\mathbf{x}'\| = r }} \left\{ - \frac{\langle (\nabla \log \pi)(\mathbf{x}) - (\nabla \log \pi)(\mathbf{x}') , \mathbf{x} - \mathbf{x}' \rangle}{\|\mathbf{x}-\mathbf{x}'\|^2}  \right\} > 0 ,
$$
for which Wasserstein-1 contraction \index{contraction} of the overdamped Langevin diffusion \index{Langevin diffusion} \eqref{eq: overdamped lang diff} can still be established \citep[see][]{lindvall1986coupling,eberle2016reflection}.
The next Lemma demonstrates that distant dissipativity is indeed a generalisation of strong log-concavity\index{log-concave}:

\begin{lemma}
\label{lem: compact dd proof}
If $\nabla \log \pi$ is bounded on a compact set $S \subset \mathbb{R}^d$ and $\pi$ is strongly log-concave on the boundary of and outside of the set $S$, then $\pi$ is distantly dissipative\index{distantly dissipative}.
\end{lemma}
\begin{proof}
Let $\mathbf{b}(\mathbf{x}) := (\nabla \log \pi)(\mathbf{x})$, let $S$ be the compact set in the statement, and let $\mathrm{int}(S)$ denote the interior of $S$.
From the strong log-concavity assumption, there exists $c > 0$ such that for all $\mathbf{x},\mathbf{x}' \notin \mathrm{int}(S)$ we have
\begin{align}
- \frac{ \langle \mathbf{b}(\mathbf{x}) - \mathbf{b}(\mathbf{x}') , \mathbf{x} - \mathbf{x}' \rangle }{ \|\mathbf{x}-\mathbf{x}'\|^2 } > c . \label{eq: strong log conc out S}
\end{align}
Since $\mathbf{b}$ is bounded on $S$, we may pick $B > \sup_{\mathbf{x} \in S} \|\mathbf{b}(\mathbf{x})\|$.
Since $S$ is compact, $S$ is contained in $\{\mathbf{x} : \|\mathbf{x}\| \leq r/2 \}$ for some sufficiently large $r > 0$, and in particular the diameter of $S$ is at most $r$; we may further suppose that $r > 2 B/c$, so that $c' := \frac{1}{2} \left( c - \frac{2B}{r} \right) > 0$.

Consider arbitrary $\mathbf{x},\mathbf{x}'$ such that $\|\mathbf{x}-\mathbf{x}'\| > 2r$.
If $\mathbf{x},\mathbf{x}' \notin S$, condition \eqref{eq: strong log conc out S} is satisfied.
Thus consider the other case, where without loss of generality $\mathbf{x} \in S$; then $\mathbf{x}' \notin S$, since the diameter of $S$ is at most $r < \|\mathbf{x}-\mathbf{x}'\|$.
Let $\mathbf{x}''$ be the point of the line segment from $\mathbf{x}$ to $\mathbf{x}'$ that belongs to $S$ and is closest to $\mathbf{x}'$; such a point exists because $S$ is closed, and $\mathbf{x}'' \notin \mathrm{int}(S)$, since otherwise a point slightly further along the segment would also belong to $S$.
Writing $\mathbf{x} - \mathbf{x}' = \frac{\|\mathbf{x}-\mathbf{x}'\|}{\|\mathbf{x}''-\mathbf{x}'\|} ( \mathbf{x}'' - \mathbf{x}' )$ in the second term, we have
\begin{align*}
    - \langle \mathbf{b}(\mathbf{x}) - \mathbf{b}(\mathbf{x}') , \mathbf{x} - \mathbf{x}' \rangle & = - \langle \mathbf{b}(\mathbf{x}) - \mathbf{b}(\mathbf{x}'') , \mathbf{x} - \mathbf{x}' \rangle \\
    & \qquad - \frac{\|\mathbf{x}-\mathbf{x}'\|}{\|\mathbf{x}''-\mathbf{x}'\|} \langle \mathbf{b}(\mathbf{x}'') - \mathbf{b}(\mathbf{x}') , \mathbf{x}'' - \mathbf{x}' \rangle \\
    & >  - 2 B \|\mathbf{x}-\mathbf{x}'\|  \; + \; c \|\mathbf{x}-\mathbf{x}'\| \, \|\mathbf{x}''-\mathbf{x}'\|  
\end{align*}
where in the final line the Cauchy--Schwarz \index{Cauchy--Schwarz} and triangle inequalities, together with $\mathbf{x},\mathbf{x}'' \in S$, were used to bound the first term, and for the second term \eqref{eq: strong log conc out S} was applied to $\mathbf{x}',\mathbf{x}'' \notin \mathrm{int}(S)$.
Since $\mathbf{x}$ and $\mathbf{x}''$ both belong to $S$, we have $\|\mathbf{x}-\mathbf{x}''\| \leq r$ and hence $\|\mathbf{x}''-\mathbf{x}'\| \geq \|\mathbf{x}-\mathbf{x}'\| - r$.
Thus 
\begin{align*}
    - \frac{\langle \mathbf{b}(\mathbf{x}) - \mathbf{b}(\mathbf{x}') , \mathbf{x} - \mathbf{x}' \rangle}{\|\mathbf{x}-\mathbf{x}'\|^2} & > - \frac{2B}{\|\mathbf{x}-\mathbf{x}'\|} + c \frac{\|\mathbf{x}''-\mathbf{x}'\|}{\|\mathbf{x}-\mathbf{x}'\|} \\
    & \geq c - \frac{2B + cr}{\|\mathbf{x}-\mathbf{x}'\|} 
    > c - \frac{2B + cr}{2r} = \frac{1}{2} \left( c - \frac{2B}{r} \right) = c' .
\end{align*}
Combining these two results, we have shown that for all $\mathbf{x},\mathbf{x}'$ with $\|\mathbf{x}-\mathbf{x}'\| > 2r$, 
$$
- \frac{\langle \mathbf{b}(\mathbf{x}) - \mathbf{b}(\mathbf{x}') , \mathbf{x} - \mathbf{x}' \rangle}{\|\mathbf{x}-\mathbf{x}'\|^2} > \min(c,c') > 0
$$
and thus the distant dissipativity of $\pi$ is established.
\end{proof}

The most commonly used kernel Stein discrepancies \index{kernel Stein discrepancy} do not offer control of the Wasserstein-1 distance\index{Wasserstein metric}, though it is possible, through a careful choice of kernel, to obtain Wasserstein-1 control; we return to this point at the end of the present section.
Rather, the most common kernel Stein discrepancies offer control of the (weaker) \emph{Dudley metric} \index{Dudley metric}
\begin{align*}
    \mathsf{d}_D(\pi,\nu) := \sup_{\substack{ g \in C^0(\mathbb{R}^d,\mathbb{R}) \\ M_0(g) + M_1(g) \leq 1 }} \int g(\mathbf{x}) \; \mathrm{d}\pi(\mathbf{x}) - \int g(\mathbf{x}) \; \mathrm{d}\nu(\mathbf{x}) 
\end{align*}
on $\mathcal{P}(\mathbb{R}^d)$, at least in certain scenarios where $\pi$ is distantly dissipative \index{distantly dissipative} and the reproducing kernel $k$ is carefully selected.
For a bounded function $g : \mathbb{R}^d \rightarrow \mathbb{R}^p$, let $M_0(g) = \sup_{\mathbf{x} \in \mathbb{R}^d} \|g(\mathbf{x})\|$ and recall that $M_1(g)$ denotes the Lipschitz constant, from \eqref{eq: def m1}. 

\begin{theorem}[Weak convergence control; Theorem 8 of \citealp{gorham2017measuring}]
\label{thm: weak conv control} \index{convergence control}
Let $\pi$ be distantly dissipative \index{distantly dissipative} and $\nabla \log \pi$ be Lipschitz.
Let $\mathsf{K}(\mathbf{x},\mathbf{x}') = \phi(\mathbf{x}-\mathbf{x}') \mathbf{I}_d$ where $\phi(\mathbf{z}) = (1 + \|\mathbf{z}/\sigma\|^2)^{-\beta}$ for some $\sigma > 0$, $\beta \in (0,1)$.
Then $\mathcal{D}_{\mathsf{k}_\pi}(\nu_n) \rightarrow 0$ implies that $\mathsf{d}_D(\pi,\nu_n) \rightarrow 0$.
\end{theorem}

Of course, if $\nabla \log \pi$ is Lipschitz, then $\nabla \log \pi$ is automatically bounded on any compact set.
The set of test functions that define the Dudley metric \index{Dudley metric} is smaller than that for the Wasserstein-1 metric\index{Wasserstein metric}, and as a result, the Dudley metric is weaker than the Wasserstein-1 metric.
Indeed, the Dudley metric actually metrises the so-called \emph{weak convergence} of distributions, meaning that $\nu_n$ converges weakly (or \emph{in distribution}) to $\pi$ if and only if $\mathsf{d}_D(\pi,\nu_n) \rightarrow 0$.
The kernel Stein discrepancy itself does not provide an upper bound on $\mathsf{d}_D$ in this context, but an explicit non-linear transformation of the kernel Stein discrepancy \emph{does} still constitute an explicit upper bound.
 \citet{gorham2017measuring} for full details.

The reproducing kernel \index{kernel (scalar-valued)} appearing in Theorem \ref{thm: weak conv control} is called the \emph{inverse multi-quadric} \index{inverse multi-quadric} reproducing kernel, and its use was not accidental.
The careful analysis in Theorem 6 of \citet{gorham2017measuring} demonstrates that reproducing kernels with lighter tails can fail to control weak convergence, at least in dimension $d \geq 3$.
Moreover, the inverse multi-quadric reproducing kernel is computationally straightforward, and in fact the property of weak convergence \index{weak convergence} control extends to the parametric family of inverse multi-quadric reproducing kernels of the form
\begin{align}
\mathsf{K}_{\textnormal{IMQ}}(\mathbf{x},\mathbf{x}') = \big( 1 +\|\boldsymbol{\Sigma}^{-1/2} (\mathbf{x}-\mathbf{x}')\|^2 \big)^{-\beta} \mathbf{I}_d, \qquad \boldsymbol{\Sigma} \succ 0, \; \beta \in (0,1) ,  \label{eq: gen inv mult quad}
\end{align}
where $\boldsymbol{\Sigma}$ is a symmetric positive definite matrix, with the former case recovered when $\boldsymbol{\Sigma} = \mathbf{I}_d$; see Theorem 4 in \citet{chen2019stein}.
For the extended family of inverse multi-quadric reproducing kernels in \eqref{eq: gen inv mult quad}, the explicit form in \eqref{eq: AAk} becomes 
\begin{align*}
    \mathsf{k}_\pi(\mathbf{x},\mathbf{x}') & = - \frac{4 \beta (\beta + 1) \|\boldsymbol{\Sigma}^{-1} (\mathbf{x}-\mathbf{x}')\|^2 }{ (1 + \|\boldsymbol{\Sigma}^{-1/2}(\mathbf{x}-\mathbf{x}')\|^2)^{\beta + 2} } \\
    & \qquad + 2 \beta \left[ \frac{ \mathrm{tr}(\boldsymbol{\Sigma}^{-1}) + \langle (\nabla \log \pi)(\mathbf{x}) - (\nabla \log \pi)(\mathbf{x}') , \boldsymbol{\Sigma}^{-1} (\mathbf{x}-\mathbf{x}') \rangle }{ (1 + \|\boldsymbol{\Sigma}^{-1/2}(\mathbf{x}-\mathbf{x}')\|^2)^{\beta + 1} } \right] \\
    & \qquad + \frac{ \langle (\nabla \log \pi)(\mathbf{x}) , (\nabla \log \pi)(\mathbf{x}') \rangle }{ (1 + \|\boldsymbol{\Sigma}^{-1/2} (\mathbf{x}-\mathbf{x}')\|^2)^\beta } ,
\end{align*}
which can be readily computed provided that the gradient \index{gradient} $\nabla \log \pi$ can be pointwise evaluated.

To illustrate the performance of kernel Stein discrepancies\index{kernel Stein discrepancy}, consider again the Gaussian mixture distribution $\pi$ from \eqref{eq: mixture example}.
This distribution is distantly dissipative\index{distantly dissipative}\footnote{The class of distantly dissipative distributions includes all finite Gaussian mixtures with common covariance matrix; see \citet{gorham2019measuring}.} and has a log-density whose gradient is Lipschitz, meaning we are in the setting of Theorem \ref{thm: weak conv control}.
Proceeding with the inverse multi-quadric \index{inverse multi-quadric} reproducing kernel \eqref{eq: gen inv mult quad} with parameters $\Sigma = 1$, $\beta = 0.5$, we therefore have a guarantee that convergence of the kernel Stein discrepancy $\mathcal{D}_{\mathsf{k}_\pi}(\nu_n)$ to 0 implies the weak convergence \index{weak convergence} of $\nu_n$ to $\pi$.
In what follows, we let $\nu_n = \frac{1}{n} \sum_{k=1}^n \delta_{X_k}$ be the empirical distribution associated to a Markov chain sample path $(X_k)_{1 \leq k \leq n}$ and use kernel Stein discrepancy to determine whether or not $\nu_n$ converges to $\pi$.
 Figure \ref{fig: performance KSD} displays typical realisations of the Markov chain sample path (top), while underneath the associated kernel Stein discrepancy (as a function of $n$) is displayed.
In addition, the figure includes corresponding results for a Markov chain that leaves only the first component of $\pi$ invariant (d), and thus does not provide a consistent approximation of $\pi$.
Asymptotically, it can be seen that the kernel Stein discrepancy correctly distinguishes between the two scenarios, in which the Markov chain does and does not leave $\pi$ invariant.
However, focusing on the biased Markov chain, at small sample sizes the discrepancy does not detect that the chain has only explored one mixture component, and the discrepancy appears to decrease smoothly as more samples are collected.
It is only once a sufficient number of samples have been collected that the failure of the Markov chain to explore the second mixture component is detected, and the discrepancy ceases decreasing to reflect that.

\begin{figure}
    \centering
    \includegraphics[width=\textwidth]{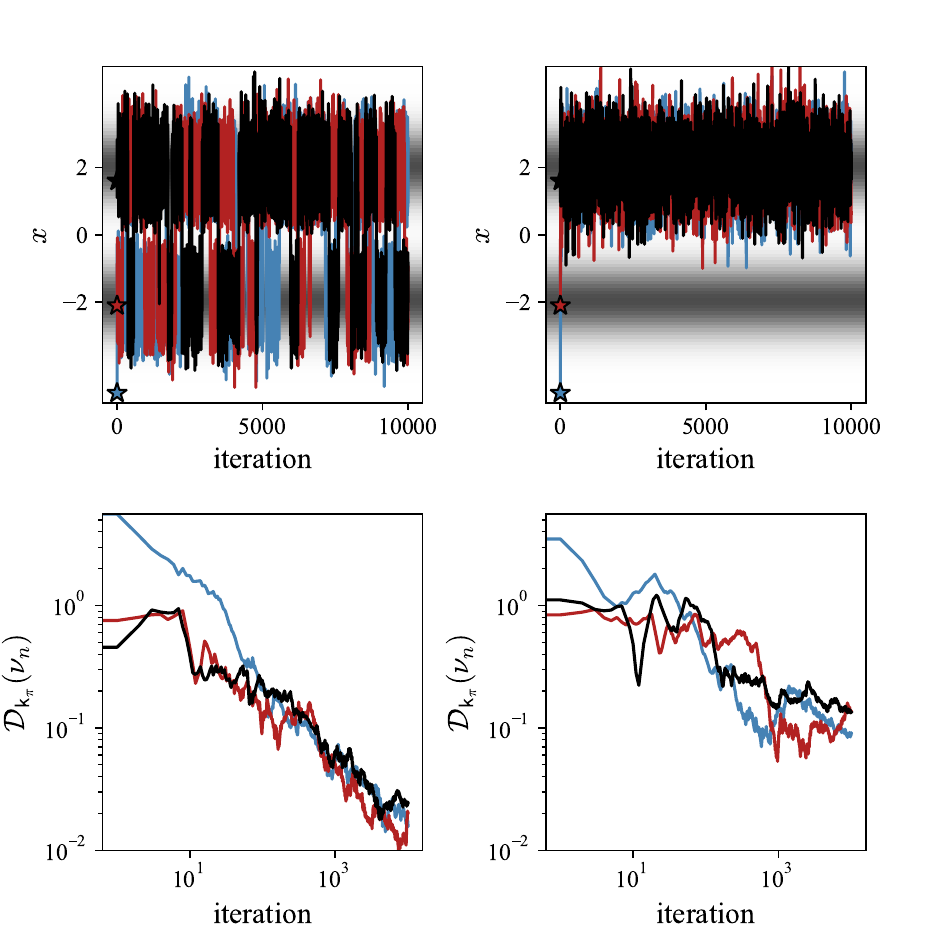}
    \caption{Performance of kernel Stein discrepancy.
    Three unbiased (left) and biased (right) Markov chains were simulated. 
    In the unbiased case, the kernel Stein discrepancy $\mathcal{D}_{\mathsf{k}_\pi}(\nu_n)$ correctly detects that the Markov chains are converging to the target.
    In the biased scenario, the kernel Stein discrepancy detects that the Markov chains have not converged to the correct target, but this becomes clear only after a sufficient number of iterations have been performed.
    [Stars indicate the initial state of each Markov chain.
    The density $\pi$ is shaded.]
    }
    \label{fig: performance KSD}
\end{figure}

The small $n$ behaviour of the kernel Stein discrepancy \index{kernel Stein discrepancy} observed in Figure \ref{fig: performance KSD} has been termed \emph{blindness to mixing proportions} in \cite{wenliang2020blindness}, and provides an important note of caution that, when using kernel Stein discrepancies to assess MCMC output, the failure of the Markov chain to explore distant high-probability regions may only be detected if the length $n$ of the Markov chain output is large enough.
Our formal argument in Section \ref{subsec: formal arg} provides insight into this phenomenon; the Wasserstein-1 \index{Wasserstein metric} contraction \index{contraction} rate constant of the overdamped Langevin diffusion\index{Langevin diffusion}, denoted $c$ in \eqref{eq: discrepancy}, can be extremely small for distributions such as $\pi$ for which the blindness phenomenon is encountered, since for this diffusion a move between the effective support of the distinct mixture components is a rare event.
As a consequence, although kernel Stein discrepancy does provide convergence control\index{convergence control}, in unfavourable settings it may provide only a loose form of control.

In multi-dimensional settings, an appropriate choice of the matrix $\boldsymbol{\Sigma}$ appearing in the inverse multi-quadric \index{inverse multi-quadric} reproducing kernel \index{kernel (scalar-valued)} \eqref{eq: gen inv mult quad} can be important.
In situations where $\nu_n$ can be interpreted as an approximation of $\pi$, the present authors continue to recommended the use of $\boldsymbol{\Sigma} = \mathbf{I}_d$, following the application of a data-dependent transformation
\begin{align}
(\mathbf{x}_i, \nabla \log \pi(\mathbf{x}_i)) & \mapsto ( \boldsymbol{\Gamma}_n^{-1} \mathbf{x}_i , \boldsymbol{\Gamma}_n \nabla \log \pi(\mathbf{x}_i) ), \label{eq: KSD transform}
\end{align}
where $\boldsymbol{\Gamma}_n$ is the diagonal matrix whose diagonal entries are the mean absolute deviation of the corresponding coordinates of $(\mathbf{x}_k)_{1 \leq k \leq n}$, the states on which $\nu_n$ is supported.
This transformation amounts to performing the change of variables $\mathbf{x} \mapsto \widetilde{\mathbf{x}} := \boldsymbol{\Gamma}_n^{-1} \mathbf{x}$ prior to computing the kernel Stein discrepancy with $\boldsymbol{\Sigma} = \mathbf{I}_d$.
Indeed, denoting by $\widetilde{\pi}$ the transformed probability density function, the change-of-variables formula gives that
\begin{align*}
\nabla_{\widetilde{\mathbf{x}}} \log \widetilde{\pi}(\widetilde{\mathbf{x}}) & = \nabla_{\widetilde{\mathbf{x}}} \log \left[ \mathrm{det}(\boldsymbol{\Gamma}_n) \pi(\mathbf{x}) \right] \\
& = \nabla_{\widetilde{\mathbf{x}}} \log \pi(\mathbf{x}) \\
& = \nabla_{\widetilde{\mathbf{x}}} \log \pi(\boldsymbol{\Gamma}_n \widetilde{\mathbf{x}})  
 = \boldsymbol{\Gamma}_n \nabla \log \pi(\boldsymbol{\Gamma}_n \widetilde{\mathbf{x}}) 
 = \boldsymbol{\Gamma}_n \nabla \log \pi(\mathbf{x}) .
\end{align*}
In this recommendation, the mean absolute deviation is used to provide a robust estimate for the unknown scale of the standard deviation of each coordinate in $\pi$.
One may be tempted to consider extending this recommendation to the more general class of invertible linear transformations, but the authors of \citet{riabiz2022optimal} cautioned that if considerable additional sample-based variability is introduced in estimating a general invertible linear transform, this can act as an undesirable confounding factor when the resulting discrepancies are to be interpreted for assessment of MCMC.
In a similar spirit, so-called \emph{sliced} kernel Stein discrepancies \index{kernel Stein discrepancy (sliced)} have recently been developed for high-dimensional applications \citep{gong2020sliced}; however, at the time of writing the convergence control \index{convergence control} of these discrepancies has yet to be established.

Aside from the specific limitations just discussed, there are a myriad of statistical applications where kernel Stein discrepancies \index{kernel Stein discrepancy} can and have been successfully applied.
Two distinct uses will be discussed in this chapter; optimal weighting of Markov chain output (Section \ref{sec: optimal weights}), and optimal thinning of Markov chain output (Section \ref{sec: optimal thinning}).
To close this section, we highlight that stronger modes of convergence can also be controlled by kernel Stein discrepancies.
The following result, which is a special case of \citet{kanagawa2023controlling}, indicates how a suitable \emph{tilting} of the reproducing kernel enforces moment convergence control\index{convergence control}:

\begin{theorem}[(Moment convergence control; Corollary 3.4 in \citet{kanagawa2023controlling}]
\label{thm: kanagawa result}
Let $\pi$ be distantly dissipative \index{distantly dissipative} and $\nabla \log \pi$ be Lipschitz.
Let $q \in \mathbb{N}$, $\mathbf{x}_0 \in \mathbb{R}^d$, and adopt the shorthand $w_r(\mathbf{x}) := (1 + \|\mathbf{x} - \mathbf{x}_0\|^2)^{(r-1)/2}$.
Let 
$$
\mathsf{K}(\mathbf{x},\mathbf{x}') = w_q(\mathbf{x}) w_q(\mathbf{x}') \mathsf{K}_{\textnormal{IMQ}}(\mathbf{x},\mathbf{x}') + w_{q-1}(\mathbf{x}) w_{q-1}(\mathbf{x}') (1 + \langle \mathbf{x} - \mathbf{x}_0 , \mathbf{x}' - \mathbf{x}_0 \rangle ) \mathbf{I}_d
$$
where $\mathsf{K}_{\textnormal{IMQ}}$ is the inverse multi-quadric \index{inverse multi-quadric} reproducing kernel from \eqref{eq: gen inv mult quad}.
Let $(\nu_n)_{n \in \mathbb{N}}$ be a sequence of distributions whose moments up to order $q$ exist.
Then $\mathcal{D}_{\mathsf{k}_\pi}(\nu_n) \rightarrow 0$ implies that both $\mathsf{d}_D(\pi,\nu_n) \rightarrow 0$ and the moments of $\nu_n$ up to order $q$ converge to those of $\pi$.
\end{theorem}

\noindent In other words, the kernel Stein discrepancies \index{kernel Stein discrepancy} constructed in this manner have control over the Wasserstein-$q$ distance, which is equivalent to weak convergence \index{weak convergence} plus the convergence of moments up to $q$th order.
The principal requirement for making use of the kernel Stein discrepancies in Theorem \ref{thm: kanagawa result} is to pick a location $\mathbf{x}_0 \in \mathbb{R}^d$.
Theoretical guidance tells us that this kernel Stein discrepancy provides tightest control over moments when $\mathbf{x}_0$ is in a region of high probability for $\pi$, since otherwise the weightings \index{weight} $w_q$ and $w_{q-1}$ become approximately constant and we recover the standard kernel.
The difficulty of finding such a location $\mathbf{x}_0$ will be context-dependent.

\subsection{Stochastic Gradient Stein Discrepancy}
\label{sec: ssd}

The aim of this section is to discuss how kernel Stein discrepancies may be extended to the so-called \emph{tall data} \index{tall data} setting, where algorithms such as stochastic gradient \index{stochastic gradient} MCMC from Chapter \ref{chap:sgld} are used.
The principal challenge in this setting is that computation of the gradient \index{gradient} $\nabla \log \pi$ is associated with a high computational cost $O(N)$ due to the form of the likelihood
$$
L(\mathbf{x}; \mathcal{D}) = \prod_{j=1}^N L(\mathbf{x} ; \mathbf{y}_j)
$$
as a product of a large number $N$ of terms that each need to be differentiated.
Performing Bayesian inference with a prior $\pi_0(\mathbf{x})$, the posterior \index{posterior distribution} distribution takes the form
\begin{align*}
\pi(\mathbf{x}) \propto \prod_{j=1}^N \pi_j(\mathbf{x}), \qquad \pi_j(\mathbf{x}) \propto \pi_0(\mathbf{x})^{1/N} L(\mathbf{x} ; \mathbf{y}_j )  ,
\end{align*}
where we assume each $\pi_j$ can be properly normalised.
Let $\nu_n = \frac{1}{n} \sum_{k=1}^n \delta_{\mathbf{x}_k}$ define a sequence $(\nu_n)_{n \in \mathbb{N}} \subset \mathcal{P}(\mathbb{R}^d)$ of discrete distributions in terms of a sequence $(\mathbf{x}_k)_{k \in \mathbb{N}} \subset \mathbb{R}^d$.
Fix a \emph{batch size} \index{batch} $m \ll N$ and, for each $k$, independently select a uniformly random subset $\mathcal{S}_m^{(k)}$ of size $m$ from $\{1,\dots,N\}$.
Then 
\begin{align*}
\widehat{\mathbf{b}}_k := \frac{N}{m} \sum_{j \in \mathcal{S}_m^{(k)} } \nabla \log \pi_j(\mathbf{x}_k)
\end{align*}
is a stochastic approximation to the gradient $\mathbf{b}(\mathbf{x}_k) = \nabla \log \pi(\mathbf{x}_k)$ that can be computed at a relatively lower $O(m)$ cost.
It is then tempting to replace the exact gradients \index{gradient} $\mathbf{b}(\mathbf{x}_i)$ with their stochastic counterparts $\widehat{\mathbf{b}}_i$ within the construction of kernel Stein discrepancy.
For reproducing kernels of the form $\mathsf{K}(\mathbf{x},\mathbf{x}') = \mathsf{k}(\mathbf{x},\mathbf{x}') \mathbf{I}_d$, this construction leads to the following stochastic approximation of \eqref{eq: AAk}:
\begin{align*}
    \widehat{\mathsf{k}}_\pi(\mathbf{x}_i,\mathbf{x}_j) & := \left. \nabla_{\mathbf{x}} \cdot \nabla_{\mathbf{x}'} \mathsf{k}(\mathbf{x},\mathbf{x}') \right|_{\mathbf{x} = \mathbf{x}_i, \mathbf{x}' = \mathbf{x}_j} + \left\langle \widehat{\mathbf{b}}_i , \left. \nabla_{\mathbf{x}'} \mathsf{k}(\mathbf{x}_i,\mathbf{x}') \right|_{\mathbf{x}' = \mathbf{x}_j} \right\rangle \\
    & \qquad + \left\langle \widehat{\mathbf{b}}_j , \left. \nabla_{\mathbf{x}} \mathsf{k}(\mathbf{x},\mathbf{x}_j) \right|_{\mathbf{x} = \mathbf{x}_i} \right\rangle + \left\langle \widehat{\mathbf{b}}_i , \widehat{\mathbf{b}}_j \right\rangle \mathsf{k}(\mathbf{x}_i,\mathbf{x}_j)
\end{align*}
\citet{gorham2020stochastic} defined the \emph{stochastic kernel Stein discrepancy} \index{kernel Stein discrepancy (stochastic)} in this context as
\begin{align*}
\mathcal{D}_{\hat{\mathsf{k}}_\pi}(\nu_n) = \sqrt{ \frac{1}{n^2} \sum_{k=1}^n \sum_{k'=1}^n \hat{\mathsf{k}}_\pi(\mathbf{x}_k,\mathbf{x}_{k'})  } .
\end{align*}
An immediate question is whether or not the introduction of stochasticity into the gradients jeopardised the weak convergence \index{weak convergence} control \index{convergence control} property established in the case of exact gradients in Theorem \ref{thm: weak conv control}.
It turns out that, provided each $\pi_1,\dots,\pi_N$ satisfies the conditions of Theorem \ref{thm: weak conv control}, a form of weak convergence control continues to hold.
Specifically, if each $\pi_i$ is distantly dissipative, and each $\nabla \log \pi_i$ is Lipschitz, then with $\mathsf{k}(\mathbf{x},\mathbf{x}') = (1 + \|(\mathbf{x}-\mathbf{x}')/\sigma\|^2)^{-\beta}$, $\sigma > 0$, $\beta \in (0,1)$, \citet[][Theorem 4]{gorham2020stochastic} shows that
$$
\mathcal{D}_{\hat{\mathsf{k}}_\pi}(\nu_n) \rightarrow 0 \qquad \implies \qquad \mathsf{d}_{\mathsf{D}}(\pi,\nu_n) \rightarrow 0
$$
almost surely.
This result justifies the use of stochastic kernel Stein discrepancies in their own right, not merely as approximations to kernel Stein discrepancy that becomes exact when $m = N$.
Indeed, in principle, only a batch \index{batch} size of $m=1$ is required.

Figure \ref{fig: SSD} displays stochastic kernel Stein discrepancies computed for the sequence of empirical distributions $\nu_n$ produced using stochastic gradient \index{stochastic gradient} Langevin dynamics applied to the logistic regression example from Section \ref{sec:ch3-logistic-regression}.
It can be seen that, even for a batch size $m=100$, which is much less than the size $N = 10^4$ of the dataset, the stochastic kernel Stein discrepancy is capable of providing similar information on the performance of the sampler compared to when larger batch sizes, such as $m = 10^3$ are used.
The predictable decrease of the discrepancy indicates that the intrinsic bias of stochastic gradient Langevin dynamics is negligible relative to the error incurred by using only $n$ of these samples to construct $\nu_n$.
Nevertheless, at the time of writing, there remains scope to improve these stochastic discrepancies, not least through the use of reduced-variance stochastic approximations to the gradient.

\begin{figure}[t]
    \centering
    \includegraphics[width = 0.7\textwidth]{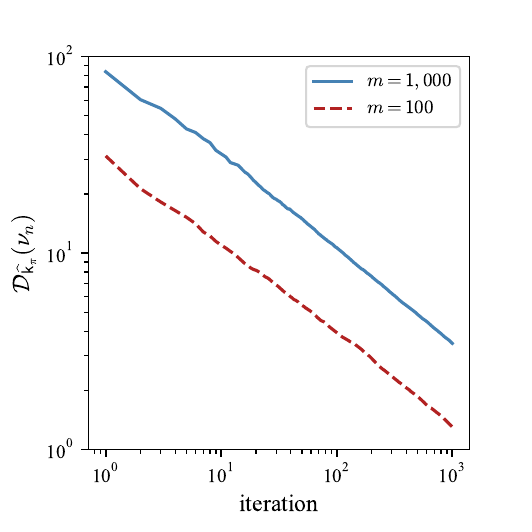}
    \caption{Stochastic gradient Stein discrepancies.  Stochastic gradient Langevin dynamics was used to generate approximate samples from the posterior distribution $\pi$ in the logistic regression example from Section \ref{sec:ch3-logistic-regression}, and stochastic gradient Stein discrepancies were used to measure the discrepancy between the empirical distribution $\nu_n$ of the approximate samples and $\pi$. Here $m$ indicates the size of the data subsets that were used to approximate the gradient.
    }
    \label{fig: SSD}
\end{figure}

\section{Optimal Weights for MCMC}
\label{sec: optimal weights}

At this point, we have seen how computable discrepancies may be constructed and used to \emph{passively} assess the performance of MCMC.
Now, we turn to how such discrepancies might be used to \emph{actively} improve output from MCMC.
Specifically, in this section, we explore how, given a realisation $(\mathbf{x}_k)_{1 \leq k \leq n}$ of a Markov chain and a target distribution $\pi$, we may exploit the kernel Stein discrepancy \index{kernel Stein discrepancy} to assign a weight \index{weight} $w_k$ to each $\mathbf{x}_k$ in such a manner that the discrepancy between the weighted empirical distribution and $\pi$ is minimised.
This is loosely analogous to importance sampling \index{importance sampling} (c.f. Section \ref{sec.MCBayes}), except here the analogue of the importance distribution is the distribution of the MCMC sample path, which, like the posterior itself, is implicitly defined.
As such, the methods we will discuss were termed \emph{Black Box Importance Sampling} \index{black-box importance sampling} in \citet{liu2017black}.
Surprisingly, we will see that such retrospective re-weighting can be used to remove the bias of approximate sampling algorithms, such as stochastic gradient MCMC from Chapter \ref{chap:sgld}.

Let $\mathsf{k}_\pi : \mathbb{R}^d \times \mathbb{R}^d \rightarrow \mathbb{R}$ be a scalar-valued reproducing kernel \index{kernel (scalar-valued)} for which $\int \mathsf{k}_\pi(\mathbf{x},\mathbf{y}) \; \mathrm{d}\pi(\mathbf{x}) = 0$ for all $\mathbf{y} \in \mathbb{R}^d$; an example being \eqref{eq: AAk}.
The weights \index{weight} that we consider are the solution of the following optimisation problem:
\begin{align}
\mathbf{w}^\star := \left( \begin{array}{c} w_1^\star \\ \vdots \\w_n^\star \end{array} \right) \in \argmin_{\substack{w_1 , \dots , w_n \geq 0 \\  w_1 + \dots + w_n = 1}} \mathcal{D}_{\mathsf{k}_\pi}\left(\sum_{k=1}^n w_k \delta_{\mathbf{x}_k} \right)  \label{eq: non neg weights}
\end{align}
Let $\nu_n$ be the general weighted empirical distribution appearing on the right-hand side of \eqref{eq: non neg weights}.
From \eqref{eq: sd double integral} we have  $\mathcal{D}_{\mathsf{k}_\pi}(\nu_n)^2 = \langle \mathbf{w} , \mathsf{K}_\pi \mathbf{w} \rangle$, where $\mathsf{K}_\pi$ is the $n \times n$ matrix with entries $[\mathsf{K}_\pi]_{i,j} = \mathsf{k}_\pi(\mathbf{x}_i,\mathbf{x}_j)$.
If the matrix $\mathsf{K}_\pi$ is positive definite
then $\mathbf{w}^\star$ is unique and, although not available in closed form, $\mathbf{w}^\star$ can be computed by solving a linearly constrained quadratic optimisation problem over the positive orthant of $\mathbb{R}^n$.
The optimally weighted distribution will be denoted $\nu_n^\star$ in the sequel.

A natural first question is whether the weighted approximation to $\pi$, obtained by retrospectively assigning weights to output from MCMC, is consistent.
This turns out to be true, under appropriate assumptions, and moreover, optimal weights can provide bias correction in settings where the Markov chain is not $\pi$-invariant.
Indeed, the recent work of \citet{riabiz2022optimal} established that, in the case of a $\mu$-invariant, time-homogeneous, ergodic Markov chain $(\mathbf{X}_k)_{k \in \mathbb{N}} \subset \mathbb{R}^d$, then the existence of certain moments of the ratio $\pi / \mu$ can be used to deduce that
$$
    \mathcal{D}_{\mathsf{k}_\pi} \left( \sum_{k=1}^n w_k^\star \delta_{\mathbf{X}_k} \right) \rightarrow 0
$$
almost surely as $n \rightarrow \infty$.
Importantly, the biased target $\mu$ of the Markov chain $(\mathbf{X}_k)_{k \in \mathbb{N}}$ does not need to be known to perform Black Box Importance Sampling in \eqref{eq: non neg weights}.

It can sometimes be convenient to relax the non-negativity constraint, to consider 
\begin{align}
\widetilde{\mathbf{w}}^\star := \left( \begin{array}{c} \widetilde{w}_1^\star \\ \vdots \\ \widetilde{w}_n^\star \end{array} \right) \in \argmin_{\substack{\mathbf{w} \in \mathbb{R}^n \\ w_1 + \dots + w_n = 1}} \mathcal{D}_{\mathsf{k}_\pi}\left(\sum_{k=1}^n w_k \delta_{\mathbf{x}_k} \right) . \label{eq: weighted}
\end{align}
The weights $\widetilde{\mathbf{w}}^\star$ may be negative, and thus the associated $\tilde{\nu}_n^\star$ is a \emph{signed measure} in general.
Signed measures may not pose a problem if the goal of computation is to approximate posterior expectations of interest, but if the goal is to approximate $\pi$ itself then a proper probability distribution may be preferred.
The main advantage of the formulation in \eqref{eq: weighted} is that, provided $\mathsf{K}_\pi \succ 0$, the relaxed optimisation problem has a unique and explicit solution
\begin{align}
\widetilde{\mathbf{w}}^\star = \frac{ \mathsf{K}_\pi^{-1} \mathbf{1} }{ \mathbf{1}^\top \mathsf{K}_\pi^{-1} \mathbf{1}  }. \label{eq: optimal signed weights}
\end{align}
This formulation is closely related to  \emph{kernel cubature} \index{kernel cubature} (also known as \emph{Bayesian quadrature}), and specifically coincides with the \emph{normalised} kernel cubature of \citet{karvonen2018bayes}. From \eqref{eq: optimal signed weights}, we deduce that the computational complexity of obtaining optimal weights is in general $O(n^3)$.
This can preclude the use of optimal weights \index{weight} on desktop computational hardware when $n$ is larger than a few thousand.
However, we will see in Section \ref{sec: optimal thinning} how accurate sparse approximations to the optimally weighted distribution can be efficiently constructed.

To demonstrate the effect of re-weighting\index{weight}, consider the following \emph{Rosenbrock} \index{Rosenbrock} target
\begin{align*}
\pi(x,y) \propto \exp(- (x-a)^2 - b(y-x^2)^2 )
\end{align*}
where here we take $a = 0$, $b = 3$.
The distribution $\pi$ might be referred to as a \emph{horseshoe} distribution, due to the curved shape of its level sets.
To represent output from a biased sampler, we generate sequence $(\mathbf{X}_k)_{k \in \mathbb{N}}$ of independent samples from $\nu = \mathsf{N}(\mathbf{0},\mathbf{I}_2)$ and assign a weight $w_i$ to each state $\mathbf{X}_i$ according to either \eqref{eq: non neg weights} or \eqref{eq: weighted}.
For these experiments we took $\mathbf{\Sigma} = \mathbf{I}_2$, $\beta = 0.5$, and the transformation in \eqref{eq: KSD transform} was applied.
Figure \ref{fig: opt weights} displays the qualitative properties of the weights $\mathbf{w}^\star$ defined by \eqref{eq: non neg weights} (left) and the weights $\widetilde{\mathbf{w}}^\star$ defined by \eqref{eq: weighted} (right).
In both cases, it can be seen that states $\mathbf{X}_k$ for which the probability under $\pi$ is low are typically assigned a small weight.
On the other hand, optimal weights are not independent, and the over-representation of states in a local region due to Monte Carlo sampling variability is partially mitigated.
Comparison of the kernel Stein discrepancies $\mathcal{D}_{\mathsf{k}_\pi}(\nu_n^\star)$ and $D_{\mathsf{k}_\pi}(\widetilde{\nu}_n^\star)$ \index{kernel Stein discrepancy} indicates that the non-negative weights $\mathbf{w}^\star$ perform nearly as well as the signed weights $\widetilde{\mathbf{w}}^\star$, while both sets of weights lead to a substantial decrease in kernel Stein discrepancy compared to the use of (inconsistent) uniform weights in this experiment.

\begin{figure}
    \centering
    \includegraphics[width=\textwidth]{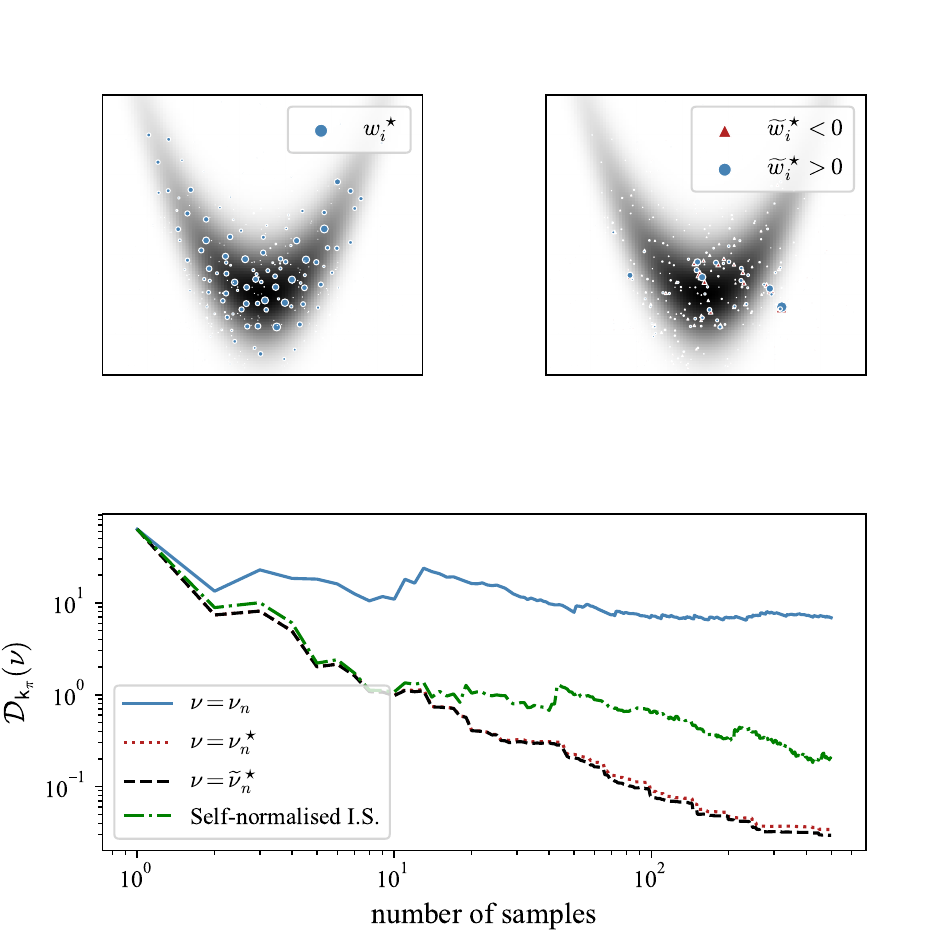}
    \caption{Optimal weights for MCMC.
    Samples from a standard Gaussian distribution were assigned either the optimal sign-constrained weights $\mathbf{w}^\star$ (left) or the optimal unconstrained weights $\widetilde{\mathbf{w}}^\star$ (right), to obtain consistent approximations $\nu_n^\star$ and $\widetilde{\nu}_n^\star$ respectively of the horseshoe distribution $\pi$ (shaded).
    These approximations each demonstrate convergence in the sense of kernel Stein discrepancy as the number of samples is increased, while the distribution $\nu$ of the original samples does not provide a consistent approximation of $\pi$.
    }
    \label{fig: opt weights}
\end{figure}

Of course, in this toy example, one has access to the sampling density of the $\mathbf{X}_k$ and self-normalised importance sampling \index{importance sampling} could trivially be used.
That is, to each sample $\mathbf{X}_k$ we assign weights proportional to $\pi(\mathbf{X}_k)/\nu(\mathbf{X}_k)$, and we normalise these weights to sum to 1.
The performance of self-normalised importance sampling is displayed in the lower panel of Figure \ref{fig: opt weights}, where it is seen to be inferior to the discrepancy-based methods that we have discussed.
Where has this performance gap come from?
Well, in addition to bias correction, the discrepancy-based methods additionally perform variance reduction, in the sense that the random vectors $\mathbf{w}^\star$ and $\widetilde{\mathbf{w}}^\star$ each contain components that are strongly inter-dependent.
This means that if a region is over-represented with samples, then the overall weight of these samples can be collectively reduced to better approximate the target $\pi$.
In contrast, self-normalised importance sampling has to rely on the long-run frequency of independent sampling to ensure that different regions of the domain are assigned an equal amount of probability mass, and this negatively affects its finite sample performance.

The satisfactory performance of optimal weights \index{weight} is observed in settings such as Figure \ref{fig: opt weights}, where pathological behaviours of kernel Stein discrepancy (such as blindness to mixing proportions; see Section \ref{subsec: conv control kernel}) are not encountered.
However, outside this setting the use of optimal weights can fail.
Further, if the kernel Stein discrepancy has weak convergence \index{weak convergence} control \index{convergence control} but not moment control, then there is no guarantee that moments computed using the weighted empirical approximations $\nu_n^\star$ or $\widetilde{\nu}_n^\star$ will be convergent in the $n \rightarrow \infty$ limit.
These remarks emphasise that a certain degree of caution is needed when optimal weights are employed.

\section{Optimal Thinning for MCMC}
\label{sec: optimal thinning}

The output from a sampling algorithm is often used for subsequent computation, for example, to approximate the posterior expectation of a quantity of interest.
In scenarios where this subsequent computation incurs a non-trivial computational cost, it is usually desirable to work with as small a number $n$ of samples as possible, provided that these continue to provide an accurate approximation to the posterior target.
Standard practice for MCMC is to retain the subset of states visited along the sample path whose indices are $(\sigma(i))_{1 \leq i \leq m}$, where $\sigma(i) = b + c i$, $b$ is the duration of a burn-in \index{burn-in} period and $c$ is the \emph{thinning} \index{thinning} period.
However, this does not directly attempt to arrive at a compressed representation of the posterior target.
The aim of this section is to discuss how one might select indices $(\sigma(i))_{1 \leq i \leq m}$ to optimally approximate the target. 
In doing so, we will also arrive at a convenient sparse approximation to the optimally weighted distributions studied in Section \ref{sec: optimal weights}.

Given output $(\mathbf{X}_k)_{1 \leq k \leq n}$ from a $\pi$-invariant MCMC algorithm, we aim to construct an approximation $\nu_{n,m} = \frac{1}{m}  \sum_{k=1}^m \delta_{\mathbf{X}_{\sigma(k)}}$ is sparse, meaning that $m \ll n$.
For concreteness, we consider the setting where the index sequence $\sigma$ is greedily determined according to
\begin{align*}
\sigma(j) \in \argmin_{k = 1,\dots,n} \; \mathcal{D}_{\mathsf{k}_\pi}\left( \frac{1}{j} \delta_{\mathbf{X}_k} + \frac{1}{j} \sum_{j'=1}^{j-1} \delta_{\mathbf{X}_{\sigma(j')}} \right) 
\end{align*}
for each $j \in \mathbb{N}$.
Using \eqref{eq: sd double integral}, and ignoring terms that do not depend on $\mathbf{X}_k$, the greedy algorithm \index{greedy algorithm} is equivalent to
\begin{align*}
\sigma(j) \in \argmin_{k = 1,\dots,n} \; \frac{\mathsf{k}_\pi(\mathbf{X}_k,\mathbf{X}_k)}{2} + \sum_{j'=1}^{j-1} \mathsf{k}_\pi( \mathbf{X}_k , \mathbf{X}_{\pi(j')} )
\end{align*}
where, in the event of a tie, it does not matter how a state is selected.
The approximation $\nu_{n,m}$, under appropriate regularity conditions, converges to $\nu_n^\star$ in the $m \rightarrow \infty$ limit; see Theorem 1 of \citet{riabiz2022optimal}.
Thus, a finite run of this greedy algorithm could in principle be used as a sparse alternative to optimal weighting \index{weight} of states from Section \ref{sec: optimal weights}.
Further, the computational complexity of this greedy algorithm is $O(nm^2)$, which would improve on the $O(n^3)$ of the optimal weights from Section \ref{sec: optimal weights} when $m \ll n$.
But how large should $m$ be for $\nu_{n,m}$ to be a sufficiently accurate approximation of $\nu_n^\star$ to be useful?
This question was answered with a theoretical argument in \citet{riabiz2022optimal}, who established that
$$
\mathcal{D}_{\mathsf{k}_\pi}(\nu_{n,m})   -  \mathcal{D}_{\mathsf{k}_\pi}(\nu_n^\star)  \rightarrow 0
$$ 
almost surely as $m,n \rightarrow \infty$, under conditions that include requiring $m$ to increase at least as fast as $(\log n)^{2/\beta}$ for some $\beta \in (0,1)$.
This is a relatively mild constraint on $m$, and in this sense the relative size of $m$ compared to $n$ can be small.

To perform an empirical comparison of optimal weighting and the greedy algorithm \index{greedy algorithm} just described, we return to the Rosenbrock \index{Rosenbrock} example from Section \ref{sec: optimal weights}.
Using the same MCMC output, we contrast the approximations produced using the weights $\mathbf{w}^\star$ with the approximation produced using the greedy algorithm just described.
Figure \ref{fig: opt thinning} demonstrates the convergence of the sparse approximation to the optimally weighted approximation as $m$ is increased.
This convergence occurs reasonably quickly, suggesting that a faster, sparse approximation may often be preferred compared to the weighted approximations from Section \ref{sec: optimal weights}.
For these experiments, we took $\mathbf{\Sigma} = \mathbf{I}_2$, $\beta = 0.5$, and the transformation in \eqref{eq: KSD transform} was applied.
Although in this toy example, sampling was not a computational bottleneck, in more challenging examples the use of these greedy algorithms can provide an automatic method to both identify and remove an initial burn-in \index{burn-in} period, and to compress the sampler output.

\begin{figure}
    \centering
    \includegraphics[width=\textwidth]{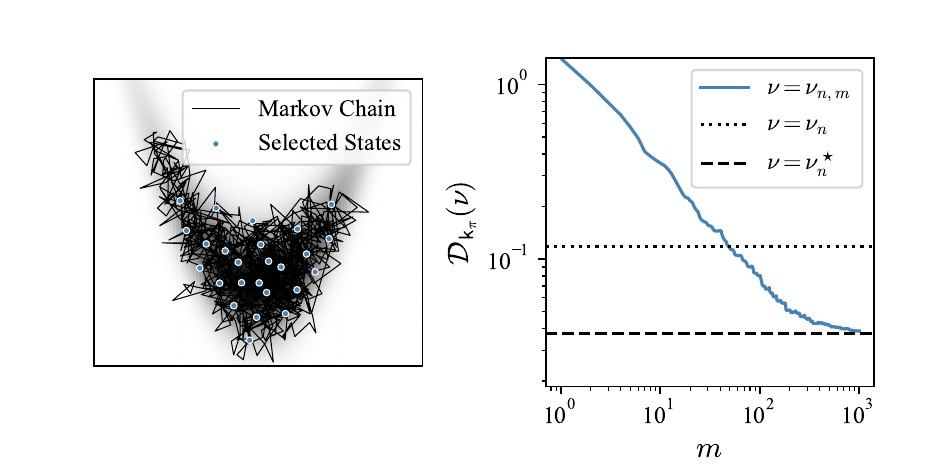}
    \caption{Optimal thinning for MCMC.
    A subset of size $m$ was selected from the sample path $(\mathbf{X}_k)_{1 \leq k \leq n}$ of a $\pi$-invariant Markov chain in such a way that the kernel Stein discrepancy between the associated empirical distribution (circles) and $\pi$ (shaded) was greedily minimised.
    The selected states are shown on the left panel, while on the right panel the kernel Stein discrepancy of the resulting empirical distribution $\nu_{n,m}$ is seen to converge to that of the optimally weighted empirical distribution $\nu_n^\star$ that uses the full Markov chain output.
    Here $n = 10^3$.
    }
    \label{fig: opt thinning}
\end{figure}

\section{Chapter Notes}

The development of sophisticated sampling algorithms, including those described in this book, should be guided by a qualitative assessment of their empirical performance over a variety of realistic distributional targets.
The purpose of this chapter was to demonstrate how one can construct explicit upper bounds on the ``closeness'' of the sampler output to the target.
In particular, kernel Stein discrepancies emerged as a computationally convenient performance measure, which can be computed provided that the gradient of the target log-density can be evaluated pointwise.
Except for scenarios where the posterior contains distant high-probability regions, the kernel Stein discrepancy can provide an accurate indication of sampler performance.
Further, we described two different scenarios in which output from MCMC can be actively improved using these techniques; optimal weighting and optimal thinning of MCMC output.

The literature on diagnostic checks for MCMC is almost as old as the literature on MCMC.
Our brief discussion in Section \ref{sec: diagnostics} barely scratched the surface of this topic, and we refer the interested reader to more detailed treatments such as \citet{cowles1996markov}.
The convergence diagnostics \index{convergence diagnostic} we presented are due to \citet{gelman1992, brooks1998general, gelman2014}.
The somewhat arbitrary choices of $\delta = 0.1$ and $\delta = 0.01$ were used, respectively, in \citet{gelman2014,vats2018revisiting} and \citet{vehtari2021rank}. 
Generalisations of these convergence diagnostics to the case of a multivariate target, and other improvements, can be found in e.g. \citet{brooks1998general,vats2018revisiting,vehtari2021rank}.

The construction of computable convergence bounds has received limited historical attention, from authors that include \citet{meyn1994computable, rosenthal1995minorization, roberts1999bounds, jones2001honest}.
The convergence bound we presented in Section \ref{sec: convergence bound} was somewhat novel, in the sense that earlier work has tended to motivate and derive such bounds as a consequence of \emph{Stein's method}\index{Stein's method}.
Introduced in \citet{stein1972bound}, this technique from applied probability has been extensively used to study various instances of approximation among random variables.
However, the last decade has seen an explosion of research investigating the \emph{computational} uses of Stein's method\index{Stein's method}, sparked by the formalisation of the Stein discrepancy in \citet{gorham2015measuring}.
A myriad of computational applications of Stein discrepancies have now been explored, and a recent overview is provided in \citet{anastasiou2021stein}.

The use of reproducing kernels led us to a discrepancy that could be explicitly computed.
However, there are some technical challenges associated with the use of the resulting kernel Stein discrepancies.
First, the computational complexity of evaluating the kernel Stein discrepancy between $\pi$ and a distribution $\nu_n$ supported on $n$ discrete states is $O(n^2)$; c.f. \eqref{eq: sd double sum}.
However, this complexity can in fact be reduced to near-linear using the \emph{random features} \index{random features} approach developed in \citet{huggins2018random}, whose discussion was beyond the scope of this book.
Second, one must ensure that the required properties of the discrepancy hold in the relevant applied context.
Our discussion focused on weak convergence \index{weak convergence} control\index{convergence control}, but other relevant properties include \emph{separation}\index{separation}, meaning that $D_\pi(\nu) = 0$ if and only if $\pi=\nu$, and \emph{convergence detection}\index{convergence detection}, meaning that $D_\pi(\nu_n) \rightarrow 0$ whenever $\nu_n$ converges to $\pi$ in an appropriate sense to be specified.
A discrepancy for which both convergence control and convergence detection are satisfied may be used to compare and select between competing sampling algorithms, as investigated in \citet{gorham2015measuring,gorham2017measuring}.
To this end, a rigorous technical presentation of kernel Stein discrepancies and their theoretical properties can be found in \citet{barp2022targeted}.

The use of isotropic reproducing kernels can lead to a \emph{curse of dimension}\index{curse of dimension}, meaning that differences between probability distributions become more difficult to detect as the dimension $d$ of the state space is increased.
A generalisation that replaces the overdamped Langevin \index{Langevin diffusion} in \eqref{eq: overdamped lang diff} with a more general class of $\pi$-invariant diffusion processes on $\mathbb{R}^d$ was studied in \citet{gorham2019measuring}, and in the case of kernel Stein discrepancy, this is equivalent to the use of certain non-isotropic reproducing kernels, however, the selection of a suitable diffusion to address the curse of dimension has not been explored.
In a constructive attempt to improve the performance of Stein discrepancy in the high-dimensional context, \citet{grathwohl2020learning} proposed to substitute the reproducing kernel Hilbert space in the integral probability metric \index{integral probability metric} \eqref{eq: ipm} with the set of test functions spanned by an appropriately differentiable parametric neural network.
Such an approach trades the potentially better detection properties of the discrepancy with both a lack of theoretical guarantees and the additional computational complexity involved in the adversarial training of a neural network.
Further research will be required to understand this trade-off in detail.

To limit the scope, we discussed only algorithms for optimal weighting and optimal thinning, in each case for probability distributions defined on $\mathbb{R}^d$.
Optimal weighting was introduced in \citet{liu2017black} and its consistency was first established in \citet{hodgkinson2020reproducing}.
Optimal thinning was introduced and analysed in \citet{riabiz2022optimal}, and mini-batching strategies were proposed and studied to further reduce the $O(nm^2)$ cost in \citet{teymur2021optimal}.
Both algorithms can be generalised to non-Euclidean domains $\mathcal{X}$ through the identification of a suitable Markov process on $\mathcal{X}$ with an explicit generator $\mathcal{L}_\pi$; some Markov processes suitable for discrete domains $\mathcal{X}$ are described in e.g. \citet{shi2022gradient}.
In related work, \citet{fisher2022gradient} demonstrated how consistent approximation using optimal weights and optimal thinning can even be achieved \emph{without} access to gradients of the target, provided that gradients of a suitable approximating distribution can be obtained.


  \backmatter
  

  \addtocontents{toc}{\vspace{\baselineskip}}

  \bibliographystyle{abbrvnat}
  \bibliography{references}\label{refs}

@book{bernardo2009bayesian,
  title={Bayesian Theory},
  author={Bernardo, Jos{\'e} M and Smith, Adrian FM},
  year={2009},
  publisher={John Wiley \& Sons}
}

@book{meyn2012markov,
  title={Markov Chains and Stochastic Stability},
  author={Meyn, Sean P and Tweedie, Richard L},
  year={2012},
  publisher={Springer Science \& Business Media}
}

@misc{johndrow2020no,
  title={No free lunch for approximate MCMC},
  author={Johndrow, James E and Pillai, Natesh S and Smith, Aaron},
  note={arXiv:2010.12514},
  year={2020}
}

@book{ripley2009stochastic,
  title={Stochastic Simulation},
  author={Ripley, Brian D},
  year={2009},
  publisher={John Wiley \& Sons}
}

@book{rubinstein2008,
title={Simulation and the {M}onte {C}arlo Method},
author={Rubinstein, R Y and Kroese, D P},
year=2008,
  publisher={John Wiley \& Sons}
}

@book{robert1999monte,
  title={Monte {C}arlo Statistical Methods},
  author={Robert, Christian P and Casella, George},
  year=1999,
  publisher={Springer}
}

@book{gamerman2006markov,
  title={{Markov chain Monte Carlo: stochastic simulation for Bayesian inference}},
  author={Gamerman, Dani and Lopes, Hedibert F},
  year={2006},
  publisher={CRC press}
}

@book{brooks2011handbook,
  title={Handbook of {M}arkov chain {M}onte {C}arlo},
  author={Brooks, Steve and Gelman, Andrew and Jones, Galin and Meng, Xiao-Li},
  year={2011},
  publisher={CRC press}
}

@book{robert2007bayesian,
  title={The {B}ayesian Choice: from Decision-Theoretic Foundations to Computational Implementation},
  author={Robert, Christian P },
  year={2007},
  publisher={Springer}
}

@article{caflisch1998monte,
  title={Monte {C}arlo and Quasi-{M}onte {C}arlo {M}ethods},
  author={Caflisch, Russel E},
  journal={Acta Numerica},
  volume={7},
  pages={1--49},
  year={1998},
  publisher={Cambridge University Press}
}

@incollection{l2002recent,
  title={Recent advances in randomized quasi-{M}onte {C}arlo methods},
  author={L’Ecuyer, Pierre and Lemieux, Christiane},
  booktitle={Modeling Uncertainty: An Examination of Stochastic Theory, Methods, and Applications},
 editor={Dror, Moshe and L'Ecuyer, Pierre and Szidarovszky, Ferenc},
  year={2002},
  publisher={Springer}
}

@book{norris1998markov,
  title={Markov Chains},
  author={Norris, James R},
  year={1998},
  publisher={Cambridge University Press}
}

@book{oksendal2013stochastic,
  title={Stochastic Differential Equations: An Introduction with Applications},
  author={Oksendal, Bernt},
  year={2013},
  publisher={Springer Science \& Business Media}
}

@book{rogers2000diffusions,
  title={Diffusions, {M}arkov Processes, and Martingales: Volume 1, Foundations},
  author={Rogers, Leonard CG and Williams, David},
  year={2000},
  publisher={Cambridge University Press}
}

@book{rogers2000diffusions2,
  title={Diffusions, {M}arkov Processes, and Martingales: Volume 2, {I}to Calculus},
  author={Rogers, Leonard CG and Williams, David},
  year={2000},
  publisher={Cambridge University Press}
}

@article{gorham2020stochastic,
  title={Stochastic {S}tein discrepancies},
  author={Gorham, Jackson and Raj, Anant and Mackey, Lester},
  journal={Advances in Neural Information Processing Systems},
  volume={33},
  pages={17931--17942},
  year={2020}
}

@misc{chevallier2021pdmp,
  title={{PDMP Monte Carlo methods for piecewise-smooth densities}},
  author={Chevallier, Augustin and Power, Sam and Wang, Andi Q and Fearnhead, Paul},
  note={arXiv:2111.05859},
  year={2021}
}

@article{bierkens2018piecewise,
  title={{Piecewise deterministic Markov processes for scalable Monte Carlo on restricted domains}},
  author={Bierkens, Joris and Bouchard-C{\^o}t{\'e}, Alexandre and Doucet, Arnaud and Duncan, Andrew B and Fearnhead, Paul and Lienart, Thibaut and Roberts, Gareth and Vollmer, Sebastian J},
  journal={Statistics \& Probability Letters},
  volume={136},
  pages={148--154},
  year={2018},
  publisher={Elsevier}
}

@article{brooks1998general,
  title={General methods for monitoring convergence of iterative simulations},
  author={Brooks, Stephen P and Gelman, Andrew},
  journal={Journal of computational and graphical statistics},
  volume={7},
  number={4},
  pages={434--455},
  year={1998},
  publisher={Taylor \& Francis Group}
}

@inproceedings{gong2020sliced,
  title={Sliced Kernelized Stein Discrepancy},
  author={Gong, Wenbo and Li, Yingzhen and Hern{\'a}ndez-Lobato, Jos{\'e} Miguel},
  booktitle={International Conference on Learning Representations},
  year={2020}
}

@article{hodgkinson2020reproducing,
  title={The reproducing Stein kernel approach for post-hoc corrected sampling},
  author={Hodgkinson, Liam and Salomone, Robert and Roosta, Fred},
  journal={arXiv preprint arXiv:2001.09266},
  year={2020}
}

@inproceedings{stein1972bound,
  title={A bound for the error in the normal approximation to the distribution of a sum of dependent random variables},
  author={Stein, Charles},
  booktitle={Proceedings of the 6th Berkeley Symposium on Mathematical Statistics and Probability, Volume 2: Probability Theory},
  volume={6},
  pages={583--603},
  year={1972},
  organization={University of California Press}
}

@article{fisher2022gradient,
  title={Gradient-free kernel {S}tein discrepancy},
  author={Fisher, Matthew and Oates, Chris J},
  journal={Advances in Neural Information Processing Systems},
  volume={36},
  year={2024}
}

@article{huggins2018random,
  title={Random feature Stein discrepancies},
  author={Huggins, Jonathan and Mackey, Lester},
  journal={Advances in Neural Information Processing Systems},
  volume={31},
  year={2018}
}

@inproceedings{grathwohl2020learning,
  title={Learning the stein discrepancy for training and evaluating energy-based models without sampling},
  author={Grathwohl, Will and Wang, Kuan-Chieh and Jacobsen, J{\"o}rn-Henrik and Duvenaud, David and Zemel, Richard},
  booktitle={International Conference on Machine Learning},
  pages={3732--3747},
  year={2020},
  organization={PMLR}
}

@article{shi2022gradient,
  title={Gradient estimation with discrete Stein operators},
  author={Shi, Jiaxin and Zhou, Yuhao and Hwang, Jessica and Titsias, Michalis and Mackey, Lester},
  journal={Advances in neural information processing systems},
  volume={35},
  pages={25829--25841},
  year={2022}
}

@article{cowles1996markov,
  title={Markov chain Monte Carlo convergence diagnostics: a comparative review},
  author={Cowles, Mary Kathryn and Carlin, Bradley P},
  journal={Journal of the American Statistical Association},
  volume={91},
  number={434},
  pages={883--904},
  year={1996},
  publisher={Taylor \& Francis}
}

@inproceedings{wenliang2020blindness,
  title={Blindness of score-based methods to isolated components and mixing proportions},
  author={Wenliang, Li K and Kanagawa, Heishiro},
  booktitle={Proceedings of the NeurIPS Workshop ``Your Model is Wrong: Robustness and Misspecification in Probabilistic Modeling'''},
  year={2021}
}

@inproceedings{liu2017black,
  title={Black-box importance sampling},
  author={Liu, Qiang and Lee, Jason},
  booktitle={Artificial Intelligence and Statistics},
  pages={952--961},
  year={2017},
  organization={PMLR}
}

@inproceedings{teymur2021optimal,
  title={Optimal quantisation of probability measures using maximum mean discrepancy},
  author={Teymur, Onur and Gorham, Jackson and Riabiz, Marina and Oates, Chris},
  booktitle={International Conference on Artificial Intelligence and Statistics},
  pages={1027--1035},
  year={2021},
  organization={PMLR}
}

@article{karvonen2018bayes,
  title={A Bayes-Sard cubature method},
  author={Karvonen, Toni and Oates, Chris J and Sarkka, Simo},
  journal={Advances in Neural Information Processing Systems},
  volume={31},
  year={2018}
}

@inproceedings{chwialkowski2016kernel,
  title={A kernel test of goodness of fit},
  author={Chwialkowski, Kacper and Strathmann, Heiko and Gretton, Arthur},
  booktitle={International conference on machine learning},
  pages={2606--2615},
  year={2016},
  organization={PMLR}
}

@article{oates2017control,
  title={Control functionals for Monte Carlo integration},
  author={Oates, Chris J and Girolami, Mark and Chopin, Nicolas},
  journal={Journal of the Royal Statistical Society: Series B (Statistical Methodology)},
  volume={79},
  number={3},
  pages={695--718},
  year={2017},
  publisher={Wiley Online Library}
}

@inproceedings{liu2016kernelized,
  title={A kernelized Stein discrepancy for goodness-of-fit tests},
  author={Liu, Qiang and Lee, Jason and Jordan, Michael},
  booktitle={International Conference on Machine Learning},
  pages={276--284},
  year={2016},
  organization={PMLR}
}

@article{riabiz2022optimal,
  title={Optimal thinning of {MCMC} output},
  author={Riabiz, Marina and Chen, Wilson Ye and Cockayne, Jon and Swietach, Pawel and Niederer, Steven A and Mackey, Lester and Oates, Chris J},
  journal={Journal of the Royal Statistical Society Series B: Statistical Methodology},
  volume={84},
  number={4},
  pages={1059--1081},
  year={2022},
  publisher={Oxford University Press}
}

@article{anastasiou2021stein,
  title={Stein’s method meets computational statistics: {A} review of some recent developments},
  author={Anastasiou, Andreas and Barp, Alessandro and Briol, Fran{\c{c}}ois-Xavier and Ebner, Bruno and Gaunt, Robert E and Ghaderinezhad, Fatemeh and Gorham, Jackson and Gretton, Arthur and Ley, Christophe and Liu, Qiang and Mackey, Lester and Oates, Chris J. and Reinert, Gesine and Swan, Yvik},
  journal={Statistical Science},
  volume={38},
  number={1},
  pages={120--139},
  year={2023},
  publisher={Institute of Mathematical Statistics}
}

@article{kanagawa2023controlling,
  title={Controlling Moments with Kernel Stein Discrepancies},
  author={Heishiro Kanagawa and Alessandro Barp and Carl-Johann Simon-Gabriel and Arthur Gretton and Lester Mackey},
  journal={arXiv preprint arXiv:2211.05408v4},
  year={2024}
}

@article{lindvall1986coupling,
  title={Coupling of multidimensional diffusions by reflection},
  author={Lindvall, Torgny and Rogers, L Cris G},
  journal={The Annals of Probability},
  pages={860--872},
  year={1986},
  publisher={JSTOR}
}

@inproceedings{chen2019stein,
  title={Stein point {M}arkov chain {M}onte {C}arlo},
  author={Chen, Wilson Ye and Barp, Alessandro and Briol, Fran{\c{c}}ois-Xavier and Gorham, Jackson and Girolami, Mark and Mackey, Lester and Oates, Chris},
  booktitle={International Conference on Machine Learning},
  pages={1011--1021},
  year={2019},
  organization={PMLR}
}

@article{carmeli2006vector,
  title={Vector valued reproducing kernel Hilbert spaces of integrable functions and Mercer theorem},
  author={Carmeli, Claudio and De Vito, Ernesto and Toigo, Alessandro},
  journal={Analysis and Applications},
  volume={4},
  number={04},
  pages={377--408},
  year={2006},
  publisher={World Scientific}
}

@article{von2005transport,
  title={Transport inequalities, gradient estimates, entropy and Ricci curvature},
  author={von Renesse, Max-K and Sturm, Karl-Theodor},
  journal={Communications on pure and applied mathematics},
  volume={58},
  number={7},
  pages={923--940},
  year={2005},
  publisher={Wiley Online Library}
}

@Article{jones2001honest,
  author    = {Jones, Galin L and Hobert, James P},
  journal   = {Statistical Science},
  title     = {Honest exploration of intractable probability distributions via {M}arkov chain {M}onte {C}arlo},
  year      = {2001},
  number    = {4},
  pages     = {312--334},
  volume    = {16},
  publisher = {JSTOR},
}

@Article{roberts1999bounds,
  author    = {Roberts, Gareth O and Tweedie, Richard L},
  journal   = {Stochastic Processes and Their Applications},
  title     = {Bounds on regeneration times and convergence rates for {M}arkov chains},
  year      = {1999},
  number    = {2},
  pages     = {211--229},
  volume    = {80},
  publisher = {Elsevier},
}

@Article{rosenthal1995minorization,
  author    = {Rosenthal, Jeffrey S},
  journal   = {Journal of the American Statistical Association},
  title     = {Minorization conditions and convergence rates for {M}arkov chain {M}onte {C}arlo},
  year      = {1995},
  number    = {430},
  pages     = {558--566},
  volume    = {90},
  publisher = {Taylor \& Francis Group},
}

@inproceedings{barp2022targeted,
  title={Targeted separation and convergence with kernel discrepancies},
  author={Barp, Alessandro and Simon-Gabriel, Carl-Johann and Girolami, Mark and Mackey, Lester},
  booktitle={NeurIPS 2022 Workshop on Score-Based Methods},
  year={2022}
}

@article{vehtari2021rank,
  title={Rank-normalization, folding, and localization: {A}n improved $\hat{R}$ for assessing convergence of {MCMC}},
  author={Vehtari, Aki and Gelman, Andrew and Simpson, Daniel and Carpenter, Bob and B{\"u}rkner, Paul-Christian},
  journal={Bayesian Analysis},
  volume={1},
  number={1},
  pages={1--28},
  year={2021},
  publisher={International Society for Bayesian Analysis}
}

@article{vats2018revisiting,
  title={Revisiting the {G}elman--{R}ubin diagnostic},
  author={Vats, Dootika and Knudson, Christina},
  journal={Statistical Science},
  volume={36},
  number={4},
  pages={518--529},
  year={2021},
  publisher={Institute of Mathematical Statistics}
}

@Article{gelman1992,
  author    = {Gelman, Andrew and Rubin, Donald B},
  journal   = {Statistical science},
  title     = {Inference from iterative simulation using multiple sequences},
  year      = {1992},
  number    = {4},
  pages     = {457--472},
  volume    = {7},
  publisher = {Institute of Mathematical Statistics},
}

@book{gelman2014,
  title={Bayesian Data Analysis},
  author={Gelman, Andrew and Carlin, John B and Stern, Hal S and Dunson, David B and Vehtari, Aki and Rubin, Donald B},
  year={2014},
  volume={2},
  publisher={CRC press}
}

@article{eberle2016reflection,
  title={Reflection couplings and contraction rates for diffusions},
  author={Eberle, Andreas},
  journal={Probability theory and related fields},
  volume={166},
  number={3},
  pages={851--886},
  year={2016},
  publisher={Springer}
}

@inproceedings{li2016scalable,
  title={{Scalable MCMC for mixed membership stochastic blockmodels}},
  author={Li, Wenzhe and Ahn, Sungjin and Welling, Max},
  booktitle={Artificial Intelligence and Statistics},
  pages={723--731},
  year={2016}
}

@inproceedings{ahn2015large,
  title={{Large-scale distributed Bayesian matrix factorization using stochastic gradient MCMC}},
  author={Ahn, Sungjin and Korattikara, Anoop and Liu, Nathan and Rajan, Suju and Welling, Max},
  booktitle={Proceedings of the 21th ACM SIGKDD international conference on knowledge discovery and data mining},
  pages={9--18},
  year={2015},
  organization={ACM}
}

@article{Teh:2016,
  title={Consistency and fluctuations for stochastic gradient Langevin dynamics},
  author={Teh, Yee Whye and Thiery, Alexandre H and Vollmer, Sebastian J},
  journal={The Journal of Machine Learning Research},
  volume={17},
  number={1},
  pages={193--225},
  year={2016},
  publisher={JMLR. org}
}

@article{Cheng:2017,
  title={{Underdamped Langevin MCMC: A non-asymptotic analysis}},
  author={Cheng, Xiang and Chatterji, Niladri S and Bartlett, Peter L and Jordan, Michael I},
  journal={arXiv preprint arXiv:1707.03663},
  year={2017}
}

@article{Beskos:2013,
  title={{Optimal tuning of the hybrid Monte Carlo algorithm}},
  author={Beskos, Alexandros and Pillai, Natesh and Roberts, Gareth and Sanz-Serna, Jesus-Maria and Stuart, Andrew},
  journal={Bernoulli},
  volume={19},
  number={5A},
  pages={1501--1534},
  year={2013},
  publisher={Bernoulli Society for Mathematical Statistics and Probability}
}

@InCollection{Neal:2011,
 author = {Neal, Radford M},
 title = {{MCMC using Hamiltonian dynamics}},
 booktitle = {{Handbook of Markov chain Monte Carlo}},
 pages = {113--162},
 publisher = {CRC Press},
 year = {2011},
 editor = {Brooks, Steve and Gelman, Andrew and Jones, Galin L and Meng, Xiao-Li}
}

@article{green2001delayed,
  title={Delayed rejection in reversible jump {Metropolis--Hastings}},
  author={Green, Peter J and Mira, Antonietta},
  journal={Biometrika},
  volume={88},
  number={4},
  pages={1035--1053},
  year={2001},
  publisher={Biometrika Trust}
}

@inproceedings{Welling:2011,
  title={{Bayesian learning via stochastic gradient Langevin dynamics}},
  author={Welling, Max and Teh, Yee W},
  booktitle={Proceedings of the 28th International Conference on Machine Learning (ICML-11)},
  pages={681--688},
  year={2011}
}

@inproceedings{Dubey:2016,
  title={{Variance reduction in stochastic gradient Langevin dynamics}},
  author={Dubey, Kumar Avinava and Reddi, Sashank J and Williamson, Sinead A and Poczos, Barnabas and Smola, Alexander J and Xing, Eric P},
  booktitle={Advances in Neural Information Processing Systems},
  pages={1154--1162},
  year={2016}
}

@article{Hastings:1970,
  title={{Monte Carlo sampling methods using Markov chains and their applications}},
  author={Hastings, W Keith},
  year={1970},
  journal={Biometrika},
  volume=57,
  pages={97--109},
  publisher={Oxford University Press}
}

@article{Roberts:1996,
  title={{Exponential convergence of Langevin distributions and their discrete approximations}},
  author={Roberts, Gareth O and Tweedie, Richard L},
  journal={Bernoulli},
  volume={2},
  number={4},
  pages={341--363},
  year={1996},
  publisher={Bernoulli Society for Mathematical Statistics and Probability}
}

@article{Roberts:1998,
  title={{Optimal scaling of discrete approximations to Langevin diffusions}},
  author={Roberts, Gareth O and Rosenthal, Jeffrey S},
  journal={Journal of the Royal Statistical Society: Series B (Statistical Methodology)},
  volume={60},
  number={1},
  pages={255--268},
  year={1998},
  publisher={Wiley Online Library}
}

@article{Roberts:2001,
  title={{Optimal scaling for various Metropolis-Hastings algorithms}},
  author={Roberts, Gareth O and Rosenthal, Jeffrey S},
  journal={Statistical science},
  volume={16},
  number={4},
  pages={351--367},
  year={2001},
  publisher={Institute of Mathematical Statistics}
}

@article{Baker:2017,
  title={Control variates for stochastic gradient MCMC},
  author={Baker, Jack and Fearnhead, Paul and Fox, Emily B and Nemeth, Christopher},
  journal={Statistics and Computing},
  volume={29},
  number={3},
  pages={599--615},
  year={2019},
  publisher={Springer}
}

@article{Bierkens:2019,
  title={The Zig-Zag process and super-efficient sampling for Bayesian analysis of big data},
  author={Bierkens, Joris and Fearnhead, Paul and Roberts, Gareth O},
  journal={Annals of statistics},
  volume={47},
  number={3},
  pages={1288--1320},
  year={2019},
  publisher={Inst Mathematical Statistics}
}

@article{Pollock:2016,
  title={Quasi-stationary Monte Carlo and the ScaLE algorithm},
  author={Pollock, Murray and Fearnhead, Paul and Johansen, Adam M and Roberts, Gareth O},
  journal={Journal of the Royal Statistical Society Series B: Statistical Methodology},
  volume={82},
  number={5},
  pages={1167--1221},
  year={2020},
  publisher={Oxford University Press}
}

@article{Fearnhead:2018,
  title={{Piecewise deterministic Markov processes for continuous-time Monte Carlo}},
  author={Fearnhead, Paul and Bierkens, Joris and Pollock, Murray and Roberts, Gareth O and others},
  journal={Statistical Science},
  volume={33},
  number={3},
  pages={386--412},
  year={2018},
  publisher={Institute of Mathematical Statistics}
}

@article{wu2020coordinate,
  title={{Coordinate sampler: a non-reversible Gibbs-like MCMC sampler}},
  author={Wu, Changye and Robert, Christian P},
  journal={Statistics and Computing},
  volume={30},
  number={3},
  pages={721--730},
  year={2020},
  publisher={Springer}
}

@article{Dalalyan:2017,
  title={User-friendly guarantees for the Langevin Monte Carlo with inaccurate gradient},
  author={Dalalyan, Arnak S and Karagulyan, Avetik},
  journal={Stochastic Processes and their Applications},
  volume={129},
  number={12},
  pages={5278--5311},
  year={2019},
  publisher={Elsevier}
}

@article{peters2012rejection,
  title={Rejection-free Monte Carlo sampling for general potentials},
  author={Peters, Elias A J F and {de With}, G},
  journal={Physical Review E},
  volume={85},
  number={2},
  pages={026703},
  year={2012},
  publisher={APS}
}

@misc{pagani2020nuzz,
  title={{NuZZ: numerical Zig-Zag sampling for general models}},
  author={Pagani, Filippo and Chevallier, Augustin and Power, Sam and House, Thomas and Cotter, Simon},
 note={arXiv:2003.03636},
  year={2020}
}

@misc{vanetti2017piecewise,
  title={{Piecewise-deterministic Markov chain Monte Carlo}},
  author={Vanetti, Paul and Bouchard-C{\^o}t{\'e}, Alexandre and Deligiannidis, George and Doucet, Arnaud},
  note={arXiv:1707.05296},
  year={2017}
}

@article{grenander1994representations,
  title={Representations of knowledge in complex systems},
  author={Grenander, Ulf and Miller, Michael I},
  journal={Journal of the Royal Statistical Society: Series B (Methodological)},
  volume={56},
  number={4},
  pages={549--581},
  year={1994},
  publisher={Wiley Online Library}
}

@article{bierkens2023sticky,
  title={Sticky {PDMP} samplers for sparse and local inference problems},
  author={Bierkens, Joris and Grazzi, Sebastiano and Meulen, Frank van der and Schauer, Moritz},
  journal={Statistics and Computing},
  volume={33},
  number={1},
  pages={8},
  year={2023},
  publisher={Springer}
}

@article{chevallier2022reversible,
  title={Reversible jump PDMP samplers for variable selection},
  author={Chevallier, Augustin and Fearnhead, Paul and Sutton, Matthew},
  journal={Journal of the American Statistical Association},
  volume={118},
  number={544},
  pages={2915--2927},
  year={2023},
  publisher={Taylor \& Francis}
}

@article{stephens2000bayesian,
  title={Bayesian analysis of mixture models with an unknown number of components-an alternative to reversible jump methods},
  author={Stephens, Matthew},
  journal={Annals of Statistics},
  pages={40--74},
  year={2000}
}

@incollection{phillips1996bayesian,
  title={Bayesian model comparison via jump diffusions},
  author={Phillips, David B and Smith, Adrian FM},
  booktitle={{Markov chain Monte Carlo in practice}},
  editor={Gilks, Wally R and Richardson, Sylvia and  Spiegelhalter, David},
  pages={215-240},
  year={1996},
  publisher={Chapman \& Hall, CRC}
}

@inproceedings{bierkens2020boomerang,
  title={The boomerang sampler},
  author={Bierkens, Joris and Grazzi, Sebastiano and Kamatani, Kengo and Roberts, Gareth},
  booktitle={International Conference on Machine Learning},
  pages={908--918},
  year={2020},
  organization={PMLR}
}

@book{press2007numerical,
  title={Numerical recipes in C++: The art of scientific computing},
  author={Press, William H and Teukolsky, Saul A and Vetterling, William T and Flannery, Brian P},
  year={2007},
 publisher={Cambridge University Press}
}

@article{corbella2022automatic,
  title={Automatic {Zig-Zag} sampling in practice},
  author={Corbella, Alice and Spencer, Simon EF and Roberts, Gareth O},
  journal={Statistics and Computing},
  volume={32},
  number={6},
  pages={107},
  year={2022},
  publisher={Springer}
}

@article{sutton2023concave,
  title={Concave-convex PDMP-based sampling},
  author={Sutton, Matthew and Fearnhead, Paul},
  journal={Journal of Computational and Graphical Statistics},
  volume={32},
  number={4},
  pages={1425--1435},
  year={2023},
  publisher={Taylor \& Francis}
}

@inproceedings{Chatterji:2018,
  title={On the theory of variance reduction for stochastic gradient Monte Carlo},
  author={Chatterji, Niladri and Flammarion, Nicolas and Ma, Yian and Bartlett, Peter and Jordan, Michael},
  booktitle={International Conference on Machine Learning},
  pages={764--773},
  year={2018},
  organization={PMLR}
}

@inproceedings{Brosse:2018,
  title={{The promises and pitfalls of Stochastic Gradient Langevin Dynamics}},
  author={Brosse, Nicolas and Durmus, Alain and Moulines, {\'E}ric},
  booktitle={Advances in Neural Information Processing Systems},
  pages={8278--8288},
  year={2018}
}

@misc{Nagapetyan:2017,
  title={The true cost of stochastic gradient Langevin dynamics},
  author={Nagapetyan, Tigran and Duncan, Andrew B and Hasenclever, Leonard and Vollmer, Sebastian J and Szpruch, Lukasz and Zygalakis, Konstantinos},
  howpublished={arXiv:1706.02692},
  year={2017}
}

@inproceedings{Huggins:2016,
  title={Quantifying the accuracy of approximate diffusions and Markov chains},
  author={Huggins, Jonathan and Zou, James},
  booktitle={Artificial Intelligence and Statistics},
  pages={382--391},
  year={2017},
  organization={PMLR}
}

@inproceedings{raginsky2017non,
  title={Non-convex learning via stochastic gradient langevin dynamics: a nonasymptotic analysis},
  author={Raginsky, Maxim and Rakhlin, Alexander and Telgarsky, Matus},
  booktitle={Conference on Learning Theory},
  pages={1674--1703},
  year={2017},
  organization={PMLR}
}

@article{majka2018non,
  title={Non-asymptotic bounds for sampling algorithms without log-concavity},
  author={Majka, Mateusz B and Mijatovi{\'c}, Aleksandar and Szpruch, {\L}ukasz},
  journal={The Annals of Applied Probability},
  volume={30},
  number={4},
  pages={1534--1581},
  year={2020}
}

@inproceedings{ma2015complete,
  title={A complete recipe for stochastic gradient MCMC},
  author={Ma, Yi-An and Chen, Tianqi and Fox, Emily},
  booktitle={Advances in Neural Information Processing Systems},
  pages={2917--2925},
  year={2015}
}

@inproceedings{chen2014stochastic,
  title={{Stochastic gradient Hamiltonian Monte Carlo}},
  author={Chen, Tianqi and Fox, Emily and Guestrin, Carlos},
  booktitle={International Conference on Machine Learning},
  pages={1683--1691},
  year={2014}
}

@inproceedings{patterson2013stochastic,
  title={Stochastic gradient Riemannian Langevin dynamics on the probability simplex},
  author={Patterson, Sam and Teh, Yee Whye},
  booktitle={Advances in Neural Information Processing Systems},
  pages={3102--3110},
  year={2013}
}

@inproceedings{brosse2017sampling,
  title={Sampling from a log-concave distribution with compact support with proximal Langevin Monte Carlo},
  author={Brosse, Nicolas and Durmus, Alain and Moulines, {\'E}ric and Pereyra, Marcelo},
  booktitle={Conference on learning theory},
  pages={319--342},
  year={2017},
  organization={PMLR}
}

@article{bubeck2018sampling,
  title={Sampling from a log-concave distribution with Projected Langevin Monte Carlo},
  author={Bubeck, S{\'e}bastien and Eldan, Ronen and Lehec, Joseph},
  journal={Discrete \& Computational Geometry},
  volume={59},
  number={4},
  pages={757--783},
  year={2018},
  publisher={Springer}
}

@inproceedings{hsieh2018mirrored,
  title={Mirrored Langevin Dynamics},
  author={Hsieh, Ya-Ping and Kavis, Ali and Rolland, Paul and Cevher, Volkan},
  booktitle={Advances in Neural Information Processing Systems},
  pages={2883--2892},
  year={2018}
}

@article{beck2003mirror,
  title={Mirror descent and nonlinear projected subgradient methods for convex optimization},
  author={Beck, Amir and Teboulle, Marc},
  journal={Operations Research Letters},
  volume={31},
  number={3},
  pages={167--175},
  year={2003},
  publisher={Elsevier}
}

@article{blei2003latent,
  title={Latent dirichlet allocation},
  author={Blei, David M and Ng, Andrew Y and Jordan, Michael I},
  journal={Journal of machine Learning research},
  volume={3},
  number={Jan},
  pages={993--1022},
  year={2003}
}

@inproceedings{baker2018large,
  title={Large-Scale Stochastic Sampling from the Probability Simplex},
  author={Baker, Jack and Fearnhead, Paul and Fox, Emily and Nemeth, Christopher},
  booktitle={Advances in Neural Information Processing Systems},
  pages={6721--6731},
  year={2018}
}

@article{cox1985theory,
  title={A Theory of the Term Structure of Interest Rates},
  author={Cox, John and Ingersoll Jr, Jonathan E and Ross, Stephen A},
  journal={Econometrica},
  volume={53},
  number={2},
  pages={385--408},
  year={1985}
}

@inproceedings{ma2017stochastic,
  title={Stochastic gradient MCMC methods for hidden Markov models},
  author={Ma, Yi-An and Foti, Nicholas J and Fox, Emily B},
  booktitle={International Conference on Machine Learning},
  pages={2265--2274},
  year={2017},
  organization={PMLR}
}

@article{aicher2019stochastic,
  title={Stochastic gradient MCMC for state space models},
  author={Aicher, Christopher and Ma, Yi-An and Foti, Nicholas J and Fox, Emily B},
  journal={SIAM Journal on Mathematics of Data Science},
  volume={1},
  number={3},
  pages={555--587},
  year={2019},
  publisher={SIAM}
}

@inproceedings{gorham2017measuring,
  title={Measuring sample quality with kernels},
  author={Gorham, Jackson and Mackey, Lester},
  booktitle={Proceedings of the 34th International Conference on Machine Learning},
  pages={1292--1301},
  year={2017},
  organization={PMLR}
}

@article{gorham2019measuring,
  title={Measuring sample quality with diffusions},
  author={Gorham, Jackson and Duncan, Andrew B and Vollmer, Sebastian J and Mackey, Lester},
  journal={The Annals of Applied Probability},
  volume={29},
  number={5},
  pages={2884--2928},
  year={2019},
  publisher={JSTOR}
}

@inproceedings{gorham2015measuring,
  title={Measuring sample quality with Stein's method},
  author={Gorham, Jackson and Mackey, Lester},
  booktitle={Advances in Neural Information Processing Systems},
  pages={226--234},
  year={2015}
}

@article{hoffman2014no,
  title={The No-U-Turn sampler: adaptively setting path lengths in Hamiltonian Monte Carlo.},
  author={Hoffman, Matthew D and Gelman, Andrew},
  journal={Journal of Machine Learning Research},
  volume={15},
  number={1},
  pages={1593--1623},
  year={2014}
}

@article{aicher2023stochastic,
  title={Stochastic gradient MCMC for nonlinear state space models},
  author={Aicher, Christopher and Putcha, Srshti and Nemeth, Christopher and Fearnhead, Paul and Fox, Emily},
  journal={Bayesian Analysis},
  volume={1},
  number={1},
  pages={1--23},
  year={2023},
  publisher={International Society for Bayesian Analysis}
}

@article{vasdekis2023speed,
  title={Speed up {Zig-Zag}},
  author={Vasdekis, Giorgos and Roberts, Gareth O},
  journal={The Annals of Applied Probability},
  volume={33},
  number={6A},
  pages={4693--4746},
  year={2023},
  publisher={Institute of Mathematical Statistics}
}

@article{vasdekis2022note,
  title={A note on the polynomial ergodicity of the one-dimensional {Zig-Zag} process},
  author={Vasdekis, Giorgos and Roberts, Gareth O},
  journal={Journal of Applied Probability},
  volume={59},
  number={3},
  pages={895--903},
  year={2022},
  publisher={Cambridge University Press}
}

@article{andrieu2021subgeometric,
  title={Subgeometric hypocoercivity for piecewise-deterministic {Markov process Monte Carlo methods}},
  author={Andrieu, Christophe and Dobson, Paul and Wang, Andi Q},
  journal={Electronic Journal of Probability},
  volume={26},
  pages={1--26},
  year={2021},
  publisher={The Institute of Mathematical Statistics and the Bernoulli Society}
}

@article{bertazzi2022adaptive,
  title={Adaptive schemes for piecewise deterministic {Monte Carlo} algorithms},
  author={Bertazzi, Andrea and Bierkens, Joris},
  journal={Bernoulli},
  volume={28},
  number={4},
  pages={2404--2430},
  year={2022},
  publisher={Bernoulli Society for Mathematical Statistics and Probability}
}

@article{scott2016bayes,
  title={Bayes and big data: The consensus Monte Carlo algorithm},
  author={Scott, Steven L and Blocker, Alexander W and Bonassi, Fernando V and Chipman, Hugh A and George, Edward I and McCulloch, Robert E},
  journal={International Journal of Management Science and Engineering Management},
  volume={11},
  number={2},
  pages={78--88},
  year={2016},
  publisher={Taylor \& Francis}
}

@article{bierkens2019ergodicity,
  title={Ergodicity of the zigzag process},
  author={Bierkens, Joris and Roberts, Gareth O and Zitt, Pierre-Andr{\'e}},
  journal={The Annals of Applied Probability},
  volume={29},
  number={4},
  pages={2266--2301},
  year={2019},
  publisher={JSTOR}
}

@article{deligiannidis2021randomized,
  title={Randomized {Hamiltonian Monte Carlo} as scaling limit of the bouncy particle sampler and dimension-free convergence rates},
  author={Deligiannidis, George and Paulin, Daniel and Bouchard-C{\^o}t{\'e}, Alexandre and Doucet, Arnaud},
  journal={The Annals of Applied Probability},
  volume={31},
  number={6},
  pages={2612--2662},
  year={2021},
  publisher={Institute of Mathematical Statistics}
}

@inproceedings{bierkens2022spectral,
  title={Spectral analysis of the zigzag process},
  author={Bierkens, Joris and Verduyn Lunel, Sjoerd M},
  booktitle={Annales de l'Institut Henri Poincare (B) Probabilites et statistiques},
  volume={58},
  number={2},
  pages={827--860},
  year={2022},
  organization={Institut Henri Poincar{\'e}}
}

@article{bierkens2017limit,
  title={Limit theorems for the zig-zag process},
  author={Bierkens, Joris and Duncan, Andrew},
  journal={Advances in Applied Probability},
  volume={49},
  number={3},
  pages={791--825},
  year={2017},
  publisher={Cambridge University Press}
}

@article{bierkens2022high,
  title={High-dimensional scaling limits of piecewise deterministic sampling algorithms},
  author={Bierkens, Joris and Kamatani, Kengo and Roberts, Gareth O},
  journal={The Annals of Applied Probability},
  volume={32},
  number={5},
  pages={3361--3407},
  year={2022},
  publisher={Institute of Mathematical Statistics}
}

@article{andrieu2021hypocoercivity,
  title={Hypocoercivity of piecewise deterministic {Markov process-Monte Carlo}},
  author={Andrieu, Christophe and Durmus, Alain and N{\"u}sken, Nikolas and Roussel, Julien},
  journal={The Annals of Applied Probability},
  volume={31},
  number={5},
  pages={2478--2517},
  year={2021},
  publisher={Institute of Mathematical Statistics}
}

@article{deligiannidis2019exponential,
  title={Exponential ergodicity of the bouncy particle sampler},
  author={Deligiannidis, George and Bouchard-C{\^o}t{\'e}, Alexandre and Doucet, Arnaud},
  year={2019},
journal={The Annals of Statistics},
volume={47},
pages={1268–1287}
}

@article{nemeth2018merging,
  title={Merging MCMC subposteriors through Gaussian-process approximations},
  author={Nemeth, Christopher and Sherlock, Chris },
  journal={Bayesian Analysis},
  volume={13},
  number={2},
  pages={507--530},
  year={2018},
  publisher={International Society for Bayesian Analysis}
}

@article{sherlock2022discrete,
  title={A discrete bouncy particle sampler},
  author={Sherlock, Chris and Thiery, Alexandre H},
  journal={Biometrika},
  volume={109},
  number={2},
  pages={335--349},
  year={2022},
  publisher={Oxford University Press}
}

@misc{bierkens2023scaling,
      title={Scaling of Piecewise Deterministic Monte Carlo for Anisotropic Targets}, 
      author={Joris Bierkens and Kengo Kamatani and Gareth O. Roberts},
      year={2023},
      eprint={2305.00694},
      archivePrefix={arXiv},
      primaryClass={stat.ME}
}

@article{bouchard2018bouncy,
  title={The bouncy particle sampler: A nonreversible rejection-free Markov chain Monte Carlo method},
  author={Bouchard-C{\^o}t{\'e}, Alexandre and Vollmer, Sebastian J and Doucet, Arnaud},
  journal={Journal of the American Statistical Association},
  volume={113},
  number={522},
  pages={855--867},
  year={2018},
  publisher={Taylor \& Francis}
}

@article{davis1984piecewise,
  title={Piecewise-deterministic {M}arkov processes: {A} general class of non-diffusion stochastic models},
  author={Davis, Mark H A},
  journal={Journal of the Royal Statistical Society: Series B (Methodological)},
  volume={46},
  number={3},
  pages={353--376},
  year={1984},
  publisher={Wiley Online Library}
}

@article{lewis1979simulation,
  title={Simulation of nonhomogeneous {P}oisson processes by thinning},
  author={Lewis, PA W and Shedler, Gerald S},
  journal={Naval Research Logistics Quarterly},
  volume={26},
  number={3},
  pages={403--413},
  year={1979},
  publisher={Wiley Online Library}
}

@inproceedings{bardenet2014towards,
  title={Towards scaling up Markov chain Monte Carlo: an adaptive subsampling approach},
  author={Bardenet, R{\'e}mi and Doucet, Arnaud and Holmes, Chris},
  booktitle={International Conference on Machine Learning (ICML)},
  pages={405--413},
  year={2014}
}

@article{metropolis1953equation,
  title={Equation of state calculations by fast computing machines},
  author={Metropolis, Nicholas and Rosenbluth, Arianna W and Rosenbluth, Marshall N and Teller, Augusta H and Teller, Edward},
  journal={The journal of chemical physics},
  volume={21},
  number={6},
  pages={1087--1092},
  year={1953},
  publisher={AIP}
}

@article{gustafson1998guided,
  title={A guided walk Metropolis algorithm},
  author={Gustafson, Paul},
  journal={Statistics and Computing},
  volume={8},
  number={4},
  pages={357--364},
  year={1998},
  publisher={Springer}
}

@article{girolami2011riemann,
  title={Riemann manifold {L}angevin and {H}amiltonian {M}onte {C}arlo methods},
  author={Girolami, Mark and Calderhead, Ben},
  journal={Journal of the Royal Statistical Society: Series B (Statistical Methodology)},
  volume={73},
  number={2},
  pages={123--214},
  year={2011},
  publisher={Wiley Online Library}
}

@article{meyn1994computable,
  title={Computable bounds for geometric convergence rates of Markov chains},
  author={Meyn, Sean P and Tweedie, Robert L and others},
  journal={The Annals of Applied Probability},
  volume={4},
  number={4},
  pages={981--1011},
  year={1994},
  publisher={Institute of Mathematical Statistics}
}

@article{roberts2004general,
  title={General state space Markov chains and MCMC algorithms},
  author={Roberts, Gareth O and Rosenthal, Jeffrey S},
  journal={Probability Surveys},
  volume={1},
  pages={20--71},
  year={2004},
  publisher={The Institute of Mathematical Statistics and the Bernoulli Society}
}

@article{robbins1951stochastic,
  title={A stochastic approximation method},
  author={Robbins, Herbert and Monro, Sutton},
  journal={The annals of mathematical statistics},
  pages={400--407},
  year={1951},
  publisher={JSTOR}
}

@article{bierkens2016non,
  title={Non-reversible {M}etropolis-{H}astings},
  author={Bierkens, Joris},
  journal={Statistics and Computing},
  volume={26},
  number={6},
  pages={1213--1228},
  year={2016},
  publisher={Springer}
}

@article{diaconis2000analysis,
  title={Analysis of a nonreversible {M}arkov chain sampler},
  author={Diaconis, Persi and Holmes, Susan and Neal, Radford M},
  journal={Annals of Applied Probability},
  pages={726--752},
  year={2000},
  volume=10,
  number=3
}

@inproceedings{chen1999lifting,
  title={Lifting {M}arkov chains to speed up mixing},
  author={Chen, Fang and Lov{\'a}sz, L{\'a}szl{\'o} and Pak, Igor},
  booktitle={Proceedings of the thirty-first annual ACM symposium on Theory of computing},
  pages={275--281},
  year={1999}
}

@article{neal2004improving,
  title={Improving asymptotic variance of {MCMC} estimators: {N}on-reversible chains are better},
  author={Neal, Radford M},
  journal={arXiv preprint math/0407281},
  year={2004}
}

@article{horowitz1991generalized,
  title={A generalized guided {M}onte {C}arlo algorithm},
  author={Horowitz, Alan M},
  journal={Physics Letters B},
  volume={268},
  number={2},
  pages={247--252},
  year={1991},
  publisher={Elsevier}
}

@inproceedings{sun2010improving,
  title={Improving the asymptotic performance of {Markov chain Monte-Carlo} by inserting vortices},
  author={Sun, Yi and Schmidhuber, J{\"u}rgen and Gomez, Faustino},
  booktitle={Advances in Neural Information Processing Systems},
  volume={23},
  editor={Lafferty, J. and Williams, C.K.I. and Shawe-Taylor, J. and Zemel, R.S. and Culotta, A.},
  pages={2235--2243},
  year={2010}
}

@article{suwa2010markov,
  title={Markov chain {Monte Carlo} method without detailed balance},
  author={Suwa, Hidemaro and Todo, Synge},
  journal={Physical Review Letters},
  volume={105},
  number={12},
  pages={120603},
  year={2010},
  publisher={APS}
}

@article{turitsyn2011irreversible,
  title={Irreversible {Monte Carlo} algorithms for efficient sampling},
  author={Turitsyn, Konstantin S and Chertkov, Michael and Vucelja, Marija},
  journal={Physica D: Nonlinear Phenomena},
  volume={240},
  number={4-5},
  pages={410--414},
  year={2011},
  publisher={Elsevier}
}

@article{Geyer1992,
author = {Charles J. Geyer},
title = {{Practical Markov chain Monte Carlo}},
volume = {7},
journal = {Statistical Science},
number = {4},
publisher = {Institute of Mathematical Statistics},
pages = {473 -- 483},
year = {1992},
doi = {10.1214/ss/1177011137},
URL = {https://doi.org/10.1214/ss/1177011137}
}

@article{michel2014generalized,
  title={Generalized event-chain {Monte Carlo}: Constructing rejection-free global-balance algorithms from infinitesimal steps},
  author={Michel, Manon and Kapfer, Sebastian C and Krauth, Werner},
  journal={The Journal of Chemical Physics},
  volume={140},
  number={5},
  year={2014},
  publisher={AIP Publishing}
}

@article{bierkens2017piecewise,
  title={A piecewise deterministic scaling limit of lifted {Metropolis--Hastings in the Curie--Weiss} model},
  author={Bierkens, Joris and Roberts, Gareth},
journal={The Annals of Applied Probability},
volume={27},
pages={846--882},
  year={2017}
}

@article{neal2003slice,
  title={Slice sampling},
  author={Neal, Radford M},
  journal={The Annals of Statistics},
  volume={31},
  number={3},
  pages={705--767},
  year={2003},
  publisher={Institute of Mathematical Statistics}
}

@misc{wu2017generalized,
  title={Generalized bouncy particle sampler},
  author={Wu, Changye and Robert, Christian P},
  note={arXiv:1706.04781},
  year={2017}
}

@article{michel2020forward,
  title={Forward event-chain {Monte Carlo}: Fast sampling by randomness control in irreversible {M}arkov chains},
  author={Michel, Manon and Durmus, Alain and S{\'e}n{\'e}cal, St{\'e}phane},
  journal={Journal of Computational and Graphical Statistics},
  volume={29},
  number={4},
  pages={689--702},
  year={2020},
  publisher={Taylor \& Francis}
}

@article{besag1994,
    author = {Besag, Julian},
    title = {Comments on "Representations of knowledge in complex systems" by U. Grenander and M.I. Miller},
    journal = {Journal of the Royal Statistical Society Series B},
    year = {1994},
    volume = {56},
    pages = {591-592}
}

@article{BeskosRobertsStuart:2009,
author={Alex Beskos and Gareth Roberts and Andrew Stuart},
title={Optimal scalings for local {M}etropolis-{H}astings chains on non-product targets in high dimensions.},
journal={Annals of Applied Probability},
volume={19},
number={3},
pages={863--898},
year={2009}
}

@article{RobertsGelmanGilks1997,
author={Roberts, Gareth O. and Gelman, Andrew and Gilks, Walter R.},
title={Weak convergence and optimal scaling of random walk {M}etropolis algorithms},
journal={The Annals of Applied Probability},
year=1997,
volume=7,
pages={110-120}}

@article{Sherlock/Roberts:2009,
    AUTHOR = {Chris Sherlock and Gareth Roberts},
     TITLE = {Optimal scaling of the random walk {M}etropolis on elliptically symmetric unimodal targets},
   JOURNAL = {Bernoulli},
  FJOURNAL = {Bernoulli},
      YEAR = 2009,
    VOLUME = 15,
    NUMBER = 3,
     PAGES = {774-798},
      ISSN = {1350-7265},
       DOI = {10.3150/08-BEJ176},
     ARXIV = {0909.0856},
      SICI = {1350-7265(2009)15:3<774:OSOTRW>2.0.CO;2-Z},
}

@article{Sherlock:2015,
    author={Sherlock, C. and Thiery, A. H. and Roberts, G. O. and Rosenthal, J. R.},
    TITLE = {On the efficiency of pseudo-marginal random walk {M}etropolis algorithms},
      YEAR = {2015},
      VOLUME  = {43},
      NUMBER  = 1,
     PAGES = {238--275},
     journal={Annals of Statistics},
}

@article{ChrRobRos2005,
 ISSN = {13697412, 14679868},
 URL = {http://www.jstor.org/stable/3647577},
 author = {Ole F. Christensen and Gareth O. Roberts and Jeffrey S. Rosenthal},
 journal = {Journal of the Royal Statistical Society. Series B (Statistical Methodology)},
 number = {2},
 pages = {253--268},
 publisher = {[Royal Statistical Society, Wiley]},
 title = {Scaling Limits for the Transient Phase of Local Metropolis-Hastings Algorithms},
 urldate = {2023-09-21},
 volume = {67},
 year = {2005}
}

@article{Kamatani2020,
title = {Random walk Metropolis algorithm in high dimension with non-Gaussian target distributions},
journal = {Stochastic Processes and their Applications},
volume = {130},
number = {1},
pages = {297-327},
year = {2020},
issn = {0304-4149},
doi = {https://doi.org/10.1016/j.spa.2019.03.002},
url = {https://www.sciencedirect.com/science/article/pii/S0304414919301383},
author = {Kengo Kamatani}
}

@article{Creutz1988,
  title = {Global Monte Carlo algorithms for many-fermion systems},
  author = {Creutz, Michael},
  journal = {Phys. Rev. D},
  volume = {38},
  issue = {4},
  pages = {1228--1238},
  numpages = {0},
  year = {1988},
  month = {Aug},
  publisher = {American Physical Society},
  doi = {10.1103/PhysRevD.38.1228},
  url = {https://link.aps.org/doi/10.1103/PhysRevD.38.1228}
}

@article{SheUrbLud2023,
  title={The apogee to apogee path sampler},
  author={Sherlock, Chris and Urbas, Szymon and Ludkin, Matthew},
  journal={Journal of Computational and Graphical Statistics},
  volume={32},
  number={4},
  pages={1436--1446},
  year={2023},
  publisher={Taylor \& Francis}
}

@misc{DurVog2023,
      title={Metropolis Adjusted Langevin Trajectories: a robust alternative to Hamiltonian Monte Carlo}, 
      author={Riou-Durand, Lionel and Vogrinc, Jure},
      year={2023},
      eprint={2202.13230},
      archivePrefix={arXiv},
      primaryClass={stat.CO}
}

@InProceedings{HofRadSou2021,
  title = 	 { An Adaptive-MCMC Scheme for Setting Trajectory Lengths in Hamiltonian Monte Carlo },
  author =       {Hoffman, Matthew and Radul, Alexey and Sountsov, Pavel},
  booktitle = 	 {Proceedings of The 24th International Conference on Artificial Intelligence and Statistics},
  pages = 	 {3907--3915},
  year = 	 {2021},
  editor = 	 {Banerjee, Arindam and Fukumizu, Kenji},
  volume = 	 {130},
  series = 	 {Proceedings of Machine Learning Research},
  month = 	 {13--15 Apr},
  publisher =    {PMLR},
  url = 	 {https://proceedings.mlr.press/v130/hoffman21a.html}
}

@article{BouSan2017,
 ISSN = {10505164},
 URL = {http://www.jstor.org/stable/26361544},
 author = {Nawaf Bou-Rabee and Jesús María Sanz-Serna},
 journal = {The Annals of Applied Probability},
 number = {4},
 pages = {2159--2194},
 publisher = {Institute of Mathematical Statistics},
 title = {RANDOMIZED HAMILTONIAN MONTE CARLO},
 urldate = {2023-09-21},
 volume = {27},
 year = {2017}
}

@article{LivZan2022,
    author = {Livingstone, Samuel and Zanella, Giacomo},
    title = "{The Barker Proposal: Combining Robustness and Efficiency in Gradient-Based MCMC}",
    journal = {Journal of the Royal Statistical Society Series B: Statistical Methodology},
    volume = {84},
    number = {2},
    pages = {496-523},
    year = {2022},
    month = {01},
    issn = {1369-7412},
    doi = {10.1111/rssb.12482},
    url = {https://doi.org/10.1111/rssb.12482},
    eprint = {https://academic.oup.com/jrsssb/article-pdf/84/2/496/49322274/jrsssb\_84\_2\_496.pdf},
}

@InProceedings{SohMud2014,
  title={Hamiltonian Monte Carlo without detailed balance},
  author={Sohl-Dickstein, Jascha and Mudigonda, Mayur and DeWeese, Michael},
  booktitle={International Conference on Machine Learning},
  pages={719--726},
  year={2014},
  organization={PMLR}
}

@article{LudShe2022,
    author = {Ludkin, M and Sherlock, C},
    title = "{Hug and hop: a discrete-time, nonreversible Markov chain Monte Carlo algorithm}",
    journal = {Biometrika},
    volume = {110},
    number = {2},
    pages = {301-318},
    year = {2022},
    month = {07},
    issn = {1464-3510},
    doi = {10.1093/biomet/asac039},
    url = {https://doi.org/10.1093/biomet/asac039},
    eprint = {https://academic.oup.com/biomet/article-pdf/110/2/301/50311360/asac039.pdf},
}

@inproceedings{putcha2023preferential,
  title={Preferential Subsampling for Stochastic Gradient Langevin Dynamics},
  author={Putcha, Srshti and Nemeth, Christopher and Fearnhead, Paul},
  booktitle={International Conference on Artificial Intelligence and Statistics},
  pages={8837--8856},
  year={2023},
  organization={PMLR}
}

@article{coullon2023efficient,
  title={Efficient and generalizable tuning strategies for stochastic gradient MCMC},
  author={Coullon, Jeremie and South, Leah and Nemeth, Christopher},
  journal={Statistics and Computing},
  volume={33},
  number={3},
  pages={66},
  year={2023},
  publisher={Springer}
}

@article{doucet2009tutorial,
  title={A tutorial on particle filtering and smoothing: Fifteen years later},
  author={Doucet, Arnaud and Johansen, Adam M and others},
  journal={Handbook of nonlinear filtering},
  volume={12},
  number={656-704},
  pages={3},
  year={2009}
}

@article{nemeth2016particle,
  title={Particle approximations of the score and observed information matrix for parameter estimation in state--space models with linear computational cost},
  author={Nemeth, Christopher and Fearnhead, Paul and Mihaylova, Lyudmila},
  journal={Journal of Computational and Graphical Statistics},
  volume={25},
  number={4},
  pages={1138--1157},
  year={2016},
  publisher={Taylor \& Francis}
}

@article{nemeth2021stochastic,
  title={Stochastic gradient markov chain monte carlo},
  author={Nemeth, Christopher and Fearnhead, Paul},
  journal={Journal of the American Statistical Association},
  volume={116},
  number={533},
  pages={433--450},
  year={2021},
  publisher={Taylor \& Francis}
}

@book{lecam1986,
  author = {LeCam, {Lucien}},
  publisher = {Springer},
  series = {Springer series in statistics},
  title = {Asymptotic methods in statistical decision theory},
  year = 1986
}

@article{DuaneHMC1987,
title = {Hybrid Monte Carlo},
journal = {Physics Letters B},
volume = {195},
number = {2},
pages = {216-222},
year = {1987},
issn = {0370-2693},
doi = {https://doi.org/10.1016/0370-2693(87)91197-X},
url = {https://www.sciencedirect.com/science/article/pii/037026938791197X},
author = {Simon Duane and A.D. Kennedy and Brian J. Pendleton and Duncan Roweth}
}

@article{Xifaraetal2014,
title = {Langevin diffusions and the Metropolis-adjusted Langevin algorithm},
journal = {Statistics \& Probability Letters},
volume = {91},
pages = {14-19},
year = {2014},
issn = {0167-7152},
doi = {https://doi.org/10.1016/j.spl.2014.04.002},
url = {https://www.sciencedirect.com/science/article/pii/S0167715214001333},
author = {T. Xifara and C. Sherlock and S. Livingstone and S. Byrne and M. Girolami}
}

@book{RasmussenWilliams2005,
    author = {Rasmussen, Carl Edward and Williams, Christopher K. I.},
    title = "{Gaussian Processes for Machine Learning}",
    publisher = {The MIT Press},
    year = {2005},
    month = {11},
    isbn = {9780262256834},
    doi = {10.7551/mitpress/3206.001.0001},
    url = {https://doi.org/10.7551/mitpress/3206.001.0001}
}

@book{Conway2010,
    author = {Conway, John B.},
    title = "{A Course in Functional Analysis}",
    publisher = {Springer},
    year = {2010},
     isbn = {9781441930927},
    edition = {Second}
}

@article{HeidWelch1981,
title = {A Spectral Method for
Confidence Interval
Generation and Run
Length Control in
Simulations},
journal = {Communications of the ACM},
volume = {24},
pages = {233-245},
year = {1981},
number = {4},
author = {Philip Heidelberger and Peter D. Welch}
}

@article{coullon2022sgmcmcjax,
  title={SGMCMCJax: a lightweight JAX library for stochastic gradient Markov chain Monte Carlo algorithms},
  author={Coullon, Jeremie and Nemeth, Christopher},
  journal={Journal of Open Source Software},
  volume={7},
  number={72},
  pages={4113},
  year={2022}
}

@article{cabezas2024blackjax,
  title={BlackJAX: Composable Bayesian inference in JAX},
  author={Cabezas, Alberto and Corenflos, Adrien and Lao, Junpeng and Louf, R{\'e}mi},
  journal={arXiv preprint arXiv:2402.10797},
  year={2024}
}

@article{Sun2005,
title = {Mercer theorem for RKHS on noncompact sets},
journal = {Journal of Complexity},
volume = {21},
number = {3},
pages = {337-349},
year = {2005},
issn = {0885-064X},
doi = {https://doi.org/10.1016/j.jco.2004.09.002},
url = {https://www.sciencedirect.com/science/article/pii/S0885064X04000822},
author = {Hongwei Sun}
}

@article{vyner2023swiss,
  title={SwISS: A scalable Markov chain Monte Carlo divide-and-conquer strategy},
  author={Vyner, Callum and Nemeth, Christopher and Sherlock, Chris},
  journal={Stat},
  volume={12},
  number={1},
  pages={e523},
  year={2023},
  publisher={Wiley Online Library}
}

@article{leimkuhler2013rational,
  title={Rational construction of stochastic numerical methods for molecular sampling},
  author={Leimkuhler, Benedict and Matthews, Charles},
  journal={Applied Mathematics Research eXpress},
  volume={2013},
  number={1},
  pages={34--56},
  year={2013}
}

@article{sharrock2025tuning,
  title={Tuning-free sampling via optimization on the space of probability measures},
  author={Sharrock, Louis and Nemeth, Christopher},
  journal={arXiv preprint arXiv:2510.25315},
  year={2025}
}

  \cleardoublepage


 \printindex


\end{document}